\documentclass[11pt]{article}    

\usepackage{graphicx}
\usepackage{url}
\usepackage{amsmath}    
\usepackage{amssymb}
\usepackage{fullpage}
\usepackage{enumerate}
\usepackage{bbm}
\usepackage{bm}
\usepackage{color}
\usepackage{mathtools}
\usepackage[table]{xcolor}
\usepackage{verbatim}
\usepackage{multirow}
\usepackage{subcaption}
\usepackage{setspace}
\usepackage[authoryear,round,comma,compress]{natbib}
\bibpunct{(}{)}{,}{a}{}{,}
\usepackage[english]{babel}
\usepackage[utf8]{inputenc}
\usepackage{etoolbox}
\usepackage{graphicx}
\usepackage[colorinlistoftodos]{todonotes}
\usepackage[version=3]{mhchem}
\usepackage{chemfig}
\usepackage{fancyhdr}
\usepackage{amsfonts}
\usepackage{microtype}
\usepackage{setspace}
\usepackage[usenames,dvipsnames]{pstricks}
\usepackage{epsfig}
\usepackage{pst-grad} 
\usepackage{pst-plot} 
\usepackage[space]{grffile} 
\usepackage{etoolbox} 
\usepackage{float}
\usepackage{mathrsfs}
\usepackage[margin=1in]{geometry}
\usepackage{bookmark}
\usepackage{hyperref}
\usepackage{pgfplots}
\usepackage{framed}
\usepackage{tikz}
\usetikzlibrary{arrows,intersections}
\usepackage{xcolor}
\usepackage{enumitem}
\usepackage{filecontents}
\usepackage{clipboard}
\newclipboard{myclipboard}
\usepackage[bbgreekl]{mathbbol}
\usepackage[thinlines]{easytable}
\usepackage{booktabs,
	makecell}                         
\setcellgapes{3pt}                            
\usepackage{isotope}
\usepackage{multirow}
\usepackage{diagbox}

\newcommand{\bunderline}[1]{\underline{#1\mkern-4mu}\mkern4mu }
\makeatletter
\def\blfootnote{\xdef\@thefnmark{}\@footnotetext}
\makeatother

\DeclareMathOperator*{\argmax}{arg\,max}
\DeclareMathOperator*{\argmin}{arg\,min}

\newcommand {\ualpha}{\bar{\alpha}}
\newcommand {\lalpha}{\underline{\alpha}}
\newcommand {\cS}{{\cal S}}

\newcommand {\beq}{\begin{equation}}
\newcommand {\eeq}{\end{equation}}
\newcommand {\beqn}{\begin{equation*}}
\newcommand {\eeqn}{\end{equation*}}
\newcommand {\bear}{\begin{eqnarray}}
\newcommand {\eear}{\end{eqnarray}}
\newcommand {\bearn}{\begin{eqnarray*}}
	\newcommand {\eearn}{\end{eqnarray*}}
\newcommand {\bpmatrix}{\begin{pmatrix}}
	\newcommand {\epmatrix}{\end{pmatrix}}
\newcommand{\Expect}{\mathbb{E}}
\newcommand{\Prob}{\mathbb{P}}

\newcommand {\rR}{\mathbb{R}}
\newcommand {\rE}{\mathbb{E}}
\newcommand {\rN}{\mathbb{N}}
\newcommand {\rP}{\mathbb{P}}
\newcommand{\cP}{\mathcal{P}}
\newcommand{\cA}{\mathcal{A}}
\newcommand{\cR}{\mathcal{R}}
\newcommand{\cK}{\mathcal{K}}

\newcommand{\cE}{\mathcal{E}}
\newcommand{\cV}{\mathcal{V}}
\newcommand{\cN}{\mathcal{N}}
\newcommand{\cC}{\mathcal{C}}
\newcommand{\cH}{\mathcal{H}}
\newcommand{\cT}{\mathcal{T}}

\renewcommand{\l}{\left}
\renewcommand{\r}{\right}
\newcommand{\Indlr}[1]{\mathbbm{1}\l\{#1 \r\}}
\newcommand {\known}{\text{\tiny$\mathrm{known}$}}
\newcommand {\aux}{\mathrm{aux}}

\newcommand{\intd}{\mathrm{d}}
\newcommand{\KL}{\mathrm{KL}}
\newcommand{\AIE}{\mathrm{AIE}}

\newcommand{\rPlr}[1]{\rP\l\{ #1\r\}}

\newcommand\numberthis{\addtocounter{equation}{1}\tag{\theequation}}

\newtheorem{theorem}{Theorem}
\newtheorem{definition}{Definition}

\newtheorem{lemma}{Lemma}
\newtheorem{prop}{Proposition}
\newtheorem{corollary}{Corollary}

\newtheorem{example}{Example}
\newtheorem{remark}{Remark}
\def\qed{ \ \vrule width.2cm height.2cm depth0cm\iffalse \smallskip\fi }
\newenvironment{proof}{\noindent {\bf Proof.\/}}{$\qed$\vskip 0.01in}

\renewcommand{\vector}[1]{\ensuremath \boldsymbol{\mathrm #1}}
\newcommand{\vectorgreek}[1]{\ensuremath \boldsymbol{#1}}

\long\def\symbolfootnote[#1]#2{\begingroup%
	\def\thefootnote{\fnsymbol{footnote}}\footnote[#1]{#2}\endgroup}

\makeatletter
\newcommand\mathgreek[1]{\expandafter\@mathgreek\csname c@#1\endcsname}
\newcommand\@mathgreek[1]{%
  \ifcase#1\or C_1\or C_2\or C_3\or C_4\or C_5\or C_6\or
   C_7\or C_8\or C_9\or C_{10}\or C_{11}\or C_{12}\or C_{13}\or  C_{14}\or C_{15}\or
   C_{16}\or C_{17}\or C_{18}\or C_{19}\or C_{20}\or C_{21}\or C_{22}\or C_{23} \or C_{24} \or C_{25} \or C_{26} \or C_{27} \or C_{28} \or C_{29} \or C_{30} \or C_{31} \or C_{32} \or C_{33} \or C_{34} \or C_{35}\else
  \@ctrerror\fi}
\newcounter{greekvars}
\renewcommand\thegreekvars{\mathgreek{greekvars}}
\newcommand\constvar[1][]{%
  \if\relax\detokenize{#1}\relax
    \stepcounter{greekvars}%
  \else
    \refstepcounter{greekvars}\ltx@label{#1}%
  \fi
  \thegreekvars
}
\makeatother
\newcommand{\constref}[1]{\ref*{#1}}

\normalsize

\begin{document}
\title{\vspace{-1.0cm}Adaptive Sequential Experiments \\ with Unknown Information Arrival Processes\vspace{0.3cm}}

\author{
{\sf Yonatan Gur}
\\Stanford University\and
{\sf Ahmadreza Momeni}\thanks{
Correspondence: {\tt ygur@stanford.edu}, {\tt amomenis@stanford.edu}.}
\\Stanford University
\vspace{0.4cm}}

\vspace{-.4cm}
\date{\today}

\maketitle

\vspace{-0.65cm}
\begin{abstract}\setstretch{1.00}
\noindent Sequential experiments are often characterized by an exploration-exploitation tradeoff that is captured by the multi-armed bandit (MAB) framework. This framework has been studied and applied, typically when at each time period feedback is received only on the action that was selected at that period. However, in many practical settings additional data may become available \emph{between} decision epochs.
We introduce a generalized MAB formulation, which considers a broad class of distributions that are informative about mean rewards, and allows observations from these distributions to arrive according to an arbitrary and a priori unknown arrival process.
When it is known how to map auxiliary data to reward estimates, by obtaining matching lower and upper bounds we characterize a spectrum of minimax complexities for this class of problems as a function of the information arrival process, which captures how salient characteristics of this process impact achievable performance.
In terms of achieving optimal performance, we establish that 
upper confidence bound and posterior sampling policies possess natural robustness with respect to the information arrival process without any adjustments, which uncovers a novel property of these popular policies and further lends credence to their appeal. 
When the mappings connecting auxiliary data and rewards are a priori unknown, we characterize necessary and sufficient conditions under which auxiliary information allows performance improvement. We devise a new policy that is based on two different upper confidence bounds (one that accounts for auxiliary observation and one that does not) and establish the near-optimality of this policy. We use data from a large media site to analyze the value that may be captured in practice by leveraging auxiliary data for designing content recommendations.

\vspace{0.15cm}

\end{abstract}

\setstretch{1.45}

\vspace{-0.2cm}
\section{Introduction} \label{sec-intro}\vspace{-0.1cm}
\subsection{Background and Motivation}\label{subsec:background}\vspace{-0.1cm}

Effective design of sequential experiments requires balancing between new information acquisition (\emph{exploration}), and optimizing payoffs based on available information (\emph{exploitation}). A well-studied framework that captures the trade-off between exploration and exploitation is the one of multi-armed bandits (MAB) that first emerged in \cite{thompson1933likelihood} in the context of drug testing, and was later extended by \cite{robbins1952aspects} to a more general setting. In this framework, an agent repeatedly chooses between different actions (arms), where each action is associated with an a priori unknown reward distribution, and the objective is to maximize cumulative return over a certain time horizon.

The MAB framework focuses on balancing exploration and exploitation under restrictive assumptions on the future information collection process: in every round noisy observations are collected only on the action that was selected at that period. Correspondingly, policy design is typically predicated on that information collection procedure. However, in many practical settings to which bandit models have been applied for designing sequential experimentation, additional information might be realized and leveraged over time. Examples of such settings include dynamic pricing \citep{bastani2019meta, bu2019online}, retail management \citep{zhang2010crafting}, clinical trials \citep{bertsimas2016analytics, anderer2019adaptive}, as well as many machine learning domains (see a list of relevant settings in the survey by \citealt{pan2009survey}).

To illustrate one particular such setting, consider the problem of designing product recommendations with limited data. Product recommendation systems are widely deployed over the web with the objective of helping users navigate through content and consumer products while increasing volume and revenue for service and e-commerce platforms. These systems commonly apply techniques that leverage information such as explicit and implicit user preferences, product consumption, and consumer ratings (see, e.g., \citealt{hill1995recommending, konstan1997grouplens, breese1998empirical}). While effective when ample relevant information is available, these techniques tend to under-perform when encountering products that are new to the system and have little or no trace of activity. This phenomenon, known as the \emph{cold-start} problem, has been documented and studied in the literature; see, e.g., \cite{schein2002methods}, \cite{park2009pairwise}.

In the presence of new products, a recommendation system needs to strike a balance between maximizing instantaneous performance indicators (such as revenue) and collecting valuable information that is essential for optimizing future recommendations. To illustrate this tradeoff (see Figure \ref{fig:additional-info}), consider a consumer (A) that is searching for a product using the organic recommendation system of an online retailer (e.g., Amazon). The consumer provides key product characteristics (e.g., ``headphones" and ``bluetooth"), and based on this description (and, perhaps, additional factors such as the browsing history of the consumer) the retailer recommends products to the consumer. While some of the candidate products that fit the desired description may have already been recommended many times to consumers, and the mean returns from recommending them are known, other candidate products might be new brands that were not recommended to consumers before. Evaluating the mean returns from recommending new products requires experimentation that is essential for improving future recommendations, but might be costly if consumers are less likely to purchase these new products.
	
	\begin{figure}[t]
		\centering
		\includegraphics[height=1.9in]{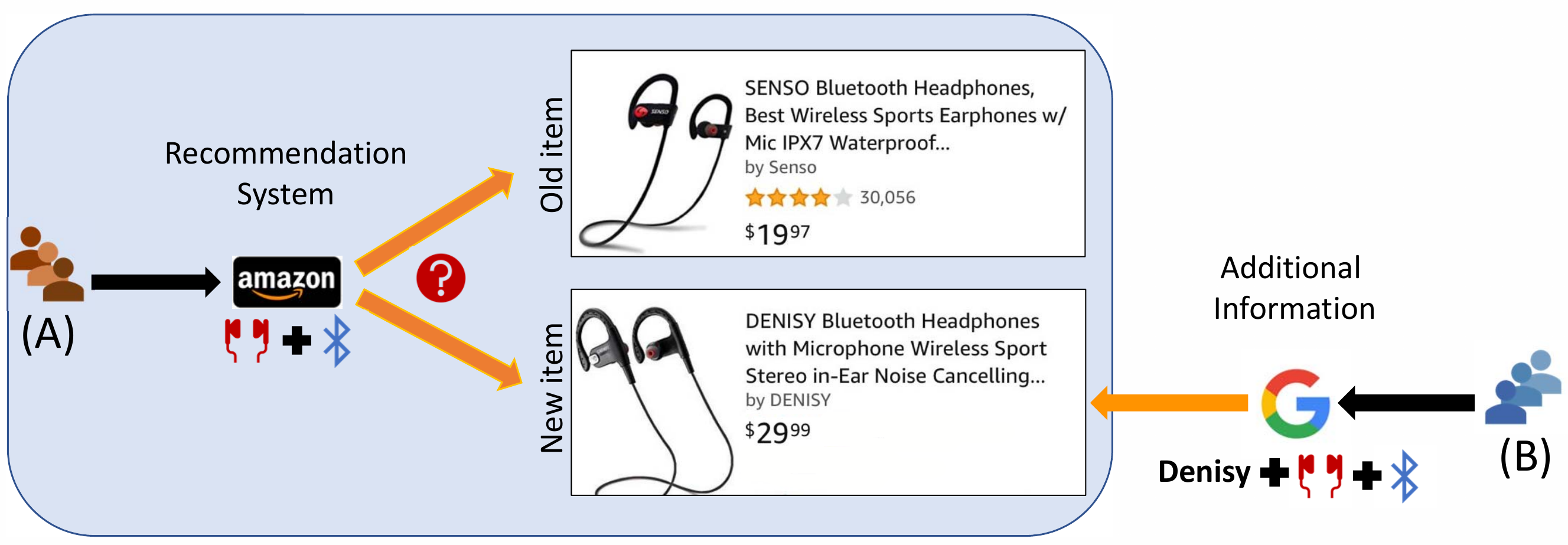}\vspace{-0.1cm}
		\caption{\small Additional information in product recommendations. In the depicted scenario, consumers of type (A) sequentially arrive to an online retailer's organic recommendation system to search for products that match the terms ``headphones" and ``bluetooth." Based on these terms, the retailer can recommend one of two products: an old brand that was already recommended many times to consumers, and a new brand that was never recommended before. In parallel, consumers of type (B) arrive to the new brand's product page directly from a search engine (e.g., Google) by searching for ``headphones," ``bluetooth," and ``Denisy" (the name of the new brand).} \label{fig:additional-info}\vspace{-0.3cm}
	\end{figure}
	
Several MAB settings were suggested and applied for designing recommendation algorithms that effectively balance information acquisition and instantaneous revenue maximization, where arms represent candidate recommendations; see the overview in \cite{madani2005contextual}, as well as later studies by \cite{agarwal2009online}, \cite{caron2013mixing}, \cite{tang2014ensemble}, and \cite{wang2017biucb}. Aligned with traditional MAB frameworks, these studies consider settings where in each time period information is obtained only on items that are recommended at that period, and suggest recommendation policies that are predicated on this information collection process. However, in many instances additional browsing and consumption information may be maintained \emph{in parallel} to the sequential recommendation process, as a significant fraction of website traffic may take place through means other than its organic recommender system; see browsing data analysis in \cite{sharma2015estimating} that estimate that product recommendation systems are responsible for at most 30\% of site traffic and revenue and that substantial fraction of remaining traffic arrives to product pages directly from search engines, as well as \cite{besbes2015optimization} that report similar findings for content recommendations in media sites.

To illustrate an additional source of traffic and revenue information, let us revert to Figure~\ref{fig:additional-info} and consider a different consumer (B) that \emph{does not} use the organic recommender system of the retailer, but rather arrives to a certain product page directly from a search engine (e.g., Google) by searching for that particular brand. While consumers (A) and (B) may have inherently different preferences, the actions taken by consumer (B) after arriving to the product page (e.g., whether or not she bought the product) can be informative about the returns from recommending that product to other consumers. This additional information could potentially be used to improve the performance of recommendation algorithms, especially when the browsing history that is associated with some products is limited.

The above example illustrates that additional information often becomes available \emph{between} decision epochs of recommender systems, and that this information might be essential for designing effective product recommendations. In comparison with the restrictive information collection process that is assumed in most MAB formulations, the availability of such auxiliary information (and even the prospect of obtaining it in the future) might impact the achievable performance and the design of sequential experiments and learning policies. When additional information is available one may potentially need to ``sacrifice" less decisions for exploration in order to obtain effective estimators for the mean rewards. While this intuition suggests that exploration can be reduced in the presence of additional information, it is a priori not clear what type of performance improvement can be expected given different information arrival processes, and how policy design should depend on these processes.

Moreover, monitoring exploration rates in real time in the presence of an arbitrary information arrival process introduces additional challenges that have distinct practical relevance. Most importantly, an optimal exploration rate may depend on several characteristics of the information arrival process, such as the amount of information that arrives and the time at which it appears (e.g., early on versus later on along the decision horizon). Since it may be hard to predict upfront the salient characteristics of the information arrival process, an important challenge is to identify the extent to which it is possible to adapt in real time to an a priori unknown information arrival process, in the sense of approximating the performance that is achievable under prior knowledge of the sample path of information arrivals. This paper is concerned with addressing these challenges.

The research questions that we study in this paper are along two lines. First, we study how the information arrival process and its characteristics impact the performance one could aspire to when designing a sequential experiment. Second, we study what type of policies should one adopt for designing sequential experiments in the presence of general and unknown information arrival processes.

\vspace{-0.0cm}
\subsection{Main Contributions}\vspace{-0.0cm}

The main contribution of this paper lies in (1) introducing a new, generalized MAB framework with unknown and arbitrary information arrival process; (2) characterizing the minimax regret complexity of this broad class of problems as a function of the information arrival process; and (3) identifying policies (and proposing new ones) that are rate-optimal in the presence of general and a priori unknown information arrival processes. More specifically, our contribution is along the following lines.

\textit{(1) Modeling.} To capture the salient features of settings with auxiliary information arrivals we formulate a new class of MAB problems that generalizes the classical stochastic MAB framework, by relaxing strong assumptions that are typically imposed on the information collection process. Our formulation considers a broad class of distributions that are informative about mean rewards, and allows observations from these distributions to arrive at arbitrary and a priori unknown rates and times.

More particularly, given available actions $k=1,\ldots,K$ and horizon $t=1,\ldots T$, we consider an a priori unknown \textit{information arrival matrix} $\vector{H}$, with each integer entry $h_{k,t}$ capturing the number of auxiliary observations collected on action $k$ between time $t-1$ and time $t$. We assume that for each action $k$ there are mappings $\left\{\phi_k:\;k=1,\ldots,K\right\}$ that map auxiliary observations to unbiased estimates of mean rewards, where we distinguish between the cases where these mappings are known or unknown a priori to the decision maker. Our model captures a large variety of real-world phenomena, yet maintains tractability that allows for a crisp characterization of the performance improvement that can be expected when more information might become available throughout the sequential decision process.

\textit{(2) Analysis of achievable performance.} When it is known how to map auxiliary data to reward estimates, we establish lower bounds on the performance that is achievable by \emph{any} non-anticipating policy, where performance is measured in terms of regret relative to the performance of an oracle that constantly selects the action with the highest mean reward. We further show that our lower bounds can be achieved through suitable policy design. These results identify the spectrum of minimax complexities associated with the MAB problem with unknown information arrival processes, as 
the following function of the realized information arrival process itself:\vspace{-0.15cm}  
	\[
	\sum_{k=1}^K\log\left(\sum\limits_{t=1}^T \vspace{-0.0cm} \exp\left(-c\sum\limits_{s=1}^th_{k,s}\right)\right),\vspace{-0.15cm}
	\]
where $c$ is a constant that depends on other problem parameters. The identified spectrum of minimax complexities ranges from the classical regret rates that appear in the stochastic MAB literature when there is no auxiliary information, to regret that is uniformly bounded over time when information arrives frequently and/or early enough. This provides a yardstick for evaluating the performance of policies, and a sharp criterion for optimality in the presence of auxiliary information arrivals.

When the mappings $\{\phi_k\}_{k=1}^{K}$ are unknown, one faces a difficult problem that essentially requires to learn synchronously \emph{from} auxiliary data and \emph{on} that data and how to interpret it. Nevertheless, focusing on the practical family of linear mappings, we identify a necessary and sufficient condition delineating, on the one hand, when auxiliary data still allows the performance improvement that is characterized by the case where these mappings are known, and on the other hand, when performance cannot improve based on auxiliary information, regardless of the timing and frequency of information arrivals.

\textit{(3) Policy design.} We observe that, whether the mappings $\{\phi_k\}_{k=1}^{K}$ are known or unknown, policies that are rate optimal in the absence of auxiliary data may not be rate optimal anymore in the presence of auxiliary data. Nevertheless, when it is known how to map auxiliary data to reward estimates, we establish that polices that are based on upper confidence bounds and posterior sampling leverage auxiliary data naturally, without any prior knowledge on the information arrival process. In that sense, polices such as UCB1 and Thompson sampling exhibit remarkable robustness with respect to the information arrival process: without any adjustments, these policies achieve the best performance that can be achieved, guaranteeing rate optimality \emph{uniformly} over a general class of information arrival processes. Moreover, the regret rates they guarantee cannot be improved even with prior knowledge of the information arrival process. This identifies a new important property of these popular policies, extends the class of problems to they apply, and further lends credence to their appeal.

When the mappings $\{\phi_k\}_{k=1}^{K}$ connecting auxiliary data to reward estimates are a priori unknown, we devise a new policy (termed 2-UCBs) that is adaptive with respect to both the unknown information arrival process and the unknown mappings, and establish its near optimality for the class of linear mappings. 2-UCBs is based on two different upper confidence bounds: one that accounts only for decisions-driven observations, and one that also accounts (in an optimistic manner) for auxiliary data.

Using data from a large media site, we analyze the value that may be captured by leveraging auxiliary information to recommend content with unknown impact on the future browsing path of readers. The auxiliary data we utilize consist of users that arrived to media pages directly from search, and the browsing path of these users following their arrival. We compare the performance of the 2-UCBs algorithm relative to non-adaptively relaying on estimates of $\left\{\phi_k\right\}_{k=1}^{K}$ formed based on historical data.

\vspace{-0.0cm}
\subsection{Related Work}\vspace{-0.0cm}

\noindent
\textbf{Multi-Armed Bandits.} Since its inception, the MAB framework has been adopted for studying a variety of applications including clinical trials \citep{zelen1969play}, strategic pricing \citep{bergemann1996learning}, assortment selection \citep{caro2007dynamic}, online auctions  \citep{kleinberg2003value}, online advertising \citep{pandey2007bandits}, and product recommendations \citep{madani2005contextual, li2010contextual}. For an overview of MAB formulations we refer the readers to \cite{berry1985bandit} and \cite{gittins2011multi} for Bayesian 
settings, as well as \cite{cesa2006prediction} and \cite{bubeck2012regret} for the 
so-called adversarial setting. A regret characterization for the more traditional formulation adopted in this paper, often referred to as the stochastic MAB problem, was first established by \cite{lai1985asymptotically}, followed by analysis of policies such as $\epsilon_t$-greedy, UCB1, and Thompson sampling; see, e.g., \cite{auer2002finite-time} and \cite{agrawal2013further}.

The MAB framework focuses on balancing exploration and exploitation under restrictive assumptions on the information collection process. Correspondingly, optimal policy design is typically predicated on the assumption that at each period a reward observation is collected only on the action that is selected at that time period (exceptions to this common information collection process are discussed below). Such policy design does not account for information that may become available \emph{between} decision epochs, and that might be essential for achieving good performance in many practical settings. In the current paper we relax these assumptions on the information arrival process by allowing arbitrary information arrival processes. Our focus is on studying the impact of the information arrival characteristics on achievable performance and policy design when the information arrival process is a priori unknown.

Few MAB settings deviate from the information collection process described above by considering \emph{specific} information arrival processes that are \emph{known a priori}, as opposed to our formulation where the information arrival process is arbitrary and a priori unknown. In a well-studied contextual MAB setting \citep{woodroofe1979one} 
at each trial the decision maker observes a context carrying information about other arms.  In the full-information \emph{adversarial} MAB setting, rewards are arbitrary and can even be selected by an adversary (\citealt{auer1995gambling}, \citealt{freund1997decision}). In this setting, after each decision epoch the agent has access to a reward observation from each arm. The adversarial nature of this setting makes it fundamentally different from ours in terms of achievable performance, analysis, and policy design. \cite{shivaswamy2012multi} consider what could be viewed as a special case of our framework, where some initial reward observations are available at the beginning of the horizon.


\medskip
\noindent \textbf{Balancing and Regulating Exploration.} Several papers have considered settings of dynamic optimization with partial information and distinguished between cases where myopic policies guarantee optimal performance, and cases where exploration is essential, in the sense that myopic policies may lead to incomplete learning and large losses; examples include several application domains, including dynamic pricing (\citealt{araman2009dynamic}, \citealt{farias2010dynamic}, \citealt{Har-Kes-Zee}, \citealt{den2013simultaneously}), inventory management \citep{Bes-Muh}, product recommendations \citep{bastani2018sequential}, and technology development \citep{LeeLeeLee}. Furthermore it was shown that myopic policies maintain notions of optimality when future decisions are discounted (e.g., \citealt{woodroofe1979one}, 
\citealt{sarkar1991one}), or with access to sufficiently diverse contexts \citep{bastani2017exploiting}.
On the other hand, few papers have studied cases where exploration is not only essential but should be enhanced in order to maintain optimality; prominent examples include the partial monitoring feedback structure by \cite{cesa2006regret}, and the non-stationary MAB framework by \cite{Bes-Gur-Zee-2018}.

The above studies demonstrate a variety of practical settings where the extent of exploration that is required to maintain optimality depends on problem characteristics that may often be a priori unknown to the decision maker. This introduces the challenge of dynamically endogenizing the rate at which a decision-making policy explores to approximate the best performance that is achievable under ex ante knowledge of the underlying problem characteristics. In this paper we address this challenge from an information collection perspective. We identify the connection between the information arrival process and the degree of exploration that is required to maintain rate optimality, and identify adaptive policies that guarantee rate optimality without prior knowledge on the information arrival process.

\medskip
\noindent\textbf{Recommender Systems.} An active stream of literature has been focusing on modelling and maintaining connections between users and products in recommender systems; see, e.g., \cite{ansari2000internet}, the survey by \cite{adomavicius2005toward}, and a book by \cite{ricci2011introduction}. One key element that impacts the performance of recommender systems is the often limited data that is available. Focusing on the prominent information acquisition aspect of the problem, several studies (to which we referred earlier) have addressed sequential recommendation problems using a MAB framework where at each time period information is obtained only on items that are recommended at that period. Another approach is to identify and leverage additional sources of relevant information. Following that avenue, \cite{farias2017learning} consider the problem of estimating user-item propensities, and propose a method to incorporate auxiliary data such as browsing and search histories to enhance the predictive power of recommender systems. While their work concerns with an offline prediction setting, our paper focuses on the impact of auxiliary information streams on policy design and performance in a sequential dynamic setting.

\medskip
\noindent\textbf{Conference Proceedings.}
An extended abstract that appeared in \cite{Gur-Momeni2018} included some initial findings for the special case where at most one auxiliary observation is obtained for each arm at each time period, and where auxiliary data must consist of iid reward observations. Considering that setting, the focus of \cite{Gur-Momeni2018} is on using virtual time indexes for developing an $\varepsilon_t$-greedy-type policy that adapts to an unknown information arrival matrix, albeit predicated on prior knowledge of~$\Delta$, the minimal separation in mean rewards between the optimal action and suboptimal ones.\footnote{For completeness, in Appendix~\ref{section:A near-optimal adaptive policy} we describe the virtual time indexes method and detail its analysis.} The current paper considers a general framework where, in particular, auxiliary data can be realized from a general family of distributions that could be mapped to rewards estimates, and analyzes the cases where the underlying mappings are known or unknown a priori. Furthermore, our focus is on algorithms based on upper confidence bounds and posterior sampling that guarantee near optimality without requiring prior knowledge of $\Delta$, including in the setting considered in \cite{Gur-Momeni2018}.

\vspace{-0.15cm}
\section{Problem Formulation}\label{sec: prob_form}
\vspace{-0.15cm}

We next formulate the class of multi-armed bandit problems with auxiliary information arrivals. 
Let~$\mathcal{K} = \{1,\dots,K\}$ be a set of arms (actions) and let $\mathcal{T} = \{1,\dots,T\}$ denote a sequence of decision epochs. At each time period $t\in\cal{T}$, a decision maker selects one of the $K$ arms. When selecting an arm~$k\in\cal{K}$ at time $t\in\cal{T}$, a reward~$X_{k,t}\in \mathbb{R}$ is realized and observed. For each $t\in\cal T$ and $k\in\cal K$, the reward $X_{k,t}$ is assumed to be independently drawn from a distribution $\nu_k\in \cV$ with mean $\mu_{k}$, where $\cV$ is the collection of $\sigma^2$-sub-Gaussian distributions.\footnote{A real-valued random variable $X$ is said to be sub-Gaussian if there is some $\sigma > 0$ such that for every $\lambda \in \mathbb{R}$ one has $\mathbb{E}e^{\lambda \left(X - \mathbb{E}X \right)} \le e^{\sigma^2\lambda^2/2}$. This broad class of distributions includes, for instance, Gaussian random variables, as well as any random variable with a bounded support (if $X\in[a,b]$ then $X$ is $\frac{(a-b)^2}{4}$-sub-Gaussian) such as Bernoulli random variables. In particular, if a random variable is $\sigma^2$-sub-Gaussian, it is also $\tilde{\sigma}^2$-sub-Gaussian for all $\tilde{\sigma}>\sigma$.}
We denote by~$\vectorgreek{\nu} = \left( \nu_1 , \dots , \nu_K\right)^\top$ the distribution of the rewards profile $\left(X_{1,t},\dots, X_{K,t}\right)^\top$. We assume that rewards are independent across time periods and arms. For simplicity, we assume that there exists a unique best action, and denote the highest expected reward and the best arm by $\mu^\ast$ and~$k^\ast$, respectively:\footnote{For simplicity, when using the $\argmin$ and $\argmax$ operators we assume ties to be broken in favor of the smaller index.}\vspace{-0.1cm}
\[
\mu^\ast \coloneqq \max_{k\in\mathcal{K}} \{\mu_k\} , \qquad k^\ast \coloneqq \argmax_{k\in \cal K} \mu_k.\vspace{-0.0cm}
\]
We denote by $\Delta_k = \mu^\ast - \mu_k$ the difference in expected reward between the best arm and arm $k$ and assume it is bounded. We denote by $\Delta$ a lower bound such that $0 < \Delta \le \min\limits_{k\in\mathcal{K}\setminus \{k^\ast\}}\Delta_k$. (We make no assumption on the separability of mean rewards and allow $\Delta$ to be a decreasing function of $T$.)


\medskip
\noindent\textbf{Information Arrival Process.} Before each epoch $t$, the agent may or may not observe auxiliary information for some of the arms without pulling them. For each period $t$ and arm $k$, we denote by~$h_{k,t}\in \rN^{+} := \{0\} \cup \rN$ the number of auxiliary information arrivals on arm $k$ between decision epochs $t-1$ and~$t$. We denote 
by~$\vector{H}$ 
the information arrival matrix with entries $h_{k,t}$; we assume that this matrix is independent of the policy's actions. If $h_{k,t} \ge 1$, then between decision epochs $t-1$ and~$t$ the agent observes $h_{k,t}$ independent random variables $Y_{k,t,m} \sim \nu_k^{\mathrm{aux}}\in \cV^{\mathrm{aux}}$ with mean $y_{k}$, where $m \in \{1,\dots, h_{k,t}\}$, and where $\cV^{\mathrm{aux}}$ is an a priori known class of allowed auxiliary observation distributions. For each arm~$k$ we assume that there exists a mapping~$\phi_k:\rR \rightarrow \rR$, through which an unbiased estimator for the mean reward $\mu_k$ can be constructed, that is, $\rE \l[\phi_k(Y_{k,t,m})\r] = \mu_k$, and that $\phi_k(Y_{k,t,m})$ is $\hat{\sigma}^2$-sub-Gaussian for some $\hat{\sigma} > 0$ and any $\nu_k^{\mathrm{aux}}\in \cV^{\mathrm{aux}}$. (The mappings $\{\phi_k\}$ can be either known or unknown a priori to the decision maker, and we will analyze both cases; see discussion in \S\ref{subsubsec:extentions}.) We denote the class of allowed mappings that satisfy the former criterion by $\Phi$. We assume that the random variables $Y_{k,t,m}$ are independent across time periods and arms, and are also independent from reward realizations. We denote the vector of average information received between epochs $t-1$ and~$t$ by~$\vector{Z}_t = \left(Z_{1,t},\dots, Z_{K,t}\right)^\top$ where for any $k$ one has~$Z_{k,t} =   \sum_{m=1}^{h_{k,t}}\phi_k(Y_{k,t,m})/\max\{1,h_{k,t}\}$.

\begin{example}\label{subsubsec:mappings}\textbf{\textup{(Linear mappings)}}\label{exp:linear-mappins}
A practical class of mappings $\Phi_L$ consists of linear mappings with an intercept at zero. Then, auxiliary data consists of realizations of independent random variables $Y_{k,t,m}$ for which there is a vector $(\alpha_1,\ldots,\alpha_K) \in~\rR^K$ such that for all $k$, $t$, and $m$ one has $\rE[\alpha_k Y_{k,t,m}] = \mu_k$.
\end{example}

\noindent This simple class of mappings illustrates the possibility of utilizing available data in many domains, including maximizing conversions of product recommendations (for example, in terms of purchases of recommended products). The practicality of using linear mappings for improving product recommendations will be demonstrated empirically in \S\ref{section-real data} using data from a large media site.

\medskip
\noindent\textbf{Policies and Performance.} \Copy{admissible policy}{Let $U$ be a random variable defined over a probability space $(\mathbb{U}, \mathcal{U}, \bm{\mathrm{P}}_u)$, and denote $\vectorgreek{h}_t = \left(h_{1,t}, \dots, h_{K,t}\right)^\top$. Let $\pi_t:\mathbb{R}^{t-1} \times \mathbb{R}^{K \times t} \times \{\rN^{+}\}^{K \times t} \times \mathbb{U} \rightarrow \mathcal{K}$ for $t = 1,2,\dots$ be measurable functions (with some abuse of notation we also denote the action at time $t$ by $\pi_t \in \mathcal{K}$) given by $\pi_t = \pi_t(X_{\pi_{t-1},t-1}, \dots, X_{\pi_{1},1}, \vector{Z}_{t}, \dots, \vector{Z}_1, \vectorgreek{h}_t,\dots,\vectorgreek{h}_1,U)$ for $t=2,3,\dots$, and $\pi_1 = \pi_1(\vector{Z}_{1},\vectorgreek{h}_1,U)$.

The mappings $\{\pi_t:t=1,\dots,T\}$, together with the distribution $\bm{\mathrm{P}}_u$ define the class of admissible policies denoted by $\cP$. Policies in $\cP$ depend only on the past history of actions and observations as well as auxiliary information arrivals, and allow for randomization via their dependence on~$U$.}
We denote by\vspace{-0.15cm}
	\[
	\mathcal{S} =
\l\{
(\vectorgreek{\nu}, \vectorgreek{\nu}^{\mathrm{aux}}) \in
 \prod\limits_{k\in\cK}(\cV \times \cV^{\mathrm{aux}}): \forall k \in \cK \ \ \exists \phi_k \in \Phi \text{ s.t. } \rE \l[ \phi_k(Y_{k,t,m})\r] = \mu_k, \
 \text{and } \Delta_k \ge \Delta \text{ for all }k \neq k^\ast
 \r\}\vspace{-0.15cm}
\]
 the class of problems that includes pairs of allowed reward distribution profiles $\l(\vectorgreek{\nu}, \vectorgreek{\nu}^{\mathrm{aux}}\r)$.

We evaluate the performance of a policy $\pi\in\cP$ by the regret it incurs under information arrival process $\vector{H}$ relative to the performance of an oracle that selects the arm with the highest expected reward.
For a given instance $(\vectorgreek{\nu}, \vectorgreek{\nu}^{\mathrm{aux}}) \in \cS$, we define the cased-dependent \emph{regret} by:\vspace{-0.2cm}
   \begin{equation}\label{eq:regret-def}
	\cR^{\pi}_{\vectorgreek{\nu}, \vectorgreek{\nu}^{\mathrm{aux}}}(T,\vector{H}) \coloneqq \mathbb{E}^{\pi}_{
		\vectorgreek{\nu}, \vectorgreek{\nu}^{\mathrm{aux}}
		} \left[ \sum\limits_{t=1}^T(\mu^{\ast} - \mu_{\pi_t}) \right],\vspace{-0.2cm}
\end{equation}
where the expectation $\mathbb{E}^\pi_{\vectorgreek{\nu}, \vectorgreek{\nu}^{\mathrm{aux}}}[\cdot]$ is taken with respect to the noisy rewards and noisy auxiliary observations, as well as to the policy's actions (throughout the paper we will denote by $\mathbb{P}^\pi_{\vectorgreek{\nu}, \vectorgreek{\nu}^{\mathrm{aux}}}$, $\mathbb{E}^\pi_{\vectorgreek{\nu}, \vectorgreek{\nu}^{\mathrm{aux}}}$, and $\mathbb{R}^\pi_{\vectorgreek{\nu}, \vectorgreek{\nu}^{\mathrm{aux}}}$ the probability, expectation, and regret when the arms are selected according to policy $\pi$, rewards are distributed according to $\vectorgreek{\nu}$, and auxiliary observations are distributed according to $\vectorgreek{\nu}^{\mathrm{aux}}$). We note that regret is incurred over decisions made in epochs $t=1,\ldots,$T; the main distinction relative to classical regret formulations is that in (\ref{eq:regret-def}) the mappings $\left\{\pi_t;\;t=1,\ldots,T\right\}$ can be measurable with respect to the $\sigma$-field that also includes information that arrives \emph{between} decision epochs, as captured by the matrix~$\vector{H}$. Finally, given an information arrival matrix~$\vector{H}$, we denote by\vspace{-0.2cm}
\[
\cR^{\pi}_{\cS}(T,\vector{H}) \coloneqq  \sup\limits_{\l(\vectorgreek{\nu}, \vectorgreek{\nu}^{\mathrm{aux}}\r)\in\cS}\mathcal{R}^\pi_{\vectorgreek{\nu}, \vectorgreek{\nu}^{\mathrm{aux}}}(T,\vector{H});\quad\quad\text{ and }\quad\quad \cR^{\ast}_{\cS}(T,\vector{H}) \coloneqq  \inf\limits_{\pi\in\cal{P}}\cR^{\pi}_{\cS}(T,\vector{H})\vspace{-0.2cm}
\]
the worst-case regret guaranteed by policy $\pi$, and the best guaranteed performance, respectively.

\vspace{-0.0cm}
\subsection{Discussion: Prior Knowledge on Mappings}\label{subsubsec:extentions}\vspace{-0.0cm}

\noindent
In our model, the decision maker may estimate mean rewards based on auxiliary observations through the mappings $\left\{\phi_k\right\}$. In  \S\ref{sec:impact-info-arrival-known-mappings} and \S\ref{sec:natural-adaptation}, to study in isolation the implications of the information arrival matrix~$\vector H$, we assume that these mappings are a priori known. The interpretability of feedback signals for the purpose of forming mean rewards estimates is a simplifying modeling assumption that is essential for tractability in many online dynamic optimization frameworks, including the MAB literature.\footnote{Typically in the MAB literature feedback is known to consist of reward observations. However, even when this relation between feedback and rewards is more general, e.g., in \cite{russo2018tutorial}, it is assumed to be a priori known.} Therefore this part could be viewed as extending prior work by relaxing assumptions on the information collection process that are common in the MAB literature, while offering assumptions on the information structure (and what is known about it) that are comparable to what is assumed in existing literature.


The challenge of adapting to unknown mappings is the focus of \S\ref{sec:unknown-mappings}, where we develop an approach for leveraging auxiliary information when these mappings are unknown. The approaches from \S\ref{sec:natural-adaptation} and \S\ref{sec:unknown-mappings} are then compared in \S\ref{section-real data} where we use data from a large US media site with the objective of leveraging auxiliary information (actions of readers that arrived to content pages directly from a search engine) to improve in-site recommendations. Performance is compared along two characteristics of problem instances: $(i)$ the effectiveness of auxiliary information, and $(ii)$ the mapping misspecification.





\vspace{-0.0cm}
\section{Impact of Information Arrival Process on Achievable Performance}\label{sec:impact-info-arrival-known-mappings}\vspace{-0.0cm}

In this section we study the impact that the information arrival process (particularly, the matrix $\vector H$) may have on the performance one could aspire to achieve when the mappings $\{\phi_k\}$ are known a priori. Our first result formalizes what \textit{cannot} be achieved, establishing a lower bound on the best achievable performance as a function of the information arrival process.
        \begin{theorem}\label{theorem-lower bound-general} \textbf{\textup{(Lower bound on the best achievable performance)}} For any $T \ge 1$ and information arrival matrix $\vector{H}$, the worst-case regret for any admissible policy $\pi\in\cP$ (with knowledge of $\{\phi_k\}$) is bounded from below as follows:\vspace{-0.1cm}
        \[
        \mathcal{R}^\pi_{\cal S}(\vector{H}, T) \ge \frac{\constvar[lower_bound1]}{\Delta} \sum\limits_{k=1}^K\log \left(
        \frac{\constvar[lower_bound2]\Delta^2}{K} \sum\limits_{t=1}^{T} \exp\left(-\constvar[lower_bound3] \Delta^2\sum\limits_{s=1}^th_{k,s}\right)
        \right),\vspace{-0.1cm}
        \]
        where $\constref{lower_bound1}$, $\constref{lower_bound2}$, and $\constref{lower_bound3}$ are positive constants that only depend on $\sigma$ and $\hat{\sigma}$.
        \end{theorem}

\noindent The precise expressions of $\constref{lower_bound1}$, $\constref{lower_bound2}$, and $\constref{lower_bound3}$ are provided in the discussion below. Theorem \ref{theorem-lower bound-general} establishes a lower bound on the achievable performance that depends on an arbitrary path of information arrivals, as captured by the elements of the matrix $\vector{H}$. In that sense, Theorem \ref{theorem-lower bound-general} provides a spectrum of bounds on the achievable performances, mapping many potential information arrival trajectories to the best performance they allow. When there is no additional information over what is assumed in the classical MAB setting, one recovers a lower bound of order $\frac{K}{\Delta} \log T$ that was established for that setting (see \citealt{lai1985asymptotically}, \citealt{bubeck2013bounded}). Theorem~\ref{theorem-lower bound-general} further suggests that in the presence of auxiliary observations, achievable regret rates may become lower, and that the impact of information arrivals on the achievable performance depends on the \textit{frequency} of these arrivals, but also on the \textit{time} at which these arrivals occur; we further discuss these observations in \S\ref{subsec:examples}.

\medskip\noindent
\textbf{Key Ideas in the Proof.} The proof of Theorem \ref{theorem-lower bound-general} adapts to our framework ideas of identifying a worst-case nature ``strategy" (see, e.g., \citealt{bubeck2013bounded}). While the full proof is deferred to the appendix, we next illustrate its key ideas using the special case of two arms. 
We assume that the decision maker is a priori informed that the first arm generates rewards according to a normal distribution with standard variation $\sigma$ and a mean that is either $-\Delta$ (according to distribution denoted by $\vectorgreek{\nu}$) or $+\Delta$ (according to a distribution denoted by $\tilde{\vectorgreek{\nu}}$), and that the second arm generates rewards with normal distribution of standard variation $\sigma$ and mean zero. To quantify a notion of distance between the possible profiles of reward distributions we use the Kullback-Leibler (KL) divergence. The KL divergence between two positive measures $\rho$ and $\rho^{\mathrm{aux}}$ with $\rho$ absolutely continuous with respect to $\rho^{\prime}$, is defined as:\vspace{-0.15cm}
    \[
    \mathrm{KL}(\rho,\rho^{\prime} ) \coloneqq \int \log\left(\frac{d\rho}{d\rho^{\prime}}\right) d\rho = \mathbb{E}_{\rho} \log\left(\frac{d\rho}{d\rho^{\prime}}(X)\right),\vspace{-0.15cm}
    \]
where $\mathbb{E}_{\rho}$ denotes the expectation with respect to $\rho$. Using Lemma 2.6 from \cite{Tsybakov2008introduction} that connects the KL divergence to error probabilities, we establish that at each period $t$ the probability of selecting a suboptimal arm must be at least\vspace{-0.15cm}
\[
p^{\mathrm{sub}}_t = \frac{1}{4}\exp\left(-\frac{2\Delta^2}{\sigma^2}\left( \mathbb{E}_{\vectorgreek{\nu},\vectorgreek{\nu}^{\aux}} [\tilde{n}_{1,T}] + \sum\limits_{s=1}^t \frac{\sigma^2}{\hat{\sigma}^2}h_{1,s}\right)\right),\vspace{-0.15cm}
\]
where $\tilde{n}_{1,t}$ denotes the number of times the first arm is selected up to time $t$. Each selection of suboptimal arm contributes $\Delta$ to the regret, and therefore the cumulative regret must be at least $\Delta \sum\limits_{t=1}^Tp^{\mathrm{sub}}_t$. We further observe that if arm 1 has a mean reward of $-\Delta$, the cumulative regret must also be at least $\Delta \cdot \mathbb{E}_{\vectorgreek{\nu},\vectorgreek{\nu}^{\aux}}[\tilde{n}_{1,T}]$. Therefore the regret is lower bounded by $\frac{\Delta}{2} \left( \sum\limits_{t=1}^Tp^{\mathrm{sub}}_t + \mathbb{E}_{\vectorgreek{\nu},\vectorgreek{\nu}^{\aux}}[\tilde{n}_{1,T}]\right)$, which is greater than $\frac{\sigma^2}{4\Delta} \log \left( \frac{\Delta^2}{2\sigma^2}\sum\limits_{t=1}^{T} \exp\left(-\frac{2 \Delta^2}{\hat{\sigma}^2}\sum\limits_{s=1}^th_{1,s}\right)\right)$. The argument can be repeated by switching arms $1$ and $2$. For $K$ arms, we follow the above lines and average over the established bounds to obtain:\vspace{-0.15cm}
\[
\mathcal{R}^\pi_{\cal S}(\vector{H}, T) \ge \frac{\sigma^2(K-1)}{4K\Delta} \sum\limits_{k=1}^K\log \left( \frac{\Delta^2}{\sigma^2K}\sum\limits_{t=1}^{T} \exp\left(-\frac{2\Delta^2}{\hat{\sigma}^2}\sum\limits_{s=1}^th_{k,s}\right)\right),\vspace{-0.15cm}
\]
which establishes the result.

\vspace{-0.0cm}
\subsection{Discussion and Subclasses of Information Arrival Processes}\label{subsec:examples}\vspace{-0.0cm}

Theorem \ref{theorem-lower bound-general} suggests that auxiliary information arrivals may be leveraged to reduce regret rates, and that their impact on the achievable performance increases when information arrives more frequently, and earlier. 
We next demonstrate these implications 
using two information arrival processes of natural interest: a process with a fixed arrival rate, and a process with a decreasing arrival rate.
																							
\vspace{-0.00cm}
\subsubsection{Stationary Information Arrival Process}\label{example-stochastic-lower bound}\vspace{-0.05cm}
	Assume that $h_{k,t}$'s are i.i.d. 
random variables with mean $\lambda$. Then, for any $T \ge 1$ and admissible policy $\pi\in\cP$, one obtains the following lower bound for the achievable performance:\vspace{-0.1cm}
    \begin{enumerate}
    \item If $\lambda \le \frac{\hat{\sigma}^2}{4\Delta^2T}$, then\vspace{-0.2cm}
    \[
        \mathbb{E}_{\vector{H}}\left[ \mathcal{R}^\pi_{\cal S}(\vector{H}, T) \right]
        \ge
         \frac{\sigma^2(K-1)}{4\Delta} \log \left( \frac{(1-e^{-1/2})\Delta^2 T}{K}\right).\vspace{-0.2cm}
        \]
    \item If $\lambda \ge \frac{\hat{\sigma}^2}{4\Delta^2T}$, then\vspace{0.0cm}
    \[
        \mathbb{E}_{\vector{H}}\left[ \mathcal{R}^\pi_{\cal S}(\vector{H}, T) \right]
         \ge
          \frac{\sigma^2(K-1)}{4\Delta} \log \left( \frac{1-e^{-1/2}}{2\lambda K\sigma^2/\hat{\sigma}^2}\right).\vspace{-0.1cm}
        \]
    \end{enumerate}

\noindent This class includes instances in which, on average, information arrives at a constant rate $\lambda$. Analyzing these arrival processes reveals two different regimes. When the information arrival rate is sufficiently low (there are no more than order $\Delta^{-2}$ information arrivals over $T$ time periods), auxiliary observations become essentially ineffective, and one recovers the performance bounds established in the absence of auxiliary observations. 
When $\Delta$ is fixed and independent of the horizon length $T$, the lower bound scales logarithmically with $T$. When~$\Delta$ can scale with $T$, a bound of order $\sqrt{T}$ is recovered when $\Delta $ is of order $T^{-1/2}$. In both cases, there are known policies that guarantee rate-optimal performance (for details see policies, analysis, and discussion in \citealt{auer2002finite-time}).

On the other hand, when there are more than order $\Delta^{-2}$ observations over $T$ periods, the lower bound on the regret depends on the arrival rate $\lambda$. When the arrival rate is independent of the horizon length $T$, the regret is bounded by a constant that is independent of $T$, and a myopic policy (e.g., a policy that for the first $K$ periods pulls each arm once, and at each later period pulls the arm with the current highest estimated mean reward, while randomizing to break ties) is optimal; For analysis and details see sections \ref{appendix-proof-of-example-Stationary-uniform} and \ref{appendix-proof-of-myopic-optimality]} of the appendix.

\vspace{-0.00cm}
\subsubsection{Diminishing Information Arrival Process} \label{example-Decayin-rate}\vspace{-0.0cm}
Assume that $h_{k,t}$'s are random variables such that for each arm $k \in \mathcal{K}$ and time step $t$,\vspace{-0.15cm}
\[
\mathbb{E}\left[\sum\limits_{s=1}^t h_{k,s} \right] = \left\lfloor \frac{\hat{\sigma}^2 \kappa }{2\Delta^2} \log t \right\rfloor,\vspace{-0.15cm}
\]
for some fixed $\kappa> 0$. Then, for any $T \ge 1$ and admissible policy $\pi\in\cP$, one obtains the following lower bounds for the achievable performance:\vspace{-0.0cm}
		\begin{enumerate}
			\item If $\kappa < 1$ then:\vspace{-0.35cm}
			\[
			 \mathbb{E}_{\vector{H}}\left[\mathcal{R}^\pi_{\cal S}(\vector{H}, T)\right]
			\ge
			\frac{\sigma^2(K-1)}{4\Delta} \log\left(
			\frac{\Delta^2/K\sigma^2}{1-\kappa } \left(
			(T+1)^{1-\kappa} - 1
			\right)
			\right).\vspace{-0.2cm}
			\]
			\item If $\kappa > 1$ then:\vspace{-0.25cm}
			\[
		 \mathbb{E}_{\vector{H}}\left[\mathcal{R}^\pi_{\cal S}(\vector{H}, T)\right]
			\ge
			\frac{\sigma^2(K-1)}{4\Delta} \log\left(
			\frac{\Delta^2/K\sigma^2}{\kappa - 1} \left(
			1-\frac{1}{(T+1)^{\kappa-1}}
			\right)
			\right).\vspace{-0.2cm}
			\]
		\end{enumerate}

\noindent This class includes information arrival processes under which the expected number of information arrivals up to time~$t$ is of order~$\log t$. Therefore, it demonstrates the impact of the \emph{timing} of information arrivals on the achievable performance, and suggests that a constant regret may be achieved when the rate of information arrivals is decreasing. Whenever $\kappa < 1$, the lower bound on the regret is logarithmic in $T$, and there are known policies that guarantee rate-optimal performance (see~\citealt{auer2002finite-time}). When~$\kappa > 1$, the lower bound becomes a constant, and one may observe that when~$\kappa$ is large enough a myopic policy is asymptotically optimal. (In the limit~$\kappa\rightarrow 1$ the lower bound is of order~$\log\log T$.) For analysis and details see sections~\ref{appendix-proof-of-example-decaying} and~\ref{appendix-analysis-of-myopic-example-decaying} of the appendix.

\vspace{-0.05cm}
\subsubsection{Discussion}\label{subsec:lower-disc}\vspace{-0.0cm}

One may contrast the classes of information arrival processes described in \S\ref{example-stochastic-lower bound} and \S\ref{example-Decayin-rate} by selecting $\kappa = \frac{2\Delta^2\lambda T}{\sigma^2\log T}$. Then, in both settings the total number of information arrivals for each arm is $\lambda T$ on average. However, while in the first class the information arrival rate is fixed over the horizon, in the second class this arrival rate is higher in the beginning of the horizon and decreases over time. The different timing of the $\lambda T$ information arrivals may lead to different regret rates. For example, selecting $\lambda = \frac{\sigma^2\log T}{\Delta^2 T}$ implies $\kappa=2$. Then, the lower bound in \S\ref{example-stochastic-lower bound} is logarithmic in $T$ (establishing the impossibility of constant regret in that setting), but the lower bound in \S\ref{example-Decayin-rate} is constant and independent of $T$ (in \S4 we show that constant regret is indeed achievable in this setting). This observation conforms to the intuition that earlier observations have higher impact on achievable performance, as at early periods there is only little information that is available (and therefore the marginal impact of an additional observation is larger), and since earlier information can be used for more decisions (as the remaining horizon is longer).\footnote{The subclasses described in \S\ref{example-stochastic-lower bound} and \S\ref{example-Decayin-rate} are special cases of the following setting. Let $h_{k,t}$'s be independent random variables such that for each arm $k$ and time $t$, the expected number of information arrivals up to time $t$ satisfies:\vspace{-0.10cm}
\[
\mathbb{E}\left[\sum_{s=1}^{t} h_{k,s}\right] = \lambda T \frac{t^{1-\gamma} - 1}{T^{1-\gamma} - 1}.\vspace{-0.10cm}
\]
While the expected number of total information arrivals for each arm, $\lambda T$, is determined by the parameter $\lambda$, the concentration of arrivals is governed by the parameter $\gamma$. When $\gamma = 0$ the arrival rate is constant, corresponding to the class described in \S\ref{example-stochastic-lower bound}. As $\gamma$ increases, information arrivals concentrate in the beginning of the horizon, and $\gamma \rightarrow 1$ leads to $\mathbb{E}\left[\sum_{s=1}^{t} h_{k,s} \right] = \lambda T \frac{\log t}{\log T}$, corresponding to the class in \S\ref{example-Decayin-rate}. Then, when $\lambda T$ is of order $T^{1-\gamma}$ or higher, the lower bound is a constant independent of $T$.}

The analysis above suggests that effective optimal policy design might depend on the information arrival process: while policies that explore (and in that sense are not myopic) may be rate-optimal in some cases, a myopic policy that does not explore (except perhaps in a small number of periods in the beginning of the horizon) can be rate-optimal in other cases. However, this identification of rate-optimal policies relies on \emph{prior knowledge} of the information arrival process. In the following sections we therefore study the \emph{adaptation} of polices to unknown information arrival processes, in the sense of guaranteeing rate-optimality without any prior knowledge on the information arrival process.

\vspace{-0.0cm}
\section{Natural Adaptation to the Information Arrival Process}\label{sec:natural-adaptation}\vspace{-0.0cm}
In this section we establish that, when the mappings $\{\phi_k\}$ are known, policies based on upper confidence bounds and posterior sampling exhibit natural robustness to the information arrival process in the sense of achieving the lower bound of Theorem~\ref{theorem-lower bound-general} uniformly over the general class of information arrival processes we consider, without any prior knowledge on the information arrival matrix~$\vector H$.

\vspace{-0.0cm}
\subsection{Robustness of UCB} \label{section-ucb1}\vspace{-0.0cm}
We next consider a UCB1 policy \citep{auer2002finite-time} in which in addition to updating mean rewards and observation counters after the policy's actions, these are also updated after auxiliary observations.  Denote by $n^{\known} _{k,t}$ and $\bar{X}_{k,t}^{\known} $ the weighted number of times a sample from arm $k$ has been observed and the weighted empirical average reward of arm $k$ up to time $t$, respectively:\vspace{-0.25cm}
\begin{equation}\label{eq-def-of-counters-and-empirical-mean}
	n^{\known}_{k,t} \coloneqq \sum\limits_{s=1}^{t-1}  \mathbbm{1}\{\pi_s = k\}  + \sum\limits_{s=1}^{t} \frac{\sigma^2}{\hat{\sigma}^2}h_{k,s}, \quad
	\bar{X}_{k,t}^{\known} \coloneqq \bar{X}_{k,n^{\known}_{k,t}}^{\known} \coloneqq \frac{\sum\limits_{s=1}^{t-1}\frac{1}{\sigma^2}\mathbbm{1}\{\pi_s = k\}X_{k,s} +  \sum\limits_{s=1}^{t} \frac{1}{\hat{\sigma}^2}h_{k,s} Z_{k,s}}{\sum\limits_{s=1}^{t-1} \frac{1}{\sigma^2} \mathbbm{1}\{\pi_s = k\}  + \sum\limits_{s=1}^{t} \frac{1}{\hat{\sigma}^2}h_{k,s} }.\vspace{-0.8cm}
\end{equation} \vspace{-0.0cm}
\begin{framed}\small
	\noindent \textbf{UCB1 with Auxiliary Observations (aUCB1).}\label{Naive-UCB-policy} Inputs: a constant $c$. 
	\vspace{-0.1cm}
	\begin{enumerate}
		\item At each period $t = 1,\dots,T$:\vspace{-0.1cm}
		\begin{enumerate}
			\item Observe the vectors $\vectorgreek{h}_t$  and $\vector{Z}_t$
			\item Select the arm\vspace{-0.1cm}
			\[
			\pi_t = \begin{cases*}
				t & if  $t \le K$  \\
				\arg \max_{k\in\mathcal{K}}\left\{  U_{k,t} \coloneqq \bar{X}_{k,t}^{\known} + \sqrt{\frac{c\sigma^2 \log t}{n^{\known} _{k,t}}} \right\} & if $t > K$
			\end{cases*}\vspace{-0.05cm}
			\]
			\item Receive and observe a reward $X_{\pi_t,t}$\vspace{-0.4cm}
		\end{enumerate}
	\end{enumerate}
\end{framed}\normalsize
\vspace{-0.1cm}

\noindent The next result establishes that by updating counters and reward estimates after both policy's actions and auxiliary information arrivals, the UCB1 policy guarantees rate-optimal performance over the class of general information arrival processes defined in~\S\ref{sec: prob_form}.

\begin{theorem}[Near optimality of UCB1 with auxiliary observations]\label{theorem-UCB1-upper bound}
	Assume that the mappings $\{\phi_k\}_{k=1}^K$ are known, and let $\pi$ be UCB1 with auxiliary observations, tuned by $c>2$. Then, for any $T \ge 1$, $K \ge 2$, and additional information arrival matrix $\vector{H}$:\vspace{-0.05cm}
\[
\mathcal{R}^\pi_{\cal S}(\vector{H}, T) \le
\sum\limits_{k\in\mathcal{K}}
\frac{ \constvar[ucb1] }{\Delta^2} \log \left( \sum\limits_{t=0}^T \exp\left(\frac{-\Delta^2}{\constref{ucb1}}\sum\limits_{s = 1}^{t} h_{k,s}\right)  \right)
+
\constvar[ucb2] \vspace{-0.05cm}
,\]
where $\constref{ucb1}$, and $\constref{ucb2}$ are positive constants that depend only on $\sigma$ and $\hat{\sigma}$.
\end{theorem}

\noindent
The upper bound in Theorem~\ref{theorem-UCB1-upper bound} holds for any arbitrary trajectory of information arrivals that is captured by the matrix $\vector{H}$, and matches the lower bound in Theorem \ref{theorem-lower bound-general} with respect to that trajectory (captured by $h_{k,t}$'s), and with respect to the time horizon~$T$, the number of arms~$K$, the minimum expected reward difference~$\Delta$, as well as $\sigma$ and $\hat{\sigma}$. This establishes that, by updating counters and reward estimates after both policy's actions and auxiliary information arrivals, UCB1 guarantees rate optimality \emph{uniformly} over the general class of information arrival processes defined in~\S\ref{sec: prob_form}. 

When put together, Theorems~\ref{theorem-lower bound-general} and~\ref{theorem-UCB1-upper bound} establish the spectrum of minimax complexities associated with our class of problems, which is characterized as follows.\vspace{-0.1cm}

\begin{remark}\label{remark-big-oh} \textbf{\textup{(Minimax complexity)}} Theorems~\ref{theorem-lower bound-general} and~\ref{theorem-UCB1-upper bound} together identify the spectrum of minimax regret rates for the class of MAB problem with unknown information arrival process (but with prior knowledge of mappings $\{\phi_k\}$), as a function of the realized information arrival process matrix~$\vector{H}$:\vspace{-0.1cm}
	\[
	\mathcal{R}^\ast_{\cal S}(\vector{H}, T) \asymp \sum\limits_{k\in\mathcal{K}}\log\left(\sum\limits_{t=1}^T \vspace{-0.0cm} \exp\left(-c\sum\limits_{s=1}^th_{k,s}\right)\right),\vspace{-0.1cm}
	\]
	where $c$ is a constant that depends on problem parameters such as $K$, $\Delta$, and $\sigma$.\vspace{-0.05cm}
\end{remark}

\noindent
Remark~\ref{remark-big-oh} has a couple of important implications. First, it identifies a spectrum of minimax regret rates ranging from the classical regret rates that were established in the stochastic MAB literature when there is no auxiliary information, to regret that is uniformly bounded over time when information arrives frequently and/or early enough. Furthermore, identifying the minimax complexity provides a yardstick for evaluating the performance of policies, and a sharp criterion for identifying rate-optimal ones. This already identifies UCB1 as rate optimal, but will shortly be used also for establishing the optimality (or sub-optimality) of other policies.

\medskip
\noindent
\textbf{Key Ideas in the Proof.} The proof of Theorem~\ref{theorem-UCB1-upper bound} adjusts the analysis of UCB1 in \cite{auer2002finite-time}. Pulling a suboptimal arm $k$ at time step $t$ implies at least one of the following: $(i)$ the empirical average of the best arm deviates from its mean; $(ii)$ the empirical mean of arm $k$ deviates from its mean; or $(iii)$ arm $k$ has not been pulled sufficiently often in the sense that\vspace{-0.2cm}
\[
\tilde{n}_{k,t-1} \le \hat  l_{k,t} - \sum\limits_{s=1}^t \frac{\sigma^2}{\hat{\sigma}^2} h_{k,s},\vspace{-0.2cm}
\]
where $\hat l_{k,t} = \frac{4c\sigma^2\log \left(\tau_{k,t}\right)}{ \Delta^2}$ with $\tau_{k,t} \coloneqq \sum\limits_{s=1}^t \exp\left(
{\frac{ \Delta_k^2}{4c\hat{\sigma}^2}\sum\limits_{\tau = s}^t h_{k,\tau}}
\right)$, and $\tilde{n}_{k,t-1}$ is the number of times arm $k$ is pulled up to time $t$. The probability of the first two events can be bounded using Chernoff-Hoeffding inequality, and the probability of the third one can be bounded using:\vspace{-0.05cm}
\[
\sum\limits_{t=1}^T \mathbbm{1} \left\{ \pi_t = k, \tilde{n}_{k,t-1} \le  \hat l_{k,t} - \sum\limits_{s=1}^t \frac{\sigma^2}{\hat{\sigma}^2} h_{k,s} \right\}
\le
\max_{1\le t \le T}\left\{ \hat l_{k,t} - \sum\limits_{s=1}^t \frac{\sigma^2}{\hat{\sigma}^2}h_{k,s}\right\}\vspace{-0.05cm}
\]
Applying these bounds, we established the Theorem for $c>2$.

\vspace{-0.05cm}
\subsection{Robustness of Thompson Sampling} \label{section-thompson-sampling}\vspace{-0.05cm}
Consider a Thompson sampling policy with Gaussian priors \citep{agrawal2012analysis,agrawal2013further}, but where posteriors are updated both after the policy's actions and after auxiliary observations. We denote by $n^{\known} _{k,t}$ and $\bar{X}_{k,t}^{\known}$ the observation counters and mean reward estimates as defined in~\eqref{eq-def-of-counters-and-empirical-mean}.\vspace{-0.2cm}

\begin{framed}\small
	\textbf{Thompson sampling with auxiliary observations (aTS).}\label{Thompson-Sampling-policy} Inputs: a tuning constant $c$.
	\begin{enumerate}
		\item Initialization: set initial counters $n_{k,0} = 0$ and initial empirical means $\bar{X}_{k,0} = 0$ for all $k\in \cal K$
		
		\item At each period $t = 1,\dots,T$:
		\begin{enumerate}
			\item Observe the vectors $\vectorgreek{h}_t$ and $\vector{Z}_t$
			\item Sample $\theta_{k,t} \sim \mathcal{N}(\bar{X}_{k,t}^{\known} , c\sigma^2 (n^{\known} _{k,t}+1)^{-1})$ for all $k \in \cal K$, and select the arm\vspace{-0.15cm}
			\[
			\pi_t = \arg \max_k \theta_{k,t}.\vspace{-0.15cm}
			\]
			
			\item Receive and observe a reward $X_{\pi_t,t}$\vspace{-0.25cm}
		\end{enumerate}
	\end{enumerate}
\end{framed}\vspace{-0.3cm}\normalsize

\noindent The next result establishes that by deploying a Thompson sampling policy that updates posteriors after the policy's actions \emph{and} whenever auxiliary information arrives, one may guarantee rate optimal performance in the presence of unknown information arrival processes.\vspace{-0.1cm}

\begin{theorem}\label{theorem-Thompson-Sampling-regret-upper bound} \textbf{\textup{(Near optimality of Thompson sampling with auxiliary observations)}}
	Assume that the mappings $\{\phi_k\}_{k=1}^K$ are known, and let $\pi$ be a Thompson sampling with auxiliary observations policy, with $c>0$. For every $T \ge 1$, and auxiliary information arrival matrix $\vector{H}$:\vspace{-0.15cm}
\[
\mathcal{R}^\pi_{\cal S}(\vector{H}, T) \le
\sum\limits_{k\in\mathcal{K}\setminus \{k^\ast\}}
\left(
\frac{\constvar[TS1]}{\Delta} \log
\left(
\sum\limits_{t=0}^T \exp\left(\frac{-\Delta^2}{\constref{TS1}}\sum\limits_{s = 1}^t h_{k,s}\right)
\right)
+
\frac{\constvar[TS2]}{\Delta^3}
+
\constvar[TS3]
\right),\vspace{-0.15cm}
\]
for some absolute positive constants $\constref{TS1}$, $\constref{TS2}$, and $\constref{TS3}$ that depend on only $c$, $\sigma$ and $\hat{\sigma}$.\vspace{-0.1cm}
\end{theorem}

\noindent
The upper bound in Theorem \ref{theorem-Thompson-Sampling-regret-upper bound} holds for any trajectory of information arrivals that is captured by the matrix $\vector{H}$, and matches the lower bound in Theorem \ref{theorem-lower bound-general} with respect to the dependence on that trajectory (captured by $h_{k,t}$'s), as well as $T$, $K$, $\Delta$, $\sigma$, and $\hat{\sigma}$. This establishes that, similarly to UCB1, by updating posteriors after both policy's actions and auxiliary information arrivals, Thompson sampling guarantees rate optimality uniformly over the general class of information arrival processes defined in \S\ref{sec: prob_form}. 

\medskip
\noindent
\textbf{Key Ideas in the Proof.} To establish Theorem~\ref{theorem-Thompson-Sampling-regret-upper bound} we decompose the regret associated with each suboptimal arm $k$ into three components: $(i)$ Regret from selecting arm $k$ when its empirical mean deviates from its expectation; $(ii)$ Regret from selecting arm $k$ when its empirical mean does not deviate from its expectation, but $\theta_{k,t}$ deviates from the empirical mean; and $(iii)$ Regret from selecting arm $k$ when its empirical mean does not deviate from its expectation and $\theta_{k,t}$ does not deviate from the empirical mean. Following the analysis in \cite{agrawal2013thompson} one may use concentration inequalities 
to bound the cumulative regret associated with $(i)$ and $(iii)$ uniformly over time.

To analyze $(ii)$, we consider an arrival process with a rate that is decreasing exponentially with each arrival. To bound the number of arrivals we establish the following lemma.

\begin{lemma}\label{lemma-Arrival-process-with-exponentially-decaying-rate}
	Fix an integer $T\geq 1$. For every integer $1\le t \le T$, let $A_t\in\{0,1\}$ be a binary random variable. Let $w_t \ge 0 $, $1\le t \le T$, be some deterministic constants, and define $w_{T} = 0$. Define $N_t = \sum\limits_{s=1}^{t-1} \left(A_s+w_s \right)+w_t$, and assume $\mathbb{E}\left[A_t \;\middle|\; N_t \right] \le p^{N_t+1}$ for some $p\le1$. Then, for every $1\le t \le T+1$:\vspace{-0.35cm}
	\[
	\mathbb{E}
	\left[
	N_t
	\right]
	\le
	\log_{\frac{1}{p}}
	\left(
	1 + \sum\limits_{s=1}^{t-1}p^{\sum\limits_{\tau=1}^s w_\tau}
	\right)
	+
	\sum\limits_{s=1}^t w_s.\vspace{-0.05cm}
	\]
\end{lemma}

\proof \; \;
Consider the sequence $\{p^{-N_t}:\;t=1,\ldots,T\}$. For every $1\le t \le T$,\vspace{-0.15cm}
\begin{align*}
	\mathbb{E}
	\left[
	p^{-N_t} \;\middle|\; p^{-N_{t-1}}
	\right]
	&=
	\mathbb{P}\left\{ A_{t-1}=1 \;\middle|\; N_{t-1} \right\}
	\cdot p^{-\left(N_{t-1}+w_t +1\right)}
	+
	\mathbb{P}\left\{A_{t-1}=0 \;\middle|\; N_{t-1} \right\}
	\cdot p^{-\left(N_{t-1}+w_t \right)}
	\nonumber \\
	&\le
	p^{N_{t-1}+1} \cdot p^{-\left(N_{t-1}+w_t +1\right)}
	+
	p^{-\left(N_{t-1}+w_t \right)}\\
	&=
	p^{-w_t} \cdot
	\left(
	1 + p^{-N_{t-1}}
	\right).
\end{align*}
Using this inequality and applying a simple induction, one obtains
$
\mathbb{E}
\left[
p^{-N_t}
\right]
\le
\sum\limits_{s=1}^{t} p^{-\sum\limits_{\tau=s}^t w_\tau},
$
for every $t\in\{1,\ldots,T\}$. Finally, by Jensen's inequality, one has
$
p^{-\mathbb{E}
	\left[N_t\right]}
\le
\mathbb{E}
\left[
p^{-N_t}
\right],
$
which implies\vspace{-0.2cm}
\[
\mathbb{E}
\left[N_t\right]
\le
\log_{\frac{1}{p}}
\left(
\sum\limits_{s=1}^{t} p^{-\sum\limits_{\tau=s}^t w_\tau}
\right)
=
\log_{\frac{1}{p}}
\left(
1 + \sum\limits_{s=1}^{t-1}p^{\sum\limits_{\tau=1}^s w_\tau}
\right)
+
\sum\limits_{s=1}^t w_s.\vspace{-0.3cm}
\]
This concludes the proof of the lemma.	\endproof

\medskip
\noindent
Using Lemma~\ref{lemma-Arrival-process-with-exponentially-decaying-rate} we establish that the probability of selecting a sub-optimal action when event~(\textit{ii}) occurs drops exponentially whenever that action is selected or an auxiliary observation is received (similarly to the arrival probability in the setting of Lemma~\ref{lemma-Arrival-process-with-exponentially-decaying-rate}). This, in turn, enables one to bound the expected number of times a suboptimal arm is selected when event (\textit{ii}) occurs, 
and to establish Theorem~\ref{theorem-Thompson-Sampling-regret-upper bound}.\vspace{-0.2cm}


\begin{figure}[h]
\centering		\includegraphics[height=2.25in]{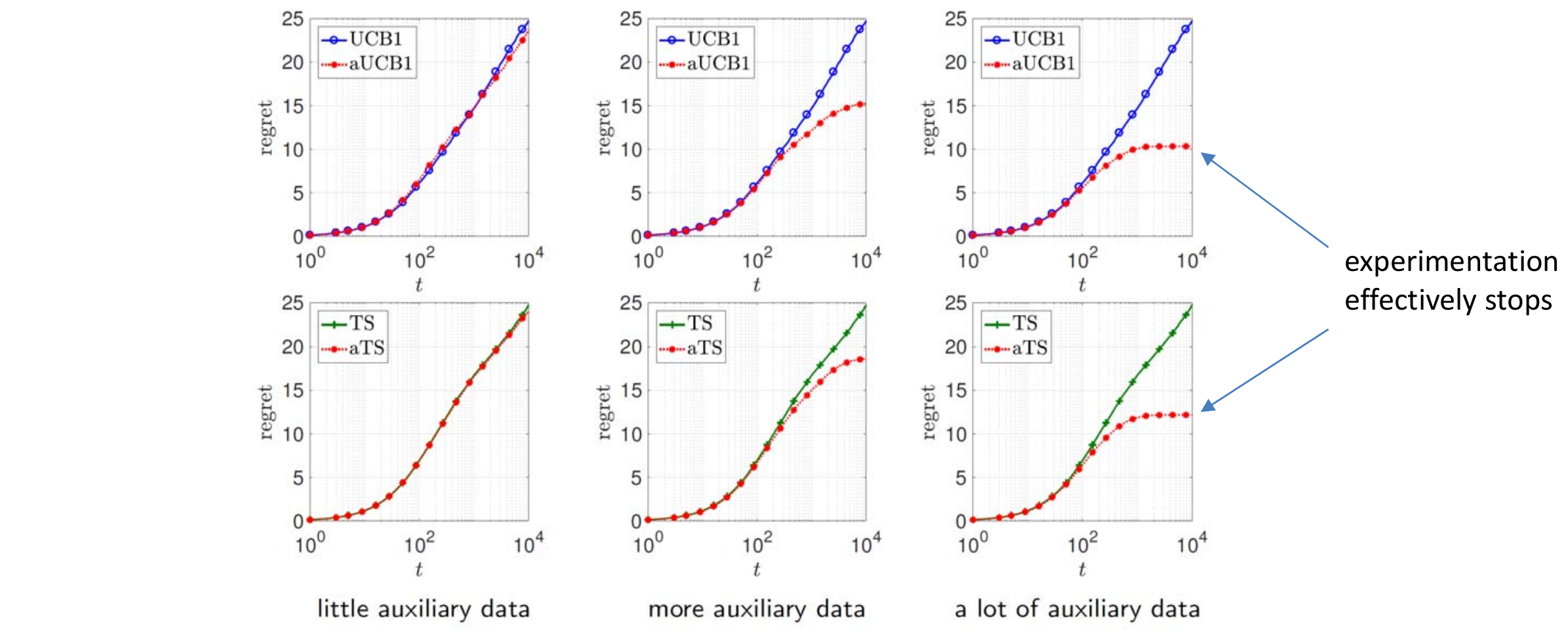}\vspace{-0.2cm}
		\caption{\small Impact of stochastic information arrivals (see \S\ref{example-stochastic-lower bound}) with different arrival rates: little auxiliary data ($\lambda = 0.001$), more auxiliary data ($\lambda = 0.01$) and a lot of auxiliary data ($\lambda = 0.05$). \emph{(Top)} Regret incurred by UCB1 tuned by $c=1.0$ with (aUCB1) and without (UCB1) leveraging auxiliary data. \emph{(Bottom)} Regret incurred by Thompson sampling tuned by $c=0.5$ with (aTS) and without (TS) leveraging auxiliary data. Setup: $T=10^4$, $K=3$, rewards are Gaussian with means $\boldsymbol{\mu} = (0.7, \; 0.5, \; 0.5)^\top$ and standard deviation $\sigma = 0.5$, auxiliary observations are i.i.d. samples from reward distributions, performance averaged over 200 replications.} \label{fig:main-text-sim}\vspace{-0.55cm}
\end{figure}
\subsection{Discussion} \label{section-discussion-on-necessity-of-e-greedy-analysis}\vspace{-0.1cm}

Theorems \ref{theorem-UCB1-upper bound} and \ref{theorem-Thompson-Sampling-regret-upper bound} establish that, when the mappings $\{\phi_k\}$ are known, both UCB1 and Thompson sampling guarantee rate optimality uniformly over the class of information arrival processes defined in~\S\ref{sec: prob_form}. 
Figure~\ref{fig:main-text-sim} visualizes the regret incurred by UCB1 and Thompson sampling for the stationary information arrivals as the rate of arrivals increases; tight upper bounds for the settings in \S\ref{example-stochastic-lower bound} and \S\ref{example-Decayin-rate} for any values of $\lambda$ and $\kappa$ are given in Appendix~\ref{subsec:app:corollaries} (Corollaries~\ref{corollary-regret-upper bound-stationary} and \ref{corollary-regret-upper bound-decaying-rate}), and numerical analysis with a variety of tuning parameters and information arrival scenarios is provided in Appendix~\ref{section-numerics}.


Our results imply that UCB1 and Thompson sampling possess remarkable robustness with respect to the information arrival process: without any adjustments they guarantee the best achievable performance (up to some multiplicative constant). Furthermore, 
the regret rate guaranteed by these policies cannot be improved by \emph{any} other policy, including ones that are predicated on accurate prior knowledge of the information arrival process. These results therefore uncover an important novel property of these policies that  extends the class of problems to which they have been applied.

To appreciate the above robustness property, it is important to observe that, in general, policies that are rate optimal in the absence of auxiliary information arrivals \emph{do not necessarily maintain optimality} under the general information arrival process considered here. For example, policies with exogenous exploration rate, such as forced sampling and $\epsilon_t$-greedy, do not exhibit such robustness (as we observe in Appendix \ref{subsec:inuition}). The main reason for the sub-optimality of these policies in the presence of auxiliary information arrivals is that regardless of the amount of auxiliary observations collected, they always incur a certain loss due to their predetermined exploration schedule. 
Nevertheless, we note that it is possible to adjust such policies by endogenizing their exploration rate and adjusting it dynamically when auxiliary information is collected. For completeness, in Appendix \S\ref{section:A near-optimal adaptive policy} we detail a method of doing so using \emph{virtual time indices} that are advanced multiplicatively upon auxiliary information arrivals. We further provide an adjusted $\epsilon_t$-greedy-type policy with virtual time indexes, and establish it guarantees rate-optimality in the presence of auxiliary information arrivals. (We note that similarly to $\epsilon_t$-greedy-type policies, this policy requires the accurate minimal expected reward difference $\Delta$ as an input.)



\color{black}
\section{Leveraging Auxiliary Information with Unknown Mappings}\label{sec:unknown-mappings}\vspace{-0.0cm}
In the previous section we have established the rate optimality of UCB1 and Thompson sampling in the presence of auxiliary information. Importantly, the rate optimality of these policies is predicated on precise prior knowledge of the mappings $\left\{\phi_k\right\}$ that allow the interpretation of auxiliary data for constructing reward estimates. When the mappings $\left\{\phi_k\right\}$ are misspecified, these policies are not rate optimal anymore, and the performance they achieve may deteriorate (and may even become worse than the performance they would have achieved without access to auxiliary data; this phenomenon will be demonstrated in~\S\ref{section-real data}).
In the current section, we address the challenge of designing adaptive policies that leverage auxiliary information without precise prior knowledge of the mappings $\left\{\phi_k\right\}$. For simplicity, we assume that mappings are known to belong to the class $\Phi_L$ of linear mappings described in Example~\ref{exp:linear-mappins} with $0 \le \lalpha \le  \alpha_k \le \ualpha$ for some known $\ualpha$ and~$\lalpha$, and assume that auxiliary observations have non-negative mean, that is, $y_k\geq 0$ for all $k\in\cal K$. (We note that extending the analysis of this section to the class of linear mappings that includes non-zero intercepts is immediate.)

It is important to observe that when the mappings $\{\phi_k\}$ are unknown, one could always construct instances where regret rates that are achievable in the absence of auxiliary information must be incurred by any admissible policy, even when auxiliary data is abundant, as we next demonstrate.\vspace{-0.05cm}

\begin{example}\textbf{\textup{(No performance gain with unknown mappings)}}\label{ex:nogain}
Let $K=2$ and let $\Phi=\Phi_L$ be a class of linear mappings from Example~\ref{exp:linear-mappins}. Assume it is known that $\alpha_k\in\{\alpha_1,\alpha_2\}$ for $k\in\{1,2\}$ and that $\mu_1\in\{\mu_2+\Delta,\mu_2-\Delta\}$. Suppose that for all $t\in\{1,\ldots,T\}$ and $m\in\{0,1,\ldots,h_{1,t}\}$, auxiliary observations obtain $Y_{1,t,m}= y_1=\frac{\mu_2+\Delta}{\alpha_2} = \frac{\mu_2-\Delta}{\alpha_1}$ with probability 1. Note that if $\mu_1 = \mu_2-\Delta$, then $\mu_1 = E[\alpha_1 Y_{1,t,m}]$, and if $\mu_1 = \mu_2+\Delta$, then $\mu_1 = E[\alpha_2 Y_{1,t,m}]$. Therefore auxiliary observations are uninformative for identifying $\mu_1$. Following the proof of Theorem~1, one obtains that for any $T\ge 1$ and information arrival matrix $\vector{H}$, the regret that any policy $\pi\in\mathcal{P}$ may guarantee is bounded from below as~follows:\vspace{-0.25cm}
	\[
	 \mathcal{R}^\pi_{\cal S}(\vector{H}, T)
	\ge
	\frac{\constref{lower_bound1}(K-1)}{\Delta} \log \left(
	\frac{\constref{lower_bound2}\Delta^2}{K} T
	\right),\vspace{-0.05cm}
	\]
where $\constref{lower_bound1}$ and $\constref{lower_bound2}$ are the same positive constants that appear in the statement of Theorem \ref{theorem-lower bound-general}.\vspace{-0.1cm}
\end{example}

\noindent We therefore turn to derive instance-dependent bounds on regret as opposed to guarantees over a class of problems. To provide meaningful instance-dependent lower bounds we focus on policies that are rate-optimal in the absence of auxiliary data. This set of policies can be characterized as follows.\footnote{A similar definition (super-fast convergent) is given in \cite{garivier2019explore}. Note that there is a variety of policies that guarantee the above regret rate; see, e.g., \cite{auer2002finite-time} that studies a class of problem instances with bounded support.}\vspace{-0.1cm}
\begin{definition}[Information-agnostic-optimal policies]\label{def:IAO-policies}
	We say a policy $\pi \in \cal{P}$ is \textit{information-agnostic-optimal (IAO)} over the class of problem instances $\cS
	$ if for any problem instance $(\vectorgreek{\nu}, \vectorgreek{\nu}^{\mathrm{aux}}) \in \cS$, any suboptimal arm $k \in \cK$, and horizon length $T\ge 2$,\vspace{-0.175cm}
	\[
	\mathbb{E}^{\pi}_{\vectorgreek{\nu}, \vectorgreek{\nu}^{\aux}}[n_{k,T+1}^{\pi}]
	\le
	\frac{C_{\pi, \cS} }{\Delta_k^2} \log T,\vspace{-0.175cm}
	\]
	for some constant $C_{\pi, \cS}>0$. We denote the class of such policies by $\cal{P}^{\mathrm{IAO}}$.\vspace{-0.1cm}
\end{definition}

\vspace{-0.1cm}
\subsection{Identifying Different Performance Regimes}\vspace{-0.1cm}
As demonstrated in Example~\ref{ex:nogain}, when the mappings $\{\phi_k\}$ are unknown it is not clear a priori if and when one should expect performance improvement by utilizing auxiliary data. As a preliminary step, in this subsection we consider a simple version of our problem where prior to time $t=1$ the decision maker has access to the means of the auxiliary observations $\{y_k\}$, and study under what circumstances better performance is achievable compared to classical regret rates that are established in the absence of auxiliary information. Concretely, we assume access to a single noiseless auxiliary observation on each arm, that is, $\hat{\sigma}=0$ and $h_{k,1} = 1$ for $k\in \cK$. Upon collecting the auxiliary means $\{y_k\}$, the agent's belief regarding the possible values of mean rewards changes. For each $k\in \cK$ the agent learns that $\mu_k \in [\lalpha \cdot y_k, \ualpha \cdot y_k]$; Figure~\ref{fig:regimes} depicts different scenarios in the case of two arms and $y_1 \ge y_2$. Depending on the value of $\mu_1$, we consider two different regimes. In Regime~1, $\mu_1$ is not in the intersection of the posterior support of the mean rewards: $\mu_1 \in ( \max\{\lalpha \cdot y_1, \ualpha \cdot y_2\}, \ualpha \cdot y_1 ]$. In Regime 2, $\mu_1$ is inside that intersection: $\mu_1 \in [\lalpha \cdot y_1, \max\{\lalpha \cdot y_1, \ualpha \cdot y_2\}]$.

\begin{figure}[h]
	\centering \includegraphics[width=.9\textwidth]{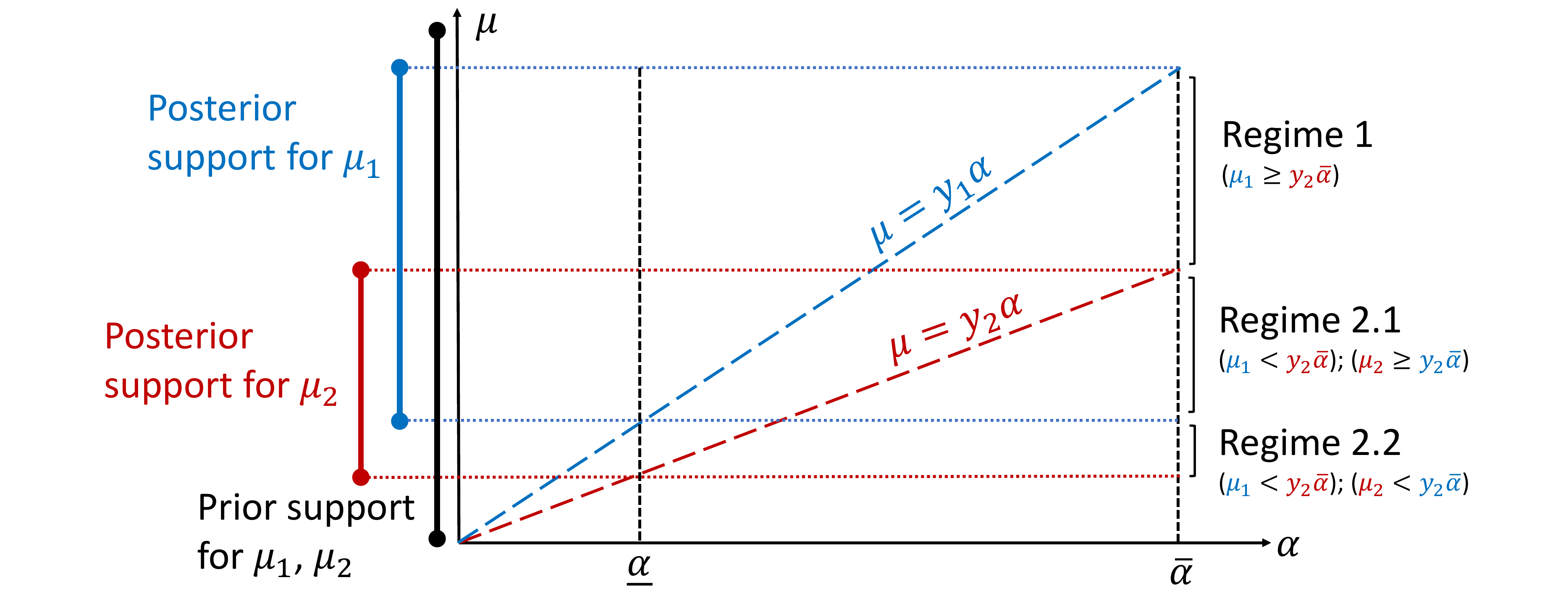}\vspace{-0.2cm}
	\caption{\small The posterior supports after receiving noiseless auxiliary information for each arms. Regime 1: $\mu_1$ is not in the intersection of the posterior supports of the mean rewards; Regime 2: $\mu_1$ is in that intersection.}\vspace{-0.1cm}
	\label{fig:regimes}
\end{figure}

Consider the following variant of UCB1, to which we refer as UCB1+, that conditional on observing the auxiliary means $\{y_k\}$, deploys for each arm $k\in\cal K$ an upper confidence bound that is bounded from above by $\ualpha \cdot y_k$, the upper bound on the posterior support of $\mu_k$. Formally, for each $k$ and $t$ we define:
\vspace{-0.2cm}
\[
U_{k,t}
=
\min\l\{
\bar{X}^{\pi}_{k,t} + \sqrt{\frac{c\sigma^2 \log t}{n^{\pi}_{k,t}}}, \ualpha \cdot y_k
\r\},\vspace{-0.2cm}
\]
where $\bar{X}^{\pi}_{k,t}$ and $n^{\pi}_{k,t}$ are the empirical mean reward and the number of pulls (see Table~\ref{table:notations}). We next characterizes the performance of this policy when the auxiliary means $\{y_k\}$ are observed before $t=1$.\vspace{-0.1cm}

\begin{theorem}\textbf{\textup{(Near optimality of UCB1+)}} \label{theorem-UCB-regret-upper-bound-unknown-extreme}\label{theorem:UCB1+}
	Let $\Phi=\Phi_L$ with $0 \le  \lalpha \le \alpha_k \le \ualpha$ for some known $\lalpha,\ualpha$, where $\hat \sigma = 0$. Let $\pi$ be the UCB1+ policy, tuned by $c>2$. Then, for every $T \ge 1$, auxiliary information arrival matrix $\vector{H}$ such that $h_{k,1} = 1$ for $k\in \cK$, and problem instance $(\vectorgreek{\nu}, \vectorgreek{\nu}^{\mathrm{aux}}) \in \cS$, one has:\vspace{-0.1cm}
	\[
	\cR^{\pi}_{\vectorgreek{\nu}, \vectorgreek{\nu}^{\aux}}(\vector{H}, T)
	\le
	\sum\limits_{\substack{k\in \mathcal{K}\setminus\{k^\ast\} \\  \ualpha y_k \ge \mu^\ast}}
	\frac{ \constref{ucb1} }{\Delta_k^2} \log T
	+
	\constref{ucb2} \vspace{-0.1cm}
	,\]
	where the constants $\constref{ucb1}$ and $\constref{ucb2}$ appear in Theorem~\ref{theorem-UCB1-upper bound}. Furthermore, this policy is optimal among the class
$\cal{P}^{\mathrm{IAO}}$ up to a multiplicative constant that is independent of $T$, and logarithmic in $K$ and $\Delta_k$'s.
\end{theorem}

\noindent Theorem \ref{theorem:UCB1+} suggests that when the mappings $\{\phi_k\}$ are unknown, auxiliary information on arm $k$ may still be leveraged to reduce the loss from selecting that arm (potentially to a regret that is uniformly bounded over time by a constant), but only when $\ualpha y_k < \mu^\ast$; only then, auxiliary information can effectively reduce the upper confidence bound associated with that arm. Recalling Figure~\ref{fig:regimes}, one may observe that performance improvement is therefore only possible in Regime 1, and the value of~$\lalpha$ does not impact the regret rate. Henceforth we simply let $\lalpha=0$.

\vspace{-0.0cm}
\subsection{Impact of Unknown Mappings on Achievable Performance}\vspace{-0.0cm}
In this subsection, we study the performance that one could aspire to achieve in the absence of prior knowledge of true mappings $\{\phi_k\}$, and under an arbitrary and unknown information arrival matrix $\vector{H}$.
The next result formalizes what \textit{cannot} be achieved in the presence of unknown mappings, establishing instance-dependent lower bounds on the best achievable performance.\vspace{-0.1cm}

\begin{theorem}\label{theorem-lower bound-unknown} \textbf{\textup{(Lower bound on achievable performance with unknown mappings)}} Let $\Phi=~\Phi_L$ with $0 \le \alpha_k \le \ualpha$ for some known~$\ualpha$. 
Then, for any $T\geq 1$ and $k\in\cal{K}$ and for any policy $\pi\in \cal{P}^{\mathrm{IAO}}$, the followings hold with $\delta_k = \mu^\ast -  \ualpha \cdot y_k $:
	\begin{enumerate}
	  \item If $\ualpha \cdot y_k > \mu^\ast$ then,
		$
		\mathbb{E}^{\pi}_{\vectorgreek{\nu}, \vectorgreek{\nu}^{\aux}_t}[n_{k,T+1}^{\pi}]
		\ge \frac{\constref{lower-bound-unknown1}}{\Delta_k^{2}}\log \l( \frac{\constref{lower-bound-unknown2}  \min\{4\Delta_k^4,(\Delta_k-\delta_k)^2 \delta_k^2\}}{ K \log T}  T\r);
		$\vspace{-0.1cm}
	  \item If $\ualpha \cdot y_k < \mu^\ast$ then,
		$
		\mathbb{E}^{\pi}_{\vectorgreek{\nu}, \vectorgreek{\nu}^{\aux}_t}[n_{k,T+1}^{\pi}]
		\ge \frac{\constref{lower-bound-unknown1}}{\Delta_k^{2}}\log \l( \frac{\constref{lower-bound-unknown2}  (\Delta_k+\delta_k)^2\delta^2_k}{ K \log T}  \sum_{t=1}^T\exp\l(- \constref{lower-bound-unknown3} \delta_k^2 \cdot\sum\limits_{s=1}^t h_{k,s} \r)\r),
$
	\end{enumerate}
where $\constvar[lower-bound-unknown1]$, $\constvar[lower-bound-unknown2]$, and $\constvar[lower-bound-unknown3]$ are positive constants that only depend on $\sigma$, $\hat{\sigma}$, and $\ualpha$.
\end{theorem}

\noindent
The product $\ualpha \cdot y_k$ represents an \textit{optimistic} mean reward estimate based on auxiliary information. Depending on whether it is larger or smaller than $\mu^\ast$, the above lower bound can be logarithmic in the horizon length $T$ or a decreasing log-sum-exp function of the cumulative information arrivals, respectively.

\medskip
\noindent \textbf{Key Ideas in the Proof.} The proof of Theorem \ref{theorem-lower bound-unknown} relies on the following lemma which provides regret lower bound for any general class of bandit problems.\vspace{-0.1cm}

\begin{lemma}\label{lemma-lower bound-unknown} \textbf{\textup{(General Lower Bound with Unknown Mappings)}} For any $T \ge 2$ and $k\in \cal{K}$, any information arrival matrix $\vector{H}$, any information-agnostic-optimal policy  $\pi\in \cal{P}^{\mathrm{IAO}}$, any problem instance $(\vectorgreek{\nu}, \vectorgreek{\nu}^{\mathrm{aux}}) \in \cS$, one has for any $\epsilon_k>0$:\vspace{-0.1cm}
	\small
	\[
		\cR^{\pi}_{\vectorgreek{\nu}, \vectorgreek{\nu}^{\aux}}(\vector{H}, T) \ge
		\sum_{k \in \cK \setminus \{k^\ast\}}
		\sup_{\substack{(\tilde{\vectorgreek{\nu}}, \tilde{\vectorgreek{\nu}}^{\mathrm{aux}}) \in \cS \\ \rE_{X \sim \tilde \nu_k}[X]=\mu^\ast+\epsilon_k }}
		\hspace{-.6cm}
		\frac{
			\Delta_k
			\log \l( \frac{\epsilon_k^2\KL(\nu_k, \tilde \nu_k)}{C_{\pi, \cS} K }  \sum\limits_{t=1}^T\exp\l[-\KL(\nu_k^\aux, \tilde \nu_k^\aux) \cdot\sum\limits_{s=1}^t h_{k,s} \r]\r) - \log \log T
		}{\KL(\nu_k, \tilde \nu_k)} -\Delta_k.\vspace{-0.1cm}
	\]
\normalsize
\end{lemma}
Lemma~\ref{lemma-lower bound-unknown} implies that in the absence of precise knowledge on mappings it is not always possible to achieve performance improvement. For example, consider arm $k$ for which there exists some problem instance $(\tilde{\vectorgreek{\nu}}, \tilde{\vectorgreek{\nu}}^{\mathrm{aux}}) \in \cS$ such that $\rE_{X \sim \tilde \nu_k}[X]>\mu^\ast $ and $\KL(\nu_k^{\mathrm{aux}}, \tilde \nu_k^{\mathrm{aux}}) = 0$. Then, the regret would be at least logarithmic in $T$, a regret rate one would incur without accounting for auxiliary information.



\subsection{Adapting to Unknown Mappings Based on Upper Confidence Bounds}\label{section-2ucb}\vspace{-0.0cm}
We next develop the \textit{2-UCBs} policy that adapts to a priori unknown linear mappings \emph{and} information arrival processes, and then establish its near optimality. 
(Relevant notation is provided in Table~\ref{table:notations}.)
\newcolumntype{C}[1]{%
	>{\vbox to 5ex\bgroup\vfill\centering}%
	p{#1}%
	<{\egroup}}
\begin{table}[]
	\centering          
	\caption{Notations for counters and empirical means}\vspace{-0.1cm}
	\label{table:notations}
	\footnotesize   
	\begin{tabular}{| l | l<{\rule[-0mm]{0pt}{5.5mm}}|}   
		\toprule[2pt]                   
		\textbf{Definition} & \textbf{Description} \\
		\midrule
		$n^{\pi}_{k,t} \coloneqq \sum_{s=1}^{t-1}\mathbbm{1}\{\pi_s = k\}$  &  Number of pulls         \\
		\hline
		$\bar{X}^{\pi}_{k,t} \coloneqq \bar{X}^{\pi}_{k,n^{\pi}_{k,t}} \coloneqq
		\l({\sum_{s=1}^{t-1}\mathbbm{1}\{\pi_s = k\}X_{k,s} }\r)/\l({\max\{1,n^{\pi}_{k,t} \}}\r)$  &   Empirical mean reward    \\
		\hline
		$n^{\aux}_{k,t} \coloneqq \sum_{s=1}^{t}  h_{k,s}$ &     Number of auxiliary observations          \\
		\hline
		$\bar{Y}_{k,t} \coloneqq \bar{Y}_{k,n^{\aux}_{k,t}} \coloneqq \l({   \sum_{s=1}^{t} \sum_{m=1}^{h_{k,s}} Y_{k,s,m}}\r) / \l({ \max\{1,n^{\aux}_{k,t}\} }\r)$    &  Empirical mean of auxiliary observations \\
		\hline
		$n^{\pi, \mathrm{aux}}_{k,t} \coloneqq n^{\pi}_{k,t} + \frac{\sigma^2}{\ualpha^2\hat \sigma^2} n^{\aux}_{k,t} $ & Weighted number of  observations\\
		\hline
		$\bar{X}^{\pi, \mathrm{aux}}_{k,t} \coloneqq
		\bar{X}^{\pi, \mathrm{aux}}_{k,n^{\pi}_{k,t}, n^{\aux}_{k,t}} \coloneqq
		\l({n^{\pi}_{k,t} \cdot \bar{X}^{\pi}_{k,n^{\pi}_{k,t}} + \frac{\sigma^2}{\ualpha^2\hat \sigma^2} \cdot  n^{\aux}_{k,t} \cdot \ualpha\bar{Y}_{k,t} }\r)/\l({\max\{1,n^{\pi, \mathrm{aux}}_{k,t} \}}\r)$  &   Optimistic empirical mean reward
		\\
		\bottomrule
	\end{tabular}\vspace{-0.3cm}
	\label{table:tree-term}
\end{table}
\normalsize

\medskip
\noindent
\textbf{2-UCBs Policy.} The 2-UCBs policy deploys the following two upper confidence bounds:\vspace{-0.1cm}
\[
	U^{\pi}_{k,t} \coloneqq
	\bar{X}^{\pi}_{k,t} + \sqrt{\frac{c\sigma^2 \log t}{n^{\pi}_{k,t}}};\quad\quad\quad
	U^{\pi, \mathrm{aux}}_{k,t} \coloneqq
	\bar{X}^{\pi, \mathrm{aux}}_{k,t} + \sqrt{\frac{c\sigma^2 \log t}{n^{\pi, \mathrm{aux}}_{k,t}}}.\vspace{-0.1cm}
\]
and tightens the final UCB by using the minimum of the two,\vspace{-0.15cm}
\[
	U_{k,t}
	\coloneqq
	\min\l\{ U^{\pi}_{k,t} , U^{\pi, \mathrm{aux}}_{k,t} \r\}.\vspace{-0.15cm}
\]
The first upper confidence bound is typical in the MAB literature (see, e.g., \citealt{auer2002finite-time}) and is based on reward observations. The second one incorporates both reward observations and auxiliary observations. Note that $\bar{X}^{\pi, \mathrm{aux}}_{k,t}$ is an optimistic estimator for the mean reward of arm $k$ since it linearly combines the estimators $\bar{X}^{\pi}_{k,t}$ and $\ualpha\bar{Y}_{k,t}$, where the latter over-estimates $\mu_k$ whenever $\alpha_k < \ualpha$. Therefore $\bar{X}^{\pi, \mathrm{aux}}_{k,t}$ is a biased estimate, while $\bar{X}^{\pi}_{k,t}$ is unbiased. On the other hand, $\bar{X}^{\pi, \mathrm{aux}}_{k,t}$ incorporates a larger number of samples compared to $\bar{X}^{\pi}_{k,t}$, and hence has a smaller stochastic error.

With only few observations, the stochastic error is a key driver of experimentation with sub-optimal actions. On the other hand, when there is abundant auxiliary data on a suboptimal arm $k$ for which $\alpha_k y_k \ge \mu^\ast$, then the upward bias of the optimistic empirical mean $\bar{X}^{\pi, \mathrm{aux}}_{k,t}$ can be substantial. 
Selecting the minimum of the two UCB's efficiently balances these opposing risks.\vspace{-0.1cm}

\begin{theorem}\label{theorem-2UCB-regret-upper-bound-unknown} \textbf{\textup{(Near optimality of 2-UCBs)}}
	Assume that $\Phi=\Phi_L$ with $0 \le \alpha_k \le \ualpha$ for some known~$\ualpha$, and let $\pi$ be the 2-UCBs policy tuned by $c>2$. Then, for any $\epsilon >\frac{1}{2}$, and every $T \ge 1$, auxiliary information arrival matrix $\vector{H}$, and problem instance $(\vectorgreek{\nu}, \vectorgreek{\nu}^{\mathrm{aux}}) \in \cS$, 
and without prior knowledge of the mappings, one has for $\delta_k = \mu^\ast - \ualpha\cdot y_k$:
\vspace{-0.0cm}
	\begin{enumerate}
		\item If $\ualpha \cdot y_k \ge \mu^\ast$, then
		$
		\mathbb{E}^{\pi}_{\vectorgreek{\nu}, \vectorgreek{\nu}^{\aux}_t}[n_{k,T+1}^{\pi}]
		\le \frac{\constvar[ucb1-unknown]}{\Delta_k^{2}}\log T;\vspace{-0.05cm}
		$
		\item If $\mu^\ast > \ualpha \cdot y_k > \mu^\ast - \epsilon\Delta_k$, then 
		$
		\mathbb{E}^{\pi}_{\vectorgreek{\nu}, \vectorgreek{\nu}^{\aux}_t}[n_{k,T+1}^{\pi}]
		\le
		\min\l\{ \frac{\constref{ucb1-unknown}}{\Delta_k^{2}}\log
		T,\; \frac{\constref{ucb1-unknown}}{\delta_k^{2}}\log \left(
		\sum\limits_{t=0}^T \exp\l(
		-\constref{ucb2-unknown}\delta_k^2\sum\limits_{s = 1}^t h_{k,s}
		\r)
		\right) +
		\constvar[ucb3-unknown]\r\};\vspace{-0.1cm}
		$
		\item If $\ualpha \cdot y_k \leq \mu^\ast - \epsilon\Delta_k$, then
		$
		\mathbb{E}^{\pi}_{\vectorgreek{\nu}, \vectorgreek{\nu}^{\aux}_t}[n_{k,T+1}^{\pi}]
		\le \frac{\constref{ucb1-unknown}}{\Delta_k^{2}}\log \left(
		\sum\limits_{t=0}^T \exp\l(
		-\constvar[ucb2-unknown]\delta_k^2\sum\limits_{s = 1}^t h_{k,s}
		\r)
		\right) +
		\constref{ucb3-unknown};
		$
	\end{enumerate}
	where $\constref{ucb1-unknown}, \constref{ucb2-unknown}$, and $\constref{ucb3-unknown}$ are positive constants that only depend on $\sigma$, $\hat{\sigma}$, $\ualpha$, $\epsilon$, and $c$.\vspace{-0.1cm}
\end{theorem}

\noindent
Theorem~\ref{theorem-2UCB-regret-upper-bound-unknown} implies that when both the information arrival matrix $\mathbf{H}$ and the mappings $\{\phi_k\}$ are a priori unknown, 2-UCBs guarantees rate-optimality with respect to arms for which either $\ualpha \cdot y_k > \mu^\ast$ or $\ualpha \cdot y_k \leq \mu^\ast - \epsilon\Delta_k$, uniformly over $(i)$ the class of arbitrary arrival processes, and $(ii)$ the class $\Phi_L$ of linear mappings. In particular, when $\ualpha \cdot y_k\ge \mu^\ast$ for some suboptimal arm $k$, then 2-UCBs guarantees the regret that is achievable without auxiliary information, and when $\ualpha \cdot y_k \leq \mu^\ast - \epsilon\Delta_k$, it guarantees regret that is a decreasing log-sum-exp function of the auxiliary information arrivals.

Recalling Figure~\ref{fig:regimes}, one may observe that the case $\mu^\ast > \ualpha \cdot y_k > \mu^\ast - \epsilon\Delta_k$ corresponds to the part of Regime 1 that is alongside Regime 2. In that case, 2-UCBs still guarantees a regret rate that is a decreasing log-sum-exp function of auxiliary information arrivals, but with a multiplicative factor that is proportional to $\delta_k^{-2}$ as opposed to $\Delta_k^{-2}$ (but never worse than the rate that is achievable without auxiliary information).

\vspace{-0.0cm}
\section{Simulation Using Content Recommendations Data}\label{section-real data}\vspace{-0.0cm}
To demonstrate the value that may be captured by leveraging auxiliary data, we use data from a large US media site to analyze performance when recommending articles that have unknown impact on the future browsing path of readers. For complete description of the setup and analysis see Appendix~\ref{sec:realdata}.

\medskip
\noindent\textbf{Background.} Content recommendations, which point readers to content they ``may like," are a form of dynamic service provided by media sites and third-party providers, with the objective of increasing readership and revenue streams for media sites. In that context, there are two key performance indicators that are fundamental for recommendations that guarantee good performance along the reading path of readers: $(i)$ the likelihood of a reader to click on a recommendation, and $(ii)$ when clicking on it, the likelihood of the reader to \emph{continue} and consume additional content \emph{afterwards} (see discussion and analysis in \citealt{besbes2015optimization}). The former indicator is the click-through rate (CTR) of the recommendation, and the latter indicator can be viewed as the conversion rate (CVR) of the recommendation.

The likelihood of readers to continue consuming content after reading an article is affected by design features of that article (e.g., its organization or the number of photos it features) that are not indicated by the recommendation linked to that article (and therefore typically do not affect CTR). These features are subject to editorial adjustments that may take place after CTR estimates are already formed (say, a few hours or days after the article's release), with the objective of increasing readership.

Rather than which \emph{article} to recommend, our setup describes sequential experiments designed for evaluating which \emph{version} of a given article to recommend, and how changes in the design of the article impact the likelihood of readers to continue consuming content \textit{after} reading it. Based on our data, we simulate the extent to which the performance of such experiments could be improved when utilizing auxiliary observations, available in the form of the browsing path of readers that arrived to these articles directly from external search (and not by clicking a recommendation). Our data consists of a list of articles, together with: $(i)$ times at which these articles were recommended and the in-site browsing path that followed these recommendations; and $(ii)$ times at which readers arrived to these articles directly from a search engine (such as Google) and the browsing path that followed these visits.\footnote{The media site at hand follows a non-subscription model, and hence common user-idiosyncratic contexts (such as age, or preferences over content) are not available to the recommender system in real time.}

\medskip
\noindent\textbf{Setup.} For each article in our data we considered a one-armed bandit setting with a known outside option to simulate experimentation with a new version of that article. Each bandit experiment was based on one day of data. We constructed a decision horizon $t=1,2,\ldots,2000$ based on the first 2,000 time-stamps of that day at which the article was recommended from the highest position (out of 5 links that are presented in each page). For each article-day pair, we assume that the CTR is identical across both versions of the article and denote it by $\mathrm{CTR}_{a,d}$; for simplicity we assume that it is known. We calculated from the data the fraction of occasions where readers continued to read another article \emph{after} arriving to article $a$ by clicking on a content recommendation, and denote it by $\mathrm{CVR}^{recom}_{a,d}$. We denote by $\mathrm{CVR}^{\text{recom}}_{a,d,0}$ the known conversion rate of article $a$ at its current design, and by $\mathrm{CVR}^{\text{recom}}_{a,d,1}$ a new (unknown) conversion rate that corresponds to a new structure that is under consideration. We assume that $\mathrm{CVR}^{recom}_{a,d,1}=\mathrm{CVR}^{recom}_{a,d}$ and that $\mathrm{CVR}^{recom}_{a,d,0} = \mathrm{CVR}^{recom}_{a,d} + s\Delta_{a,d}$, where at each replication we set $s\in\{-1,1\}$ with equal probabilities, and where $\Delta_{a,d}$ was selected in the range $[0.01, 0.04]$.

Given an article-day pair, we define the independent random variables $W_{t}\sim\mathrm{Ber}(\mathrm{CTR}_{a,d})$ and $X_{k,t} \sim~\mathrm{Ber}(\mathrm{CVR}^{recom}_{a,d,k})$, for $k\in\{0,1\}$. The average regret incurred by policy $\pi\in\cal{P}$ is denoted by $\cR^{\pi}_{a,d}$, and calculated by averaging over replications of the expression\footnote{As click-through rates are identical across article versions, this performance measure coincides with maximizing the (revenue-normalized) one-step lookahead objective, which was shown by \cite{besbes2015optimization} to approximate the recommendation value in the context of content recommendations in media sites.}\vspace{-0.2cm}
\[
\sum_{t=1}^T W_{t}\left(\max_{k\in{0,1}}\left\{\mathrm{CVR}^{\text{recom}}_{a,d,k}\right\}-  X_{\pi_t,t}\right). \vspace{-0.2cm}
\]
We assume that $X_{\pi_t,t}$ is observed only when the recommendation is clicked, that is, when $W_{t}=1$.

\medskip\noindent
\textbf{Auxiliary Information.} For each article-day pair, the matrix $\mathbf{H}$ and the sample path of auxiliary observations were \emph{fixed} and determined as follows. We extracted the trajectory of information arrivals $\left\{h_{1,t}\right\}$ from the number of readers that arrived to the article from a search engine between consecutive decision epochs. For each epoch $t$ and arrival-from-search $m\in\left\{1,\ldots,h_{1,t}\right\}$ we denote by $Y_{1,t,m}\in\left\{0,1\right\}$ an indicator of whether the reader continued to read additional content in the media site after visiting the article. We denote by $\mathrm{CVR}^{\text{search}}_{a,d}$ the fraction of readers that continued to another article after arriving to article $a$ from search during day $d$. We let\vspace{-0.15cm}
\begin{equation*}
\alpha_{a,d} \coloneqq \frac{\mathrm{CVR}^{\text{recom}}_{a,d}}{\mathrm{CVR}^{\text{search}}_{a,d}}\vspace{-0.25cm}
\end{equation*}
denote the fraction of conversion rates for users that arrive to article $a$ by clicking a recommendation, and for users that arrive to it from search.\footnote{Conversion rates of readers that clicked on a recommendation are typically higher than conversion rates of readers that arrived from search (values of $\alpha_{a,d}$ are in the range $[1,16]$). See related discussions in \cite{besbes2015optimization} on experienced versus inexperienced readers, and in \cite{caro2020managing} on followers versus new readers.} Given $\alpha_{a,d}$, auxiliary observations $\left\{Y_{1,t,m}\right\}$ can be mapped to reward observations using a mapping $\phi$ that belongs to the class of linear mappings~$\Phi_L$ from Example~\ref{subsubsec:mappings}.

We compare the performance of UCB1 without auxiliary data, with the one achieved by utilizing auxiliary data based on estimated mappings (aUCB1, see \S\ref{section-ucb1}), and the adaptive 2-UCBs policy (see \S\ref{section-2ucb}); a description of how these policies are updated under the information structure at hand is provided in Appendix~\ref{app-sec-policy-description}. It is important to clarify that the purpose of our analysis is not to study effective methods of estimating the mapping $\phi$ from data, but rather, to ($i$) demonstrate that mappings could be indeed estimated from historical data in a simple manner, and ($ii$) investigate the performance sensitivity of aUCB1 to mapping misspecification and, correspondingly, the potential benefit of deploying an adaptive policy such as 2-UCBs. With that in mind, as an estimator of $\phi$, we simply used $\hat{\alpha}_{a,d}=\alpha_{a,d-1}$, the fraction of the two conversion rates from the previous day, and deployed that estimate in the aUCB1 policy. We used $\bar{\alpha}_{a,d} = 1.1 \hat{\alpha}_{a,d}$ as the upper bound that is used in the 2-UCBs policy.\vspace{-0.025cm}

\medskip\noindent
\textbf{Evaluating Performance.} For each article-day pair we denote the \emph{Relative Improvement} (RI) achieved by policy $\pi$ relative to a UCB1 policy (that does not account for auxiliary information) by\vspace{-0.15cm}
\[
\text{RI}_{a,d}(\pi) \coloneqq	\frac{\cR^{\text{UCB1}}_{a,d} - \cR^{\pi}_{a,d} }{\cR^{\text{UCB1}}_{a,d}}.\vspace{-0.15cm}
\]
Since each article-day pair is associated with a different auxiliary information trajectory and estimation of $\alpha_{a,d}$, we present the Relative Improvement of each policy along two key dimensions. The first one is the \textit{Auxiliary Information Effectiveness} ($\AIE$) index, defined as follows:\vspace{-0.15cm}
\[
\AIE_{a,d} \coloneqq
\log T - \log \left( \sum\limits_{t=1}^{T} \exp\left(-
c
\sum\limits_{s=1}^th_{1,s}\right)\right),\vspace{-0.15cm}
\]
where $c=\tilde{c}\cdot (\Delta_{a,d}/\hat{\sigma}\alpha_{a,d})^2$, and where $\hat{\sigma}$ is set globally to $1/4$ for Bernoulli observations, $\Delta_{a,d}$ and $\alpha_{a,d}$ were set as described above, and $\tilde{c}$ is a scaling parameter.\footnote{For consistency with our analytical bounds we set $\tilde{c}$ to $4 \times (\text{UCB tuning parameter})$, which, in Figure~\ref{fig-Empirics-results-eff-idx-ucb}, equals $0.2$.}

The $\AIE$ index captures the effectiveness of the information arrival trajectory, which is impacted by frequency and timing of auxiliary information arrivals (as captured by the matrix $\vector{H}$), and other parameters that determine the informativeness of each auxiliary observation (recall that in order to focus on the impact of auxiliary information arrivals we set $T=2,000$ in all the experiments). Note that the $\AIE$ index equals $0$ in the absence of auxiliary information and is positive otherwise. When auxiliary observations arrive more frequently and earlier this index increases according to the log-sum-exp rate that determines the minimax complexity of the problem (see Remark~\ref{remark-big-oh}). The second dimension we consider is the \emph{Relative Mapping Misspecification} (RMM), defined as follows:\vspace{-0.15cm}
\[
\text{RMM}_{a,d} \coloneqq
\text{CVR}_{a,d} \left|1- \frac{\hat{\alpha}_{a,d}}{\alpha_{a,d}}\right|.\vspace{-0.15cm}
\]
The RMM captures the bias in auxiliary observations that is caused by misspecification of the mappings. It equals $0$ when the true mappings are accurately specified, and is increasing with the mapping misspecification and the true conversion rate.

\medskip
\noindent
\textbf{Results and Discussion.} The Relative Improvements as a function of the Auxiliary Information Effectiveness and the Relative Mapping Misspecification appear in Figure~\ref{fig-Empirics-results-eff-idx-ucb} (For additional summary statistics see Appendix~\ref{sec:realdata}). We note that for a given AIE, which is determined by the trajectory of auxiliary information arrivals, performance is sensitive to the particular order of auxiliary observations along that trajectory.\footnote{We further note that each of the scatter plots in Figure~\ref{fig-Empirics-results-eff-idx-ucb} depicts a projection of performances onto a single dimension: AIE or RMM; a 3D plot of RI as a function of both AIE and RMM can be viewed in \href{https://ygur.people.stanford.edu/adaptive-sequential-experiments-plot}{https://ygur.people.stanford.edu/adaptive-sequential-experiments-plot}.} 

\begin{figure} [h]
	\centering \includegraphics[width=0.87\textwidth]{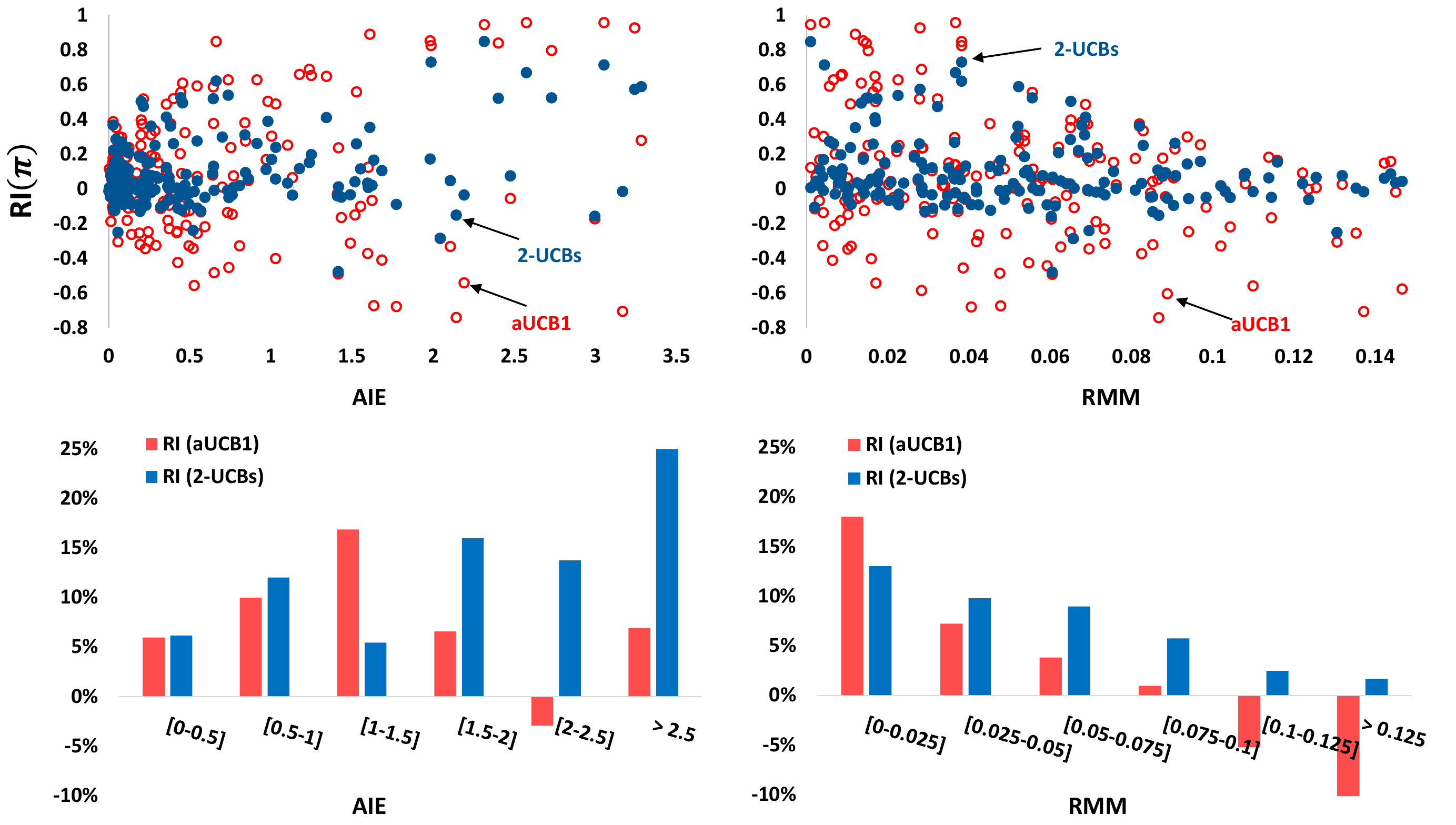}\vspace{-0.2cm}
	\caption{\small \emph{(Top)} Relative Improvement as a function of the Auxiliary Information Effectiveness and the Relative Mapping Misspecification. Each point corresponds to the Relative Improvement of an algorithm (aUCB1 or 2-UCBs) compared to UCB1 for the one-armed bandit experiment associated with a certain article-day pair, averaged over 200 replications. \emph{(Bottom)} Histograms detailing the average Relative Improvements over different ranges of AIE and RMM. (Parameters: $\Delta_{a,b} = 0.03$ globally, policies tuned by $c=0.05$.)}
	\label{fig-Empirics-results-eff-idx-ucb}\vspace{-0.35cm}
\end{figure}

Figure~\ref{fig-Empirics-results-eff-idx-ucb} reveals few important occurrences. Relative to ignoring auxiliary data, aUCB1 and 2-UCBs achieved averaged performance improvement of $7.1\%$ and $8.8\%$, respectively. With little to moderate AIE, aUCB1 achieved notable RI, but as AIE grew large the variance of RI values increased and average RI decreased. This trend demonstrates that little auxiliary data can be leveraged to significantly to improve performance even when mappings are estimated from historical data using simple methods, but also that abundant auxiliary data could lead to performance deterioration whenever auxiliary data is misinterpreted (mappings are misspecified). On the other hand, 2-UCBs adapts to unknown mappings, and its performance improves quite consistently with the effectiveness of auxiliary data (with $\text{AIE} \geq 1.5$, 2-UCBs outperformed aUCB1 and achieved average RI of $19.2\%$).

Reviewing the RI as a function of RMM reveals that the RI values of aUCB1 consistently decrease with the mapping misspecification. When the misspecification is small aUCB1 achieves significant improvement ($18\%$ when $\text{MMR} \leq 0.025$), but when the mapping misspecification grows large the performance of aUCB1 deteriorates and is eventually dominated by ignoring auxiliary data. The RI achieved by 2-UCBs also decreases with the MMR (as estimates of mappings are used to construct the upper bound $\bar{\alpha}$) but more moderately, and on average, 2-UCBs outperformed aUCB1 when $\text{MMR} \geq 0.025$.

\vspace{-0.1cm}
\section{Concluding Remarks}\label{sec:remarks}\vspace{-0.1cm}

\textbf{Summary and Implications.} We considered an extension of the MAB framework by allowing for arbitrary and unknown auxiliary information arrival processes that are relevant to many application domains, including product recommendations, pricing, health care, and machine learning. We studied the impact of the information arrival process on policy design and the performance that can be achieved. Through matching lower and upper bounds we identified a spectrum of minimax (regret) complexities for this class of problems as a function of the information arrival process, which provides a sharp criterion for identifying rate-optimal policies.

When it is known how to map auxiliary observations to reward estimates, we established that Thompson sampling and UCB1 can be leveraged to uniformly guarantee rate-optimality and, in that sense, possess natural robustness to the information arrival process (robustness which is not possessed by policies with exogenous exploration rate). Moreover, the regret rate that is achieved by Thompson sampling and UCB1 cannot be improved even with prior knowledge of the information arrival process.


When it is a priori unknown how to map auxiliary observations to reward estimates, we provide necessary and sufficient conditions that delineate when auxiliary data can still lead to performance improvement, and devise a new (and simple) 2-UCBs policy that is near optimal in that setting.

Using content recommendations data, we demonstrated the value that can be captured in practice by leveraging auxiliary information. Our analysis implies that adaptivity with respect to the mappings that interpret auxiliary data is particularly crucial in settings characterized by $(i)$ significant uncertainty about these mappings; or $(ii)$ access to abundant auxiliary data.

\subsection{Avenues for Future Research}
\textbf{Endogenous Information Arrival Processes.} One may consider information arrival processes that may depend on past decisions and observations. A class of endogenous arrival processes may include various instances, such as information arrivals that are decoupled across arms (selecting an action can impact only future information on that action), and interconnected information arrivals where each action may impact future information arrivals on all actions. An equilibrium analysis with endogenous information arrival processes is an interesting and challenging research avenue.

\medskip
\noindent \textbf{Adapting to Unknown Mappings Using Posterior Sampling.} In addition to the approach taken at~\S\ref{sec:unknown-mappings}, which is based on upper confidence bounds, performance guarantees with unknown mappings could potentially be established also based on posterior sampling. For example, one may consider the following \emph{two-dimensional Thompson sampling} policy that is based on joint posterior updates of mean rewards and the means of auxiliary observations, $(\mu_k,y_k)$. For each pair $(\mu,y)$, define the weight functions\vspace{-0.15cm}
\[
w_{k,t}(\mu, y)
\coloneqq
\exp\l(
\frac{-n^{\pi}_{k,t}(\mu - \bar{X}^{\pi}_{k,t})^2}{2\sigma^2} +
\frac{-n^{\aux}_{k,t}(y - \bar{Y}_{k,t})^2}{2\hat\sigma^2}
\r).\vspace{-0.15cm}
\]
Denote by $P_{k,t}$ the posterior of action $k$ at time $t$. Then, its density is defined as\vspace{-0.15cm}
\[
p_{k,t}(\mu,y) \coloneqq
\begin{cases}
	\frac{w_{k,t}(\mu, y)}{\int_{0}^{\infty} \int_{\frac{\mu}{\ualpha}}^{\infty} w_{k,t}(\mu^\prime, y^\prime) \intd y^\prime \intd\mu^\prime } & \text{ if } 0 \le \mu \le \ualpha y, 0\le y;
	\\
	0 & \text{ o.w.}\vspace{-0.05cm}
\end{cases}
\]
At each period, after observing the last reward and any realized auxiliary observations, one would sample $(\mu^{(\theta)}_{k,t}, y^{(\theta)}_{k,t}) \sim P_{k,t}$ for all $k \in \cal K$, and select the action that maximizes $\mu^{(\theta)}_{k,t}$. While numerical analysis suggests that performance guarantee that is similar to the one established for 2-UCBs could be established for the above policy, we note that the analysis of this policy is challenging, and requires establishing concentration inequalities across two dimensions under various conditions. 

\medskip
\noindent
\textbf{Extending Upper Bounds with Unknown Mappings.} Additional interesting paths are to extend analysis with unknown mappings to guarantee rate optimality with respect to all problem parameters (beyond $\vector{H}$ and $T$ that are at the focus of the current work) for the case $\mu^\ast > \ualpha \cdot y_k > \mu^\ast - \epsilon\Delta_k$; to design a policy that infers the upper bound $\bar{\alpha}$ throughout the decision horizon and adapts to it without prior knowledge; and to consider other classes of mappings (beyond the linear class on which we focused here).

\small
\setstretch{1.0}
\bibliographystyle{chicago}
\bibliography{references}

\newpage
\setstretch{1.4}
\appendix

\vspace{-0.1cm}
\section{Proofs of Main Results}\vspace{-0.1cm}

\vspace{-0.0cm}
\subsection{Notations}\vspace{-0.1cm}
\noindent For any policy $\pi$ and profiles $\vectorgreek{\nu}$ and $\vectorgreek{\nu}'$, let $\mathbb{P}^\pi_{\vectorgreek{\nu}, \vectorgreek{\nu}^{\mathrm{aux}}}$, $\mathbb{E}^\pi_{\vectorgreek{\nu}, \vectorgreek{\nu}^{\mathrm{aux}}}$, and $\mathbb{R}^\pi_{\vectorgreek{\nu}, \vectorgreek{\nu}^{\mathrm{aux}}}$ denote the probability, expectation, and regret when rewards are distributed according to $\vectorgreek{\nu}$, and auxiliary observations are distributed according to $\vectorgreek{\nu}^{\mathrm{aux}}$. Notation of counters and empirical means is provided in the following table:
\newcolumntype{C}[1]{%
	>{\vbox to 5ex\bgroup\vfill\centering}%
	p{#1}%
	<{\egroup}}
\begin{table}[H]
	\centering          
	\caption{Notations for counters and empirical means}\vspace{-0.1cm}
	\label{table:notations}
	\footnotesize   
	\begin{tabular}{| l | l<{\rule[-0mm]{0pt}{5.5mm}}|}   
		\toprule[2pt]                   
		\textbf{Definition} & \textbf{Description} \\
		\midrule
		$n^{\known} _{k,t} \coloneqq \sum\limits_{s=1}^{t-1}  \mathbbm{1}\{\pi_s = k\}  + \sum\limits_{s=1}^{t} \frac{\sigma^2}{\hat{\sigma}^2}h_{k,s}$  &  Number of pulls (known mappings case)        \\
		\hline
		$\bar{X}_{k,t}^{\known} \coloneqq \bar{X}_{k,n^{\known}_{k,t}}^{\known} \coloneqq \frac{\sum\limits_{s=1}^{t-1}\frac{1}{\sigma^2}\mathbbm{1}\{\pi_s = k\}X_{k,s} +  \sum\limits_{s=1}^{t} \frac{1}{\hat{\sigma}^2}h_{k,s} Z_{k,s}}{\sum\limits_{s=1}^{t-1} \frac{1}{\sigma^2} \mathbbm{1}\{\pi_s = k\}  + \sum\limits_{s=1}^{t} \frac{1}{\hat{\sigma}^2}h_{k,s} }$  &  \begin{tabular}{@{}l} Number of pulls \\ (known mappings case)
		\end{tabular}       \\
		\hline
		$n^{\pi}_{k,t} \coloneqq \sum_{s=1}^{t-1}\mathbbm{1}\{\pi_s = k\}$  &  Number of pulls         \\
		\hline
		$\bar{X}^{\pi}_{k,t} \coloneqq \bar{X}^{\pi}_{k,n^{\pi}_{k,t}} \coloneqq
		\l({\sum_{s=1}^{t-1}\mathbbm{1}\{\pi_s = k\}X_{k,s} }\r)/\l({\max\{1,n^{\pi}_{k,t} \}}\r)$  &   Empirical mean reward    \\
		\hline
		$n^{\aux}_{k,t} \coloneqq \sum_{s=1}^{t}  h_{k,s}$ &     Number of auxiliary observations          \\
		\hline
		$\bar{Y}_{k,t} \coloneqq \bar{Y}_{k,n^{\aux}_{k,t}} \coloneqq \l({   \sum_{s=1}^{t} \sum_{m=1}^{h_{k,s}} Y_{k,s,m}}\r) / \l({ \max\{1,n^{\aux}_{k,t}\} }\r)$    &  Empirical mean of auxiliary obs. \\
		\hline
		$n^{\pi, \mathrm{aux}}_{k,t} \coloneqq n^{\pi}_{k,t} + \frac{\sigma^2}{\ualpha^2\hat \sigma^2} n^{\aux}_{k,t} $ & Weighted number of  observations\\
		\hline
		$\bar{X}^{\pi, \mathrm{aux}}_{k,t} \coloneqq
		\bar{X}^{\pi, \mathrm{aux}}_{k,n^{\pi}_{k,t}, n^{\aux}_{k,t}} \coloneqq
		\l({n^{\pi}_{k,t} \cdot \bar{X}^{\pi}_{k,n^{\pi}_{k,t}} + \frac{\sigma^2}{\ualpha^2\hat \sigma^2} \cdot  n^{\aux}_{k,t} \cdot \ualpha\bar{Y}_{k,t} }\r)/\l({\max\{1,n^{\pi, \mathrm{aux}}_{k,t} \}}\r)$  &   Optimistic empirical mean reward
		\\
		\bottomrule
	\end{tabular}\vspace{-0.2cm}
	\label{table:tree-term}
\end{table}

\vspace{-0.5cm}
\subsection{Proof of Theorem \ref{theorem-lower bound-general}}\vspace{-0.1cm}

\paragraph{Step 1 (Preliminaries).} The proof adapts ideas of identifying worst-case nature strategy (see, e.g., \citealt{bubeck2013bounded}) to our setting in order to identify the precise change in the achievable performance as a function of the entries of information arrival matrix~$\vector{H}$. For $m,q\in \{1,\dots,K\}$ define the distribution profiles $\vectorgreek{\nu}^{(m,q)}$:\vspace{-0.2cm}
\[
\nu_k^{(m,q)} = \begin{cases*}
	\mathcal{N}(0,\sigma^2) & if  $k=m$  \\
	\mathcal{N}(+\Delta,\sigma^2) & if $k=q\neq m$ \\
	\mathcal{N}(-\Delta,\sigma^2) & o.w.
\end{cases*}.
\]
For example, for $m=1$, one has\vspace{-0.15cm}
\[
\vectorgreek{\nu}^{(1,1)} = \begin{pmatrix}
	\mathcal{N}(0,\sigma^2) \\ \mathcal{N}(-\Delta,\sigma^2) \\ \mathcal{N}(-\Delta,\sigma^2) \\ \vdots \\ \mathcal{N}(-\Delta,\sigma^2)
\end{pmatrix}
,\;
\vectorgreek{\nu}^{(1,2)} = \begin{pmatrix}
	\mathcal{N}(0,\sigma^2) \\ \mathcal{N}(+\Delta,\sigma^2) \\ \mathcal{N}(-\Delta,\sigma^2) \\ \vdots \\ \mathcal{N}(-\Delta,\sigma^2)
\end{pmatrix}
,\;
\vectorgreek{\nu}^{(1,3)} = \begin{pmatrix}
	\mathcal{N}(0,\sigma^2) \\ \mathcal{N}(-\Delta,\sigma^2) \\ \mathcal{N}(+\Delta,\sigma^2) \\ \vdots \\ \mathcal{N}(-\Delta,\sigma^2)
\end{pmatrix}
, \dots , \;
\vectorgreek{\nu}^{(1,K)} = \begin{pmatrix}
	\mathcal{N}(0,\sigma^2) \\ \mathcal{N}(-\Delta,\sigma^2) \\ \mathcal{N}(-\Delta,\sigma^2) \\ \vdots \\ \mathcal{N}(+\Delta,\sigma^2)
\end{pmatrix}.\vspace{-0.3cm}
\]
Similarly, assume that the auxiliary information $Y_{k,t,m}$ is distributed according to the reward distribution $\hat{\nu}_k^{(m,q)} \overset{d}{=} \nu_k^{(m,q)}$, and hence we use the notation $\mathbb{P}^\pi_{\vectorgreek{\nu}}$, $\mathbb{E}^\pi_{\vectorgreek{\nu}}$, and $\mathbb{R}^\pi_{\vectorgreek{\nu}}$ instead of $\mathbb{P}^\pi_{\vectorgreek{\nu}, \vectorgreek{\nu}^{\mathrm{aux}}}$, $\mathbb{E}^\pi_{\vectorgreek{\nu}, \vectorgreek{\nu}^{\mathrm{aux}}}$, and $\mathbb{R}^\pi_{\vectorgreek{\nu}, \vectorgreek{\nu}^{\mathrm{aux}}}$.

\paragraph{Step 2 (Lower bound decomposition).}
We note that
\begin{equation}\label{eq-lower bound-decomposition}
	\mathcal{R}^\pi_{\cal S}(\vector{H}, T) \ge \max_{m,q\in\{1,\dots,K\}}\left\{\mathcal{R}^\pi_{\vectorgreek{\nu}^{(m,q)}}(\vector{H}, T)\right\} \ge \frac{1}{K}\sum\limits_{m=1}^K \max_{q\in\{1,\dots,K\}} \left\{\mathcal{R}^\pi_{\vectorgreek{\nu}^{(m,q)}}(\vector{H}, T)\right\}.
\end{equation}

\paragraph{Step 3 (A naive lower bound for $\max\limits_{q\in\{1,\dots,K\}}\{\mathcal{R}^\pi_{\vectorgreek{\nu}^{(m,q)}}(\vector{H}, T)\}$).} We note that
\begin{equation}\label{eq-lower bound-I}
	\max_{q\in\{1,\dots,K\}}\{\mathcal{R}^\pi_{\vectorgreek{\nu}^{(m,q)}}(\vector{H}, T)\} \ge \mathcal{R}^\pi_{\vectorgreek{\nu}^{(m,m)}}(\vector{H}, T)= \Delta \cdot \sum\limits_{k\in \mathcal{K} \setminus\{ m\} } \mathbb{E}_{\vectorgreek{\nu}^{(m,m)}}[n^{\pi}_{k,T+1}].
\end{equation}

\paragraph{Step 4 (An information theoretic lower bound).} For any profile $\vectorgreek{\nu}$, denote by $\vectorgreek{\nu}_t$ the distribution of the observed rewards up to time $t$ under $\vectorgreek{\nu}$. By Lemma \ref{gerchinovitz2016refined-lemma-1}, for any $q \neq m$, one has
\begin{equation}\label{eq-KL-div}
	\mathrm{KL}(\vectorgreek{\nu}^{(m,m)}_t,\vectorgreek{\nu}^{(m,q)}_t)  =\frac{2\Delta^2}{\sigma^2} \cdot \mathbb{E}_{\vectorgreek{\nu}^{(m,m)}}[n_{q,t}] = \frac{2\Delta^2}{\sigma^2}\left(\mathbb{E}_{\vectorgreek{\nu}^{(m,m)}}[n^{\pi}_{q,t}] + \sum\limits_{s=1}^t\frac{\sigma^2}{\hat{\sigma}^2} h_{q,s}\right).
\end{equation}

\noindent One obtains:
\begin{align}\label{eq-lower bound-II}
	\max_{q\in\{1,\dots,K\}}\{\mathcal{R}^\pi_{\vectorgreek{\nu}^{(m,q)}}(\vector{H}, T)\}
	&\ge
	\frac{1}{K}\mathcal{R}^\pi_{\vectorgreek{\nu}^{(m,m)}}(\vector{H}, T) +  \frac{1}{K}\sum\limits_{q\in \mathcal{K} \setminus\{ m\} }\mathcal{R}^\pi_{\vectorgreek{\nu}^{(m,q)}}(\vector{H}, T)
	\nonumber \\
	& \ge
	\frac{\Delta}{K}\sum\limits_{t=1}^T\sum\limits_{k\in \mathcal{K} \setminus\{ m\} }\mathbb{P}_{\vectorgreek{\nu}^{(m,m)}}\{\pi_t = k\} + \frac{\Delta}{K}\sum\limits_{q\in \mathcal{K} \setminus\{ m\} }\sum\limits_{t=1}^T\mathbb{P}_{\vectorgreek{\nu}^{(m,q)}}\{\pi_t \neq q\}
	\nonumber \\
	& =
	\frac{\Delta}{K}\sum\limits_{t=1}^T\sum\limits_{q\in \mathcal{K} \setminus\{ m\} }\left(\mathbb{P}_{\vectorgreek{\nu}^{(m,m)}}\{\pi_t = q\} + \mathbb{P}_{\vectorgreek{\nu}^{(m,q)}}\{\pi_t \neq q\} \right)
	\nonumber \\
	& \overset{(a)}{\ge}
	\frac{\Delta}{2K}\sum\limits_{t=1}^{T}\sum\limits_{q\in \mathcal{K} \setminus\{ m\} } \exp(-\mathrm{KL}(\vectorgreek{\nu}^{(m,m)}_t,\vectorgreek{\nu}^{(m,q)}_t))
	\nonumber \\
	& \overset{(b)}{=}
	\frac{\Delta}{2K}\sum\limits_{t=1}^{T}\sum\limits_{q\in \mathcal{K} \setminus\{ m\} }\exp\left[-\frac{2\Delta^2}{\sigma^2}\left(\mathbb{E}_{\vectorgreek{\nu}^{(m,m)}}[n^{\pi}_{q,t-1}] + \sum\limits_{s=1}^t \frac{\sigma^2}{\hat{\sigma}^2} h_{q,s}\right)\right]
	\nonumber \\
	& \overset{(c)}{\ge}
	\frac{\Delta}{2K}\sum\limits_{t=1}^{T}\sum\limits_{q\in \mathcal{K} \setminus\{ m\} }\exp\left[-\frac{2\Delta^2}{\sigma^2}\left(\mathbb{E}_{\vectorgreek{\nu}^{(m,m)}}[n^{\pi}_{q,T+1}] + \sum\limits_{s=1}^t \frac{\sigma^2}{\hat{\sigma}^2} h_{q,s}\right)\right]
	\nonumber \\
	& =
	\sum\limits_{q\in \mathcal{K} \setminus\{ m\} } \frac{\Delta \cdot
		\exp\left(
		-\frac{2 \Delta^2}{\hat{\sigma}^2} \cdot \mathbb{E}_{\vectorgreek{\nu}^{(m,m)}}[n^{\pi}_{q,T+1}]
		\right)
	}{2K}\sum\limits_{t=1}^{T}
	\exp\left(
	-\frac{2 \Delta^2}{\sigma^2} \sum\limits_{s=1}^th_{q,s}
	\right),
\end{align}
where (a) follows from Lemma \ref{Tsybakov2008introduction-lemma-2.6}, (b) holds by (\ref{eq-KL-div}), and (c) follows from $n^\pi_{q,T+1} \ge n^\pi_{q,t}$ for $t\in \mathcal{T}$.

\paragraph{Step 5 (Unifying the lower bounds in steps 3 and 4).} Using (\ref{eq-lower bound-I}), and (\ref{eq-lower bound-II}), we establish\vspace{-0.1cm}
\begin{align}\label{eq-final-lower bound}
	&\max_{q\in\{1,\dots,K\}}\{\mathcal{R}^\pi_{\vectorgreek{\nu}^{(m,q)}}(\vector{H}, T)\}
	\nonumber
	\\
	&\hspace{1cm}\ge
	\frac{\Delta}{2} \sum\limits_{k\in \mathcal{K} \setminus\{ m\} }  \left( \mathbb{E}_{\vectorgreek{\nu}^{(m,m)}}[n^{\pi}_{k,T+1}] + \frac{
		\exp\left(
		-\frac{2 \Delta^2}{\sigma^2} \cdot \mathbb{E}_{\vectorgreek{\nu}^{(m,m)}}[n^\pi_{k,T+1}]
		\right)
	}{2K}\sum\limits_{t=1}^{T}
	\exp\left(
	-\frac{2 \Delta^2}{\hat{\sigma}^2} \sum\limits_{s=1}^th_{k,s}
	\right)
	\right)
	\nonumber \\
	&\hspace{1cm}\ge \frac{\Delta}{2} \sum\limits_{k\in \mathcal{K} \setminus\{ m\} }  \min_{x\ge0} \left( x + \frac{
		\exp\left(
		-\frac{2 \Delta^2}{\sigma^2} \cdot x
		\right)
	}{2K}\sum\limits_{t=1}^{T}
	\exp\left(
	-\frac{2 \Delta^2}{\hat{\sigma}^2} \sum\limits_{s=1}^th_{k,s}
	\right)
	\right)
	\nonumber \\
	&\hspace{1cm}\overset{(a)}{\ge} \frac{\sigma^2}{4\Delta} \sum\limits_{k\in \mathcal{K} \setminus\{ m\} } \log \left( \frac{\Delta^2}{\sigma^2K}\sum\limits_{t=1}^{T}
	\exp\left(
	-\frac{2 \Delta^2}{\hat{\sigma}^2}\sum\limits_{s=1}^th_{k,s}
	\right)
	\right),
\end{align}
where (a) follows from $x + \gamma e^{-\kappa x} \ge \frac{\log \gamma \kappa}{\kappa}$ for $\gamma, \kappa, x > 0$. (Note that the function $x + \gamma e^{-\kappa x}$ is a convex function and we can find its minimum by finding the root of its derivative) The result is then established by putting together (\ref{eq-lower bound-decomposition}), and (\ref{eq-final-lower bound}). \qed

\subsection{Proof of Theorem \ref{theorem-UCB1-upper bound}} Fix a problem instance $(\vectorgreek{\nu}, \vectorgreek{\nu}^{\aux})\in \cS$ with the mean rewards $\vectorgreek{\mu}$ and the mean $\vector{y}$ for auxiliary observations. Consider a suboptimal arm $k\neq k^\ast$. If arm $k$ is pulled at time $t$, then\vspace{-0.2cm}
\[
\bar{X}_{k,t}^{\known} + \sqrt{\frac{c\sigma^2 \log t}{n^{\known} _{k,t}}} \ge \bar{X}_{k^\ast,t}^{\known} + \sqrt{\frac{c\sigma^2 \log t}{n^{\known} _{k^\ast,t}}}.\vspace{-0.2cm}
\]
Therefore, at least one of the following three events must occur:\vspace{-0.2cm}
\begin{align*}
	\mathcal{E}_{1,t} &\coloneqq \left\{  \bar{X}_{k,t}^{\known} \ge \mu_k + \sqrt{\frac{c\sigma^2 \log t}{n^{\known} _{k,t}}} \right\}, \quad \mathcal{E}_{2,t} \coloneqq \left\{  \bar{X}_{k^\ast,t}^{\known} \le \mu_{k^\ast} - \sqrt{\frac{c\sigma^2 \log t}{n^{\known} _{k^\ast,t}}} \right\},\\
	\mathcal{E}_{3,t} &\coloneqq \left\{  \Delta_k \le 2\sqrt{\frac{c\sigma^2 \log t}{n^{\known} _{k,t}}} \right\}.
\end{align*}
To see why this is true, assume that all the above events fail. Then, we have\vspace{-0.2cm}
\[
\bar{X}_{k,t}^{\known} + \sqrt{\frac{c\sigma^2 \log t}{n^{\known} _{k,t}}} < \mu_k + 2\sqrt{\frac{c\sigma^2 \log t}{n^{\known} _{k,t}}} < \mu_k + \Delta_k = \mu_{k^\ast} < \bar{X}_{k^\ast,t}^{\known} + \sqrt{\frac{c\sigma^2 \log t}{n^{\known} _{k^\ast,t}}},\vspace{-0.2cm}
\]
which is in contradiction with the assumption that arm $k$ is pulled at time $t$.
For any sequence $\{ l_{k,t}\}_{t\in\mathcal{T}}$, and $\{\hat l_{k,t}\}_{t\in\mathcal{T}}$, such that $\hat l_{k,t} \ge l_{k,t}$ for all $t\in \mathcal{{T}}$, one has\vspace{-0.2cm}
\begin{align*}
	\mathbb{E}^\pi_{\vectorgreek{\nu}, \vectorgreek{\nu}^{\mathrm{aux}}} \left[n^\pi_{k,T+1}\right]
	&= \mathbb{E}^\pi_{\vectorgreek{\nu}, \vectorgreek{\nu}^{\mathrm{aux}}} \left[ \sum\limits_{t=1}^T \mathbbm{1} \left\{ \pi_t = k \right\}\right]
	\\
	&=
	\mathbb{E}^\pi_{\vectorgreek{\nu}, \vectorgreek{\nu}^{\mathrm{aux}}} \left[ \sum\limits_{t=1}^T \mathbbm{1} \left\{ \pi_t = k, n^{\known} _{k,t} \le l_{k,t}\right\} + \mathbbm{1} \left\{ \pi_t = k, n^{\known} _{k,t} > l_{k,t} \right\}\right]
	\\
	&\le
	\mathbb{E}^\pi_{\vectorgreek{\nu}, \vectorgreek{\nu}^{\mathrm{aux}}} \left[ \sum\limits_{t=1}^T \mathbbm{1} \left\{ \pi_t = k, n^{\known} _{k,t} \le \hat l_{k,t}\right\} + \mathbbm{1} \left\{ \pi_t = k, n^{\known} _{k,t} > l_{k,t} \right\}\right]
	\\
	&\le
	\mathbb{E}^\pi_{\vectorgreek{\nu}, \vectorgreek{\nu}^{\mathrm{aux}}} \left[ \sum\limits_{t=1}^T \mathbbm{1} \left\{ \pi_t = k, n^\pi_{k,t}\le \hat l_{k,t} - \sum\limits_{s=1}^t \frac{\sigma^2}{\hat{\sigma}^2}  h_{k,s} \right\}\right]
	+ \sum\limits_{t=1}^T \mathbb{P} \left\{ \pi_t = k, n^{\known} _{k,t} > l_{k,t} \right\}
	\\
	&\le
	\max_{1\le t \le T}\left\{\hat l_{k,t} - \sum\limits_{s=1}^t \frac{\sigma^2}{\hat{\sigma}^2}  h_{k,s}\right\}
	+ \sum\limits_{t=1}^T \mathbb{P} \left\{ \pi_t = k, n^{\known} _{k,t} > l_{k,t} \right\}.
\end{align*}
Set the values of $ l_{k,t}$ and $\hat l_{k,t}$ as follows:\vspace{-0.1cm}
\begin{equation*}\label{eq-UCB-proof-helper1}
	l_{k,t} = \frac{4c\sigma^2\log \left(
		t
		\right)}{\Delta_k^2}, \quad \text{and} \quad  \hat l_{k,t} = \frac{4c\sigma^2\log \left(
		\tau_{k,t}
		\right)}{\Delta_k^2},
\end{equation*}
where
\begin{equation*}\label{eq-UCB-proof-helper3}
	\tau_{k,t} \coloneqq \sum\limits_{s=1}^t \exp\left(
	{\frac{ \Delta_k^2}{4c\hat{\sigma}^2}\sum\limits_{\tau = s}^t h_{k,\tau}}
	\right) .\vspace{-0.2cm}
\end{equation*}
To have $n^{\known} _{k,t} > l_{k,t}$, it must be the case that $E_{3,t}$ does not occur. Therefore,\vspace{-0.3cm}
\begin{align}\label{eq-UCB-proof-helper2}
	\mathbb{E}^\pi_{\vectorgreek{\nu}, \vectorgreek{\nu}^{\mathrm{aux}}} \left[n^\pi_{k,T+1}\right]
	&\le
	\max_{1\le t \le T}\left\{\hat l_{k,t} - \sum\limits_{s=1}^t \frac{\sigma^2}{\hat{\sigma}^2}  h_{k,s}\right\} + \sum\limits_{t=1}^T \mathbb{P} \left\{ \pi_t = k, \mathcal{E}_{3,t}^c \right\}
	\nonumber \\
	&\le
	\max_{1\le t \le T}\left\{\hat l_{k,t} - \sum\limits_{s=1}^t \frac{\sigma^2}{\hat{\sigma}^2}  h_{k,s}\right\} + 	\sum\limits_{t=1}^T \mathbb{P} \left\{ \mathcal{E}_{1,t} \cup \mathcal{E}_{2,t} \right\}
	\nonumber\\
	&\le
	\max_{1\le t \le T}\left\{\hat l_{k,t} - \sum\limits_{s=1}^t \frac{\sigma^2}{\hat{\sigma}^2}  h_{k,s}\right\} + \sum\limits_{t=1}^T \frac{2}{t^{c/2}},
\end{align}
where the last inequality follows from Lemma \ref{lemma-Chernoff-Hoeffding-bound} and the union bound.
Plugging the value of $\hat l_{k,t}$ into (\ref{eq-UCB-proof-helper2}), one obtains:\vspace{-0.3cm}
\begin{align*}
	\mathbb{E}^\pi_{\vectorgreek{\nu}, \vectorgreek{\nu}^{\mathrm{aux}}} \left[n^\pi_{k,T+1}\right]
	&\le
	\max_{1\le t \le T}\left\{
	\frac{4c\sigma^2}{ \Delta_k^2} \log \left(
	\sum\limits_{s=1}^t\exp\left(
	\frac{\Delta_k^2}{4c\hat{\sigma}^2}\sum\limits_{\tau = s}^t h_{k,\tau} - \frac{ \Delta_k^2}{4c\hat{\sigma}^2}\sum\limits_{\tau = 1}^t h_{k,\tau}
	\right)
	\right)
	\right\} + \sum\limits_{t=1}^T \frac{2}{t^{c/2}}
	\nonumber
	\\
	&\le
	\frac{4c\sigma^2}{\Delta_k^2} \log \left(
	\sum\limits_{t=1}^T \exp\left(
	{\frac{- \Delta_k^2}{4c\hat{\sigma}^2}\sum\limits_{s = 1}^{t-1} h_{k,s}}
	\right)
	\right)+ \sum\limits_{t=1}^T \frac{2}{t^{c/2}},
\end{align*}
which concludes the proof. \qed

\subsection{Proof of Theorem \ref{theorem-Thompson-Sampling-regret-upper bound}}
We adapt the proof of the upper bound for Thompson sampling with Gaussian priors in \cite{agrawal2013further}.

\paragraph{Step1 (Notations and definitions).}Fix a problem instance $(\vectorgreek{\nu}, \vectorgreek{\nu}^{\aux})\in \cS$ with the mean rewards $\vectorgreek{\mu}$ and the mean $\vector{y}$ for auxiliary observations.
For every suboptimal arm $k$, we consider three parameters $x_k$, $y_k$, and $u_k$  such that $\mu_k < x_k < y_k < u_k < \mu_{k^\ast}$. We will specify these three parameters at the end of the proof. Also, define the events $E_{k,t}^\mu$, and $E^\theta_{k,t}$ to be the events on which $\bar X_{k,n^{\known} _{k,t}} \le x_k$, and $\theta_{k,t} \le y_k$, respectively. In words, the events $E_{k,t}^\mu$, and $E^\theta_{k,t}$ happen when the estimated mean reward, and the sample mean reward do not deviate from the true mean, respectively. Define the history $\mathcal{H}_t \coloneqq \left( \{X_{\pi_s,s}\}_{s=1}^{t-1}, \{\pi_s\}_{s=1}^{t-1},\{\vector{Z}_s\}_{s=1}^{t}, \{\vectorgreek{h}_s\}_{s=1}^{t}\right)$ for all $t=1, \dots$. Finally define
\[
p_{k,t} = \mathbb{P}^\pi_{\vectorgreek{\nu}, \vectorgreek{\nu}^{\mathrm{aux}}}\{\theta_{k^\ast,t} > y_k \mid\mathcal{H}_t\}.
\]

\paragraph{Step2 (Preliminaries).}
We will make use of the following lemmas from \cite{agrawal2013further} throughout this proof. The proof of the lemmas are skipped since they are simple adaptation to our setting with auxiliary information arrival.

\begin{lemma}
	\label{lemma-Agrawal-proof-of-theorem1}
	For any suboptimal arm $k$,
	\[
	\sum\limits_{t=1}^T
	\mathbb{P}^\pi_{\vectorgreek{\nu}, \vectorgreek{\nu}^{\mathrm{aux}}} \left \{ \pi_t = k, E_{k,t}^\mu, E^\theta_{k,t} \right\}
	\le
	\sum_{j=0}^{T-1} \Expect \l[
	\frac{1 - p_{k,t_j+1}}{p_{k,t_j+1}}
	\r],
	\]
	where $t_0=0$ and for $j>0$, $t_j$ is the time step at which the optimal arm $k^\ast$ is pulled for the $j$th time.
\end{lemma}
\proof \;\;\; The proof can be found in the analysis of Theorem 1 in \cite{agrawal2013further}. However, we bring the proof for completeness.
\begin{align*}
	\sum\limits_{t=1}^T
	\mathbb{P}^\pi_{\vectorgreek{\nu}, \vectorgreek{\nu}^{\mathrm{aux}}} \left \{ \pi_t = k, E_{k,t}^\mu, E^\theta_{k,t} \right\}
	&=
	\sum\limits_{t=1}^T
	\Expect^\pi_{\vectorgreek{\nu}, \vectorgreek{\nu}^{\mathrm{aux}}}\l[
	\Prob \left \{ \pi_t = k, E_{k,t}^\mu, E^\theta_{k,t} \,\middle| \, \mathcal{H}_t \right\}
	\r]
	\\
	&\overset{(a)}\le
	\sum\limits_{t=1}^T
	\Expect^\pi_{\vectorgreek{\nu}, \vectorgreek{\nu}^{\mathrm{aux}}}\l[\frac{1 - p_{k,t}}{p_{k,t}} \Prob \left \{ \pi_t = 1, E_{k,t}^\mu, E^\theta_{k,t} \,\middle| \, \mathcal{H}_t \right\}
	\r]
	\\
	&=
	\sum\limits_{t=1}^T
	\Expect^\pi_{\vectorgreek{\nu}, \vectorgreek{\nu}^{\mathrm{aux}}}\l[
	\Expect\l[
	\frac{1 - p_{k,t}}{p_{k,t}} \mathbbm{1} \left \{ \pi_t = 1, E_{k,t}^\mu, E^\theta_{k,t} \,\middle| \, \mathcal{H}_t \right\}
	\r]
	\r]
	\\
	&=
	\sum\limits_{t=1}^T
	\Expect^\pi_{\vectorgreek{\nu}, \vectorgreek{\nu}^{\mathrm{aux}}}\l[
	\frac{1 - p_{k,t}}{p_{k,t}} \mathbbm{1} \left \{ \pi_t = 1, E_{k,t}^\mu, E^\theta_{k,t} \right\}
	\r]
	\\
	&\le
	\sum\limits_{j=1}^{T-1}
	\Expect^\pi_{\vectorgreek{\nu}, \vectorgreek{\nu}^{\mathrm{aux}}}\l[
	\frac{1 - p_{k,t_j+1}}{p_{k,t_j+1}}
	\sum_{t=t_j+1}^{t_{j+1}} \mathbbm{1} \left \{ \pi_t = 1\right\}
	\r]
	\\
	&=
	\sum_{j=0}^{T-1} \Expect \l[
	\frac{1 - p_{k,t_j+1}}{p_{k,t_j+1}}
	\r],
\end{align*}
where (a) follows from Lemma 1 in \cite{agrawal2013further}.
\endproof
\begin{lemma}[\protect{\citealt[Lemma 2]
		{agrawal2013further}}]\label{lemma-Agrawal-lemma2}
	For any suboptimal arm $k$,
	\[
	\sum\limits_{t=1}^T \mathbb{P}^\pi_{\vectorgreek{\nu}, \vectorgreek{\nu}^{\mathrm{aux}}} \left \{ \pi_t = k, \bar E_{k,t}^\mu \right\}
	\le
	1 + \frac{\sigma^2}{(x_k-\mu_k)^2}.
	\]
\end{lemma}

\paragraph{Step 3 (Regret decomposition).}
Fix a profile $\vectorgreek{\nu}$. For each suboptimal arm $k$, we will decompose the number of times it is pulled by the policy as follows and upper bound each term separately:
\begin{equation}\label{eq-regert-decomposition-TS}
	\mathbb{E}^\pi_{\vectorgreek{\nu}, \vectorgreek{\nu}^{\mathrm{aux}}} \left [n^\pi_{k,T+1}\right]
	=
	\underbrace{\sum\limits_{t=1}^T \mathbb{P}^\pi_{\vectorgreek{\nu}, \vectorgreek{\nu}^{\mathrm{aux}}} \left \{ \pi_t = k, E_{k,t}^\mu, E^\theta_{k,t}\right\}}_{J_{k,1}}
	+
	\underbrace{\sum\limits_{t=1}^T \mathbb{P}^\pi_{\vectorgreek{\nu}, \vectorgreek{\nu}^{\mathrm{aux}}} \left \{ \pi_t = k, E_{k,t}^\mu, \bar E^\theta_{k,t}\right\}}_{J_{k,2}}
	+
	\underbrace{\sum\limits_{t=1}^T \mathbb{P}^\pi_{\vectorgreek{\nu}, \vectorgreek{\nu}^{\mathrm{aux}}} \left \{ \pi_t = k, \bar E_{k,t}^\mu\right\}}_{J_{k,3}}.
\end{equation}

\paragraph{Step 4 (Analysis of $J_{k,1}$).}
Given Lemma \ref{lemma-Agrawal-proof-of-theorem1}, $J_{k,1}$ can be upper bounded by analyzing $\frac{1-p_{k,t_j+1}}{p_{k,t_j+1}}$:
\begin{align*}
	p_{k,t_j+1}
	&=
	\Prob\l\{
	\mathcal{N}\l(
	\bar X_{k\ast,n_{k\ast,t_j+1}}, c\sigma^2 (n_{k\ast,t_j+1}+1)^{-1}
	\r)
	> y_k
	\r\}
	\\
	&\ge
	\Prob\l\{
	\mathcal{N}\l(
	\bar X_{k\ast,n_{k\ast,t_j+1}}, \frac{c \sigma^2}{j+1}
	\r)
	> y_k \;\middle|\;
	\bar X_{k\ast,n_{k\ast,t_j+1}} \ge u_k
	\r\}
	\cdot
	\Prob\l\{
	\bar X_{k\ast,n_{k\ast,t_j+1}} \ge u_k
	\r\} \eqqcolon Q_1 \cdot Q_2.
\end{align*}
First, we analyze $Q_1$:
\begin{align*}
	Q_1
	&\ge
	\Prob\l\{
	\mathcal{N}\l(
	u_k, \frac{c \sigma^2}{j+1}
	\r)
	> y_k
	\r\} \ge 1 - \exp\l( -\frac{(j+1) (u_k - y_k)^2}{2c\sigma^2}\r),
\end{align*}
where the last inequality follows from Chernoff-Hoeffding bound in Lemma \ref{lemma-Chernoff-Hoeffding-bound}. The term $Q_2$ can also be bounded from below using Chernoff-Hoeffding in Lemma \ref{lemma-Chernoff-Hoeffding-bound} bound as follows:
\[
Q_2 \ge  1 - \exp\l( -\frac{j (\mu_1 - u_k)^2}{2\sigma^2}\r).
\]
Define $\delta_k \coloneqq  \min\{(\mu_1 - u_k), (u_k - y_k)\}$.  The last three displays along with Lemma \ref{lemma-Agrawal-lemma2} yield
\begin{align}\label{eq-TS-step1}
	J_{k,1}
	&\le
	\frac{1}{\Prob\{ \mathcal{N}(0,c\sigma^2) >y_k \}} - 1 +
	\sum_{j=1}^{T}
	\frac{2\exp(-j\delta^2_k/2\sigma^2(c\vee1))}{1-\exp(-j\delta^2_k/2\sigma^2(c\vee1))}
	\nonumber
	\\
	&\le
	\frac{1}{\Prob\{ \mathcal{N}(0,c\sigma^2) >y_k \}} - 1 +
	\int_{1}^\infty
	\frac{2\exp(-x\delta^2_k/2\sigma^2(c\vee1))}{1-\exp(-x\delta^2_k/2\sigma^2(c\vee1))} dx
	\nonumber
	\\
	&\le
	\frac{1}{\Prob\{ \mathcal{N}(0,c\sigma^2) >y_k \}} - 1 -
	\frac{4\sigma^2(c\vee1)}{\delta^2_k} \log \l(1 - \exp(-\delta^2_k/2\sigma^2(c\vee1))  \r).
	\nonumber
	\\
	&\le
	\frac{1}{\Prob\{ \mathcal{N}(0,c\sigma^2) >y_k \}} - 1 +
	\frac{4\sigma^2(c\vee1)}{\delta^2_k} \log \l(\frac{2\sigma^2(c\vee1)}{\delta^2_k}  \r).
\end{align}
\paragraph{Step 5 (Analysis of $J_{k,2}$).}
First, we upper bound the following conditional probability
\begin{align}\label{eq-TS-step4-help1}
	\mathbb{P}^\pi_{\vectorgreek{\nu}, \vectorgreek{\nu}^{\mathrm{aux}}}
	\left \{
	\pi_t = k, E_{k,t}^\mu, \bar E^\theta_{k,t} \;\middle|\; n^{\known} _{k,t}
	\right\}
	&\le
	\mathbb{P}^\pi_{\vectorgreek{\nu}, \vectorgreek{\nu}^{\mathrm{aux}}}
	\left \{
	E_{k,t}^\mu, \bar E^\theta_{k,t} \;\middle|\; n^{\known} _{k,t}
	\right\}
	\nonumber \\
	& \le
	\mathbb{P}^\pi_{\vectorgreek{\nu}, \vectorgreek{\nu}^{\mathrm{aux}}}
	\left \{
	\bar E^\theta_{k,t} \;\middle|\; E_{k,t}^\mu, n^{\known} _{k,t}
	\right\}
	\nonumber \\
	&=
	\mathbb{P}
	\left \{
	\mathcal{N}
	\left(
	\bar{X}_{k,t}^{\known} ,c\sigma^2\l( n^{\known} _{k,t}+ 1\r)^{-1}
	\right)
	> y_k \;\middle|\; \bar{X}_{k,t}^{\known} \le x_k, n^{\known} _{k,t}
	\right\}
	\nonumber \\
	&\le
	\mathbb{P}
	\left \{
	\mathcal{N}
	\left(
	x_k,c\sigma^2\l( n^{\known} _{k,t}+ 1\r)^{-1}
	\right)
	> y_k \;\middle|\;  n^{\known} _{k,t}
	\right\}
	\nonumber \\
	&\overset{(a)}{\le}
	\exp\l(
	\frac{-n^{\known} _{k,t}(y_k-x_k)^2}{2c\sigma^2}
	\r),
\end{align}
where (a) follows from Chernoff-Hoeffding's bound in Lemma \ref{lemma-Chernoff-Hoeffding-bound}. Define
\[
n_{k,t}^\theta
\coloneqq
\sum\limits_{s=1}^t \frac{\sigma^2}{\hat{\sigma}^2} h_{k,s}
+
\sum\limits_{s=1}^{t-1}\mathbbm{1}\left\{\pi_s = k, E_{k,s}^\mu, \bar E^\theta_{k,s}\right\} \le n^{\known} _{k,t}.
\]
By (\ref{eq-TS-step4-help1}), and the definition above, one obtains:
\begin{equation*}\label{eq-TS-step4-help2}
	\mathbb{P}^\pi_{\vectorgreek{\nu}, \vectorgreek{\nu}^{\mathrm{aux}}}
	\left \{
	\pi_t = k, E_{k,t}^\mu, \bar E^\theta_{k,t}\;\middle| \;n_{k,t}^\theta
	\right\}
	\le
	\exp\left(
	{-\frac{n_{k,t}^\theta (y_k-x_k)^2}{2c\sigma^2}}
	\right)
\end{equation*}
for every $t \in \cal T$. Note that by definition,
\begin{equation*}\label{eq-TS-step4-help3}
	n_{k,t}^\theta = n_{k,t-1}^\theta +\frac{\sigma^2}{\hat{\sigma}^2} h_{k,t} + \mathbbm{1}\left\{\pi_{t-1} = k, E_{k,t-1}^\mu, \bar E^\theta_{k,t-1}\right\}.
\end{equation*}
By the above two displays, we can conclude that
\begin{align*}
	\mathbb{E}^\pi_{\vectorgreek{\nu}, \vectorgreek{\nu}^{\mathrm{aux}}}
	\left [
	n_{k,t}^\theta\;\middle| \;n_{k,t-1}^\theta
	\right ]
	&=
	n_{k,t-1}^\theta  + \frac{\sigma^2}{\hat{\sigma}^2} h_{k,t} +
	\mathbb{P}^\pi_{\vectorgreek{\nu}, \vectorgreek{\nu}^{\mathrm{aux}}}
	\left \{
	\pi_{t-1} = k, E_{k,t-1}^\mu, \bar E^\theta_{k,t-1}\;\middle| \;n_{k,t-1}^\theta
	\right\}
	\\
	&\le
	n_{k,t-1}^\theta  + \frac{\sigma^2}{\hat{\sigma}^2} h_{k,t} + e^{-\frac{n_{k,t-1}^\theta (y_k-x_k)^2}{2c\sigma^2}}.
\end{align*}
The following lemma will be deployed to analyze $n_{k,T+1}^\theta$:

Now, we can apply Lemma \ref{lemma-Arrival-process-with-exponentially-decaying-rate} to the sequence $\left \{n_{k,t}^\theta \right \}_{0\le t \le T}$, and obtain
\begin{align*}\label{eq-TS-step4-help4}
	\mathbb{E}^\pi_{\vectorgreek{\nu}, \vectorgreek{\nu}^{\mathrm{aux}}}
	\left [
	n_{k,T+1}^\theta
	\right ]
	&\le
	\log_{\exp\left(
		{(y_k-x_k)^2/2c\sigma^2}
		\right)}
	\left(
	1 + \sum\limits_{t=1}^T
	\exp\left(
	{-\frac{(y_k-x_k)^2}{2c\hat{\sigma}^2} \sum\limits_{s=1}^t h_{k,s}}
	\right)
	\right)
	+
	\sum\limits_{t=1}^T \frac{\sigma^2}{\hat{\sigma}^2}  h_{k,t}
	\nonumber \\
	&\le
	\frac{2c\sigma^2}{(y_k-x_k)^2} \log
	\left(
	1 + \sum\limits_{t=1}^T
	\exp\left(
	{-\frac{(y_k-x_k)^2}{2c\hat{\sigma}^2} \sum\limits_{s=1}^t h_{k,s}}
	\right)
	\right)
	+
	\sum\limits_{t=1}^T  \frac{\sigma^2}{\hat{\sigma}^2} h_{k,t}.
\end{align*}
By this inequality and the definition of $n_{k,t}^\theta$, we have
\begin{equation}\label{eq-TS-step4-help5}
	J_{k,2}
	=
	\sum\limits_{t=1}^T \mathbb{P}^\pi_{\vectorgreek{\nu}, \vectorgreek{\nu}^{\mathrm{aux}}} \left \{ \pi_t = k, E_{k,t}^\mu, \bar E^\theta_{k,t}\right\}
	\le
	\frac{2c\sigma^2}{(y_k-x_k)^2} \log
	\left(
	1 + \sum\limits_{t=1}^T
	\exp\left(
	{-\frac{(y_k-x_k)^2}{2c\hat{\sigma}^2} \sum\limits_{s=1}^t h_{k,s}}
	\right)
	\right).
\end{equation}

\paragraph{Step 6 (Analysis of $J_{k,3}$).} The term $J_{k,3}$ can be upper bounded using Lemma \ref{lemma-Agrawal-lemma2}.

\paragraph{Step 7 (Determining the constants)}
Finally, let $x_k = \mu_k + \frac{\Delta_k}{4}$, $y_k = \mu_k + \frac{2\Delta_k}{4}$, and $u_k = \mu_k + \frac{3\Delta_k}{4}$. Then, by putting (\ref{eq-regert-decomposition-TS}), \eqref{eq-TS-step1}, (\ref{eq-TS-step4-help5}), and  Lemma \ref{lemma-Agrawal-lemma2} back together, the result is established.  \qed

\subsection{Proof of Theorem \ref{theorem:UCB1+}}
The proof of the upper bound is a simple extension of the proof of Theorem \ref{theorem-UCB1-upper bound}. In fact, for the suboptimal arms $k\in \mathcal{K}\setminus\{k^\ast\}$  for which $\ualpha y_k \ge \mu^\ast$, one may repeat the same line of proof replacing . For the suboptimal arms $k\in \mathcal{K}\setminus\{k^\ast\}$  for which $\ualpha y_k < \mu^\ast$, one may repeat the same line of proof replacing $\bar{X}_{k,t}^{\known} + \sqrt{\frac{c\sigma^2 \log t}{n^{\known} _{k,t}}}$ with $\ualpha \cdot y_k$.

The proof of optimality among IAO policies, is an application of Theorem \ref{theorem-lower bound-unknown}. When $\hat \sigma = 0$, Theorem \ref{theorem-lower bound-unknown} reads as follows with $\delta_k = \mu^\ast -  \ualpha \cdot y_k $:
\begin{itemize}
	\item If $\ualpha \cdot y_k > \mu^\ast$ then,
	$
	\mathbb{E}^{\pi}_{\vectorgreek{\nu}, \vectorgreek{\nu}^{\aux}_t}[n_{k,T+1}^{\pi}]
	\ge\constref{lower-bound-unknown1} \Delta_k^{-2} \log \l( \frac{\constref{lower-bound-unknown2}  \min\{4\Delta_k^4,(\Delta_k-\delta_k)^2 \delta_k^2\}}{ K \log T}  T\r);
	$
	\item If $\ualpha \cdot y_k < \mu^\ast$ then,
	$
	\mathbb{E}^{\pi}_{\vectorgreek{\nu}, \vectorgreek{\nu}^{\aux}_t}[n_{k,T+1}^{\pi}]
	\ge 0,
	$
\end{itemize}
where $\constvar[lower-bound-unknown1]$ and $\constvar[lower-bound-unknown2]$ are positive constants that only depend on $\sigma$, and $\ualpha$.

\subsection{Proof of Theorem \ref{theorem-lower bound-unknown}}
In order to prove the theorem, we need to repeat the following argument for each arm $k$:

\noindent Consider  a new problem instance $(\tilde{\vectorgreek{\nu}} , \tilde{\vectorgreek{\nu}}^\mathrm{aux})\in \cS$ in Lemma \ref{lemma-lower bound-unknown}, for which arm $k$ is the optimal arm. The only difference between $(\tilde{\vectorgreek{\nu}} , \tilde{\vectorgreek{\nu}}^\mathrm{aux})$ and $(\vectorgreek{\nu}, \vectorgreek{\nu}^{\mathrm{aux}})$ is the distributions of rewards and auxiliary observations of arm $k$:

\begin{itemize}
	
	\item If $\ualpha \cdot y_k < \mu^\ast$ then,
	$(\tilde\nu_k, \tilde\nu_k^{\mathrm{aux}}) = (\cN(\mu^\ast+\epsilon_k, \sigma^2), \cN(\frac{\mu^\ast+\epsilon_k}{\ualpha},  \hat \sigma^2))$ with $\epsilon_k = \delta_k = \mu^\ast - \ualpha \cdot y_k$;

	\item If $\mu^\ast + \Delta_k \ge \ualpha \cdot y_k > \mu^\ast$ then,
	$(\tilde\nu_k, \tilde\nu_k^{\mathrm{aux}}) = (\cN(\mu^\ast+\epsilon_k, \sigma^2), \cN(y_k,  \hat \sigma^2))$ with $\epsilon_k = -\delta_k = \ualpha \cdot y_k - \mu^\ast$;
	
	\item If $  \ualpha \cdot y_k  > \mu^\ast + \Delta_k $ then,
	$(\tilde\nu_k, \tilde\nu_k^{\mathrm{aux}}) = (\cN(\mu^\ast+\epsilon_k, \sigma^2), \cN(y_k,  \hat \sigma^2))$ with $\epsilon_k = \Delta_k $.

\end{itemize}
This concludes the proof. \qed

\subsection{Proof of Lemma \ref{lemma-lower bound-unknown}}

\paragraph{Step 1 (Preliminaries).} The proof adapts ideas of identifying worst-case nature strategy to our setting in order to identify the precise change in the achievable performance as a function of the entries of information arrival matrix~$\vector{H}$. Fix a problem instance $(\vectorgreek{\nu}, \vectorgreek{\nu}^{\aux})\in \cS$ with the mean rewards $\vectorgreek{\mu}$ and the mean $\vector{y}$ for auxiliary observations, and recall that $\Delta_k = \mu^\ast-\mu_k$ and $\delta_k  = \mu^\ast - \ualpha \cdot y_k$. Consider a suboptimal arm $k$.
We consider  a new problem instance $(\tilde{\vectorgreek{\nu}} , \tilde{\vectorgreek{\nu}}^\mathrm{aux})\in \cS$ for which arm $k$ is the optimal arm such that $\rE_{X \sim \tilde \nu_k}[X]=\mu^\ast+\epsilon_k$.

\paragraph{Step 2 (Distance between problem instances).}
This step is based on the following lemma that states how informative the history can be about the true underlying problem instance.
\begin{lemma}
	Let
	$
	\cH_t
	\coloneqq
	(\vectorgreek{h}_1, \vector{Z}_{1}, \pi_1, X_{\pi_{1},1}, \dots, \vectorgreek{h}_{t-1}, \vector{Z}_{t-1}, \pi_{t-1}, X_{\pi_{t-1},t-1}, \vectorgreek{h}_t, \vector{Z}_{t})
	$
	be the history up to time step $t$. For any problem instance $(\vectorgreek{\nu}, \vectorgreek{\nu}^{\aux})$, denote by $({\bbnu}, \bbnu^{\aux})_t$ the law of $\cH_t $ under the problem instance $(\vectorgreek{\nu}, \vectorgreek{\nu}^{\aux})$. Then, for any  problem instances $(\vectorgreek{\nu}, \vectorgreek{\nu}^{\aux})$ and $(\tilde{\vectorgreek{\nu}} , \tilde{\vectorgreek{\nu}}^\mathrm{aux})$ and any time step $t\in\cT$,
	\begin{align*}
		\mathrm{KL}\l((\bbnu, \bbnu^{\aux})_t, (\tilde{\bbnu},\tilde{\bbnu}^{\aux})_t\r)
		&=
		\sum_{k\in\cK}
		\l[
		\mathrm{KL}\l(\nu_{k}, \tilde{\nu}_{k}\r) \cdot \mathbb{E}_{(\vectorgreek{\nu}_t, \vectorgreek{\nu}^{\aux}_t)}[n_{k,t}^{\pi}] +
		\mathrm{KL}\l(\nu_{k}^\mathrm{aux}, \tilde{\nu}_{k}^\mathrm{aux}\r) \cdot \sum\limits_{s=1}^th_{k,s}
		\r].
	\end{align*}
\end{lemma}
\begin{proof}
	This type of statement is well-known in the MAB literature (see, e.g., \cite{auer2002nonstochastic}, \cite{garivier2019explore}, or \cite{gerchinovitz2016refined}). However, we show how it is derived for the sake of completeness.
	For any random variables $X$ and $Y$, we denote by $(\bbnu, \bbnu^{\aux})^{(X)}$ the law of $X$, and by $(\bbnu, \bbnu^{\aux})^{(X|Y)}$ the law of $X$ conditional on $Y$ under the problem instance $(\vectorgreek{\nu}, \vectorgreek{\nu}^{\aux})$. By the chain rule of KL-divergence, one has
	\begin{align*}
		\mathrm{KL}\l((\bbnu, \bbnu^{\aux})_t, (\tilde{\bbnu},\tilde{\bbnu}^{\aux})_t\r)
		&=
		\mathrm{KL}\l((\bbnu, \bbnu^{\aux})_{t-1}, (\tilde{\bbnu},\tilde{\bbnu}^{\aux})_{t-1}\r)
		\\
		&+ 		
		\mathrm{KL}\l((\bbnu, \bbnu^{\aux})^{((\pi_{t-1}, X_{\pi_{t-1},t-1}, \vectorgreek{h}_t, \vector{Z}_{t})|\cH_{t-1})}, (\tilde{\bbnu},\tilde{\bbnu}^{\aux})^{((\pi_{t-1}, X_{\pi_{t-1},t-1}, \vectorgreek{h}_t, \vector{Z}_{t})|\cH_{t-1})}\r).
	\end{align*}
	Applying the chain rule to the second quantity on the right hand side of the above equality, one obtains
	\begin{align*}
		&\mathrm{KL}\l((\bbnu, \bbnu^{\aux})^{((\pi_{t-1}, X_{\pi_{t-1},t-1}, \vectorgreek{h}_t, \vector{Z}_{t})|\cH_{t-1})}, (\tilde{\bbnu},\tilde{\bbnu}^{\aux})^{((\pi_{t-1}, X_{\pi_{t-1},t-1}, \vectorgreek{h}_t, \vector{Z}_{t})|\cH_{t-1})}\r)
		=
		\\
		&\sum_{k\in\cK}\l[
		\mathrm{KL}\l(\nu_{k}, \tilde{\nu}_{k}\r) \cdot \mathbb{P}_{(\vectorgreek{\nu}_t, \vectorgreek{\nu}^{\aux}_t)}\l\{ \pi_t = k \r\}+
		\mathrm{KL}\l(\nu_{k}^\mathrm{aux}, \tilde{\nu}_{k}^\mathrm{aux}\r) \cdot h_{k,t}
		\r].
	\end{align*}
	By reiterating the above two displays, the desired result is derived. This concludes the proof.
\end{proof}

\noindent The above lemma yields
\begin{align}\label{eq:kl-lower-bound}
	\mathrm{KL}\l((\bbnu, \bbnu^{\aux})_t, (\tilde{\bbnu},\tilde{\bbnu}^{\aux})_t\r)
	&=
	\mathrm{KL}\l(\nu_{k}, \tilde{\nu}_{k}\r) \cdot \mathbb{E}_{(\vectorgreek{\nu}_t, \vectorgreek{\nu}^{\aux}_t)}[n_{k,t}^{\pi}] +
	\mathrm{KL}\l(\nu_{k}^\mathrm{aux}, \tilde{\nu}_{k}^\mathrm{aux}\r) \cdot \sum\limits_{s=1}^th_{k,s}.
\end{align}

\paragraph{Step 3 (A constraint for $\mathbb{E}^{\pi}_{(\vectorgreek{\nu}, \vectorgreek{\nu}^{\aux}_t)}[n_{k,T+1}^{\pi}]$).}
Note that by the assumption that $\pi$ is IAO, one has
\begin{align*}
	&\frac{C_{\pi, \cS} K \sigma^2}{\epsilon^2} \log T
	\\
	&\ge\sum_{t=1}^T \mathbb{P}^{\pi}_{(\tilde{\vectorgreek{\nu}},\tilde{\vectorgreek{\nu}}^{\aux})}\l\{\pi_t \neq k\r\}
	\\
	&\overset{(a)}{\ge}
	\sum_{t=1}^T
	\frac{1}{2}\exp\l[-\mathrm{KL}\l((\bbnu, \bbnu^{\aux})_t, (\tilde{\bbnu},\tilde{\nu}^{\aux})_t\r)  \r] -  \mathbb{P}^{\pi}_{(\vectorgreek{\nu},\vectorgreek{\nu}^{\aux})}\l\{\pi_t = k\r\}
	\\
	&\overset{(b)}{\ge}
	\sum_{t=1}^T
	\frac{1}{2}\exp\l[-\mathrm{KL}\l(\nu_{k}, \tilde{\nu}_{k}\r) \cdot \mathbb{E}_{(\vectorgreek{\nu}_t, \vectorgreek{\nu}^{\aux}_t)}[n_{k,t}^{\pi}] -
	\mathrm{KL}\l(\nu_{k}^\mathrm{aux}, \tilde{\nu}_{k}^\mathrm{aux}\r) \cdot \sum\limits_{s=1}^th_{k,s} \r] -  \mathbb{P}^{\pi}_{(\vectorgreek{\nu},\vectorgreek{\nu}^{\aux})}\l\{\pi_t = k\r\}
	\\
	&\ge
	\exp\l[-\mathrm{KL}\l(\nu_{k}, \tilde{\nu}_{k}\r) \cdot \mathbb{E}^{\pi}_{(\vectorgreek{\nu}, \vectorgreek{\nu}^{\aux})}[n_{k,T+1}^{\pi}] \r] \cdot \sum_{t=1}^T
	\frac{1}{2}\exp\l[- \mathrm{KL}\l(\nu_{k}^\mathrm{aux}, \tilde{\nu}_{k}^\mathrm{aux}\r) \cdot\sum\limits_{s=1}^t h_{k,s} \r] -    \mathbb{E}^{\pi}_{(\vectorgreek{\nu}, \vectorgreek{\nu}^{\aux}_t)}[n_{k,T+1}^{\pi}],
\end{align*}
where (a) follows from Lemma \ref{Tsybakov2008introduction-lemma-2.6} and (b) holds by \eqref{eq:kl-lower-bound}. The above inequality provides a constraint over the quantity of interest, $\mathbb{E}^{\pi}_{(\vectorgreek{\nu}, \vectorgreek{\nu}^{\aux}_t)}[n_{k,T+1}^{\pi}]$.
\paragraph{Step 4 (Lower bounding $\mathbb{E}^{\pi}_{(\vectorgreek{\nu}, \vectorgreek{\nu}^{\aux}_t)}[n_{k,T+1}^{\pi}]$).}
Given the above inequality one can see that $\mathbb{E}^{\pi}_{(\vectorgreek{\nu}, \vectorgreek{\nu}^{\aux}_t)}[n_{k,T+1}^{\pi}] $ is lower bounded by the solution of the following convex optimization problem:
\begin{subequations}
	\begin{alignat*}{2}
		&\!\min_{x}        &\qquad& x\\
		&\text{subject to} &      & \exp\l[-\mathrm{KL}\l(\nu_{k}, \tilde{\nu}_{k}\r) \cdot x \r] \cdot \sum_{t=1}^T
		\frac{1}{2}\exp\l[- \mathrm{KL}\l(\nu_{k}^\mathrm{aux}, \tilde{\nu}_{k}^\mathrm{aux}\r) \cdot\sum\limits_{s=1}^t h_{k,s} \r] -    x
		\le \frac{C_{\pi, \cS} K \sigma^2}{\epsilon^2} \log T.
	\end{alignat*}
\end{subequations}
The solution of the above optimization problem is then, lower bounded by that of the following for any $\lambda \ge 0$:
\begin{subequations}
	\begin{alignat*}{2}
		&\!\min_{x}        &\qquad& x + \lambda \l(\exp\l[-\mathrm{KL}\l(\nu_{k}, \tilde{\nu}_{k}\r) \cdot x \r] \cdot \sum_{t=1}^T
		\frac{1}{2}\exp\l[- \mathrm{KL}\l(\nu_{k}^\mathrm{aux}, \tilde{\nu}_{k}^\mathrm{aux}\r) \cdot\sum\limits_{s=1}^t h_{k,s} \r] -    x - \frac{C_{\pi, \cS} K \sigma^2}{\epsilon^2} \log T \r)\\
		&\text{subject to} &      & \exp\l[-\mathrm{KL}\l(\nu_{k}, \tilde{\nu}_{k}\r) \cdot x \r] \cdot \sum_{t=1}^T
		\frac{1}{2}\exp\l[- \mathrm{KL}\l(\nu_{k}^\mathrm{aux}, \tilde{\nu}_{k}^\mathrm{aux}\r) \cdot\sum\limits_{s=1}^t h_{k,s} \r] -    x
		\le \frac{C_{\pi, \cS} K \sigma^2}{\epsilon^2} \log T.
	\end{alignat*}
\end{subequations}
Letting  $\lambda = \frac{\epsilon^2}{C_{\pi, \cS} K \sigma^2\log T} $ and noting that $x + \gamma e^{-\kappa x} \ge \frac{\log \gamma \kappa}{\kappa}$ for $\gamma, \kappa, x > 0$, one obtains
\[
\mathbb{E}^{\pi}_{(\vectorgreek{\nu}, \vectorgreek{\nu}^{\aux}_t)}[n_{k,T+1}^{\pi}]
\ge
\frac{1}{\mathrm{KL}\l(\nu_{k}, \tilde{\nu}_{k}\r)} \log \l( \frac{\mathrm{KL}\l(\nu_{k}, \tilde{\nu}_{k}\r)\epsilon^2}{C_{\pi, \cS} K \sigma^2 \log T}  \sum_{t=1}^T\exp\l[- \mathrm{KL}\l(\nu_{k}^\mathrm{aux}, \tilde{\nu}_{k}^\mathrm{aux}\r) \cdot\sum\limits_{s=1}^t h_{k,s} \r]\r) -1.
\]
This concludes the proof. \qed

\subsection{Proof of Theorem \ref{theorem-2UCB-regret-upper-bound-unknown}}
\paragraph{Step 1 (Preliminaries).} Fix a profile $\l(\vectorgreek{\nu}, \vectorgreek{\nu}^{\mathrm{aux}}\r)$. Note that it is unlikely for the UCB of the optimal arm to be less than the largest mean reward. We will carry out the analysis based on this observation. To be more precise, define
\[
\mathcal{E}^\ast_{t} \coloneqq \left\{  U_{k^\ast, t} < \mu^\ast  \right\}
\]
to be the event on which the UCB of the optimal arm is less than the largest mean reward at time $t$. Note that by the Chernoff-Hoefding bound in Lemma \ref{lemma-Chernoff-Hoeffding-bound}, one has
\[
\rP^\pi_{\vectorgreek{\nu}, \vectorgreek{\nu}^{\mathrm{aux}}} \l\{\mathcal{E}^\ast_{t}\r\}
\le
\frac{1}{t^{\frac{c}{2}}} \qquad \forall t\ge K+1.
\]
Hence, defining $\mathcal{E}^\ast \coloneqq \bigcup_{t=1}^T \mathcal{E}^\ast_{t}$, one obtains
\begin{align*}\label{eq:2ucbs-analysis-1}
	\cR^\pi_{\vectorgreek{\nu}, \vectorgreek{\nu}^{\mathrm{aux}}}
	& \le
	\sum_{k\in \cK} \Delta_k +
	\mathbb{E}^\pi_{
		\vectorgreek{\nu}, \vectorgreek{\nu}^{\mathrm{aux}}
	} \left[ \sum\limits_{t=K+1}^T(\mu^{\ast} - \mu_{\pi_t}) \; \middle | \;  \bar{ \mathcal{E}}^\ast\right]
	+\sum_{t=K+1}\frac{\max_{k\in\mathcal{K}} \Delta_k}{t^{\frac{c}{2}}}\\
	& \le
	\sum_{k\in \cK} \Delta_k +
	\sum_{k\in \cK} \Delta_k \cdot \mathbb{E}^\pi_{
		\vectorgreek{\nu}, \vectorgreek{\nu}^{\mathrm{aux}}
	} \left[  n^\pi_{k,T+1} \; \middle | \;  \bar{ \mathcal{E}}^\ast\right]
	+\frac{\max_{k\in\mathcal{K}} \Delta_k}{\frac{c}{2}-1}, \numberthis\vspace{-0.1cm}
\end{align*}
where we have used the assumption that $c>2$. Consider a suboptimal arm $k\neq k^\ast$. We will bound $\mathbb{E}^\pi_{
	\vectorgreek{\nu}, \vectorgreek{\nu}^{\mathrm{aux}}
} \left[  n^\pi_{k,T+1} \; \middle | \;  \bar{ \mathcal{E}}^\ast\right] $, the expected number of times the suboptimal arm $k$ is pulled conditional on the event $\bar{ \mathcal{E}}^\ast$. Our analysis is based on the fact that if arm $k$ is pulled at time $t$, then\vspace{-0.2cm}
\begin{equation}\label{eq:2ucbs-analysis-2}
	U_{k,t} \ge U_{k^\ast, t}.
\end{equation}
We give two separate upper bounds for $\mathbb{E}^\pi_{\l(
	\vectorgreek{\nu}, \vectorgreek{\nu}^{\mathrm{aux}}
	\r)} \left[  n^\pi_{k,T+1} \; \middle | \;  \bar{ \mathcal{E}}^\ast\right] $. The first upper bound holds for any suboptimal arm $k$, and the second one is specific to the suboptimal arms $k\in \cK^\ast $. The following lemma will be the main tool in our analysis.

\begin{lemma}\label{lemma:double-seq-sum}
	Let $\{f_n\}_{n=1}^N$ and $\{g_n\}_{n=1}^N$ be two sequences such that for all $n\in \{1,\dots, N\}$, $f_n = \sum_{j=1}^{n-1}x_j$ with $0\le x_j \le M$ for all $j\in \{1,\dots, N\}$ and some $M\ge0$. Then, \[\sum_{n=1}^Nx_n \cdot \Indlr{f_n \le g_n } \le \max\l\{0, \max_{1\le n \le N} g_n+M \r\}.\]
	
	\begin{proof}
		If $\max\limits_{1\le n \le N} g_n+M<0$ then, the result is immediate. Assume $\max\limits_{1\le n \le N} g_n+M \ge 0 $. We prove the result through contradiction. Suppose that $\sum_{n=1}^Nx_n \cdot \Indlr{f_n \le g_n } > \max\limits_{1\le n \le N} g_n + M$. Let
		\[
		\tau
		=
		\min\l\{1 \le n \le N: \sum_{j=1}^n x_j \cdot \Indlr{f_j \le g_j } > \max_{1\le n \le N} g_n + M \r\}.
		\]
		Note that one must have $f_\tau \le g_\tau$. Since all $x_j$'s are non-negative, one has
		\[
		\sum_{j=1}^\tau  x_j \cdot \Indlr{f_j \le g_j }
		\le
		\sum_{j=1}^{\tau-1}  x_j  + x_\tau
		\le
		f_\tau + M
		\le
		g_\tau + M
		\le
		\max_{1\le n \le N} g_n  + M
		\]
		which is a contradiction. This concludes the proof.
	\end{proof}
\end{lemma}
\paragraph{Step 2 (Regret upper bound for any suboptimal arms  $k\in \mathcal{K}\setminus\{k^\ast\}$).} In this step, we deploy the UCB constructed based on reward observations. That is, \eqref{eq:2ucbs-analysis-2} yields that if arm $k$ is pulled at time $t$ then, conditional on the event $\bar{ \mathcal{E}}^\ast$, one has
\[
\bar{X}^\pi_{k,t} + \sqrt{\frac{c\sigma^2 \log t}{n^\pi_{k,t}}} = U^\pi_{k,t}\ge \mu^\ast. \vspace{-0.2cm}
\]
Therefore, at least one of the following events must occur:\vspace{-0.2cm}
\begin{align*}
	\mathcal{E}^\pi_{1,t} &\coloneqq \left\{  \bar{X}^\pi_{k,t} \ge \mu_k + \sqrt{\frac{c\sigma^2 \log t}{n^\pi_{k,t}}} \right\}, \quad
	\mathcal{E}^\pi_{2,t} \coloneqq \left\{  \Delta_k \le 2\sqrt{\frac{c\sigma^2 \log t}{n^\pi_{k,t}}} \right\}.
\end{align*}
To see why this is true, assume that all the above events fail. Then, we have\vspace{-0.2cm}
\[
\bar{X}^\pi_{k,t} + \sqrt{\frac{c\sigma^2 \log t}{n^\pi_{k,t}}} < \mu_k + 2\sqrt{\frac{c\sigma^2 \log t}{n^\pi_{k,t}}} < \mu_k + \Delta_k = \mu^\ast ,\vspace{-0.2cm}
\]
which is in contradiction with the assumption that arm $k$ is pulled at time $t$ conditional on the event $\bar{ \mathcal{E}}^\ast$.
For any sequence $\{ l_{k,t}\}_{t\in\mathcal{T}}$, one has\vspace{-0.2cm}
\begin{align*}
	\mathbb{E}^\pi_{\vectorgreek{\nu}, \vectorgreek{\nu}^{\mathrm{aux}}} \left[n^\pi_{k,T+1}\; \middle | \;  \bar{ \mathcal{E}}^\ast\right]
	&= \mathbb{E}^\pi_{\vectorgreek{\nu}, \vectorgreek{\nu}^{\mathrm{aux}}} \left[ \sum\limits_{t=1}^T \mathbbm{1} \left\{ \pi_t = k \right\}\; \middle | \;  \bar{ \mathcal{E}}^\ast\right]
	\\
	&=
	\mathbb{E}^\pi_{\vectorgreek{\nu}, \vectorgreek{\nu}^{\mathrm{aux}}} \left[ \sum\limits_{t=1}^T \mathbbm{1} \left\{ \pi_t = k, n^\pi_{k,t} \le l_{k,t}\right\} + \mathbbm{1} \left\{ \pi_t = k, n^\pi_{k,t} > l_{k,t} \right\}\; \middle | \;  \bar{ \mathcal{E}}^\ast\right]
	\\
	&\le
	\mathbb{E}^\pi_{\vectorgreek{\nu}, \vectorgreek{\nu}^{\mathrm{aux}}} \left[ \sum\limits_{t=1}^T \mathbbm{1} \left\{ \pi_t = k, n^\pi_{k,t} \le \hat l_{k,t}\right\} + \mathbbm{1} \left\{ \pi_t = k, n^\pi_{k,t} > l_{k,t} \right\} \; \middle | \;  \bar{ \mathcal{E}}^\ast \right]
	\\
	&\le
	1+ \max_{1\le t \le T}\left\{\hat l_{k,t}\right\}
	+ \sum\limits_{t=1}^T \mathbb{P}^\pi_{\vectorgreek{\nu}, \vectorgreek{\nu}^{\mathrm{aux}}} \left\{ \pi_t = k, n^\pi_{k,t} > l_{k,t} \; \middle | \;  \bar{ \mathcal{E}}^\ast \right\},
\end{align*}
where the last inequality follows from Lemma \ref{lemma:double-seq-sum}.
Set the values of $ l_{k,t}$ as follows:\vspace{-0.1cm}
\begin{equation*}\label{eq-2UCBs-proof-helper1}
	l_{k,t} = \frac{4c\sigma^2\log \left(
		t
		\right)}{\Delta_k^2}.
\end{equation*}
To have $n^\pi_{k,t} > l_{k,t}$, it must be the case that $\cE^\pi_{2,t}$ does not occur. Therefore,\vspace{-0.3cm}
\begin{align*}\label{eq-2UCBs-proof-helper2}
	\sum\limits_{t=1}^T \mathbb{P}^\pi_{\vectorgreek{\nu}, \vectorgreek{\nu}^{\mathrm{aux}}} \left\{ \pi_t = k, n^\pi_{k,t} > l_{k,t} \; \middle | \;  \bar{ \mathcal{E}}^\ast \right\}
	&\le
	\sum\limits_{t=1}^T \mathbb{P}^\pi_{\vectorgreek{\nu}, \vectorgreek{\nu}^{\mathrm{aux}}} \left\{ \pi_t = k, \bar{\mathcal{E}}^\pi_{2,t}  \; \middle | \;  \bar{ \mathcal{E}}^\ast \right\}
	\\
	&\le
	\sum\limits_{t=1}^T \mathbb{P}^\pi_{\vectorgreek{\nu}, \vectorgreek{\nu}^{\mathrm{aux}}} \left\{ \mathcal{E}^\pi_{1,t} \; \middle | \;  \bar{ \mathcal{E}}^\ast \right\}
	\overset{(a)}{\le}
	\sum\limits_{t=1}^T \frac{1}{t^{\frac{c}{2}}}
	\overset{(b)}{\le}
	\frac{1}{\frac{c}{2}-1}
	\numberthis
	\nonumber,
\end{align*}
where (a) follows from the Chernoff-Hoeffding bound in Lemma \ref{lemma-Chernoff-Hoeffding-bound} and (b) follows from the assumption that $c>2$. The last two displays yield
\[
\mathbb{E}^\pi_{\vectorgreek{\nu}, \vectorgreek{\nu}^{\mathrm{aux}}} \left[n^\pi_{k,T+1}\; \middle | \;  \bar{ \mathcal{E}}^\ast\right]
\le
\frac{4c\sigma^2\log \left(
	T
	\right)}{\Delta_k^2}
+
\frac{1}{\frac{c}{2}-1}
+ 1.
\]

\paragraph{Step 3 (Regret upper bound for suboptimal arms  $k\in \mathcal{K}\setminus\{k^\ast\}$ which satisfy $ \delta_k \ge 0$)).} For this set of arms, we deploy the UCB that incorporate both reward observations and auxiliary observations. That is, \eqref{eq:2ucbs-analysis-2} yields that if arm $k$ is pulled at time $t$ then, conditional on the event $\bar{ \mathcal{E}}^\ast$, one has
\[
\bar{X}^{\pi, \mathrm{aux}}_{k,t} + \sqrt{\frac{c\sigma^2 \log t}{n^{\pi, \mathrm{aux}}_{k,t}}}
=
U^{\pi, \mathrm{aux}}_{k,t}
\ge \mu^\ast. \vspace{-0.2cm}
\]
Therefore, at least one of the following events must occur:\vspace{-0.2cm}
\begin{align*}
	\mathcal{E}^{\pi, \mathrm{aux}}_{1,t} &\coloneqq \left\{  \bar{X}^{\pi, \mathrm{aux}}_{k,t} \ge \ualpha y_k + \sqrt{\frac{c\sigma^2 \log t}{n^{\pi, \mathrm{aux}}_{k,t}}} \right\}, \quad
	\mathcal{E}^{\pi, \mathrm{aux}}_{2,t} \coloneqq \left\{  \Delta^{\pi, \mathrm{aux}}_{k,t} \le 2\sqrt{\frac{c\sigma^2 \log t}{n^{\pi, \mathrm{aux}}_{k,t}}} \right\}.
\end{align*}
A similar analysis as in the proof of Theorem \ref{theorem-UCB1-upper bound}, except for replacing $\Delta_k$ with $\delta_k$ throughout the analysis, results in the following upper bound:
\begin{align*}
	\mathbb{E}^\pi_{\vectorgreek{\nu}, \vectorgreek{\nu}^{\mathrm{aux}}} \left[n^\pi_{k,T+1}\; \middle | \;  \bar{ \mathcal{E}}^\ast\right]
	\le
	1+\frac{4c\sigma^2}{\delta_k^2} \log \left(
	\sum\limits_{t=0}^T \exp\left(
	\frac{ -\delta_k^2}{4c\ualpha^2\hat{\sigma}^2}\sum\limits_{s = 1}^t h_{k,s}
	\r)
	\right)+ \frac{1}{\frac{c}{2}-1}.
\end{align*}

\paragraph{Step 4 (Regret upper bound for suboptimal arms  $k\in \mathcal{K}\setminus\{k^\ast\}$ which satisfy $ \delta_k \ge \frac{\Delta_k}{2}$)).} For this set of arms, we deploy the UCB that incorporate both reward observations and auxiliary observations. That is, \eqref{eq:2ucbs-analysis-2} yields that if arm $k$ is pulled at time $t$ then, conditional on the event $\bar{ \mathcal{E}}^\ast$, one has
\[
\bar{X}^{\pi, \mathrm{aux}}_{k,t} + \sqrt{\frac{c\sigma^2 \log t}{n^{\pi, \mathrm{aux}}_{k,t}}}
=
U^{\pi, \mathrm{aux}}_{k,t}
\ge \mu^\ast. \vspace{-0.2cm}
\]
Therefore, at least one of the following events must occur:\vspace{-0.2cm}
\begin{align*}
	\mathcal{E}^{\pi, \mathrm{aux}}_{1,t} &\coloneqq \left\{  \bar{X}^{\pi, \mathrm{aux}}_{k,t} \ge \mu^{\pi, \mathrm{aux}}_{k,t} + \sqrt{\frac{c\sigma^2 \log t}{n^{\pi, \mathrm{aux}}_{k,t}}} \right\}, \quad
	\mathcal{E}^{\pi, \mathrm{aux}}_{2,t} \coloneqq \left\{  \Delta^{\pi, \mathrm{aux}}_{k,t} \le 2\sqrt{\frac{c\sigma^2 \log t}{n^{\pi, \mathrm{aux}}_{k,t}}} \right\},
\end{align*}
where we define
\[
\mu^{\pi, \mathrm{aux}}_{k,t}
\coloneqq
\mu^{\pi, \mathrm{aux}}_{k,n^\pi_{k,t}, n^{\mathrm{aux}}_{k,t}}
=
\frac{ \mu_k \cdot n^\pi_{k,t} + \ualpha \cdot y_k \cdot \sum\limits_{s=1}^t \frac{\sigma^2}{\ualpha^2\hat{\sigma}^2}  h_{k,s}}{n^{\pi, \mathrm{aux}}_{k,t}};
\qquad
\Delta^{\pi, \mathrm{aux}}_{k,t}
\coloneqq
\Delta^{\pi, \mathrm{aux}}_{k,n^\pi_{k,t}, n^{\mathrm{aux}}_{k,t}}
=
\mu^\ast - \mu^{\pi, \mathrm{aux}}_{k,t}.
\]
To see why this is true, assume that all the above events fail. Then, we have\vspace{-0.2cm}
\[
\bar{X}^{\pi, \mathrm{aux}}_{k,t} + \sqrt{\frac{c\sigma^2 \log t}{n^{\pi, \mathrm{aux}}_{k,t}}} <   \mu^{\pi, \mathrm{aux}}_{k,t} + 2\sqrt{\frac{c\sigma^2 \log t}{n^{\pi, \mathrm{aux}}_{k,t}}} < \mu^{\pi, \mathrm{aux}}_{k,t} + \Delta^{\pi, \mathrm{aux}}_{k,t} = \mu^\ast ,\vspace{-0.2cm}
\]
which is in contradiction with the assumption that arm $k$ is pulled at time $t$ conditional on the event $\bar{ \mathcal{E}}^\ast$.
For any sequence $\{ l_{k,t}\}_{t\in\mathcal{T}}$, one has\vspace{-0.2cm}
\begin{align*}\label{eq:2ucbs-log-sum-exp-1}
	\mathbb{E}^\pi_{\vectorgreek{\nu}, \vectorgreek{\nu}^{\mathrm{aux}}} \left[n^\pi_{k,T+1}\; \middle | \;  \bar{ \mathcal{E}}^\ast\right]
	&= \mathbb{E}^\pi_{\vectorgreek{\nu}, \vectorgreek{\nu}^{\mathrm{aux}}} \left[ \sum\limits_{t=1}^T \mathbbm{1} \left\{ \pi_t = k \right\}\; \middle | \;  \bar{ \mathcal{E}}^\ast\right]
	\\
	&=
	\mathbb{E}^\pi_{\vectorgreek{\nu}, \vectorgreek{\nu}^{\mathrm{aux}}} \left[ \sum\limits_{t=1}^T \mathbbm{1} \left\{ \pi_t = k, n^{\pi, \mathrm{aux}}_{k,t} \le l_{k,t}\right\} + \mathbbm{1} \left\{ \pi_t = k, n^{\pi, \mathrm{aux}}_{k,t} > l_{k,t} \right\}\; \middle | \;  \bar{ \mathcal{E}}^\ast\right]
	\\
	&\le
	\mathbb{E}^\pi_{\vectorgreek{\nu}, \vectorgreek{\nu}^{\mathrm{aux}}} \left[ \sum\limits_{t=1}^T \mathbbm{1} \left\{ \pi_t = k, n^{\pi, \mathrm{aux}}_{k,t} \le  l_{k,t}\right\} + \mathbbm{1} \left\{ \pi_t = k, n^{\pi, \mathrm{aux}}_{k,t} > l_{k,t} \right\} \; \middle | \;  \bar{ \mathcal{E}}^\ast \right]
	\\
	&\le
	\mathbb{E}^\pi_{\vectorgreek{\nu}, \vectorgreek{\nu}^{\mathrm{aux}}} \left[ \sum\limits_{t=1}^T \mathbbm{1} \left\{ \pi_t = k, n^{\pi, \mathrm{aux}}_{k,t} \le  l_{k,t} \right\}\; \middle | \;  \bar{ \mathcal{E}}^\ast\right]
	\\
	&+ \sum\limits_{t=1}^T \mathbb{P}^\pi_{\vectorgreek{\nu}, \vectorgreek{\nu}^{\mathrm{aux}}} \left\{ \pi_t = k, n^{\pi, \mathrm{aux}}_{k,t} > l_{k,t} \; \middle | \;  \bar{ \mathcal{E}}^\ast\right\}.
	\numberthis
\end{align*}
We will upper bound each term on the right hand side of the above inequality separately.
Set the values of $ l_{k,t}$ as
$
l_{k,t} = \frac{4c\sigma^2\log \left(
	t
	\right)}{(\Delta^{\pi, \mathrm{aux}}_{k,t})^2}.
$
In order to upper bound the term $\mathbb{E}^\pi_{\vectorgreek{\nu}, \vectorgreek{\nu}^{\mathrm{aux}}} \left[ \sum\limits_{t=1}^T \mathbbm{1} \left\{ \pi_t = k, n^{\pi, \mathrm{aux}}_{k,t} \le  l_{k,t} \right\}\; \middle | \;  \bar{ \mathcal{E}}^\ast \right]$ in \eqref{eq:2ucbs-log-sum-exp-1}, note that the equality $\Delta^{\pi, \mathrm{aux}}_{k,t}
=
\frac{ \Delta_k \cdot n^\pi_{k,t} +  \frac{\delta_k\sigma^2}{\ualpha^2\hat{\sigma}^2}  n^{\mathrm{aux}}_{k,t}}{n^\pi_{k,t} +  \frac{\sigma^2}{\ualpha^2\hat{\sigma}^2}  n^{\mathrm{aux}}_{k,t}}$ yields
\begin{align*}\label{eq:2ucbs-log-sum-exp-2}
	n^{\pi, \mathrm{aux}}_{k,t} \le  l_{k,t}
	\;\Leftrightarrow\;
	&
	\Delta_k^2 \l(n^\pi_{k,t}\r)^2
	+
	\l(
	\frac{2\Delta_k\delta_k\sigma^2}{\ualpha^2\hat{\sigma}^2}n^{\mathrm{aux}}_{k,t}
	- 4c\sigma^2 \log t
	\r)
	n^\pi_{k,t}
	\\
	&+
	\l(
	\frac{\delta_k\sigma^2}{\ualpha^2\hat{\sigma}^2}n^{\mathrm{aux}}_{k,t}
	\r)^2
	-
	\l(4c\sigma^2 \log t\r) \l(\frac{\sigma^2}{\ualpha^2\hat{\sigma}^2}n^{\mathrm{aux}}_{k,t}\r)
	\le
	0.
	\numberthis
\end{align*}
Since the left-hand side of the above inequality is quadratic in $n^{\pi}_{k,t}$, with a positive coefficient for the quadratic term then, the above inequality holds only if
\[
n^\pi_{k,t}
\le
\frac{4c\sigma^2 \log t - \frac{2\Delta_k\delta_k\sigma^2}{\ualpha^2\hat{\sigma}^2}n^{\mathrm{aux}}_{k,t} + \sqrt{\l(4c\sigma^2 \log t\r)^2 + 4 \Delta_k (\Delta_k-\delta_k)\l(4c\sigma^2 \log t\r)\l(\frac{\sigma^2}{\ualpha^2\hat{\sigma}^2}n^{\mathrm{aux}}_{k,t}\r)}
}{2\Delta_k^2}.
\]
Now, in order to upper bound the right hand side of this inequality, note that if $\delta_k = \xi_k \Delta_k$, $\frac{1}{2}\le \xi_k \le 1$, one obtains
\begin{align*}
	- \frac{2\Delta_k\delta_k\sigma^2}{\ualpha^2\hat{\sigma}^2}n^{\mathrm{aux}}_{k,t} &+
	\sqrt{\l(4c\sigma^2 \log t\r)^2 + 4 \Delta_k (\Delta_k-\delta_k)\l(4c\sigma^2 \log t\r)\l(\frac{\sigma^2}{\ualpha^2\hat{\sigma}^2}n^{\mathrm{aux}}_{k,t}\r)}
	\\
	&\le
	- \frac{2\l(\frac{2\xi_k-1}{\xi_k^2}\r)\delta_k^2\sigma^2}{\ualpha^2\hat{\sigma}^2}n^{\mathrm{aux}}_{k,t} + 4c\sigma^2 \log t
\end{align*}
That is, the last two displays yield that \eqref{eq:2ucbs-log-sum-exp-2} holds only if
\begin{align*}
	n^\pi_{k,t}
	\le
	\frac{4c\sigma^2 \log t - \frac{\l(\frac{2\xi_k-1}{\xi_k^2}\r)\delta_k^2\sigma^2}{\ualpha^2\hat{\sigma}^2}n^{\mathrm{aux}}_{k,t}}{\Delta_k^2}
	&\le
	\frac{4c\sigma^2 \log \tau_{k,t} - \frac{\l(\frac{2\xi_k-1}{\xi_k^2}\r)\delta_k^2\sigma^2}{\ualpha^2\hat{\sigma}^2}n^{\mathrm{aux}}_{k,t}}{\Delta_k^2}
	\\
	&=
	\frac{4c\sigma^2}{\Delta_k^2} \log \l( \sum\limits_{s=1}^t \exp\l(-\frac{\l(\frac{2\xi_k-1}{\xi_k^2}\r)\delta_k^2}{4c\ualpha^2\hat{\sigma}^2} \sum_{m=1}^{s} h_{m,k}\r) \r) \eqqcolon \hat l_{k,t},
\end{align*}
where we define $\tau_{k,t} \coloneqq \sum\limits_{s=1}^t \exp\l(\frac{\l(\frac{2\xi_k-1}{\xi_k^2}\r)\delta_k^2}{4c\ualpha^2\hat{\sigma}^2} \sum\limits_{m=s}^{t-1} h_{k,m}\r)\ge t$, and the equality follows from $n^{\mathrm{aux}}_{k,t} = \sum\limits_{s=1}^{t}  h_{k,s}$. One obtains
\begin{align*}\label{eq:2ucbs-log-sum-exp-3}
	\mathbb{E}^\pi_{\vectorgreek{\nu}, \vectorgreek{\nu}^{\mathrm{aux}}} \left[ \sum\limits_{t=1}^T \mathbbm{1} \left\{ \pi_t = k, n^{\pi, \mathrm{aux}}_{k,t} \le  l_{k,t} \right\}\; \middle | \;  \bar{ \mathcal{E}}^\ast \right]
	&\le
	\mathbb{E}^\pi_{\vectorgreek{\nu}, \vectorgreek{\nu}^{\mathrm{aux}}} \left[ \sum\limits_{t=1}^T \mathbbm{1} \left\{ \pi_t = k, n^\pi_{k,t} \le  \hat l_{k,t} \r\}\; \middle | \;  \bar{ \mathcal{E}}^\ast \right]
	\\
	&\overset{(a)}{\le}
	\max\l\{0,1+ \max_{1\le t \le T} \hat l_{k,t}\r\}
	\overset{(b)}{\le}
	\max\l\{0,1+\hat l_{k,T} \r\}
	\\
	&\le
	1+\frac{4c\sigma^2}{\Delta_k \delta_k} \log \left(
	\sum\limits_{t=0}^T \exp\left(
	\frac{ -\l(\frac{2\xi_k-1}{\xi_k^2}\r)\delta_k^2 }{4c\ualpha^2\hat{\sigma}^2}\sum\limits_{s = 1}^t h_{k,s}
	\r) \r),
	\numberthis
\end{align*}
where (a) follows from Lemma \ref{lemma:double-seq-sum}, and (b) follows from the fact that $\{\hat l_{k,t}\}_t$ is an increasing sequence.

In order to upper bound the term $\sum\limits_{t=1}^T \mathbb{P}^\pi_{\vectorgreek{\nu}, \vectorgreek{\nu}^{\mathrm{aux}}} \left\{ \pi_t = k, n^{\pi, \mathrm{aux}}_{k,t} > l_{k,t} \; \middle | \;  \bar{ \mathcal{E}}^\ast\right\}$  in \eqref{eq:2ucbs-log-sum-exp-1}, we note that to have $n^{\pi, \mathrm{aux}}_{k,t} > l_{k,t}$, it must be the case that $\bar{\mathcal{E}}^{\pi, \mathrm{aux}}_{2,t}$ does not occur. On the other hand, if arm $k$ is pulled and $\mathcal{E}^{\pi, \mathrm{aux}}_{2,t}$ does not occur, then it must be the case that $\mathcal{E}^{\pi, \mathrm{aux}}_{1,t}$ occurs. Therefore,\vspace{-0.3cm}
\begin{align}\label{eq:2ucbs-log-sum-exp-4}
	\sum\limits_{t=1}^T \mathbb{P}^\pi_{\vectorgreek{\nu}, \vectorgreek{\nu}^{\mathrm{aux}}} \left\{ \pi_t = k, n^{\pi, \mathrm{aux}}_{k,t} > l_{k,t} \; \middle | \;  \bar{ \mathcal{E}}^\ast\right\}
	&=
	\sum\limits_{t=1}^T \mathbb{P}^\pi_{\vectorgreek{\nu}, \vectorgreek{\nu}^{\mathrm{aux}}} \left\{ \pi_t = k, \bar{\mathcal{E}}^{\pi, \mathrm{aux}}_{2,t} \; \middle | \;  \bar{ \mathcal{E}}^\ast \right\}
	\nonumber \\
	&\le \sum\limits_{t=1}^T \mathbb{P}^\pi_{\vectorgreek{\nu}, \vectorgreek{\nu}^{\mathrm{aux}}} \left\{ \mathcal{E}^{\pi, \mathrm{aux}}_{1,t} \; \middle | \;  \bar{ \mathcal{E}}^\ast \right\}
	\le
	\sum\limits_{t=1}^T \frac{1}{t^{c/2}}
	\le
	\frac{1}{\frac{c}{2}-1},
\end{align}
Putting back together \eqref{eq:2ucbs-log-sum-exp-1}, \eqref{eq:2ucbs-log-sum-exp-3}, and \eqref{eq:2ucbs-log-sum-exp-4}, one obtains
\begin{align*}
	\mathbb{E}^\pi_{\vectorgreek{\nu}, \vectorgreek{\nu}^{\mathrm{aux}}} \left[n^\pi_{k,T+1}\; \middle | \;  \bar{ \mathcal{E}}^\ast\right]
	\le
	1+\frac{4c\sigma^2}{\Delta_k^2} \log \left(
	\sum\limits_{t=0}^T \exp\left(
	\frac{ -\l(\frac{2\xi_k-1}{\xi_k^2}\r)\delta_k^2}{4c\ualpha^2\hat{\sigma}^2}\sum\limits_{s = 1}^t h_{k,s}
	\r)
	\right)+ \frac{1}{\frac{c}{2}-1}.
\end{align*}
This concludes the proof.
\qed

\section{Auxiliary Lemmas}\label{sec:app-auxiliary-lemmas}
In this section of the appendix, we provide new lemmas that are used in the proofs of the main results of the paper, and some auxiliary results that appear in the literature.
\begin{lemma}\label{Tsybakov2008introduction-lemma-2.6}
	Let $\rho_0,\rho_1$ be two probability distributions supported on some set $\mathbb{X}$, with $\rho_0$ absolutely continuous with respect to $\rho_1$. Then for any measurable function $\Psi : \mathbb{X} \rightarrow \{0, 1\}$, one has:\vspace{-0.1cm}
	\begin{equation*}
		\mathbb{P}_{\rho_0} \{\Psi(X) = 1\} + \mathbb{P}_{\rho_1} \{\Psi(X) = 0\} \ge \frac{1}{2} \exp(-\mathrm{KL}(\rho_0,\rho_1)).
	\end{equation*}
\end{lemma}

\proof \; \; \;
Define $\mathcal{A}$ to be the event that $\Psi(X) = 1$. One has
\begin{align*}
	\mathbb{P}_{\rho_0} \{\Psi(X) = 1\} + \mathbb{P}_{\rho_1} \{\Psi(X) = 0\}
	&=
	\mathbb{P}_{\rho_0} \{\mathcal{A}\} + \mathbb{P}_{\rho_1} \{\bar{\mathcal{A}}\}
	\\
	&\ge
	\int \min \{d\rho_0, d\rho_1\}
	\\
	&\ge
	\frac{1}{2} \exp(-\mathrm{KL}(\rho_0,\rho_1)),
\end{align*}
where the last inequality follows from \citealt[Lemma 2.6]{Tsybakov2008introduction}.
\endproof

\begin{lemma}
	\label{gerchinovitz2016refined-lemma-1}
	Consider two profiles $\vectorgreek{\nu}$, and $\hat{ \vectorgreek{\nu}}$. Denote by $\vectorgreek{\nu}_{t}$ ($\hat{ \vectorgreek{\nu}}_{t}$) the distribution of the observed rewards up to time $t$ under $\vectorgreek{\nu}$ ($\hat{ \vectorgreek{\nu}}$ respectively). Let ${n}_{k,t}$ be the number of times a sample from arm $k$ has been observed up to time $t$, that is $n^{\known} _{k,t} = h_{k,t}+ \sum\limits_{s=1}^{t-1}\left( h_{k,s}+\mathbbm{1}\{\pi_s = k\}\right)$. For any policy $\pi$, we have\vspace{-0.1cm}
	\[
	\mathrm{KL}(\vectorgreek{\nu}_t,\hat{\vectorgreek{\nu}}_t) = \sum\limits_{k=1}^K \mathbb{E}_{\vectorgreek{\nu}}\left[n^{\known} _{k,t}\right] \cdot \mathrm{KL}(\nu_k,\hat\nu_k).
	\]
\end{lemma}
\proof \;\;\; The proof is a simple adaptation of the proof of Lemma 1 in \cite{gerchinovitz2016refined}; hence, it is omitted. \endproof

\begin{lemma}[Chernoff-Hoeffding bound]\label{lemma-Chernoff-Hoeffding-bound}
	Let $X_1,\dots,X_n$ be random variables such that $X_t$ is a $\sigma^2$-sub-Gaussian random variable conditioned on $X_1, \dots, X_{t-1}$ and $\mathbb{E}\left[X_t\,\vert\, X_1,\dots,X_{t-1}\right]=\mu$. Let $S_n=X_1+\dots+X_n$. Then for all $a\ge0$\vspace{-0.1cm}
	\[
	\mathbb{P}\{S_n\ge n\mu+a\} \le e^{-\frac{a^2}{2n\sigma^2}},\qquad \text{and} \qquad \mathbb{P}\{S_n\le n\mu-a\} \le e^{-\frac{a^2}{2n\sigma^2}}.
	\]
\end{lemma}

\begin{lemma}[Bernstein inequality]\label{lemma-Bernstein-inequality}
	Let $X_1,\dots,X_n$ be random variables with range $[0,1]$ and\vspace{-0.1cm}
	\[
	\sum\limits_{t=1}^{n}\mathrm{Var}\left[X_t\,\vert\,X_{t-1},\dots,X_1\right] = \sigma^2.
	\]
	Let $S_n=X_1+\dots+X_n.$ Then for all $a\ge0$\vspace{-0.1cm}
	\[
	\mathbb{P}\{S_n\ge \mathbb{E}[S_n] +a\}
	\le \exp\left(-\frac{a^2/2}{\sigma^2+a/2} \right).
	\]
\end{lemma}

\begin{lemma}[Pinsker's inequality]\label{lemma-Pinsker's-inequality}
	Let $\rho_0,\rho_1$ be two probability distributions supported on some set $\mathbb{X}$, with $\rho_0$ absolutely continuous with respect to $\rho_1$ then,
	\[
	\left\|\rho_0-\rho_1 \right\|_1
	\le
	\sqrt{\frac{1}{2}\mathrm{KL}(\rho_0,\rho_1)},
	\]
	where $\left\|\rho_0-\rho_1 \right\|_1 = \sup_{A\subset \mathbb{X}} \left| \rho_0(A)-\rho_1(A) \right|$ is the variational distance between $\rho_0$ and $\rho_1$.
\end{lemma}

\begin{lemma}\label{lemma-expected-exponential-average-bernoulli}
	Let $X_1,\dots,X_n$ be i.i.d. Bernoulli random variable with parameters $p_1, \dots, p_n$, respectively, then, for any $\kappa >0 $:\vspace{-0.1cm}
	\[
	\mathbb{E}\left[
	e^{-\kappa\sum\limits_{j=1}^n X_j}
	\right]
	\le
	e^{-\sum\limits_{j=1}^{n}p_j/10}
	+
	e^{-\kappa \sum\limits_{j=1}^{n}p_j /2}.
	\]
\end{lemma}

\proof   \;\; \;Define event $\mathcal{E}$ to be the event that $\sum\limits_{j=1}^n X_j < \sum\limits_{j=1}^{n}p_j/2$. By Lemma \ref{lemma-Bernstein-inequality}, we have:
\begin{align*}
	\mathbb{P}\left\{ E \right\} \le e^{-\sum\limits_{j=1}^{n}p_j/10}.
\end{align*}
By the law of total expectation:
\begin{align}
	\mathbb{E}\left[ e^{-\kappa\sum\limits_{j=1}^n X_j}\right]
	&\overset{\text{}}{\le}
	\mathbb{E}\left[ e^{-\kappa\sum\limits_{j=1}^n X_j} \;\middle|\; \mathcal{E}\right] \cdot \mathbb{P}\{\mathcal{E}\} + \mathbb{E}\left[ e^{-\kappa\sum\limits_{j=1}^n X_j} \;\middle|\; \mathcal{E}^c\right] \cdot \mathbb{P}\{\mathcal{E}^c\}
	\nonumber\\
	&\overset{\text{(a)}}{\le}
	\mathbb{P}\{E\} + \mathbb{E}\left[ e^{-\kappa\sum\limits_{j=1}^n X_j} \;\middle|\; \bar{E}\right] \le e^{-\sum\limits_{j=1}^{n}p_j/10} + e^{-\kappa \sum\limits_{j=1}^{n}p_j /2}.
	\nonumber
\end{align}
\endproof

\begin{lemma}\label{sum-exponentials}
	For any $\kappa>0$ and integers $n_1$ and $n_2$, one has: $\sum\limits_{j=n_1}^{n_2}e^{-\kappa j} \le  \frac{1}{\kappa}(e^{-\kappa (n_1-1)} - e^{\kappa n_2})$.
\end{lemma}
\proof \; \; \; We can upper bound the summation using an integral as follows:
\[
\sum\limits_{j=n_1}^{n_2}e^{-\kappa j} \le \int_{n_1-1}^{n_2} e^{-\kappa x} dx = \frac{1}{\kappa}(e^{-\kappa (n_1-1)} - e^{\kappa n_2}).
\]
\endproof

\vspace{-0.1cm}
\section{Additional Auxiliary Analysis}\label{appendix:Auxiliary proofs}\vspace{-0.1cm}
\label{subsec:app:corollaries}

\begin{corollary}\label{corollary-regret-upper bound-stationary} \textbf{\textup{(Near optimality under stationary information arrival process)}}
	Let $\pi$ be Thompson sampling with auxiliary observations tuned by $c > 0$. If $h_{k,t}$'s are i.i.d. Bernoulli random variables with parameter $\lambda$ then, for every $T \ge 1$:\vspace{-0.1cm}
	\[
	\mathbb{E}_{\vector{H}}\left[ \mathcal{R}^\pi_{\cal S}(\vector{H}, T) \right]
	\le
	\left(\sum\limits_{k\in\mathcal{K}}\Delta_k \right) \left(
	\frac{18c\sigma^2}{\Delta^2} \log \left( \min\left\{T+1, \frac{18c\hat{\sigma}^2+10\Delta^2}{\Delta^2\lambda}  \right\} \right) + C\l(1+\frac{1}{\Delta^4_k}\r)
	\right),\]
	for some constant $C$ independent of $T$.
\end{corollary}

\proof
\begin{align}
	\mathbb{E}_{\vector{H}}\left[ \mathcal{R}^\pi_{\cal S}(\vector{H}, T) \right]
	&\overset{\text{(a)}}{\le}
	\mathbb{E}_{\vector{H}}\left[\sum\limits_{k\in\mathcal{K}} \frac{18c\sigma^2\Delta_k}{\Delta^2} \log \left( \sum\limits_{t=1}^T \exp\left(
	{\frac{-\Delta^2}{18c\hat{\sigma}^2}\sum\limits_{s = 1}^t h_{k,s}}
	\right)
	\right) + C\Delta_k\l(1+\frac{1}{\Delta^4_k}\r) \right]
	\nonumber\\
	&\overset{\text{(b)}}{\le}
	\sum\limits_{k\in\mathcal{K}} \left[ \frac{18c\sigma^2\Delta_k}{\Delta^2} \log \left( \mathbb{E}_{\vector{H}} \left[\sum\limits_{t=1}^T \exp\left(
	{\frac{-\Delta^2}{18c\hat{\sigma}^2}\sum\limits_{s = 1}^t h_{k,s}}
	\right)
	\right]\right) + C\Delta_k\l(1+\frac{1}{\Delta^4_k}\r) \right]
	\nonumber\\
	&\overset{\text{(c)}}{\le}
	\left(\sum\limits_{k\in\mathcal{K}} \Delta_k \right) \cdot \left( \frac{18c\sigma^2}{\Delta^2} \log \left( \frac{1-\exp\left(
		{-T\lambda/10}
		\right)
	}{\lambda/10} + \frac{1-\exp\left(
		{-\Delta^2T\lambda/18c\hat{\sigma}^2}
		\right)
	}{\Delta^2\lambda/18c\hat{\sigma}^2} \right) + C\l(1+\frac{1}{\Delta^4_k}\r) \right)
	\nonumber \\
	&{\le}
	\left(\sum\limits_{k\in\mathcal{K}} \Delta_k \right) \cdot \left( \frac{18c\sigma^2}{\Delta^2} \log \left( \frac{1}{\lambda/10} + \frac{1}{\Delta^2\lambda/18c\hat{\sigma}^2} \right) + C\l(1+\frac{1}{\Delta^4_k}\r) \right),
\end{align}
where: (a) holds by Theorem \ref{theorem-Thompson-Sampling-regret-upper bound}, (b) follows from Jensen's inequality and the concavity of $log(\cdot)$, and (c) holds by Lemmas \ref{lemma-expected-exponential-average-bernoulli} and \ref{sum-exponentials}. Noting that $\sum\limits_{t=0}^T \exp\left( {\frac{-\Delta^2}{18c\hat{\sigma}^2}\sum\limits_{s = 1}^t h_{k,s}} \right) < T + 1$, the result is established. \endproof

\begin{corollary}\label{corollary-regret-upper bound-decaying-rate} \textbf{\textup{(Near optimality under diminishing information arrival process)}}
	Let $\pi$ be Thompson sampling with auxiliary observations tuned by $c > 0$. If $h_{k,t}$'s are random variables such that for some $\kappa \in~\mathbb{R}^+$,
	$\mathbb{E}\left[\sum_{s=1}^t h_{k,s}\right] = \lceil  \frac{\hat{\sigma}^2 \kappa}{2\Delta^2} \log t \rceil$ for each arm $k \in \mathcal{K}$ at each time step $t$, then for every $T \ge 1$:\vspace{-0.1cm}
	\[
	\mathbb{E}_{\vector{H}}\left[ \mathcal{R}^\pi_{\cal S}(\vector{H}, T) \right]
	\le
	\left(\sum\limits_{k\in\mathcal{K}}\Delta_k \right) \left(
	\frac{18c\sigma^2}{\Delta^2} \log \left(2 +  \frac{T^{1-\frac{\kappa}{72c}}-1}{1-\frac{\kappa}{72c}}
	+
	\frac{T^{1-\frac{\kappa\hat{\sigma}^2}{20\Delta^2}}-1}{1-\frac{\kappa\hat{\sigma}^2}{20\Delta^2}}
	\right) + C\l(1+\frac{1}{\Delta^4_k}\r)
	\right),\vspace{-0.1cm}
	\]
	for some constant $C$ independent of $T$.
\end{corollary}

\proof
\begin{align}
	\mathbb{E}_{\vector{H}}\left[
	\mathcal{R}^\pi_{\cal S}(\vector{H}, T) \right]
	&\overset{\text{(a)}}{\le}
	\mathbb{E}_{\vector{H}}\left[
	\sum\limits_{k\in\mathcal{K}} \frac{18c\sigma^2\Delta_k}{\Delta^2} \log \left( \sum\limits_{t=1}^T \exp\left(
	{\frac{-\Delta^2}{18c\hat{\sigma}^2}\sum\limits_{s = 1}^t h_{k,s}}
	\right)
	\right) + C \Delta_k\l(1+\frac{1}{\Delta^4_k}\r) \right]
	\nonumber\\
	&\overset{\text{(b)}}{\le}
	\sum\limits_{k\in\mathcal{K}}\left[ \frac{18c\sigma^2\Delta_k}{\Delta^2} \log \left(
	\mathbb{E}_{\vector{H}}\left[ \sum\limits_{t=1}^T \exp\left(
	{\frac{-\Delta^2}{18c\hat{\sigma}^2}\sum\limits_{s = 1}^t h_{k,s}}
	\right) \right]
	\right) + C\Delta_k\l(1+\frac{1}{\Delta^4_k}\r) \right]
	\nonumber\\
	&\overset{\text{(c)}}{\le}
	\left(\sum\limits_{k\in\mathcal{K}}\Delta_k \right) \cdot \left(  \frac{18c\sigma^2\Delta_k}{\Delta^2} \log \left( \sum\limits_{t=1}^T \exp\left(
	{\frac{-\kappa}{72c}\log t}
	\right) + \exp\left(
	{\frac{-\kappa\hat{\sigma}^2}{20\Delta^2}\log t}
	\right)
	\right) + C\l(1+\frac{1}{\Delta^4_k}\r)\right)
	\nonumber\\
	&\overset{\text{}}{=}
	\left(\sum\limits_{k\in\mathcal{K}}\Delta_k \right) \cdot \left( \frac{18c\sigma^2\Delta_k}{\Delta^2} \log \left( \sum\limits_{t=1}^T t^{\frac{-\kappa}{72c}}
	+ t^{\frac{-\kappa\hat{\sigma}^2}{20\Delta^2}}
	\right)
	+ C\l(1+\frac{1}{\Delta^4_k}\r)\right)
	\nonumber\\
	&\overset{\text{}}{\le}
	\left(\sum\limits_{k\in\mathcal{K}}\Delta_k \right) \cdot \left( \frac{18c\sigma^2\Delta_k}{\Delta^2} \log \left(2 +  \int_{1}^T \left(
	t^{\frac{-\kappa}{72c}} + t^{\frac{-\kappa\hat{\sigma}^2}{20\Delta^2}}
	\right) dt
	\right)
	+ C\l(1+\frac{1}{\Delta^4_k}\r)\right)
	\nonumber\\
	&\overset{\text{}}{=}
	\left(\sum\limits_{k\in\mathcal{K}}\Delta_k \right) \cdot \left( \frac{18c\sigma^2\Delta_k}{\Delta^2} \log \left(2 +  \frac{T^{1-\frac{\kappa}{72c}}-1}{1-\frac{\kappa}{72c}}
	+
	\frac{T^{1-\frac{\kappa\hat{\sigma}^2}{20\Delta^2}}-1}{1-\frac{\kappa\hat{\sigma}^2}{20\Delta^2}}
	\right)
	+ C\l(1+\frac{1}{\Delta^4_k}\r)\right),
\end{align}
where: (a) holds by Theorem \ref{theorem-Thompson-Sampling-regret-upper bound}, (b) follows from Jensen's inequality and the concavity of $log(\cdot)$, and (c) holds by Lemma \ref{lemma-expected-exponential-average-bernoulli}. \endproof

\subsection{Proof for Subsection \ref{example-stochastic-lower bound}}\label{appendix-proof-of-example-Stationary-uniform}\vspace{-0.1cm}
We note that:
\begin{align}\label{eq-proof-example-Stationary1}
	\mathbb{E}_{\vector{H}}\left[ \mathcal{R}^\pi_{\cal S}(\vector{H}, T) \right]
	&\overset{\text{(a)}}{\ge}
	\mathbb{E}_{\vector{H}}\left[\frac{\sigma^2(K-1)}{4K\Delta} \sum\limits_{k=1}^K\log \left( \frac{\Delta^2}{\sigma^2K}\sum\limits_{t=1}^{T} e^{-\frac{2\Delta^2}{\hat{\sigma}^2}\sum\limits_{s=1}^th_{s,k}}\right)\right]
	\nonumber\\
	&\overset{\text{(b)}}{\ge}
	\frac{\sigma^2(K-1)}{4\Delta} \log \left( \frac{\Delta^2}{\sigma^2K}\sum\limits_{t=0}^{T-1} e^{-\frac{2 \Delta^2}{\hat{\sigma}^2} t\lambda}\right)
	\nonumber\\
	&\overset{\text{}}{=}
	\frac{\sigma^2(K-1)}{4\Delta} \log \left( \frac{\Delta^2}{\sigma^2K} \cdot \frac{1 - e^{-\frac{2 \Delta^2}{\hat{\sigma}^2} T\lambda}}{1 - e^{-\frac{2 \Delta^2}{\hat{\sigma}^2} \lambda}}\right)
	\nonumber\\
	&\overset{\text{(c)}}{\ge}
	\frac{\sigma^2(K-1)}{4\Delta} \log \left(\frac{1 - e^{-\frac{2 \Delta^2}{\hat{\sigma}^2} T\lambda}}{ 2K\lambda\sigma^2/\hat{\sigma}^2}\right),
\end{align}
where (a) holds by Theorem \ref{theorem-lower bound-general}, (b) follows from the fact that log-sum-exp is a convex function (see \cite[Example 3.1.5]{boyd2004convex}), and (c) follows from $1 - e^{-\frac{2 \Delta^2}{\hat{\sigma}^2} \lambda} \le 2\Delta^2\lambda / \hat{\sigma}^2$. Now we consider the following cases:
\begin{enumerate}
	\item If $2\Delta^2T\lambda \le \hat{\sigma}^2/2$ then, by inequality $1-e^{-x} \ge 2(1-e^{-1/2})x$ for $0 \le x \le 1/2$, and (\ref{eq-proof-example-Stationary1}), we have
	\[
	\mathbb{E}_{\vector{H}}\left[ \mathcal{R}^\pi_{\cal S}(\vector{H}, T) \right]
	\ge
	\frac{\sigma^2(K-1)}{4\Delta} \log \left( \frac{(1-e^{-1/2})\Delta^2 T}{\sigma^2K}\right).
	\]
	
	\item If $2\Delta^2T\lambda \ge \sigma^2/2$ then, by inequality $1-e^{-\frac{2 \Delta^2}{\sigma^2} T\lambda} \ge 1 - e^{-1/2}$, and (\ref{eq-proof-example-Stationary1}), we have
	\begin{equation} \label{eq-proof-example-Stationary2}
		\quad\quad\quad\quad\mathbb{E}_{\vector{H}}\left[ \mathcal{R}^\pi_{\cal S}(\vector{H}, T) \right]
		\ge
		\frac{\sigma^2(K-1)}{4\Delta} \log \left( \frac{1-e^{-1/2}}{2\lambda K\sigma^2/\hat{\sigma}^2}\right).\quad\quad\quad\quad\qed
	\end{equation}
\end{enumerate}

\vspace{-0.1cm}
\subsection{Analysis of a Myopic Policy under the Setting of Example \ref{example-stochastic-lower bound}.} \vspace{-0.1cm} \label{appendix-proof-of-myopic-optimality]}
Assume that $\pi$ is a myopic policy. Consider a suboptimal arm $k$. One has:
\begin{align*}
	\mathbb{P} \left\{
	\pi_t = k
	\right\}
	\le
	\mathbb{P} \left\{
	\bar{X}_{k,t}^{\known} >  \bar{X}_{k^\ast, n^{\known} _{k,t}}
	\right\}
	\le
	\mathbb{P} \left\{
	\bar{X}_{k,t}^{\known} > \mu_k + \frac{\Delta_k}{2}
	\right\}
	+ \mathbb{P} \left\{
	\bar{X}_{k^\ast,t}^{\known} < \mu_{k^\ast} + \frac{\Delta_k}{2}
	\right\}.
\end{align*}
We will upper bound $\mathbb{P} \left\{
\bar{X}_{k,t}^{\known} > \mu_k + \frac{\Delta_k}{2}
\right\}$. $\mathbb{P} \left\{
\bar{X}_{k^\ast,t}^{\known} < \mu_{k^\ast} + \frac{\Delta_k}{2}
\right\}$ can be upper bounded similarly.
\begin{align*}
	\mathbb{P} \left\{
	\bar{X}_{k,t}^{\known} > \mu_k + \frac{\Delta_k}{2}
	\right\}
	&\le
	\mathbb{P} \left\{
	\bar{X}_{k,t}^{\known} > \mu_k + \frac{\Delta_k}{2} \;\middle|\; \sum_{s=1}^{t} h_{k,s} > \frac{\lambda t}{ 2 }
	\right\}
	+
	\mathbb{P} \left\{
	\sum_{s=1}^{t} h_{k,s} \le \frac{\lambda t}{ 2 }
	\right\}
	\overset{(a)}{\le}
	e^{\frac{-\Delta_k^2 \lambda t}{16\hat{\sigma}^2}} + e^{\frac{-\lambda t}{10}},
\end{align*}
where (a) follows from Lemma \ref{lemma-Chernoff-Hoeffding-bound}, and Lemma \ref{lemma-Bernstein-inequality}. As a result the cumulative regret is upper bounded by
\[
\sum_{k\in\mathcal{K} \setminus k^\ast} \left(
\frac{32 \hat{\sigma}^2}{\lambda \Delta_k} + \frac{20\Delta_k}{\lambda}
\right).
\]
Therefore, if $\Delta$ is a constant independent of $T$ then the regret is upper bounded by a constant. \qed

\vspace{-0.1cm}
\subsection{Proof for Subsection \ref{example-Decayin-rate}.}  \label{appendix-proof-of-example-decaying}\vspace{-0.1cm}
We note that:
\begin{align}\label{eq-proof-example-Stationary1}
	\mathbb{E}_{\vector{H}}\left[
	\mathcal{R}^\pi_{\cal S}(\vector{H}, T) \right]
	&\overset{\text{(a)}}{\ge}
	\mathbb{E}_{\vector{H}}\left[
	\frac{\sigma^2(K-1)}{4K\Delta} \sum\limits_{k=1}^K\log \left( \frac{\Delta^2}{\sigma^2K}\sum\limits_{t=1}^{T} e^{-\frac{2\Delta^2}{\hat{\sigma}^2}\sum\limits_{s=1}^th_{s,k}}\right) \right]
	\nonumber\\
	&\overset{\text{(b)}}{\ge}
	\frac{\sigma^2(K-1)}{4\Delta} \log \left( \frac{\Delta^2}{\sigma^2K}\sum\limits_{t=1}^{T} e^{-\kappa \log t}\right)
	\nonumber\\
	&\overset{\text{}}{=}
	\frac{\sigma^2(K-1)}{4\Delta} \log \left( \frac{\Delta^2}{\sigma^2K} \sum\limits_{t=1}^{T} t^{-\kappa } \right)
	\nonumber\\
	&\overset{\text{}}{\ge}
	\frac{\sigma^2(K-1)}{4\Delta} \log \left( \frac{\Delta^2}{\sigma^2K} \int_{t=1}^{T+1} t^{-\kappa }dt \right)
	\nonumber\\
	&\overset{\text{}}{=}
	\frac{\sigma^2(K-1)}{4\Delta} \log \left(\frac{\Delta^2}{\sigma^2K} \cdot \frac{(T+1)^{1-\kappa} - 1}{1-\kappa} \right),
\end{align}
where (a) holds by Theorem \ref{theorem-lower bound-general}, (b) follows from the fact that log-sum-exp is a convex function (see \cite[Example 3.1.5]{boyd2004convex}).
\qed

\vspace{-0.1cm}
\subsection{Analysis of a Myopic Policy under the Setting of Example \ref{example-Decayin-rate}.} \vspace{-0.1cm} \label{appendix-analysis-of-myopic-example-decaying}
In this section we will shortly prove that if $\mathbb{E}\left[\sum\limits_{s=1}^t h_{k,s} \right]= \lceil \frac{\hat{\sigma}^2\kappa}{2\Delta^2} \log t\rceil$ for each arm $k$ and time step $t$ then a myopic policy achieves an asymptotic constant regret if $ \Delta_k^2 \kappa > 16\hat{\sigma}^2$. For any profile $\vectorgreek{\nu}$, we have:
\begin{align*}
	\quad\quad\quad\quad\mathbb{E}_{\vector{H}}\left[\mathcal{R}^\pi_{\cal S}(\vector{H}, T)\right]
	&\le
	\sum\limits_{k \in \mathcal{K}\setminus \{ k^\ast \} } \Delta_k \cdot \sum\limits_{t=1}^T \mathbb{E} e^{\frac{-\Delta^2 n^{\known} _{k,t}}{8\hat{\sigma}^2}} + \mathbb{E} e^{\frac{-\Delta^2 n^{\known} _{k^\ast,t}}{8\hat{\sigma}^2}}
	\\
	&\le
	\sum\limits_{k \in \mathcal{K}\setminus \{ k^\ast \} } \Delta_k \cdot \sum\limits_{t=1}^T t^{\frac{-\Delta_k^2 \kappa}{16\Delta^2}} + t^{\frac{-\Delta_k^2 \kappa}{16\Delta^2}} \le \sum\limits_{k \in \mathcal{K}\setminus \{ k^\ast \} } \frac{32\Delta^2}{\kappa\Delta_k}.\quad\quad\quad\quad\qed
\end{align*}

\vspace{-0.4cm}
\section{Endogenizing Exploration via Virtual Time Indices}\label{section:A near-optimal adaptive policy}\vspace{-0.1cm}

\subsection{Policy Design Intuition}\label{subsec:inuition}\vspace{-0.0cm}

In this subsection we first demonstrate, through the case of $\epsilon_t$-greedy policy, that policies with exogenous exploration rate may fail to achieve the lower bound in Theorem \ref{theorem-lower bound-general} in the presence of auxiliary information arrivals, and then develop intuition for how such policies may be adjusted to better leverage auxiliary information through the virtual time indices method. Formal policy design and analysis follow in~\S\ref{sec:formaldesign}.

\medskip
\noindent
\textbf{The Inefficiency of ``Naive" Adaptations of $\boldsymbol{\epsilon_t}$-greedy.} Despite the robustness that was established in \S4 for Thompson sampling and UCB, in general, accounting for auxiliary information while otherwise maintaining the policy structure may not suffice for achieving the lower bound established in Theorem \ref{theorem-lower bound-general}. To demonstrate this, consider the $\epsilon_t$-greedy policy (\citealt{auer2002finite-time}), which at each period~$t$ selects an arm randomly with probability~$\epsilon_t$ that is proportional to~$t^{-1}$, and with probability $1-\epsilon_t$ selects the arm with the highest reward estimate. Without auxiliary information, $\epsilon_t$-greedy guarantees rate-optimal regret of order $\log T$, but this optimality does not carry over to settings with other information arrival processes: as was visualized in Figure~\ref{fig:main-text-e-greedy}, using auxiliary observations to update estimates without appropriately adjusting the exploration rate leads to sub-optimality.\footnote{A natural alternative is to increase the time index by one (or any other constant) each time auxiliary information arrives. Such an approach, that could be implemented by arm-dependent time indices with update rule $\tau_{k,t} = \tau_{k,t-1}+1+h_{k,t}$, may lead to sub-optimality as well. For example, suppose $h_{k,1} = \left\lfloor\frac{C}{\Delta_k^2}\log T\right\rfloor$ for some constant $C>0$, and $h_{k,t} = 0$ for all $t>1$. By Theorems~2 and~3, when the constant $C$ is sufficiently large constant regret is achievable. However, following the above update rule the regret incurred due to exploration is at least of order $\log(\frac{\Delta^2_{k}T}{\log T})$.} For example, consider stationary information arrivals (described in \S\ref{example-stochastic-lower bound}), with arrival rate $\lambda \geq \frac{\hat{\sigma}^2}{4\Delta^2T}$. While constant regret is achievable in this case, without adjusting its exploration rate, $\epsilon_t$-greedy explores suboptimal arms at a rate that is independent of the number of auxiliary observations, and still incurs regret of order $\log T$.

\medskip
\noindent \textbf{Over-Exploration in the Presence of Additional Information.}
For simplicity, consider a 2-armed bandit problem with $\mu_1 > \mu_2$. At each time $t$ the $\epsilon_t$-greedy policy explores over arm 2 independently with probability $\epsilon_t = ct^{-1}$ for some constant $c>0$. As a continuous-time proxi for the minimal number of times arm $2$ is selected by time $t$, consider the function\vspace{-0.15cm}
\[
\int_{s=1}^{t} \epsilon_s ds = c\int_{s=1}^{t}\frac{ds}{s} = c\log t.\vspace{-0.15cm}
\]
The probability of best-arm misidentification at period $t$ can be bounded from above by:\footnote{This upper bound on the probability of misidentifying the best arm can be obtained using standard concentration inequalities and is formalized, for example, in Step 5 of the proof of Theorem \ref{theorem-epsilon-greedy-upper bound}.}\vspace{-0.15cm}
\[
\exp\left(-\bar{c}\int_{s=1}^{t} \epsilon_s ds\right) \leq  \tilde{c} t^{-1},\vspace{-0.15cm}
\]
for some constants $\bar{c}$, $\tilde{c}>0$. Thus, 
setting $\epsilon_t = ct^{-1}$ balances losses from exploration and exploitation.\footnote{See related discussions on balancing losses from experimenting with suboptimal arms and from best arm misidentification, e.g., in \cite{auer2002finite-time}, \cite{langford2008epoch}, \cite{goldenshluger2013linear}, and \cite{bastani2015online}.}

Next, assume that just before time $t_0$, an additional independent reward observation of arm 2 is collected. Then, at time $t_0$, the minimal number of observations from arm 2 \emph{increases} to ($1 + c\log t_0$), and the upper bound on the probability of best-arm misidentification \emph{decreases} by factor~$e^{-\bar{c}}$:\vspace{-0.15cm}
\[
\exp\left(-\bar{c}\left(1+\int_{s=1}^{t} \epsilon_s ds\right)\right)
\le   e^{-\bar{c}}\cdot\tilde{c}t^{-1},\quad\quad\forall\; t\geq t_0.\vspace{-0.15cm}
\]
Therefore, when there are many auxiliary observations, the loss from best-arm misidentification is guaranteed to diminish, but performance might still be sub-optimal due to over-exploration.

\medskip\noindent
\textbf{Endogenizing the Exploration Rate.}
Note that there exists a \emph{future} time period $\hat{t}_0 \ge t_0$ such that:\vspace{-0.2cm}
\begin{equation}\label{eq-policy-intuition}
	1 +c\int_{1}^{t_0}\frac{ds}{s} = c\int_{1}^{\hat{t}_0}\frac{ds}{s}. \vspace{-0.1cm}
\end{equation}
In words, the minimal number of observations from arm 2 by time $t_0$, including the one that arrived just before $t_0$, equals (in expectation) the minimal number of observations from arm 2 by time $\hat{t}_0$ without any additional information arrivals.
Therefore, replacing $\epsilon_{t_0} = ct_0^{-1}$ with $\epsilon_{t_0} = c\hat{t}_0^{-1}$ would adjust the exploration rate to fit the amount of information \emph{actually} collected at $t=t_0$, and the respective loss from exploitation. The regret reduction by this adjustment is illustrated in Figure~\ref{fig:adaexplor}.
We therefore adjust the exploration rate to be $\epsilon_t = c(\tau(t))^{-1}$ for some \emph{virtual time} $\tau(\cdot)$. We set $\tau(t) = t$ for all $t < t_0$, and $\tau(t) = \hat{t}_0 + (t-t_0)$ for $t \ge t_0$.
\begin{figure}[t]
	\centering
	\includegraphics[height=2.2in]{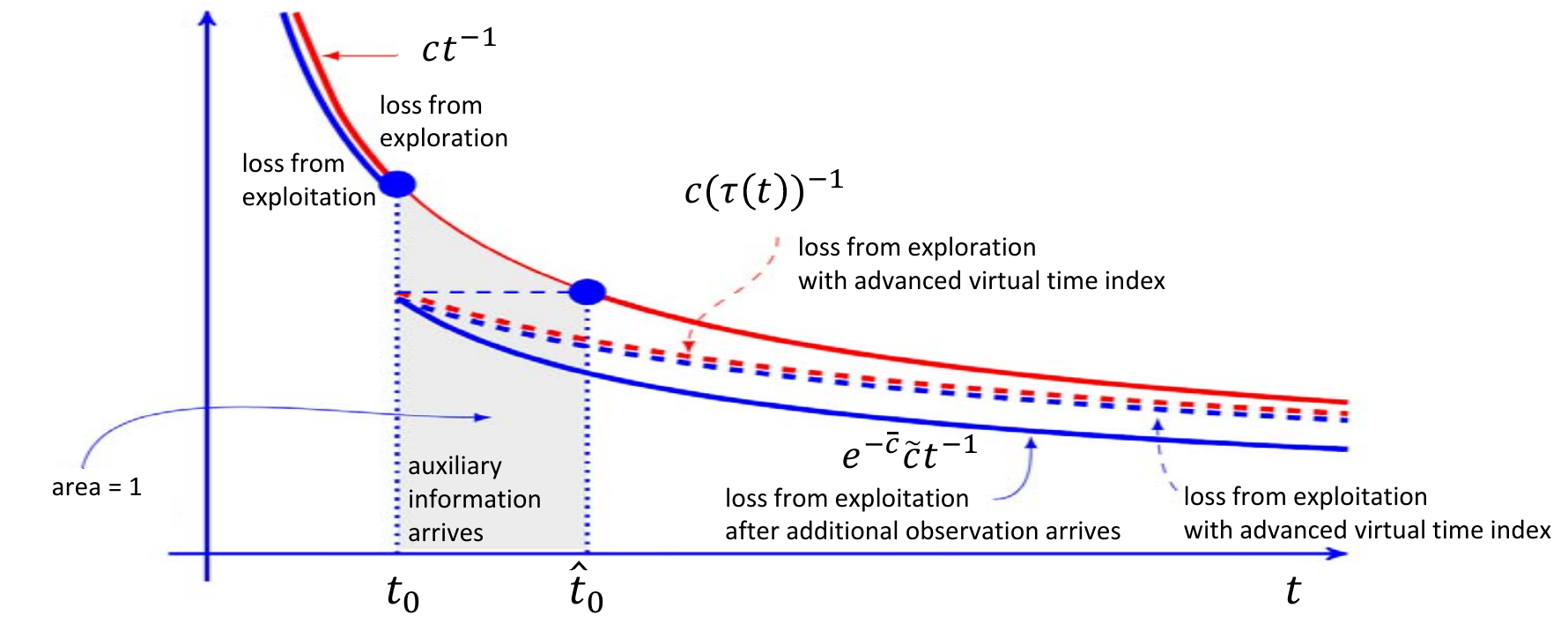}\vspace{-0.2cm}
	\caption{\small Losses from exploration and exploitation (normalized through division by $(\mu_1-\mu_2)$) when additional observation is collected just before time $t_0$, with and without replacing the standard time index $t$ with a virtual time index $\tau(t)$. With a standard time index, exploration is performed at rate $ct^{-1}$, which results in sub-optimality after time $t_0$. With a virtual time index that is advanced at time $t_0$, the exploration rate becomes $c(\tau(t))^{-1}$, which re-balances losses from exploration and exploitation (dashed lines coincide).}\label{fig:adaexplor}\vspace{-0.3cm}
\end{figure}
Solving \eqref{eq-policy-intuition} for $\hat{t}_0$, we write $\tau(t)$ in closed form:\vspace{-0.2cm}
\[
\tau(t) = \begin{cases}
	t & t < t_0\\
	c_0t_0 + ( t-t_0) & t \geq t_0,
\end{cases}\vspace{-0.2cm}
\]
for some constant $c_0>1$. Therefore, the virtual time grows together with $t$, and advanced by a multiplicative constant whenever an auxiliary observation is collected.

\vspace{-0.1cm}
\subsection{A Rate-Optimal Adaptive ${\epsilon_t}$-greedy Policy} \label{sec:formaldesign}\vspace{-0.1cm}
We apply the ideas discussed in \S\ref{subsec:inuition} to design an $\epsilon_t$-greedy policy with \textit{adaptive exploration} that dynamically adjusts the exploration rate in the presence of an unknown information arrival process. For simplicity, and consistent with standard versions of $\epsilon_t$-greedy (see, e.g., \citealt{auer2002finite-time}), the policy below assumes prior knowledge of the parameter $\Delta$.

Define $n^{\known} _{k,t}$ and $\bar{X}_{k, t}$ as in \eqref{eq-def-of-counters-and-empirical-mean}, and consider the following adaptation of the $\epsilon_t$-greedy policy.\vspace{-0.25cm}
\begin{framed}\small
	\noindent \textbf{$\epsilon_t$-greedy with adaptive exploration.}\label{Adjusted-epsilon-greedy-policy} Input: a tuning parameter $c>0$.\vspace{-0.1cm}
	\begin{enumerate}
		\item Initialization: set initial virtual times $\tau_{k,0} = 0$ for all $k\in \cal K$\vspace{-0.1cm}
		\item At each period $t = 1,2,\dots,T$:\vspace{-0.1cm}
		\begin{enumerate}
			\item Observe the vectors $\vectorgreek{h}_t$ and $\vector{Z}_t$, and update virtual time indices for all $k \in \mathcal{K}$:\vspace{-0.2cm}				\[
			\tau_{k,t} =
			\left(\tau_{k,t-1} + 1\right)\cdot \exp\left(\frac{h_{k,t}\Delta^2}{c\hat{\sigma}^2}\right)\vspace{-0.4cm}
			\]
			\item With probability $\min\left\{1, \frac{c\sigma^2}{\Delta^2}\sum\limits_{k^{\mathrm{aux}}=1}^K \tau_{k,t}^{-1}\right\}$ select an arm at random: (\emph{exploration})\vspace{-0.3cm}
			\[
			\pi_t = k \quad \mbox{with probability}\quad \frac{\tau_{k,t}^{-1}}{\sum\limits_{k=1}^K \tau_{k,t}^{-1}}, \quad \text{for all } k\in \cal K\vspace{-0.25cm}
			\]
			Otherwise, select an arm with the highest estimated reward: (\emph{exploitation})\vspace{-0.2cm}
			\[
			\pi_t = \arg \max_{k\in\mathcal{K}} \bar{X}_{k,t}^{\known} ,\vspace{-0.15cm}
			\]
			\item Receive and observe a reward $X_{\pi_t,t}$\vspace{-0.35cm}
		\end{enumerate}
	\end{enumerate}
\end{framed}\normalsize
\vspace{-0.3cm}

\noindent At every period $t$, the $\epsilon_t$-greedy with adaptive exploration policy dynamically reacts to the information sample path by advancing virtual time indices associated with different arms based on auxiliary observations that were collected since the last period. Then, the policy explores with probability that is proportional to $\sum\limits_{k^{\mathrm{aux}}=1}^K \tau_{k^{\mathrm{aux}},t}^{-1}$, and otherwise pulls the arm with the highest empirical mean reward.

Every time additional information on arm $k$ is observed, a carefully selected multiplicative factor is used to advance the virtual time index $\tau_{k,t}$ according to the update rule $\tau_{k,t} = \left(\tau_{k,t-1} + 1\right)\cdot \exp\left(\delta \cdot h_{k,t}\right)$, for some suitably selected $\delta$. In doing so, the policy effectively reduces exploration rates in order to explore over each arm $k$ at a rate that would have been appropriate \textit{without} auxiliary information arrivals at a \textit{future} time step $\tau_{k,t}$. This guarantees that the loss due to exploration is balanced with the loss due to best-arm misidentification throughout the horizon. Advancing virtual times based on the information sample path and the impact on the exploration rate of a policy are illustrated in Figure~\ref{fig:adaexplor}.
\begin{figure}[t]
	\centering
	\includegraphics[height=1.8in]{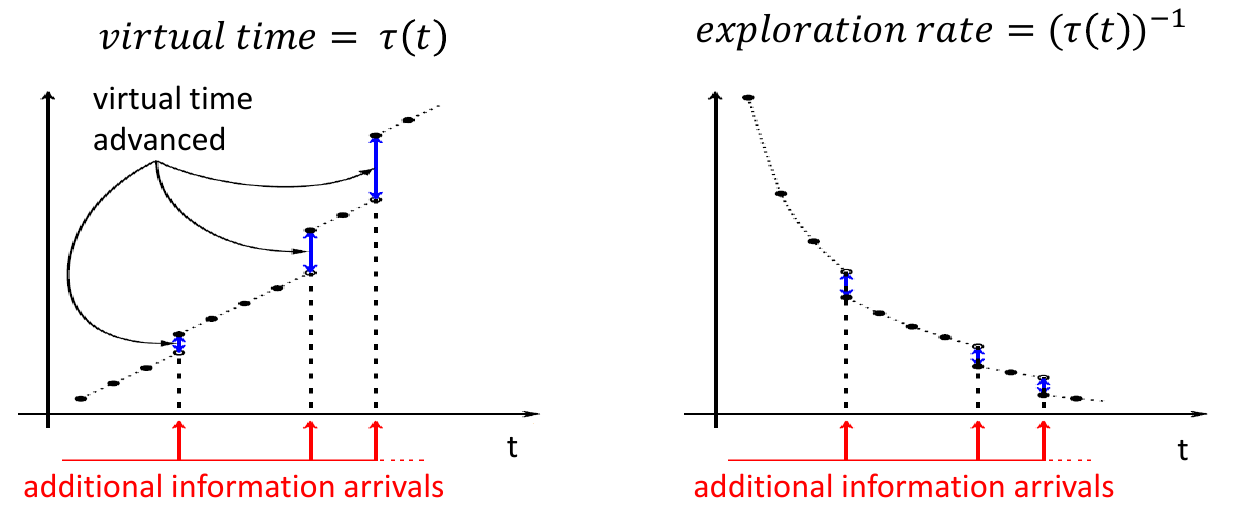}\vspace{-0.1cm}
	\caption{\small Illustration of the adaptive exploration approach. \textit{(Left)} Virtual time index $\tau(\cdot)$ is advanced using multiplicative factors whenever auxiliary information arrives. \textit{(Right)} Exploration rate decreases as a function of~$\tau(\cdot)$, exhibiting discrete ``jumps" whenever auxiliary information is collected.}\label{fig:adaexplor}\vspace{-0.3cm}
\end{figure}

The following result characterizes the guaranteed performance and establishes the rate optimality of $\epsilon_t$-greedy with adaptive exploration in the presence of unknown information arrival processes.\vspace{-0.1cm}

\begin{theorem}\label{theorem-epsilon-greedy-upper bound} \textbf{\textup{(Near optimality of $\epsilon_t$-greedy with adaptive exploration)}}
	Let $\pi$ be an $\epsilon_t$-greedy with adaptive exploration policy, tuned by $c > \max\left\{16,\frac{10\Delta^2}{\sigma^2}\right\}$. Then, for every $T \ge 1$, and for any information arrival matrix $\vector{H}$, and under prior knowledge of mappings $\{\phi_k\}$:\vspace{-0.15cm}
	\[
	\mathcal{R}^\pi_{\cal S}(\vector{H}, T)
	\le
	\sum\limits_{k\in\mathcal{K}}
	\frac{\constvar[eps-greedy1] }{\Delta}\log \left( \sum\limits_{t=t^\ast+1}^T \exp\left(-\frac{\Delta^2}{\constref{eps-greedy1}}\sum\limits_{s = 1}^t h_{k,\tau}\right)\right) +\constvar[eps-greedy2],\vspace{-0.15cm}
	\]
	where $\constref{eps-greedy1}$, and $\constref{eps-greedy2}$ are positive constants that depend only on $\sigma$ and $\hat{\sigma}$.
\end{theorem}\vspace{-0.1cm}

\noindent
The upper bound in Theorem~\ref{theorem-epsilon-greedy-upper bound} holds for any sample path of information arrivals as captured by the matrix $\vector{H}$. This establishes that, similarly to Thompson sampling and UCB policies (but unlike the standard $\epsilon_t$-greedy) the $\epsilon_t$-greedy with adaptive exploration (through virtual time indices) guarantees rate optimality uniformly over the general class of information arrival processes defined in~\S\ref{sec: prob_form}.

\medskip
\noindent
\textbf{Key Ideas in the Proof.} The proof decomposes regret into exploration and exploitation time periods. To bound the regret at exploration time periods express virtual time indices as\vspace{-0.2cm}
\[
\tau_{k,t} = \sum\limits_{s=1}^t \exp\left(
{\frac{\Delta^2}{c\hat{\sigma}^2}\sum\limits_{\tau = s}^t h_{k,\tau}}
\right).\vspace{-0.2cm}
\]
Denoting by $t_m$ the time step at which the $m^{th}$ auxiliary observation for arm $k$ was collected, we establish an upper bound on the expected number of exploration time periods for arm $k$ in the time interval $[t_m, t_{m+1}-1]$, which scales linearly with
$\frac{c\sigma^2}{\Delta^2} \log
\left(
\frac{ \tau_{k,t_{m+1}}
}{\tau_{k,t_{m}}
}
\right) - 1$.
Summing over all values of $m$, we obtain that the regret over exploration time periods is bounded from above by\vspace{-0.15cm}
\[
\sum\limits_{k\in\mathcal{K}}\Delta_k  \cdot \frac{c\sigma^2}{\Delta^2} \log \left(2\sum\limits_{t=0}^T \exp\left(\frac{-\Delta^2}{c\hat{\sigma}^2}\sum\limits_{s = 1}^t h_{k,\tau}\right)
\right).\vspace{-0.15cm}
\]

\noindent To analyze regret at exploitation time periods we first lower bound the number of observations of each arm using Bernstein inequality, and then apply Chernoff-Hoeffding inequality to bound the probability that a sub-optimal arm would have the highest estimated reward, given the minimal number of observations on each arm. When $c > \max\left\{16,\frac{10\Delta^2}{\sigma^2}\right\}$, this regret component decays at rate of at most~$t^{-1}$.

\medskip\noindent
\textbf{Numerical Analysis.} Figure~\ref{fig:main-text-e-greedy} follows the settings described in Figure~\ref{fig:main-text-sim} to visualize the impact of adjusting the exploration rate of an $\varepsilon_t$-greedy policy using virtual time indices as the amount of auxiliary information increases. Appendix \ref{section-numerics} includes further numerical analysis of $\epsilon_t$-greedy with adaptive exploration with a variety of tuning parameter and information arrival scenarios, relative to other policies. We also test robustness with respect to misspecification of the gap parameter $\Delta$.\vspace{-0.1cm}

\begin{figure}[h]
	\centering
	\includegraphics[height=1.4in]{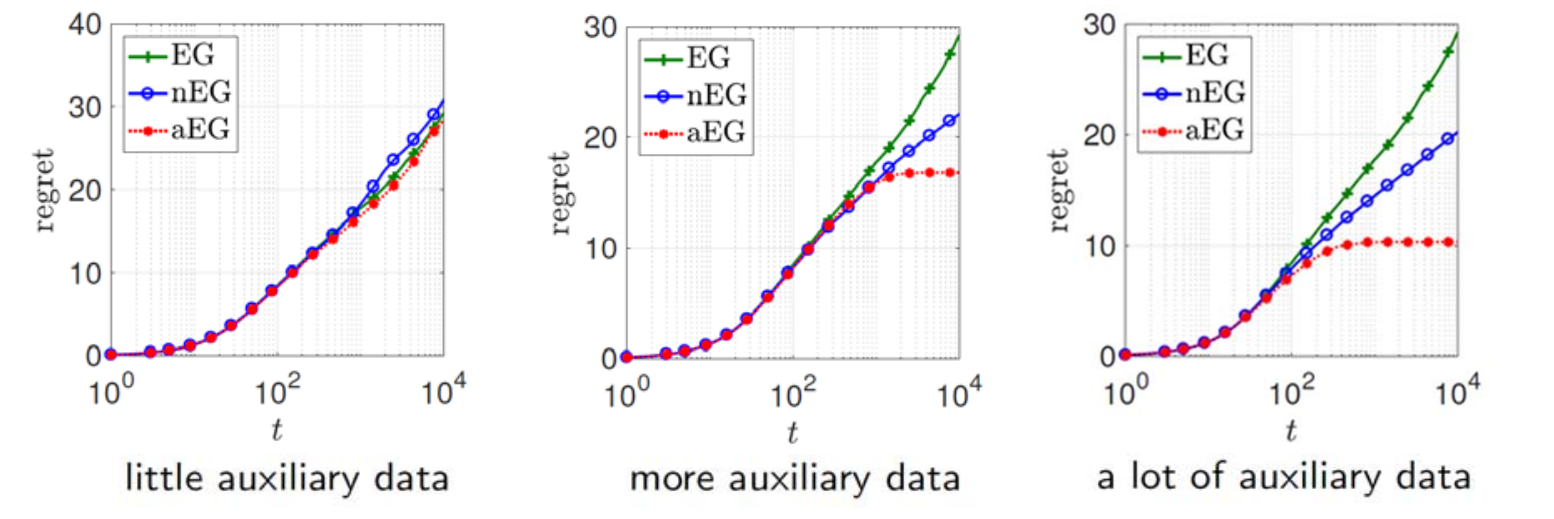}\vspace{-0.25cm}
	\caption{\small The information arrival settings described in Figure~\ref{fig:main-text-sim}, and performance of $\varepsilon_t$-greedy-type policies (tuned by $c=1.0$) that: ignore auxiliary observations (EG); use this information to adjust estimates of mean rewards without appropriately adjusting the exploration rate of the policy (nEG); and adjust reward estimates and dynamically adjusts exploration rates using virtual time indices (aEG).}\label{fig:main-text-e-greedy}\vspace{-0.35cm}
\end{figure}

\subsection{Proof of Theorem \ref{theorem-epsilon-greedy-upper bound}}
Let $t^\ast = \lceil \frac{Kc\sigma^2}{\Delta^2} \rceil$, and fix $T \ge t^\ast$, $K \ge 2$, $\vector{H} \in \{0,1\}^{K\times T}$, and fix a problem instance $(\vectorgreek{\nu}, \vectorgreek{\nu}^{\aux})\in \cS$ with the mean rewards $\vectorgreek{\mu}$ and the mean $\vector{y}$ for auxiliary observations. \vspace{-0.2cm}

\paragraph{Step 1 (Decomposing the regret).}
We decompose the regret of a suboptimal arm $k \neq k^\ast$ as follows:\vspace{-0.2cm}
\begin{align}\label{eq-regret-decomposition-epsilon-greedy}
	\Delta_k \cdot \mathbb{E}^\pi_{\vectorgreek{\nu}, \vectorgreek{\nu}^{\mathrm{aux}}} \left[ \sum\limits_{t=1}^T \mathbbm{1} \left\{ \pi_t = k  \right\}\right]
	&=
	\Delta_k \cdot \overbrace{\mathbb{E}^\pi_{\vectorgreek{\nu}, \vectorgreek{\nu}^{\mathrm{aux}}} \left[ \sum\limits_{t=1}^T  \mathbbm{1} \left\{\pi_t = k, \mbox{ policy does exploration} \right\} \right]}^{J_{k,1}}
	\nonumber
	\\
	&+
	\Delta_k \cdot \underbrace{\mathbb{E}^\pi_{\vectorgreek{\nu}, \vectorgreek{\nu}^{\mathrm{aux}}} \left[ \sum\limits_{t=1}^T \mathbbm{1} \left\{\pi_t = k,  \mbox{ policy does exploitation}  \right\} \right]}_{J_{k,2}}.
\end{align}
The component $J_{k,1}$ is the expected number of times arm $k$ is pulled due to exploration, while the the component $J_{k,2}$ is the one due to exploitation.\vspace{-0.2cm}

\paragraph{Step 2 (Analysis of $\tau_{k,t}$).}
We will later find upper bounds for the quantities $J_{k,1}$, and $J_{k,2}$ that are functions of the virtual time $\tau_{k,T}$. A simple induction results in the following expression:
\vspace{-0.2cm}
\begin{equation}\label{eq-exp-based-time-epsilon-greedy}
	\tau_{k,t} =  \sum\limits_{s=1}^t \exp\left(
	{\frac{\Delta^2}{c\hat{\sigma}^2}\sum\limits_{\tau = s}^t h_{k,\tau}}
	\right).\vspace{-0.2cm}
\end{equation}

\paragraph{Step 3 (Analysis of $J_{k,1}$).} Let $\bar{M} \coloneqq \sum\limits_{t=1}^T h_{k,t}$. For $0\le m \le \bar{M}$, let $t_{m}$ be the time step at which the $m$th auxiliary observation for arm $k$ is received, that is,\vspace{-0.1cm}
\[
t_{m} = \begin{cases}
	\min \left\{1\le t \le T\; \middle| \;\sum\limits_{s=1}^t h_{k,s} = m \right\} & \text{if } 0\le m \le \bar{M} \\
	T+1 & \text{if } m=\bar{M}+1
\end{cases}.\vspace{-0.1cm}
\]
Note that we dropped the index $k$ in the definitions above for simplicity of notation. Also define $\tau_{k,T+1} \coloneqq \left(\tau_{k,T} + 1\right)\cdot \exp\left(\frac{h_{k,t}\Delta^2}{c\hat{\sigma}^2}\right)$.  One obtains:\vspace{-0.2cm}
\begin{align}\label{eq-Jk,1-analysis-epsilon-greedy}
	J_{k,1}
	&= \sum\limits_{t=1}^T\min\left\{\frac{c\sigma^2}{\Delta^2}\sum\limits_{k=1}^K\frac{1}{\tau_{k,t}},1\right\} \cdot \frac{\frac{1}{\tau_{k,t}}}{\sum\limits_{k=1}^K\frac{1}{\tau_{k,t}}}
	\nonumber \\
	& \le
	\sum\limits_{t=1}^T \frac{c\sigma^2}{\Delta^2\tau_{k,t}}
	\nonumber \\
	& =
	\sum\limits_{m=0}^{\bar{M}} \sum\limits_{s=0}^{ t_{m+1}- t_{m}-1} \frac{c\sigma^2}{\Delta^2(\tau_{k, t_{m}}+s)}
	\nonumber \\
	&\overset{\text{(a)}}{\le}
	\sum\limits_{m=0}^{\bar M}  \frac{c\sigma^2}{\Delta^2} \log
	\left(
	\frac{\tau_{k, t_{m}}+ t_{m+1}- t_{m}-1 +1/2}{\tau_{k, t_{m}} -1/2}
	\right)
	\nonumber \\
	&\overset{\text{(b)}}{\le}
	\sum\limits_{m=0}^{\bar M}  \frac{c\sigma^2}{\Delta^2} \log
	\left(
	\frac{\exp\left(\frac{-\Delta^2}{c\hat{\sigma}^2}\right) \cdot \left(\tau_{k, t_{m+1}} -1/2 \right)}{\tau_{k, t_{m}} - 1/2}
	\right)
	\nonumber \\
	&=
	\frac{c\sigma^2}{\Delta^2} \log
	\left(
	\frac{ \tau_{k, t_{\bar{M}+1}}-1/2}{\tau_{k, t_{0}} -1/2}
	\right)
	- \sum\limits_{t=1}^T \frac{\sigma^2}{\hat{\sigma}^2} h_{k,t}
	\nonumber \\
	&=
	\frac{c\sigma^2}{\Delta^2} \log
	\left(
	\frac{ \tau_{k,T+1}-1/2}{1-1/2}
	\right)
	- \sum\limits_{t=1}^T \frac{\sigma^2}{\hat{\sigma}^2} h_{k,t}
	\nonumber \\
	&\overset{\text{(c)}}{\le}
	\frac{c\sigma^2}{\Delta^2} \log \left( 2\sum\limits_{t=0}^T \exp\left(
	{\frac{-\Delta^2}{c\hat{\sigma}^2}\sum\limits_{s = 1}^t h_{k,\tau}}
	\right)
	\right),
\end{align}
where (a), (b), and (c) follow from the following lemma, the fact that $\exp\left(\frac{\Delta^2}{c\hat{\sigma}^2}\right)\left(\tau_{k,t_{m}}+t_{m+1,k}-t_{m}\right) = \tau_{k,t_{m+1}}$, and (\ref{eq-exp-based-time-epsilon-greedy}), respectively:
\begin{lemma}\label{lemma-sum-one-over-t}
	For any $t>1/2$, and $n \in \{0,1,2,\dots\}$, we have $\log\frac{t+n+1}{t}\le\sum\limits_{s=0}^{n}\frac{1}{t+s} \le \log\frac{t+n+1/2}{t-1/2}$.
\end{lemma}
\proof \; \;
We first, show the lower through an integral:
\[
\sum\limits_{s=0}^{n}\frac{1}{t+s} \ge \int_{t}^{t+n+1} \frac{dx}{x} = \log\frac{t+n+1}{t}.
\]
Now, we show the upper bound through a similar argument:
\[
\sum\limits_{s=0}^{n}\frac{1}{t+s} \overset{(a)}{\le} \sum\limits_{s=0}^{n} \int_{t+s-1/2}^{t+s+1/2}\frac{dx}{x} = \int_{t-1/2}^{t+n+1/2}\frac{dx}{x} = \log\frac{t+n+1/2}{t-1/2},
\]
where (a) holds because we have $\frac{1}{t+s} \le \int_{t+s-1/2}^{t+s+1/2}\frac{dx}{x}$ by Jensen's inequality.
\endproof
\vspace{-0.2cm}

\paragraph{Step 4 (Analysis of $n^{\known} _{k,t}$).} To analyze $J_{k,2}$ we bound $n^{\known} _{k,t}$, the number of samples of arm $k$ at each time $t \ge t^\ast$. Fix $t \ge t^\ast$. Then, $n^{\known} _{k,t}$ is a summation of independent Bernoulli r.v.'s. One has
\begin{align}\label{eq-expected-n_{k,t}-1}
	\mathbb{E}\left[n^{\known} _{k,t}\right]
	&=
	\frac{t^\ast-1}{K} + \sum\limits_{s=1}^t \frac{\sigma^2}{\hat{\sigma}^2} h_{k,s} + \sum\limits_{s=t^\ast}^{t-1} \frac{c\sigma^2}{\Delta^2 \tau_{k,s}}.
\end{align}
The term $\sum\limits_{s=t^\ast}^{t-1} \frac{c\sigma^2}{\Delta^2 \tau_{k,s}}$ can be bounded from below similar to Step 3. Let $\bar{M}_t = \sum\limits_{s=1}^{t-1} h_{k,s}$ and $\underline{M} = \sum\limits_{s=1}^{t^\ast} h_{k,s}$. For $\bunderline{M}\le m \le \bar{M}_t$, let $t_{m,t}$ be the time step at which the $m$th auxiliary observation for arm $k$ is received up to time $t-1$, that is,
\[
t_{m,t} = \begin{cases*}
	\max\left\{t^\ast,\min \left\{0\le s \le T \;\middle| \; \sum\limits_{s=1}^q h_{k,s} = m \right\} \right\} & if $\bunderline{M}\le m \le \bar{M}_t$ \\
	t-1 & if $m=\bar{M}_t+1$
\end{cases*}.
\]
One has that\vspace{-0.3cm}
\begin{align}\label{eq-expected-n_{k,t}-2}
	\sum\limits_{s=t^\ast}^{t-1} \frac{c\sigma^2}{\Delta^2 \tau_{k,s}}
	&=
	\sum\limits_{m=\bunderline{M}}^{\bar{M}_t} \sum\limits_{s=0}^{t_{m+1,t}-t_{m,t}-1}\frac{c\sigma^2}{\Delta^2 (\tau_{k,t_{m,t}}+s)}
	\nonumber \\
	&\overset{(a)}{\ge}
	\sum\limits_{m=\bunderline{M}}^{\bar{M}_t} \frac{c\sigma^2}{\Delta^2} \log
	\left(
	\frac{\tau_{k,t_{m,t}} + t_{m+1,t}-t_{m,t}}{\tau_{k,t_{m,t}}}
	\right)
	\nonumber \\
	&=
	\sum\limits_{m=\bunderline{M}}^{\bar{M}_t} \frac{c\sigma^2}{\Delta^2} \log
	\left(
	\frac{\exp\left(\frac{-\Delta^2}{c\hat{\sigma}^2}\right)\tau_{k,t_{m+1,t}}}{\tau_{k,t_{m,t}}}
	\right)
	\nonumber \\
	&=
	\frac{c\sigma^2}{\Delta^2} \log
	\left(
	\frac{\tau_{k,t-1}}{\tau_{k,t^\ast}}
	\right)
	- \sum\limits_{s=t^\ast + 1}^{t-1} \frac{\sigma^2}{\hat{\sigma}^2}  h_{k,s}
	\nonumber \\
	&=
	\frac{c\sigma^2}{\Delta^2} \log
	\left(
	\frac{\tau_{k,t-1}}{t^\ast}
	\right)
	- \sum\limits_{s=1}^{t-1} \frac{\sigma^2}{\hat{\sigma}^2} h_{k,s},
\end{align}
where (a) follows from Lemma \ref{lemma-sum-one-over-t}. Putting together (\ref{eq-expected-n_{k,t}-1}) and (\ref{eq-expected-n_{k,t}-2}), one obtains:
\begin{align}
	\label{eq-expected-n_{k,t}-3}
	\mathbb{E}\left[n^{\known} _{k,t}\right]
	&\ge
	\frac{c\sigma^2}{\Delta^2} \log
	\left(
	\frac{\tau_{k,t-1}}{t^\ast}
	\right)
	\coloneqq x_t.
\end{align}
By Bernstein's inequality in Lemma \ref{lemma-Bernstein-inequality}, we have\vspace{-0.2cm}
\begin{align}\label{eq-expected-n_{k,t}-4}
	\mathbb{P}\left\{n^{\known} _{k,t} \le \frac{x_t}{2}\right\} \le \exp\left(\frac{-x_t}{10}\right) =  \left(\frac{t^\ast}{\tau_{k,t-1}}\right)^{\frac{-c\sigma^2}{10\Delta^2}} \le \left(\frac{t^\ast}{t-1}\right)^{\frac{-c\sigma^2}{10\Delta^2}},
\end{align}
where the last inequality follows from (\ref{eq-exp-based-time-epsilon-greedy}).\vspace{-0.1cm}

\paragraph{Step 5 (Analysis of $J_{k,2}$).} We note that\vspace{-0.2cm}
\begin{equation}\label{eq-jk,2-analysis1-epsilon-greedy}
	\sum_{k \in \mathcal{K} \setminus{k^\ast }}J_{k,2}
	\le
	t^\ast - 1 +
	\sum_{k \in \mathcal{K} \setminus{k^\ast }} \sum\limits_{t=t^\ast}^T \mathbb{P} \left\{ \bar{X}_{k,t}^{\known} >  \bar{X}_{k^\ast, t} \right\}.\vspace{-0.1cm}
\end{equation}
We upper bound each summand as follows:\vspace{-0.2cm}
\begin{align}\label{eq-jk,2-analysis2-epsilon-greedy}
	\mathbb{P} \left\{ \bar{X}_{k,t}^{\known} >  \bar{X}_{k^\ast, n^{\known} _{k,t}} \right\}
	&\le
	\mathbb{P} \left\{ \bar{X}_{k,t}^{\known} > \mu_k + \frac{\Delta_k}{2} \right\}
	\nonumber \\
	&+
	\mathbb{P} \left\{ \bar{X}_{k^\ast,t}^{\known} < \mu_{k^\ast} + \frac{\Delta_k}{2} \right\}.
\end{align}
We next bound the first term on the right-hand side of the above inequality; the second term can be treated similarly. One obtains:\vspace{-0.2cm}
\begin{align}\label{eq-jk,2-analysis3-epsilon-greedy}
	\mathbb{P} \left\{ \bar{X}_{k,t}^{\known} > \mu_k + \frac{\Delta_k}{2} \right\}
	&\overset{\text{(a)}}{\le}
	\mathbb{P} \left\{ \bar{X}_{k,t}^{\known} > \mu_k + \frac{\Delta_k}{2}\;\middle| \;n^{\known} _{k,t} > \frac{x_t}{2} \right\} + \mathbb{P} \left\{  n^{\known} _{k,t} \le \frac{x_t}{2}\right\}
	\nonumber\\
	&\overset{\text{(b)}}{\le}
	\exp\left(
	{\frac{-\Delta_k^2 x_t}{16\sigma^2}}
	\right)
	+ \left(\frac{t^\ast}{t-1}\right)^{\frac{-c\sigma^2}{10\Delta^2}}
	\nonumber \\
	&\overset{\text{}}{\le}
	\exp\left(
	{\frac{-c}{16}\log \left(
		\frac{\tau_{k,t-1}}{t^\ast}
		\right)}
	\right)
	+ \left(\frac{t^\ast}{t-1}\right)^{\frac{-c\sigma^2}{10\Delta^2}}
	\nonumber\\
	&\overset{(c)}{\le}
	\left(\frac{t^\ast}{t-1}\right)^{\frac{-c}{16}}+
	\left(\frac{t^\ast}{t-1}\right)^{\frac{-c\sigma^2}{10\Delta^2}},
\end{align}
where, (a) holds by the conditional probability definition, (b) follows from Lemma \ref{lemma-Chernoff-Hoeffding-bound} together with (\ref{eq-expected-n_{k,t}-4}), and (c) holds by (\ref{eq-exp-based-time-epsilon-greedy}). Putting together (\ref{eq-regret-decomposition-epsilon-greedy}), (\ref{eq-Jk,1-analysis-epsilon-greedy}), (\ref{eq-jk,2-analysis1-epsilon-greedy}), (\ref{eq-jk,2-analysis2-epsilon-greedy}), and (\ref{eq-jk,2-analysis3-epsilon-greedy}), the result is established. \qed

\vspace{-0.1cm}
\section{Numerical Analysis}\label{section-numerics}\vspace{-0.15cm}

\subsection{Comparing Performance of Various Policies}\label{subsec:numcomparison}
\vspace{-0.1cm}

\noindent\textbf{Setup.} \Copy{numerics-setup}{We simulate the performance of different policies using MAB instances with three arms. The reported results correspond to rewards that are Gaussian with means $\boldsymbol{\mu} = \begin{pmatrix}
		0.7 & 0.5 & 0.5
	\end{pmatrix}^\top$ and standard deviation $\sigma = 0.5$, but we note that results are very robust with respect to the reward distributions. For simplicity we drew auxiliary information from the same distributions as the rewards. \Copy{information-arrival-processes}{We considered two information arrival processes: stationary information arrival process with $h_{k,t}$'s that are i.i.d. Bernoulli random variables with mean $\lambda$, where we considered values of $\lambda \in \{500/T, 100/T, 10/T\}$; and diminishing information arrival process where $h_{k,t}$'s are i.i.d. Bernoulli random variables with mean $\frac{\kappa^{\mathrm{aux}}}{t}$, where we considered values of $\kappa^{\mathrm{aux}}\in \{4,2,1\}$ for each time period $t$ and for each arm $k$}.
	
	In this part we experimented with three policies: UCB1, Thompson sampling, and $\epsilon_t$-greedy. We considered several variants of each policy: $(i)$ versions that ignores auxiliary information arrivals (EG, UCB1, TS); $(ii)$ version that updates empirical means and cumulative number of observations upon information arrivals but uses a standard time index (nEG); and $(iii)$ version that updates empirical means and cumulative number of observations upon information arrivals and uses virtual time indices (aEG, aUCB, aTS1). The difference between these versions is essentially in manner in which empirical means and reward observation counters are computed, as well as in the type of time index that is used. Namely, for the original versions:\vspace{-0.2cm}
	\footnotesize
	\[
	\bar{X}_{k,t}^{\known} \coloneqq \frac{\sum\limits_{s=1}^{t-1}\frac{1}{\sigma^2}\mathbbm{1}\{\pi_s = k\}X_{k,s}}{\sum\limits_{s=1}^{t-1} \frac{1}{\sigma^2} \mathbbm{1}\{\pi_s = k\}}, \quad
	n^{\known} _{k,t} \coloneqq \sum\limits_{s=1}^{t-1}  \mathbbm{1}\{\pi_s = k\}, \quad
	\tau_{k,t} = \tau_{k,t-1} + 1,
	\]
	\normalsize for the versions that account for auxiliary information while using standard time index:\footnotesize\vspace{-0.1cm}
	\[
	\bar{X}_{k,t}^{\known} \coloneqq \frac{\sum\limits_{s=1}^{t-1}\frac{1}{\sigma^2}\mathbbm{1}\{\pi_s = k\}X_{k,s} +  \sum\limits_{s=1}^{t} \frac{1}{\hat{\sigma}^2}Z_{k,s}}{\sum\limits_{s=1}^{t-1} \frac{1}{\sigma^2} \mathbbm{1}\{\pi_s = k\}  + \sum\limits_{s=1}^{t} \frac{1}{\hat{\sigma}^2}h_{k,s} }, \quad
	n^{\known} _{k,t} \coloneqq \sum\limits_{s=1}^{t-1}  \mathbbm{1}\{\pi_s = k\}  + \sum\limits_{s=1}^{t} \frac{\sigma^2}{\hat{\sigma}^2}h_{k,s}, \quad
	\tau_{k,t} = \tau_{k,t-1} + 1,\vspace{-0.1cm}
	\]
	\normalsize and for the versions with virtual time indices:\footnotesize\vspace{-0.1cm}
	\[
	\bar{X}_{k,t}^{\known} \coloneqq \frac{\sum\limits_{s=1}^{t-1}\frac{1}{\sigma^2}\mathbbm{1}\{\pi_s = k\}X_{k,s} +  \sum\limits_{s=1}^{t} \frac{1}{\hat{\sigma}^2}Z_{k,s}}{\sum\limits_{s=1}^{t-1} \frac{1}{\sigma^2} \mathbbm{1}\{\pi_s = k\}  + \sum\limits_{s=1}^{t} \frac{1}{\hat{\sigma}^2}h_{k,s} }, \quad
	n^{\known} _{k,t} \coloneqq \sum\limits_{s=1}^{t-1}  \mathbbm{1}\{\pi_s = k\}  + \sum\limits_{s=1}^{t} \frac{\sigma^2}{\hat{\sigma}^2}h_{k,s}, \quad
	\tau_{k,t} = \left(\tau_{k,t-1} + 1\right)\cdot \exp\left(\frac{h_{k,t}\Delta^2}{c_\pi\hat{\sigma}^2}\right),\vspace{-0.1cm}
	\]
	\normalsize where $c_\pi$ depends on the policy type $\pi$.
	We measured the average empirical regret $\sum\limits_{t=1}^{T}(\mu^\ast - \mu_{\pi_t})$, that is, the average difference between the performance of the best arm and the performance of the policy, over a decision horizon of $T=10^4$ periods, and averaging over 400 replications.} \Copy{determine-tuning-parameter} {To determine the tuning parameter $c$ used by the policies we simulated their original versions that do not account for auxiliary information arrivals with $c = 0.1\times m$, where we considered values of $m \in \{1,2,\dots,20\}$. The results presented below include the tuning parameter that resulted in the lowest cumulative empirical regret over the entire time horizon: $c=1.0$ for $\epsilon_t$-greedy, $c=1.0$ for UCB1, and $c=0.5$ for Thompson sampling (robustness with respect to this tuning parameter is discussed in \ref{section-misspecification-tuning-parameter}).}

\medskip
\noindent\textbf{Results and Discussion.} Plots comparing the empirical regret of the different algorithms for stationary and diminishing information arrival processes appear in Figures \ref{fig-main-stationary} and \ref{fig-main-diminishing}, respectively. For a stationary information arrival process with large $\lambda$, the regret is bounded uniformly by a constant and eventually stops growing with $T$, and when $\lambda$ is smaller we transition back to logarithmic regret, as discussed in \S\ref{example-stochastic-lower bound}. The impact of the timing of information arrivals on the achievable performance can be viewed by comparing the plots that correspond to the stationary and the diminishing information arrival processes. The expected total number of information arrivals in the first, second, and third columns with diminishing information arrival process is less than or equal to that of the first, second, and third column with stationary information arrival process, and the cumulative regrets for both cases are close to one another.

\begin{figure} [H]
	\centering
	\begin{subfigure}[t]{0.23\textwidth}
		\raisebox{-\height}{\includegraphics[width=\textwidth]{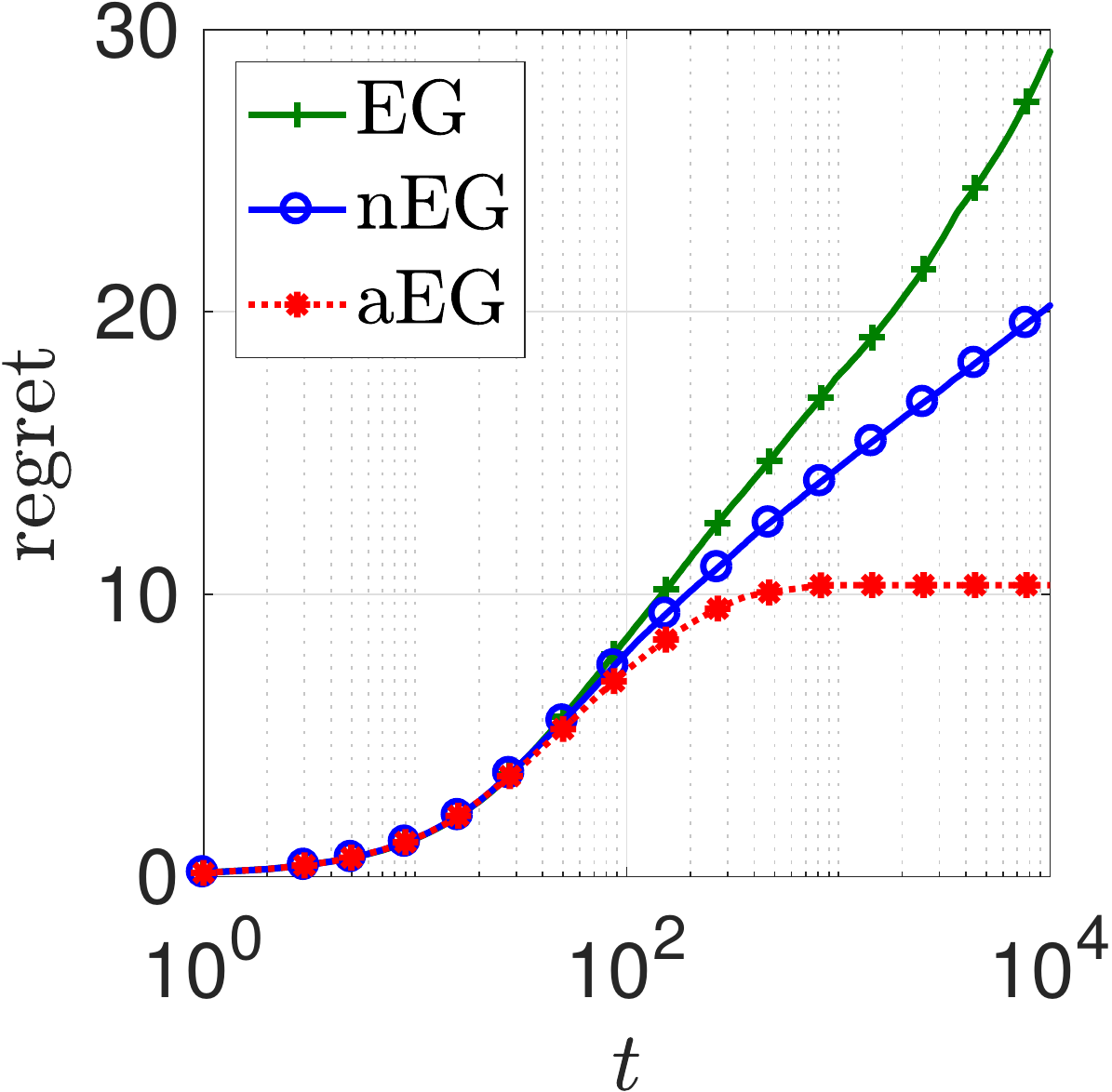}}
	\end{subfigure}
	\hfill
	\begin{subfigure}[t]{0.23\textwidth}
		\raisebox{-\height}{\includegraphics[width=\textwidth]{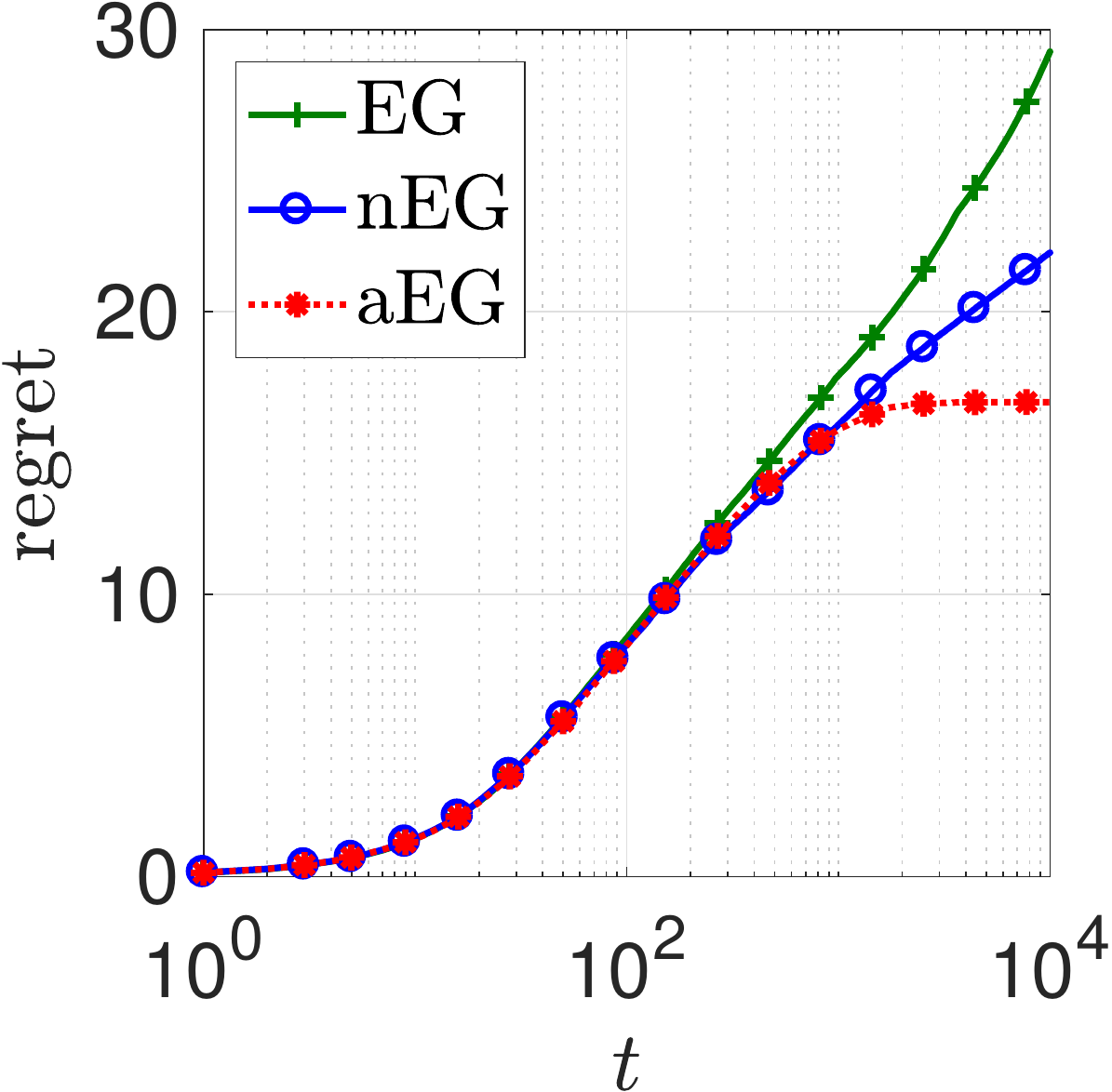}}
	\end{subfigure}
	\hfill
	\begin{subfigure}[t]{0.23\textwidth}
		\raisebox{-\height}{\includegraphics[width=\textwidth]{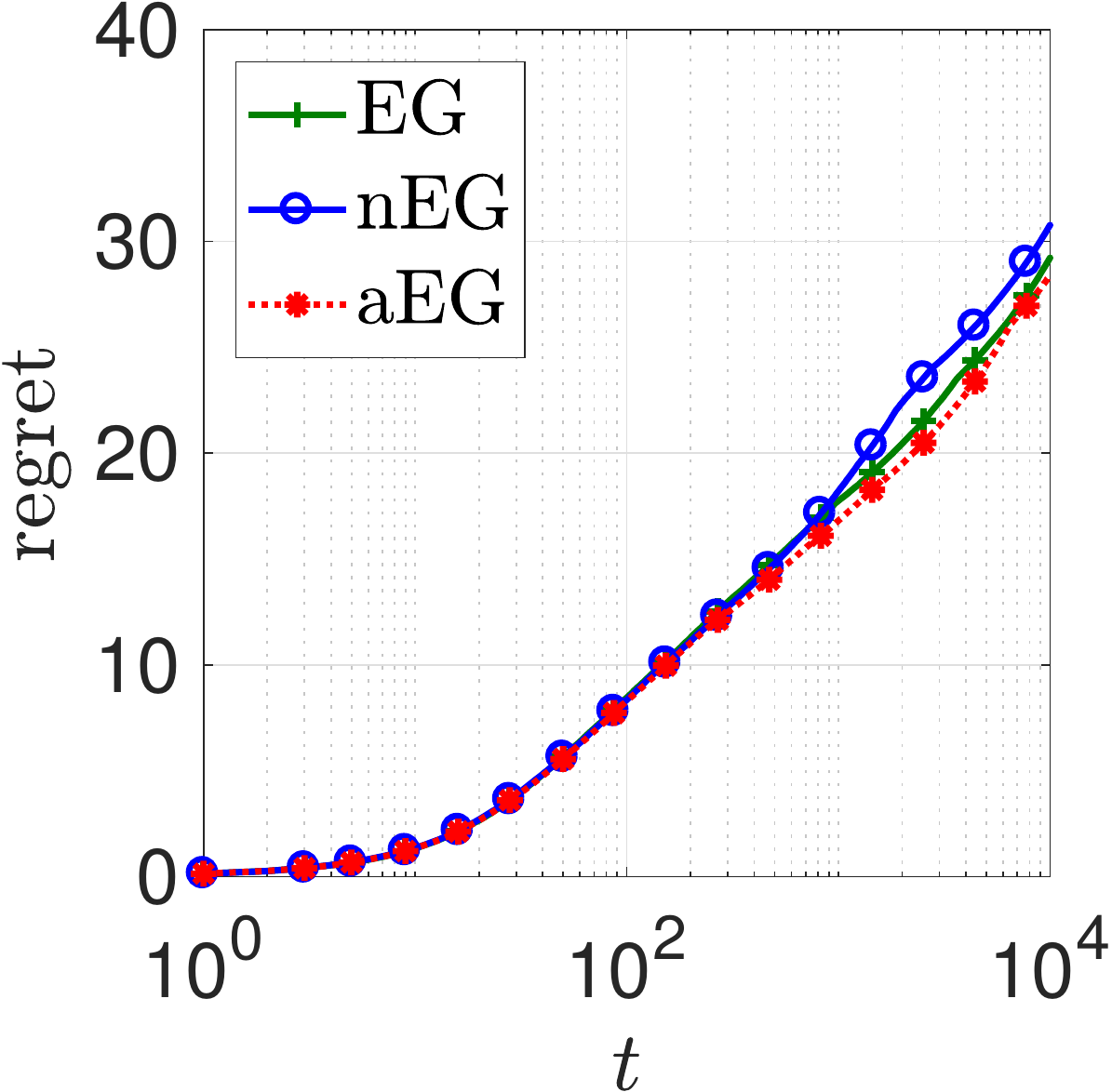}}
	\end{subfigure}
	
	\begin{subfigure}[t]{0.23\textwidth}
		\raisebox{-\height}{\includegraphics[width=\textwidth]{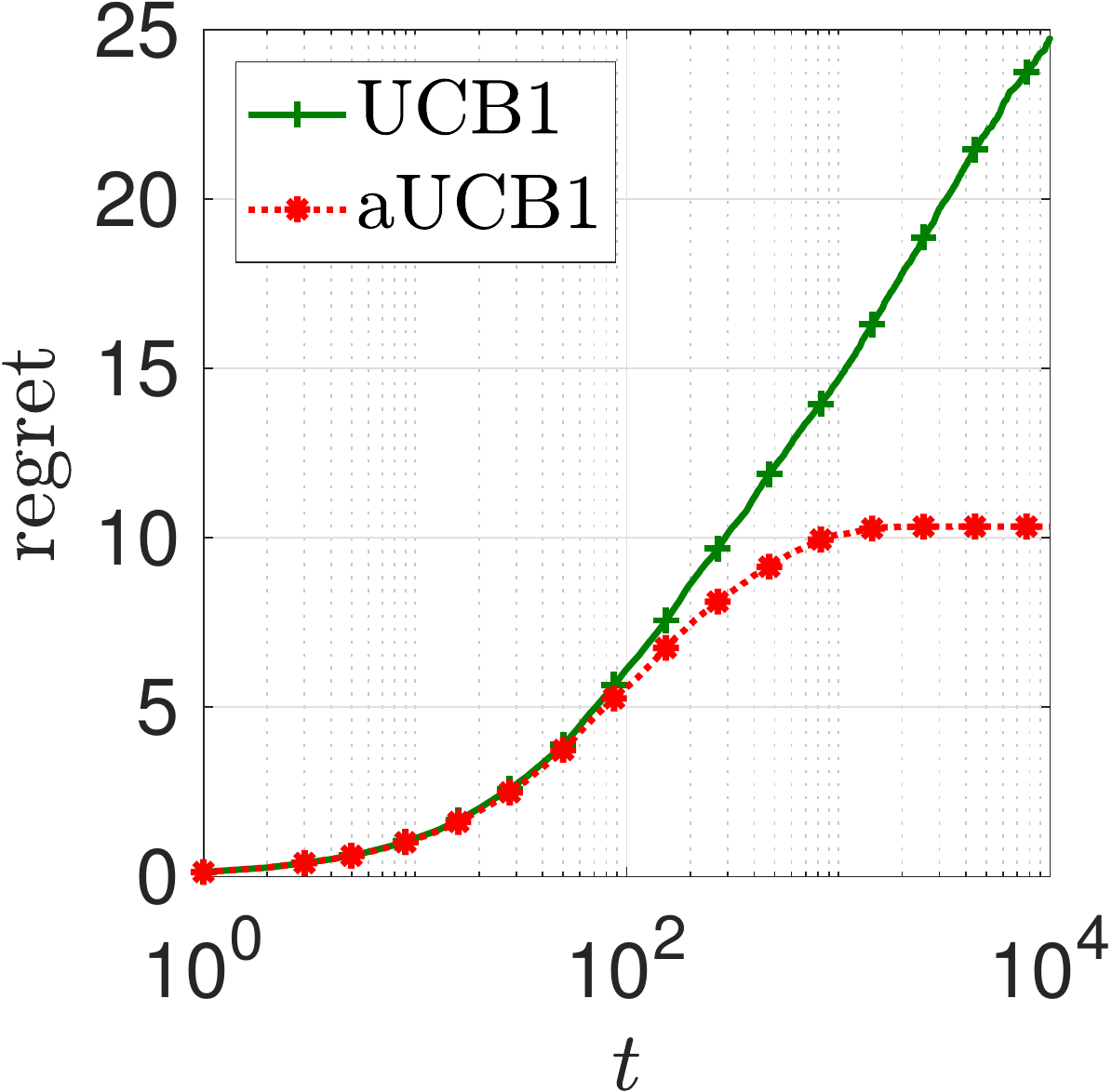}}
	\end{subfigure}
	\hfill
	\begin{subfigure}[t]{0.23\textwidth}
		\raisebox{-\height}{\includegraphics[width=\textwidth]{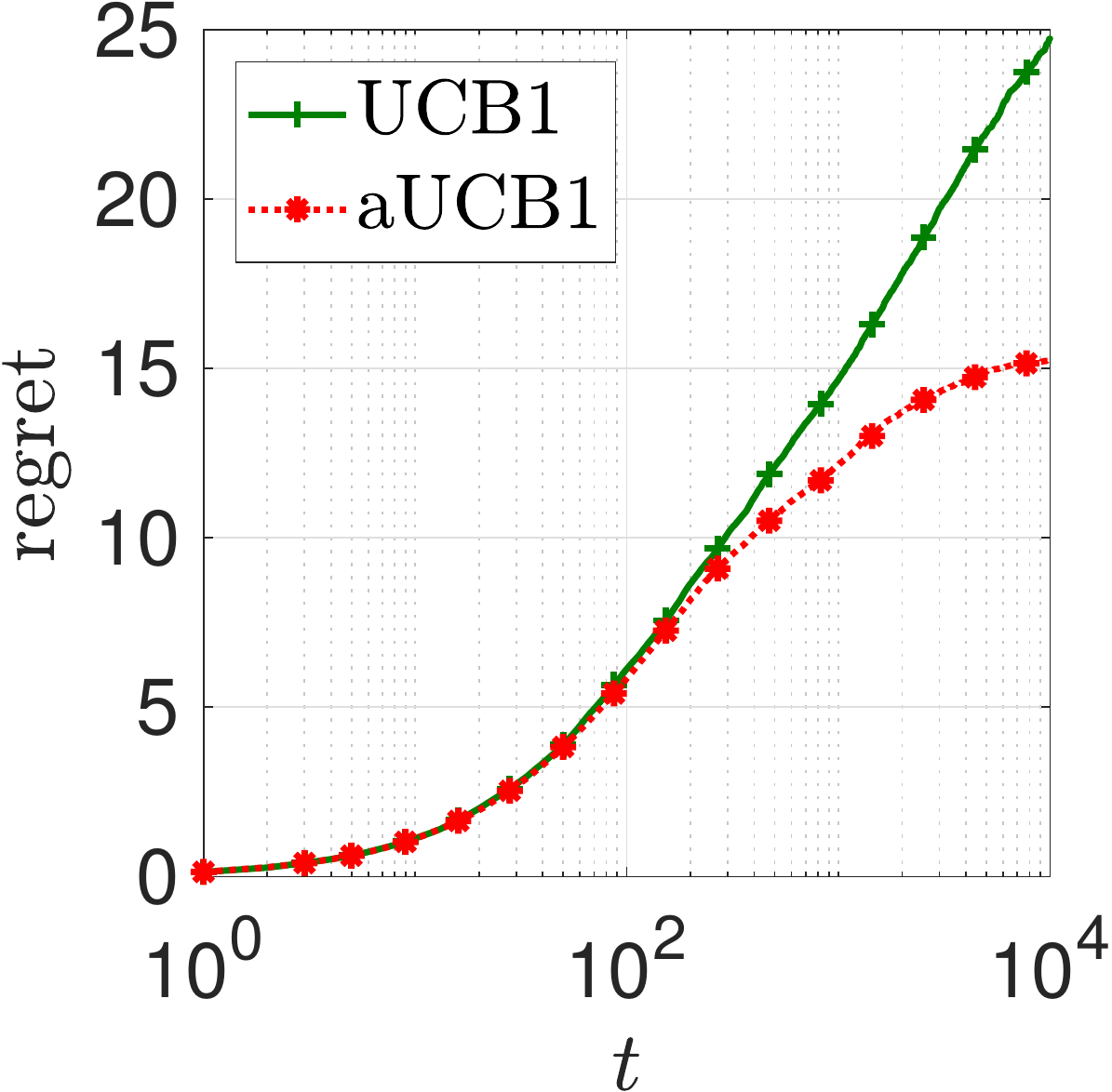}}
	\end{subfigure}
	\hfill
	\begin{subfigure}[t]{0.23\textwidth}
		\raisebox{-\height}{\includegraphics[width=\textwidth]{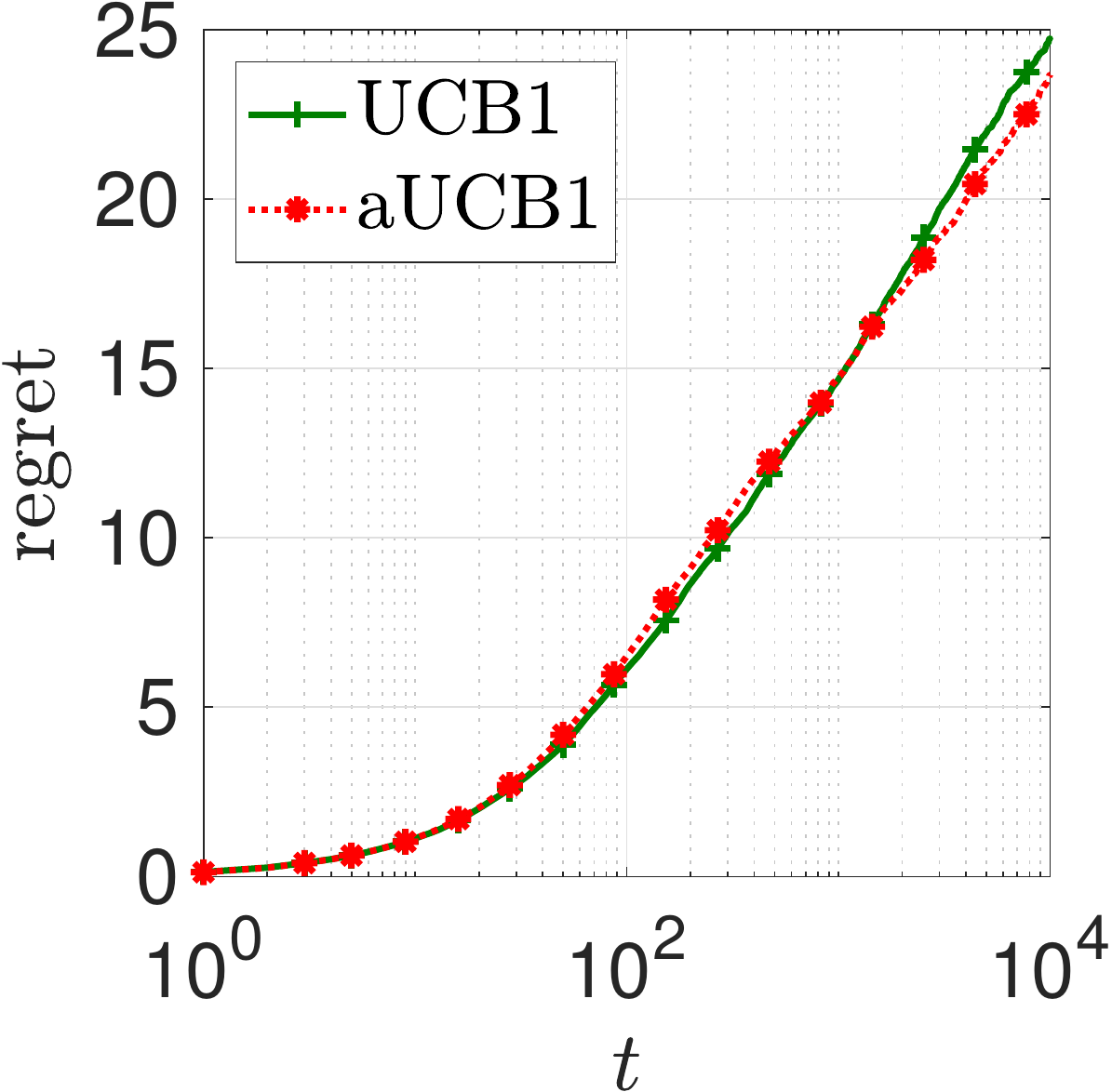}}
	\end{subfigure}
	
	\begin{subfigure}[t]{0.23\textwidth}
		\raisebox{-\height}{\includegraphics[width=\textwidth]{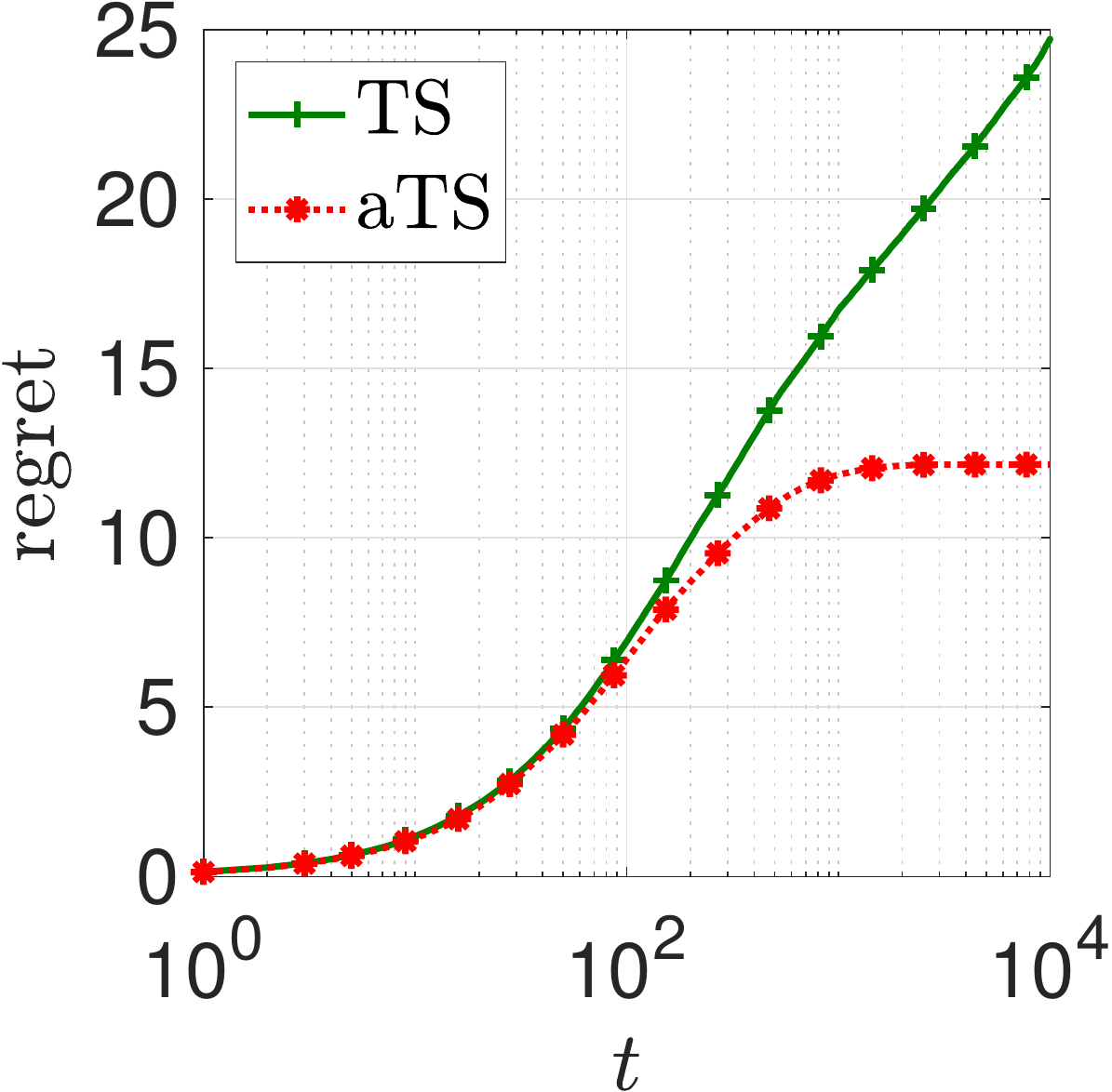}}
	\end{subfigure}
	\hfill
	\begin{subfigure}[t]{0.23\textwidth}
		\raisebox{-\height}{\includegraphics[width=\textwidth]{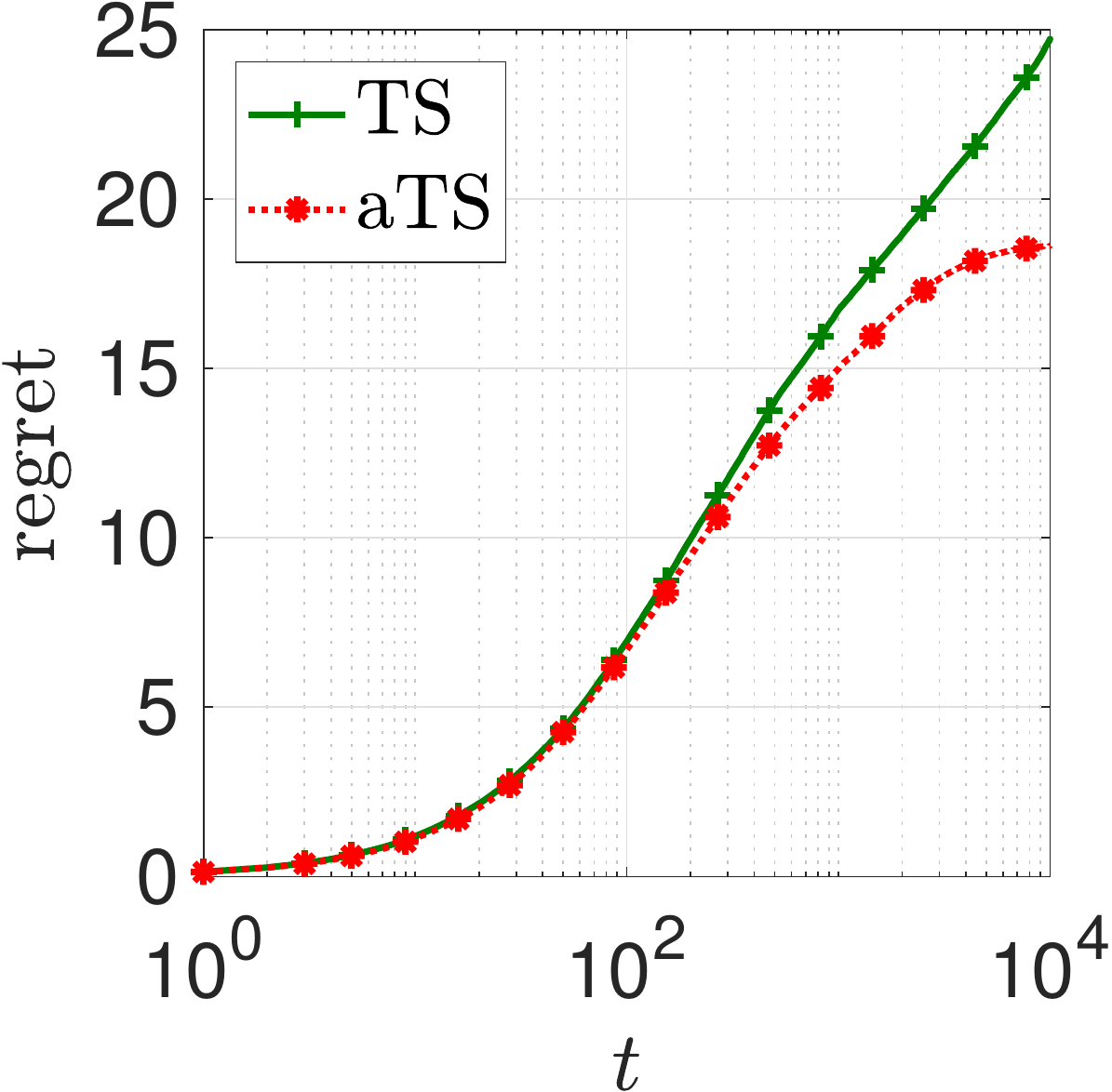}}
	\end{subfigure}
	\hfill
	\begin{subfigure}[t]{0.23\textwidth}
		\raisebox{-\height}{\includegraphics[width=\textwidth]{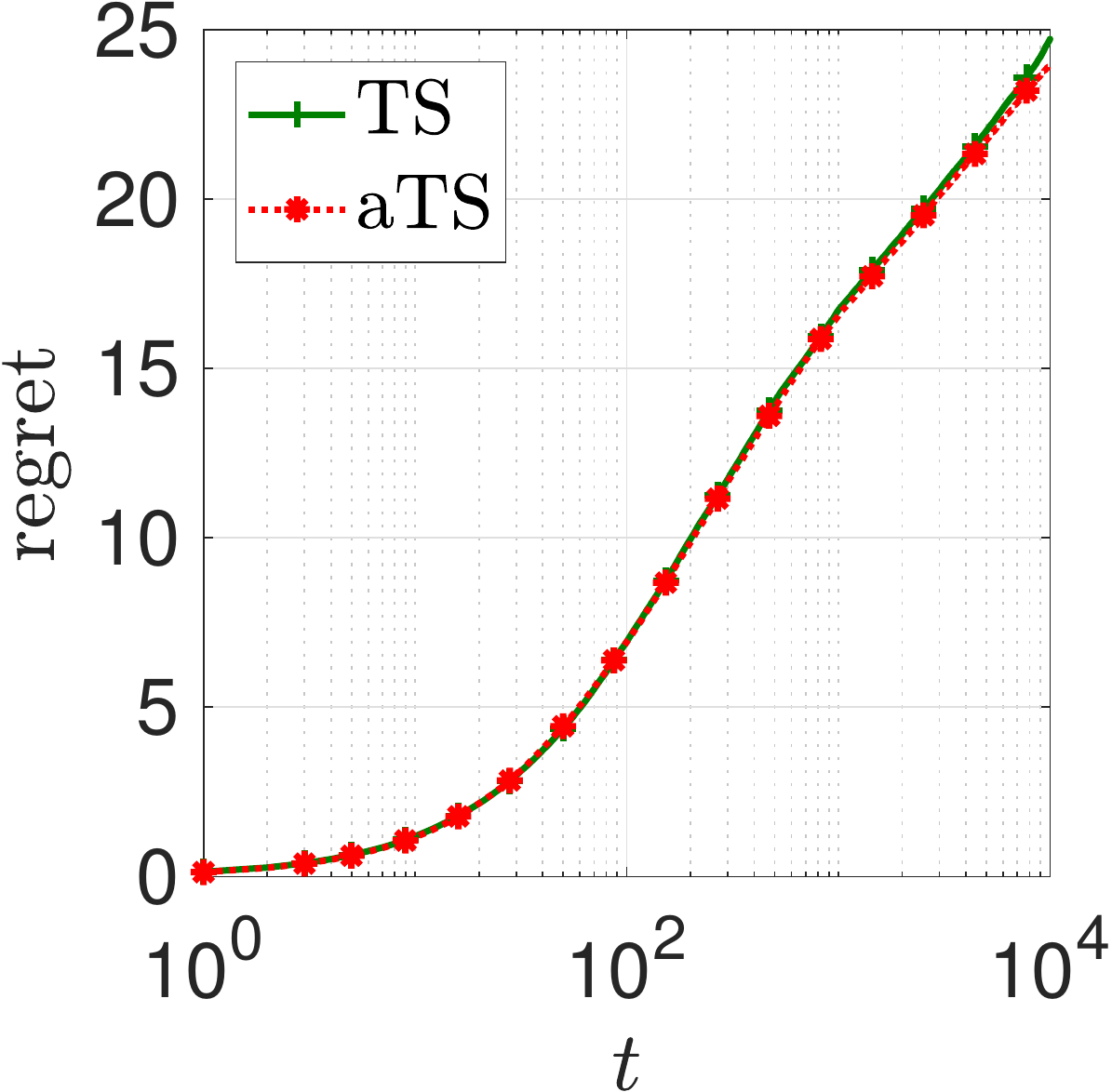}}
	\end{subfigure}
	
	\caption{\small Cumulative regret trajectory of the policies studied in the paper under a stationary information arrival process at a constant rate $\lambda$. \textit{Left column}: $\lambda = \frac{500}{T}$; \textit{Middle column}: $\lambda = \frac{100}{T}$; \textit{Right column}: $\lambda = \frac{10}{T}$.}\label{fig-main-stationary}
\end{figure}

\begin{figure}[H]
	\centering
	\begin{subfigure}[t]{0.23\textwidth}
		\raisebox{-\height}{\includegraphics[width=\textwidth]{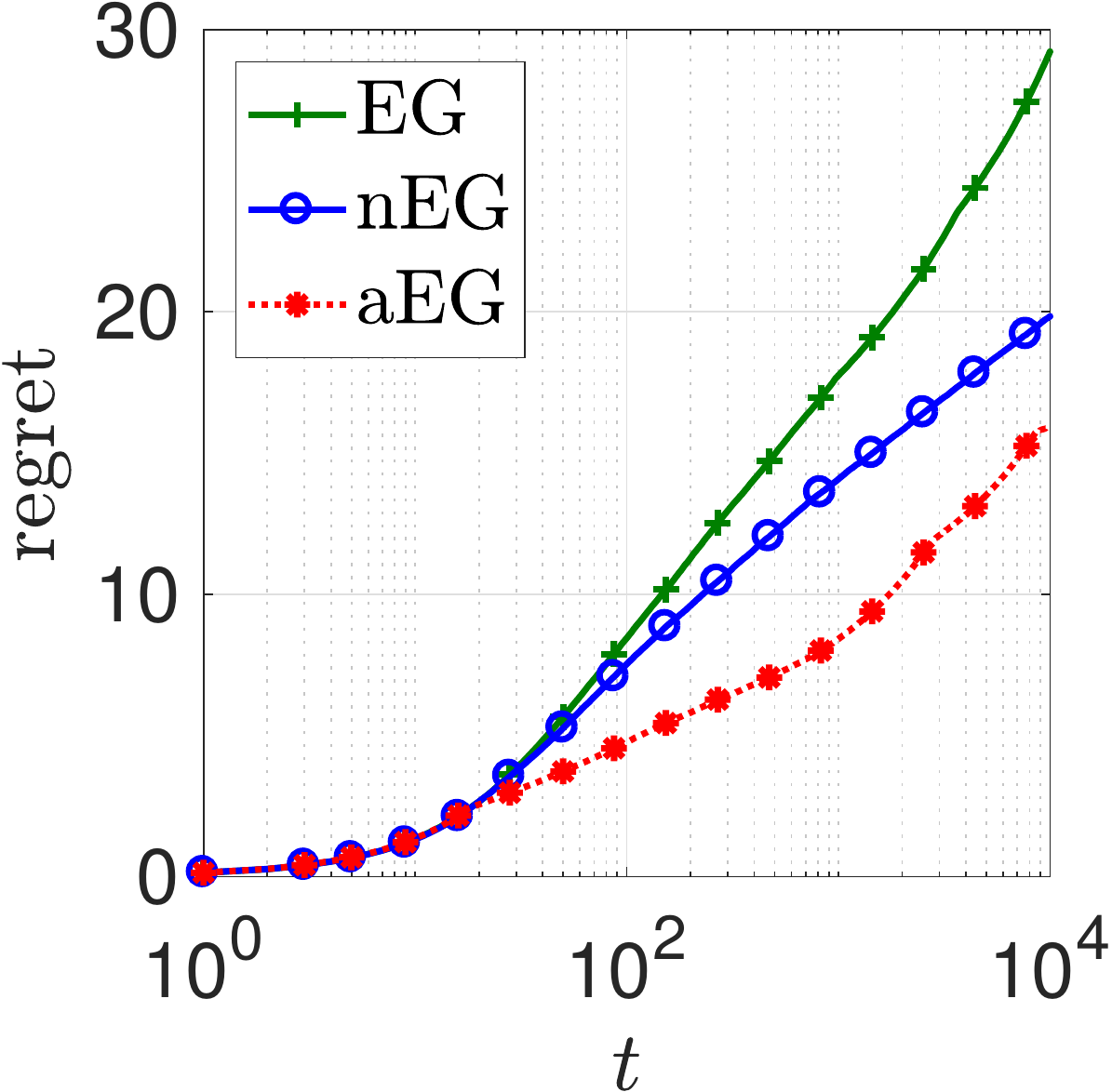}}
	\end{subfigure}
	\hfill
	\begin{subfigure}[t]{0.23\textwidth}
		\raisebox{-\height}{\includegraphics[width=\textwidth]{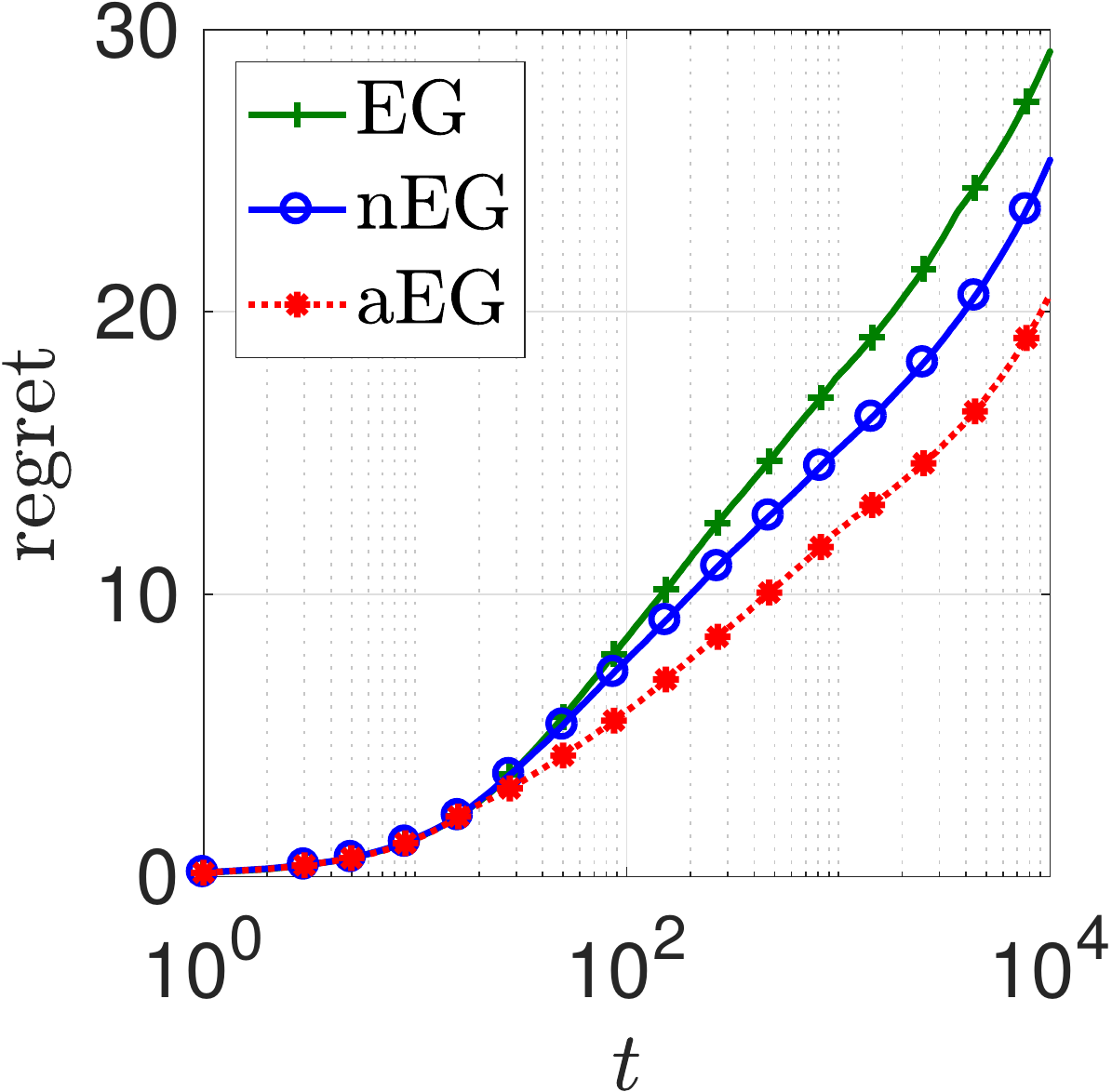}}
	\end{subfigure}
	\hfill
	\begin{subfigure}[t]{0.23\textwidth}
		\raisebox{-\height}{\includegraphics[width=\textwidth]{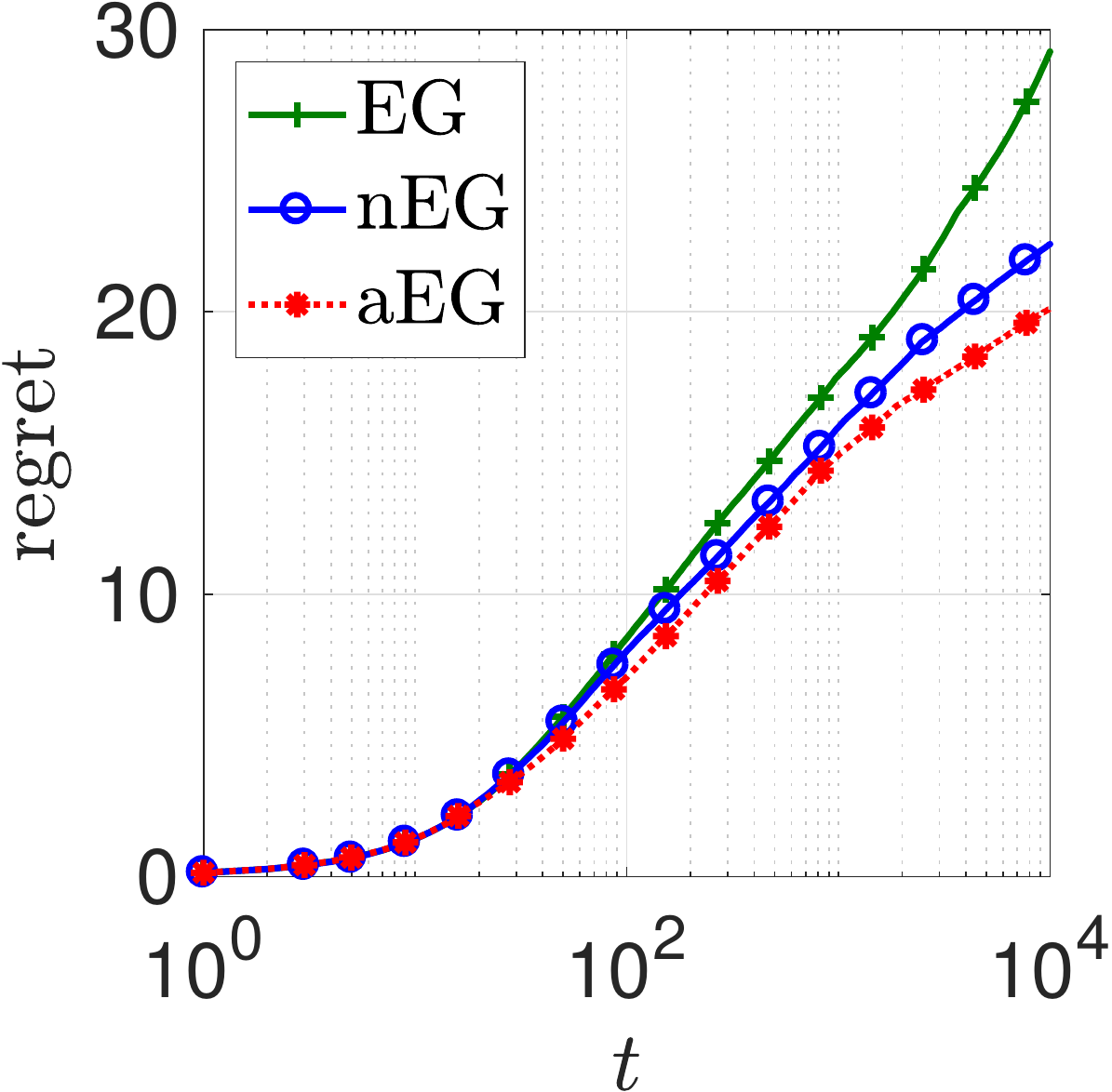}}
	\end{subfigure}
	
	\begin{subfigure}[t]{0.23\textwidth}
		\raisebox{-\height}{\includegraphics[width=\textwidth]{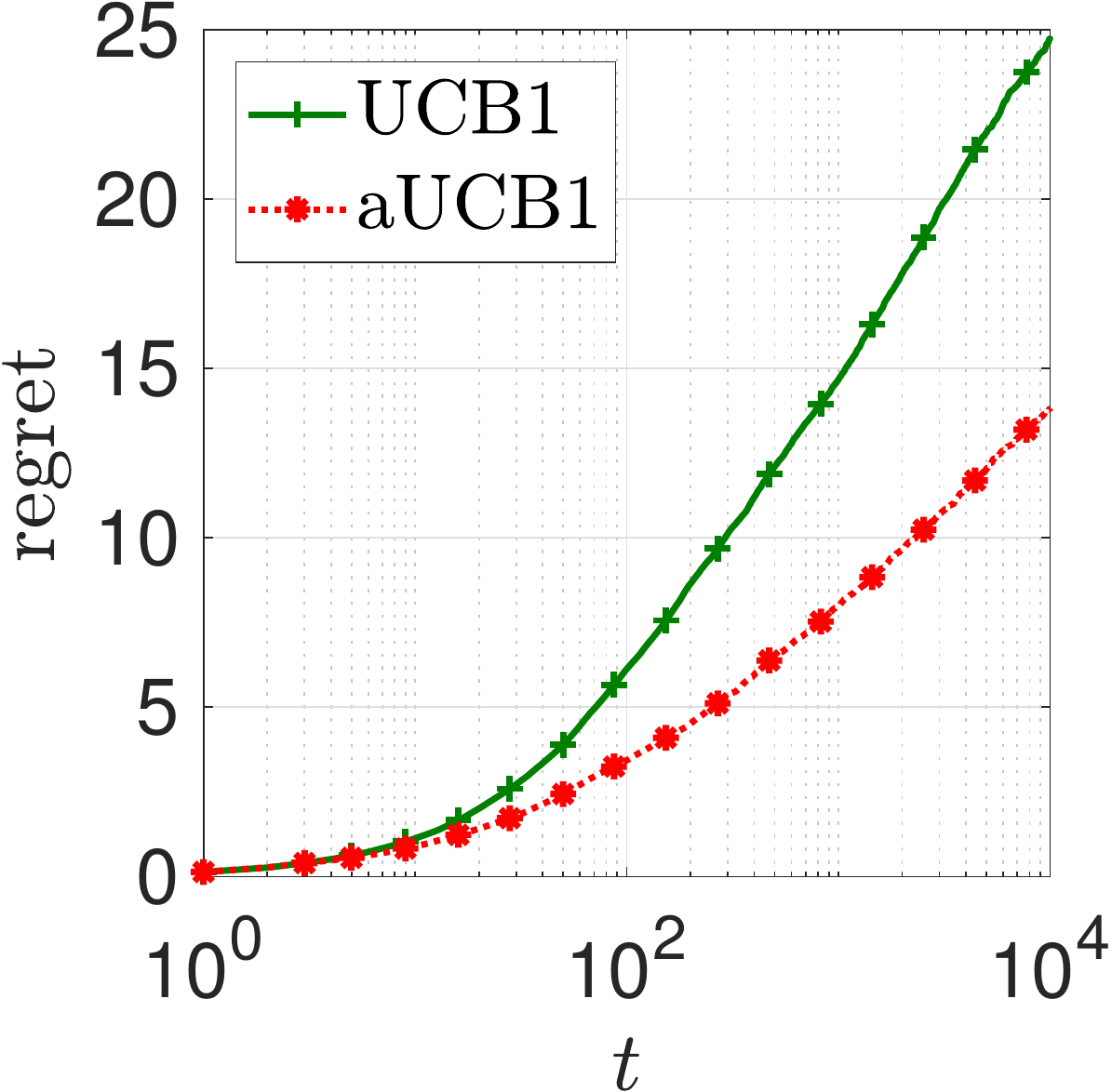}}
	\end{subfigure}
	\hfill
	\begin{subfigure}[t]{0.23\textwidth}
		\raisebox{-\height}{\includegraphics[width=\textwidth]{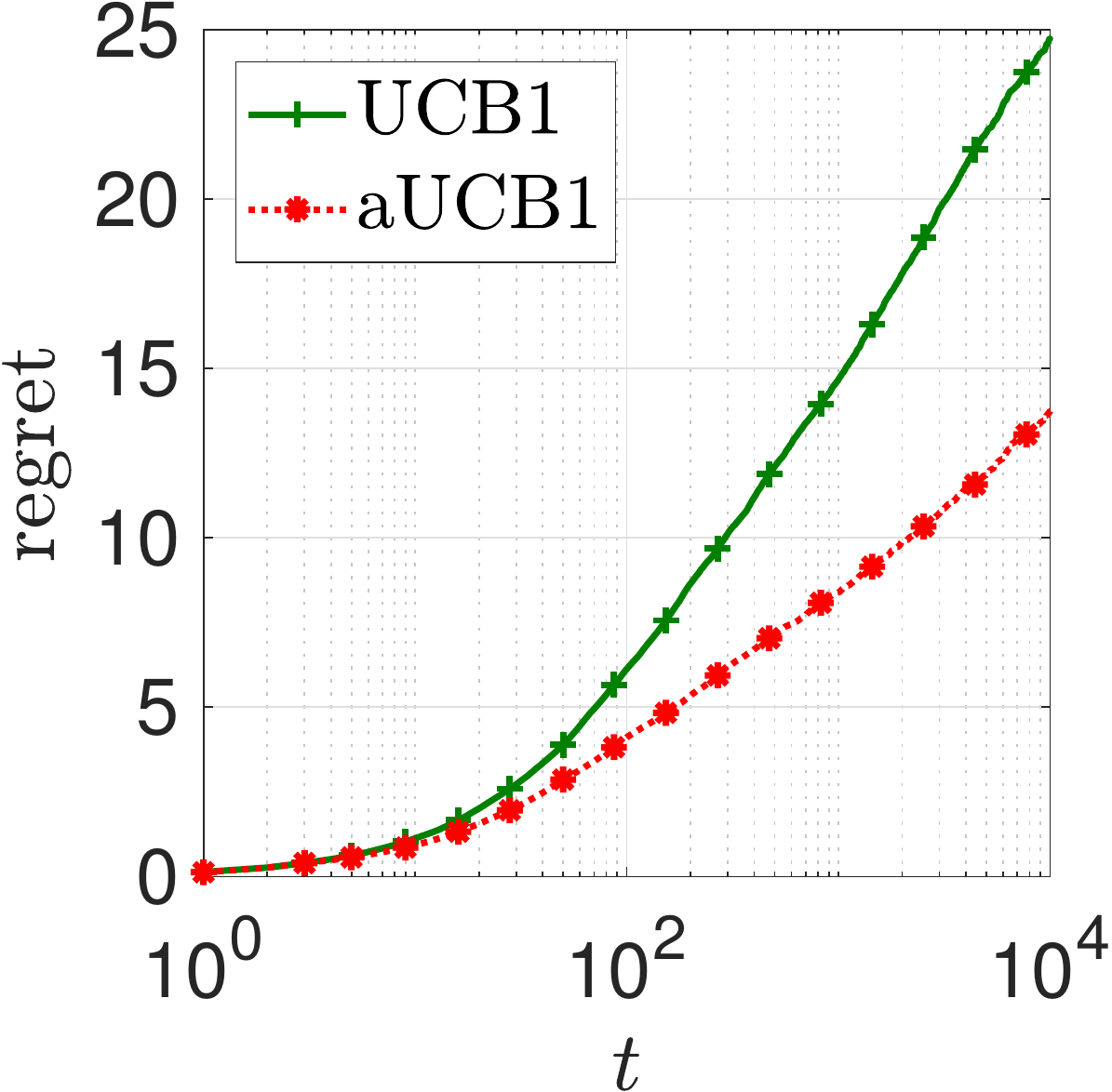}}
	\end{subfigure}
	\hfill
	\begin{subfigure}[t]{0.23\textwidth}
		\raisebox{-\height}{\includegraphics[width=\textwidth]{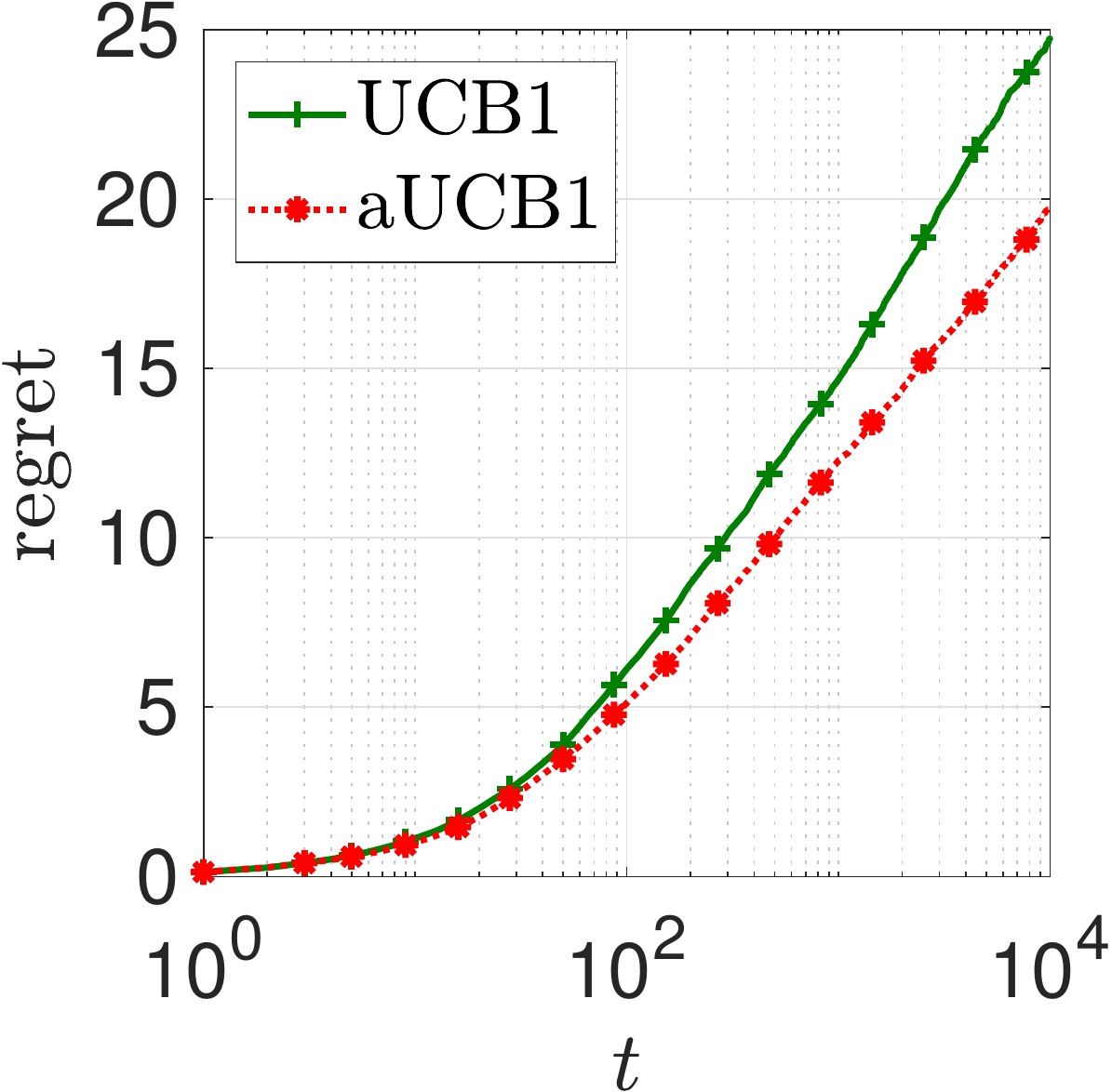}}
	\end{subfigure}
	
	\begin{subfigure}[t]{0.23\textwidth}
		\raisebox{-\height}{\includegraphics[width=\textwidth]{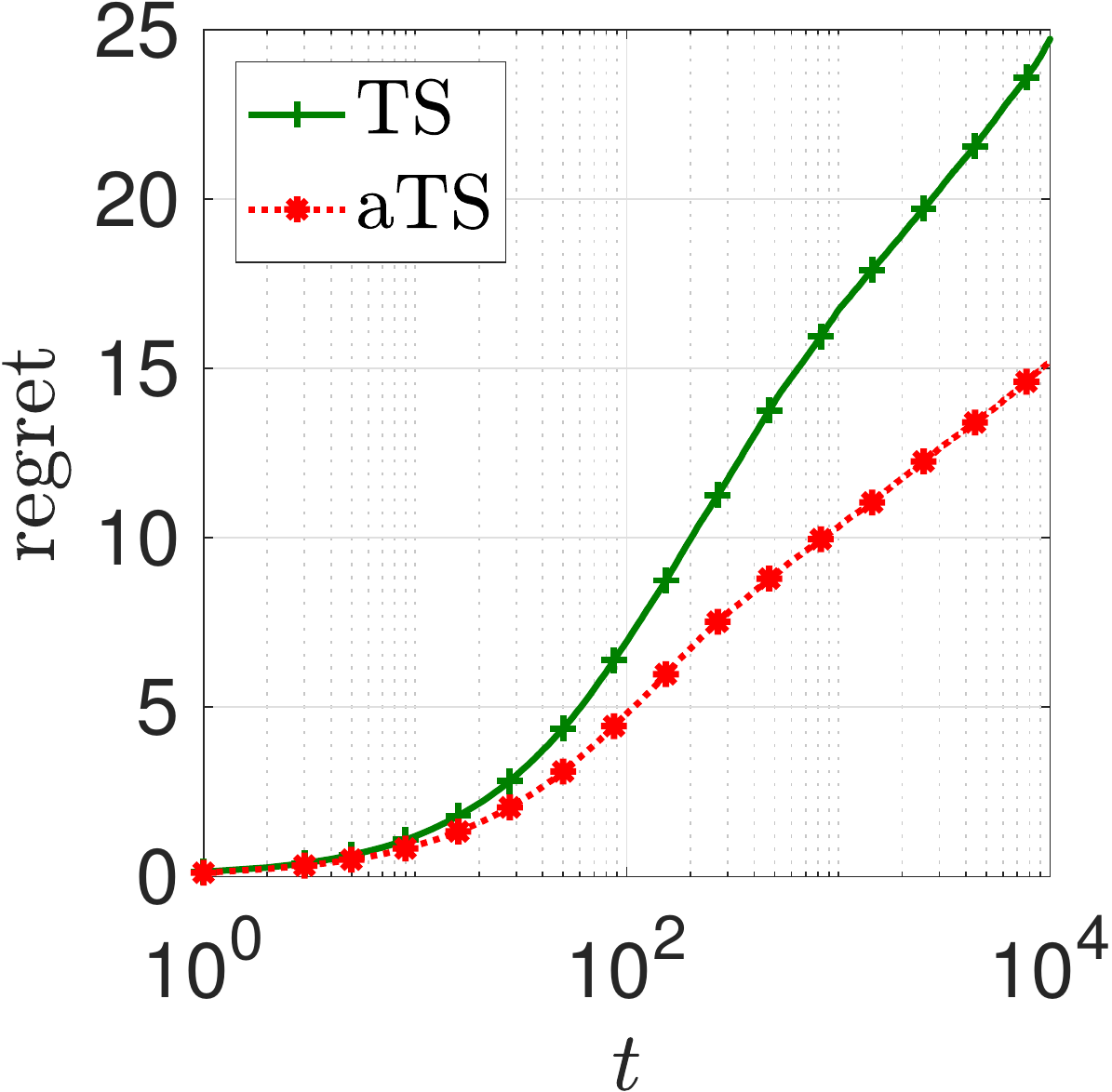}}
	\end{subfigure}
	\hfill
	\begin{subfigure}[t]{0.23\textwidth}
		\raisebox{-\height}{\includegraphics[width=\textwidth]{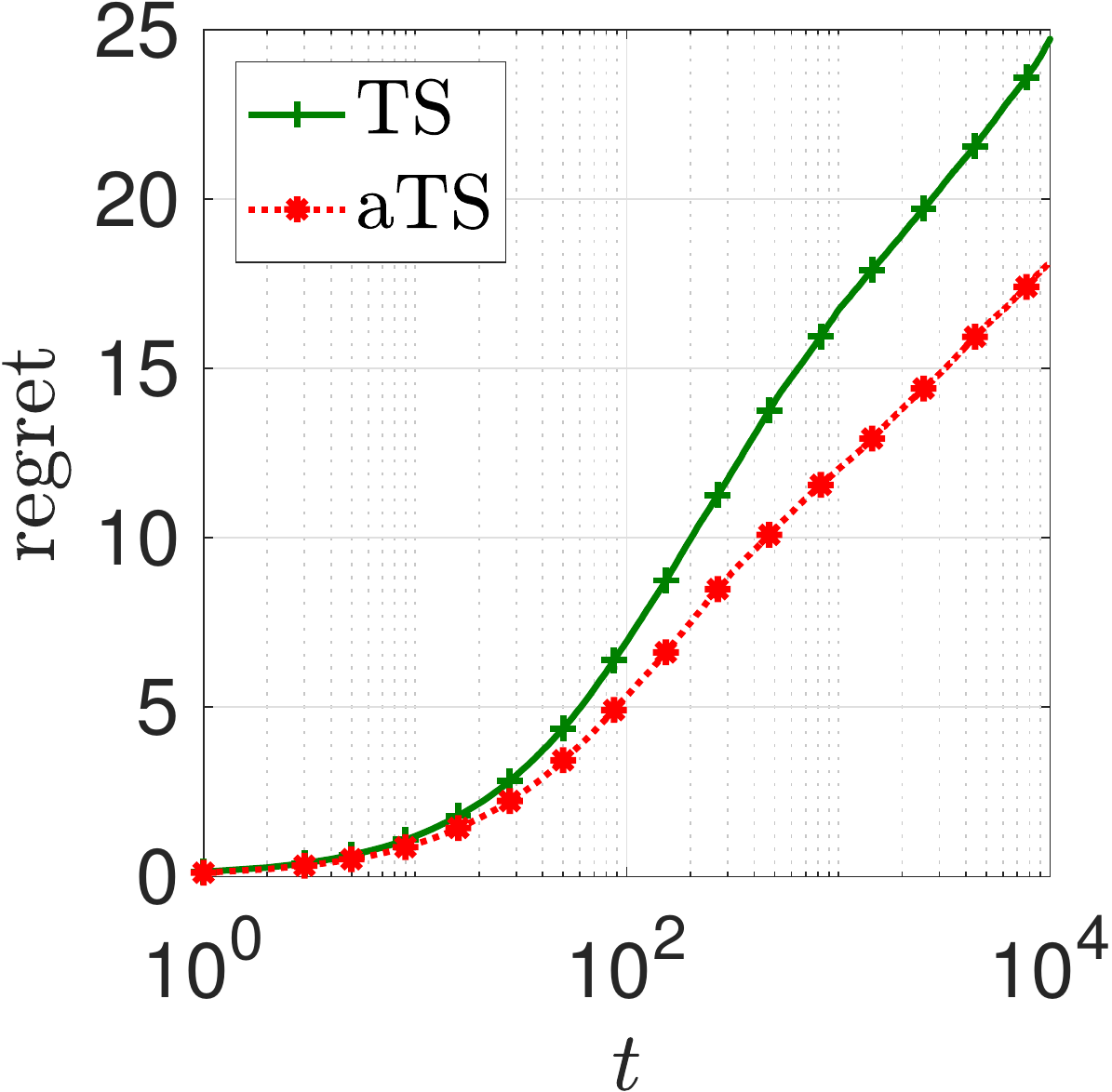}}
	\end{subfigure}
	\hfill
	\begin{subfigure}[t]{0.23\textwidth}
		\raisebox{-\height}{\includegraphics[width=\textwidth]{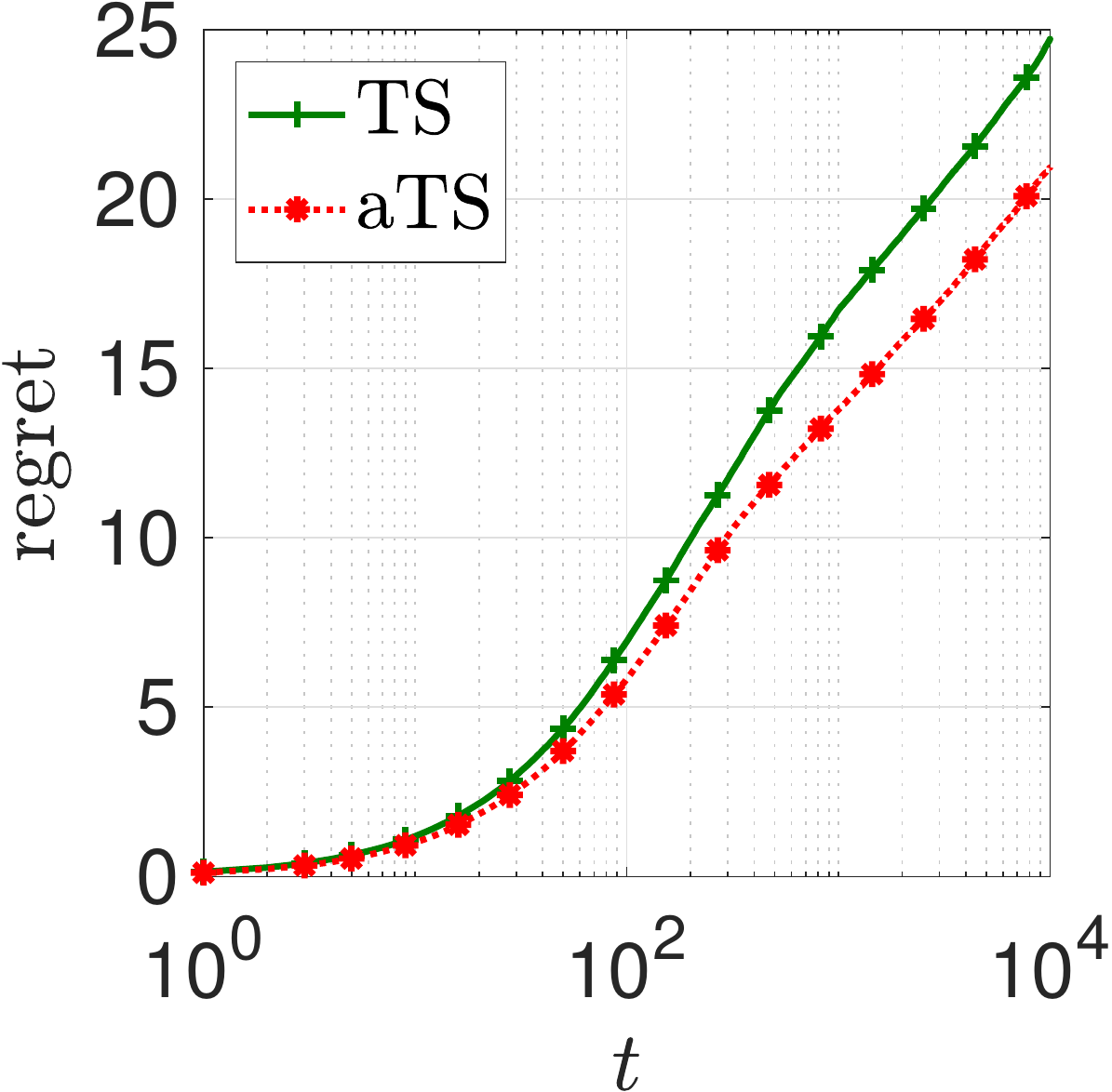}}
	\end{subfigure}
	\vspace{-0.2cm}
	\caption{\small Cumulative regret trajectory under a diminishing information arrival process with rate $\frac{\kappa^{\mathrm{aux}}}{t}$ at each time step $t$. \textit{Left column}: $\kappa^{\mathrm{aux}} = 4$; \textit{Middle column}: $\kappa^{\mathrm{aux}} = 2$; \textit{Right column}: $\kappa^{\mathrm{aux}} = 1$.}\label{fig-main-diminishing}\vspace{-0.3cm}
\end{figure}

\vspace{-0.2cm}
\subsection{Different Values of the Tuning Parameter c}\label{section-misspecification-tuning-parameter}\vspace{-0.1cm}

\noindent \textbf{Setup.} We use the setup discussed in \S\ref{subsec:numcomparison} to study the effect of using different values for the tuning parameter $c$. 
The plots here detail results for parametric values of $c=\{0.4, 1.0, 1.6\}$ for $\epsilon_t$-greedy, $c\in\{0.4, 1.0, 1.6\}$ for UCB1, and $c=\{0.1, 0.5, 0.7\}$ for Thompson sampling, and using four different information arrival instances: stationary with $\lambda \in \{500/T, 10/ T\}$ and diminishing with $\kappa^{\mathrm{aux}} \in \{4.0, 1.0\}$.

\medskip
\noindent \textbf{Results and Discussion.} Plots comparing the empirical regret accumulation of policies with different tuning parameters $c$ for the stationary and diminishing information arrival processes appear in Figure \ref{fig-miss-c}. \Copy{misspecification-tuning-parameter}{For $\epsilon_t$-greedy and Thompson sampling, one may observe that when there is a lot of auxiliary information (the first and third columns) lower values of $c$ lead to better performance. However, when there is less auxiliary information (the second and fourth columns) very small values of $c$ may result in high regret. The reason for this is that when $c$ is small these policies tend to explore less, hence, in the presence of a lot of auxiliary information they incur little regret but if auxiliary information is rare these policies under-explores and incur a lot of regret. We observe a similar effect for UCB1 as well.}

\begin{figure} [H]
	\centering
	\begin{subfigure}[t]{0.23\textwidth}
		\raisebox{-\height}{\includegraphics[width=\textwidth]{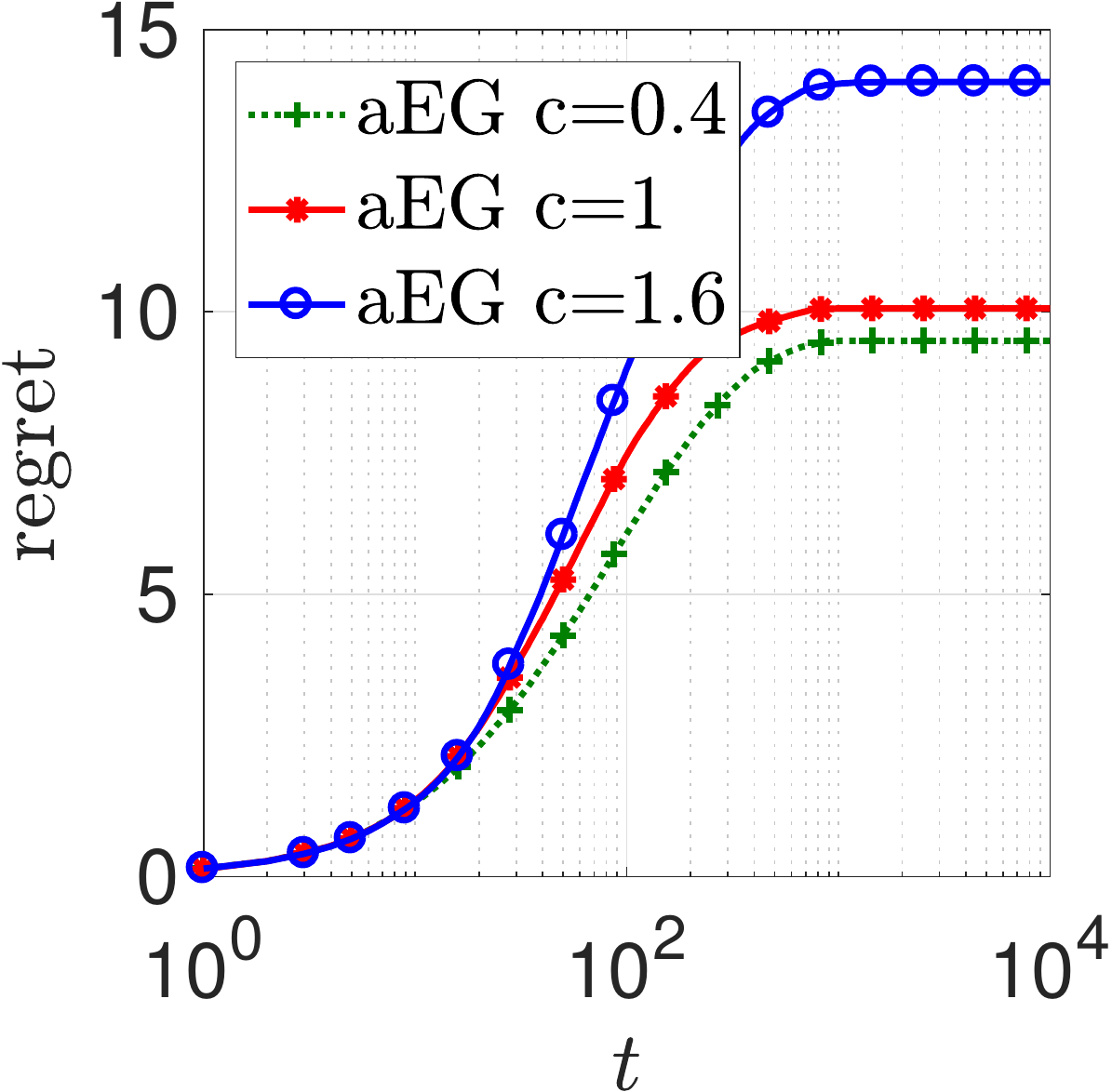}}
	\end{subfigure}
	\hfill
	\begin{subfigure}[t]{0.23\textwidth}
		\raisebox{-\height}{\includegraphics[width=\textwidth]{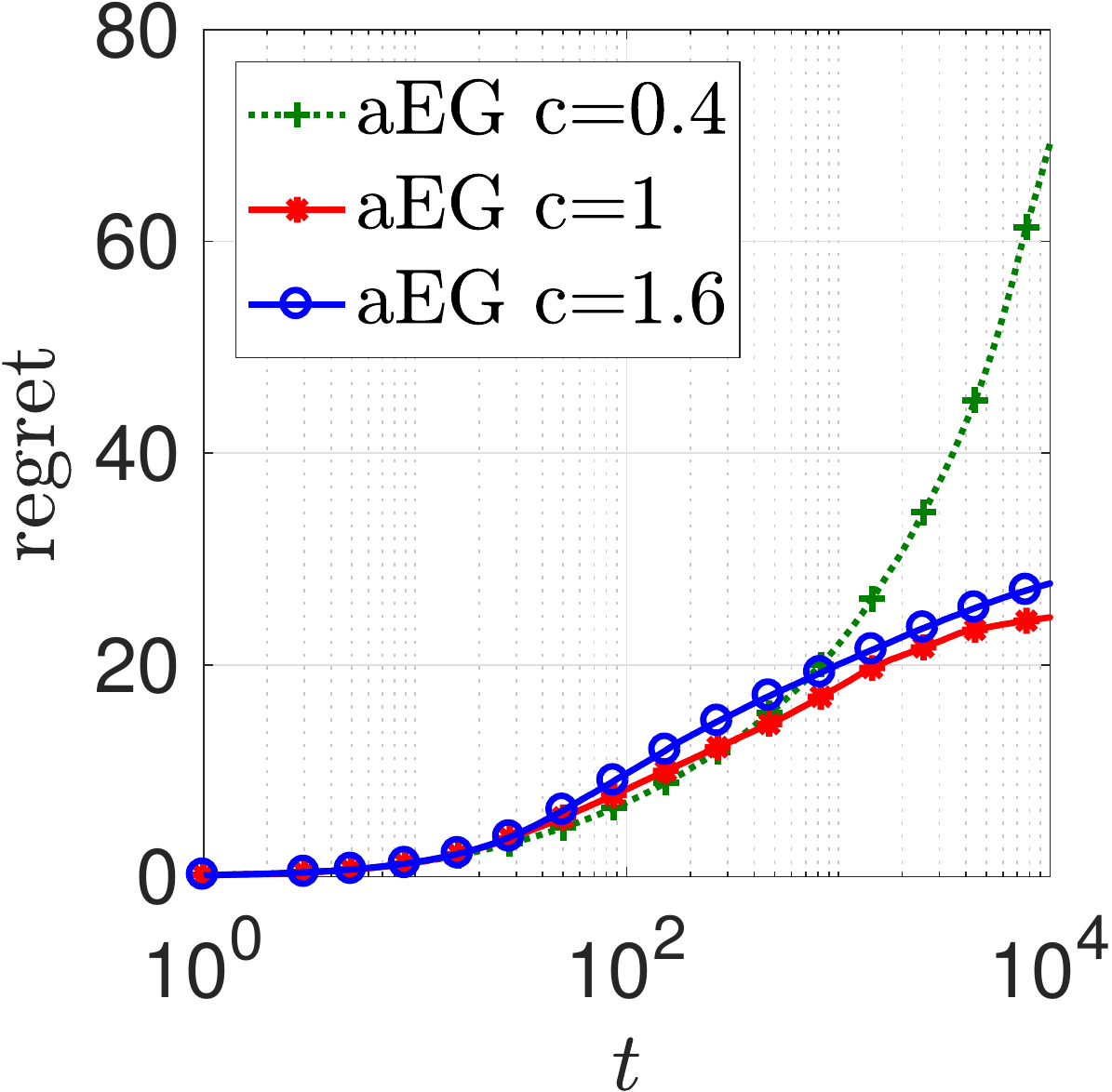}}
	\end{subfigure}
	\hfill
	\begin{subfigure}[t]{0.23\textwidth}
		\raisebox{-\height}{\includegraphics[width=\textwidth]{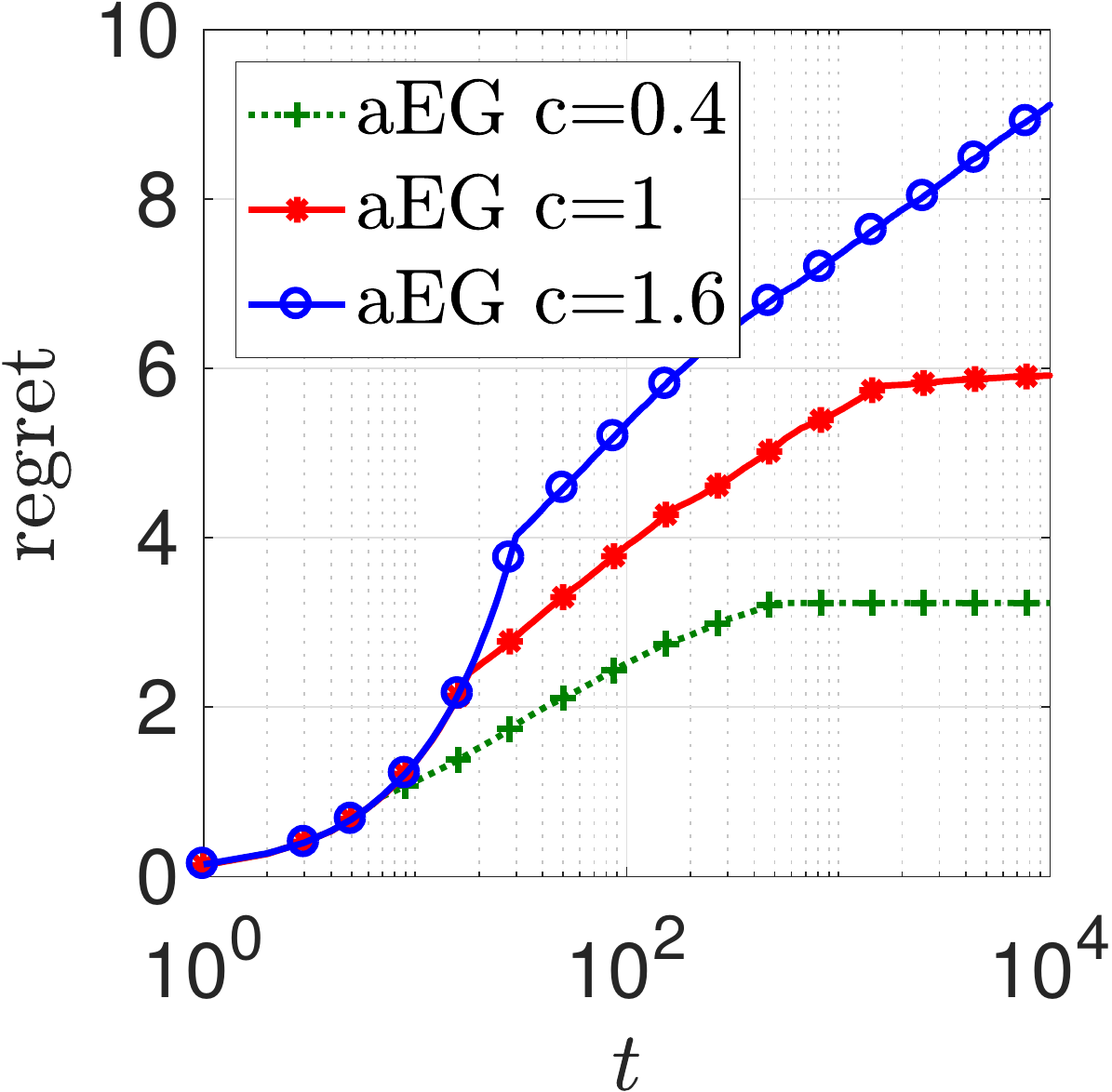}}
	\end{subfigure}
	\hfill
	\begin{subfigure}[t]{0.23\textwidth}
		\raisebox{-\height}{\includegraphics[width=\textwidth]{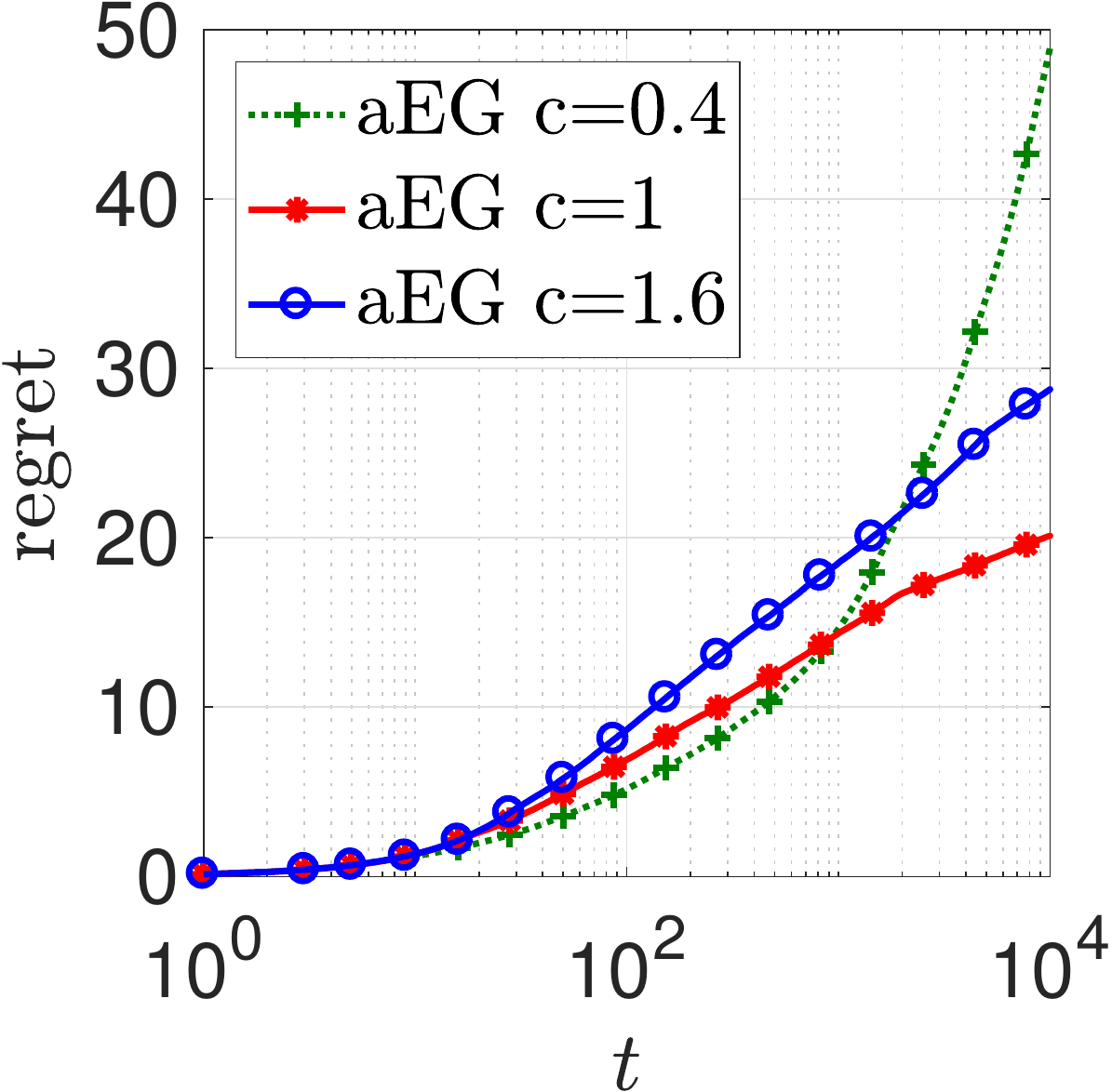}}
	\end{subfigure}
	
	\begin{subfigure}[t]{0.23\textwidth}
		\raisebox{-\height}{\includegraphics[width=\textwidth]{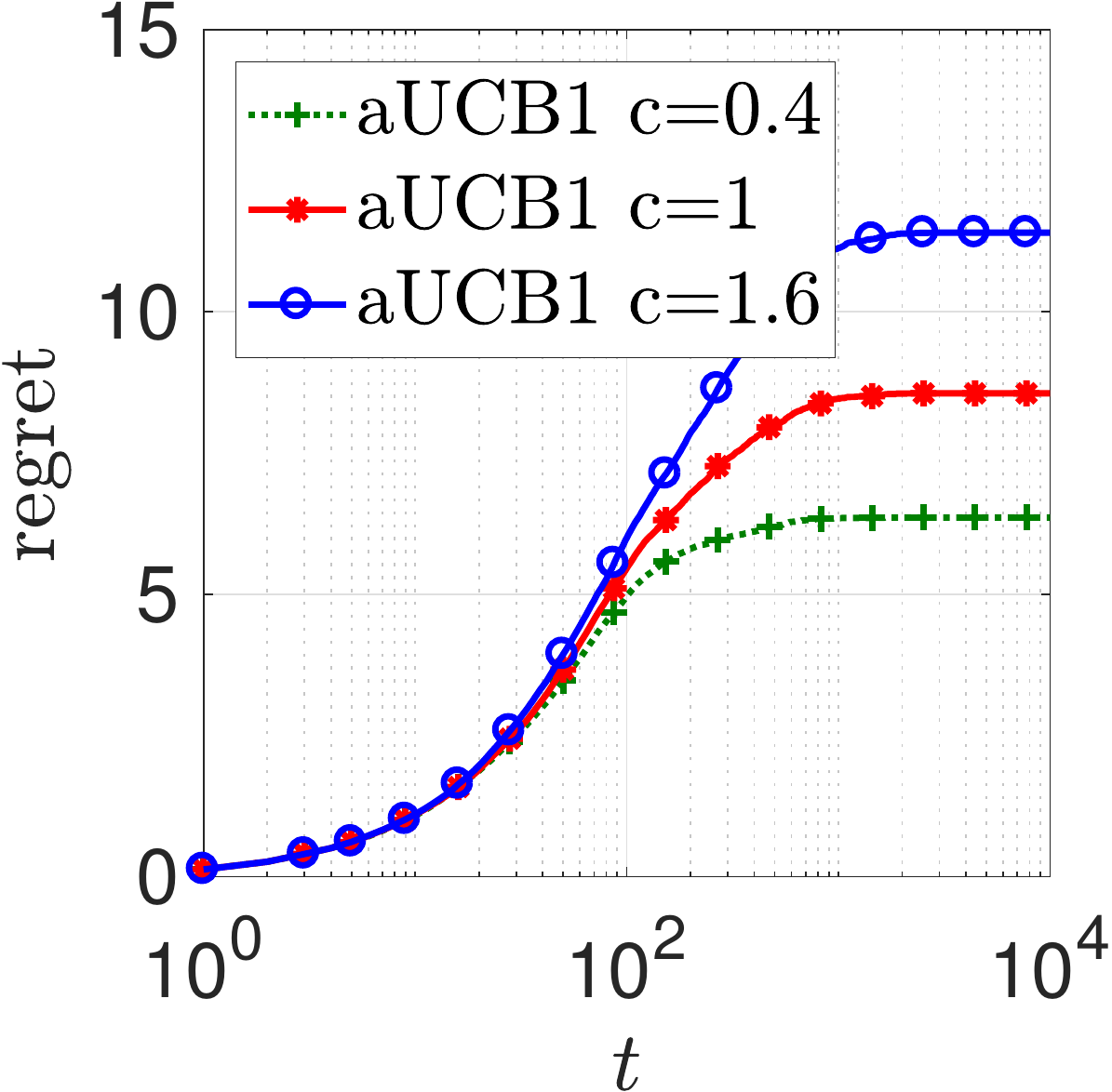}}
	\end{subfigure}
	\hfill
	\begin{subfigure}[t]{0.23\textwidth}
		\raisebox{-\height}{\includegraphics[width=\textwidth]{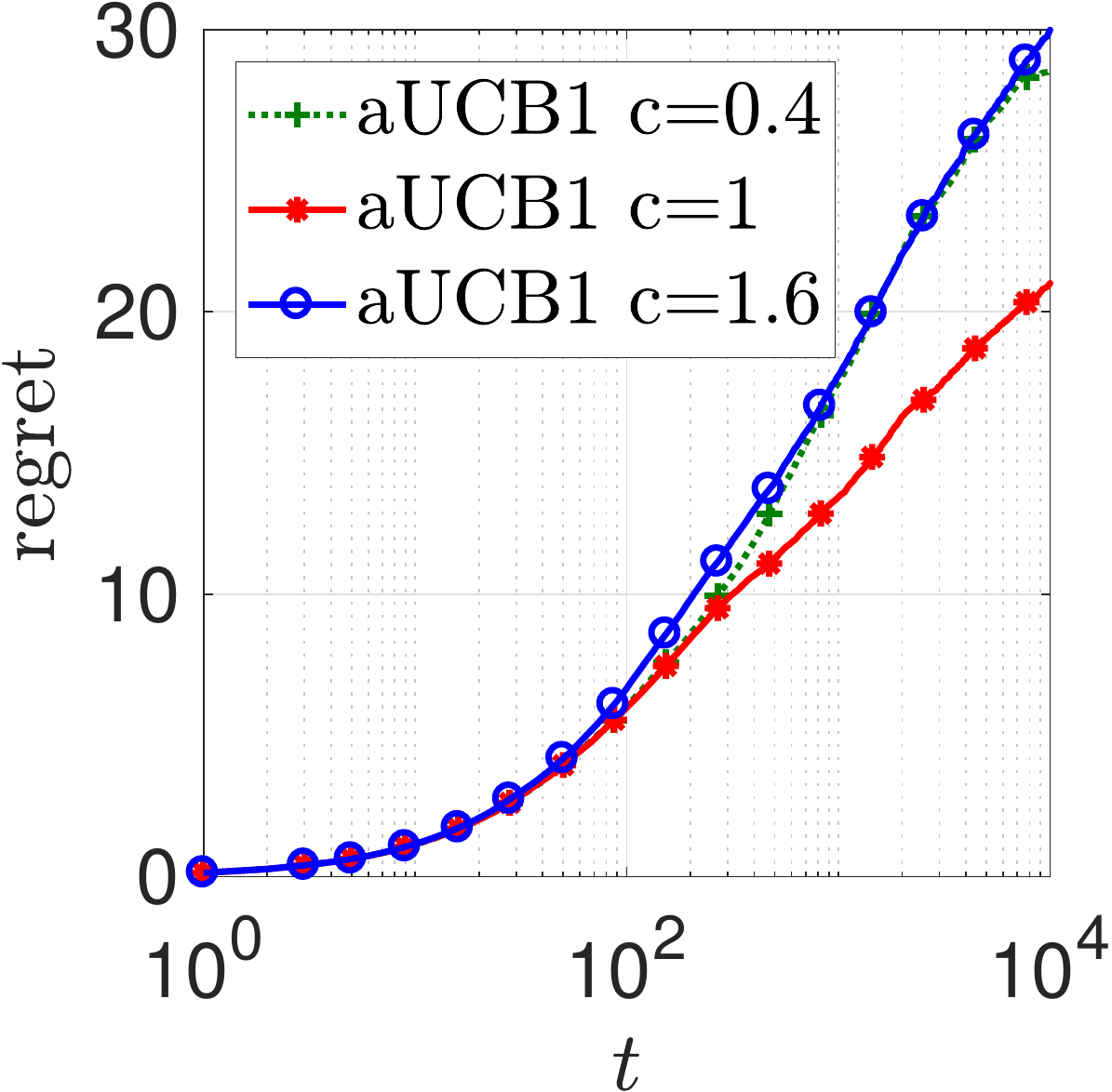}}
	\end{subfigure}
	\hfill
	\begin{subfigure}[t]{0.23\textwidth}
		\raisebox{-\height}{\includegraphics[width=\textwidth]{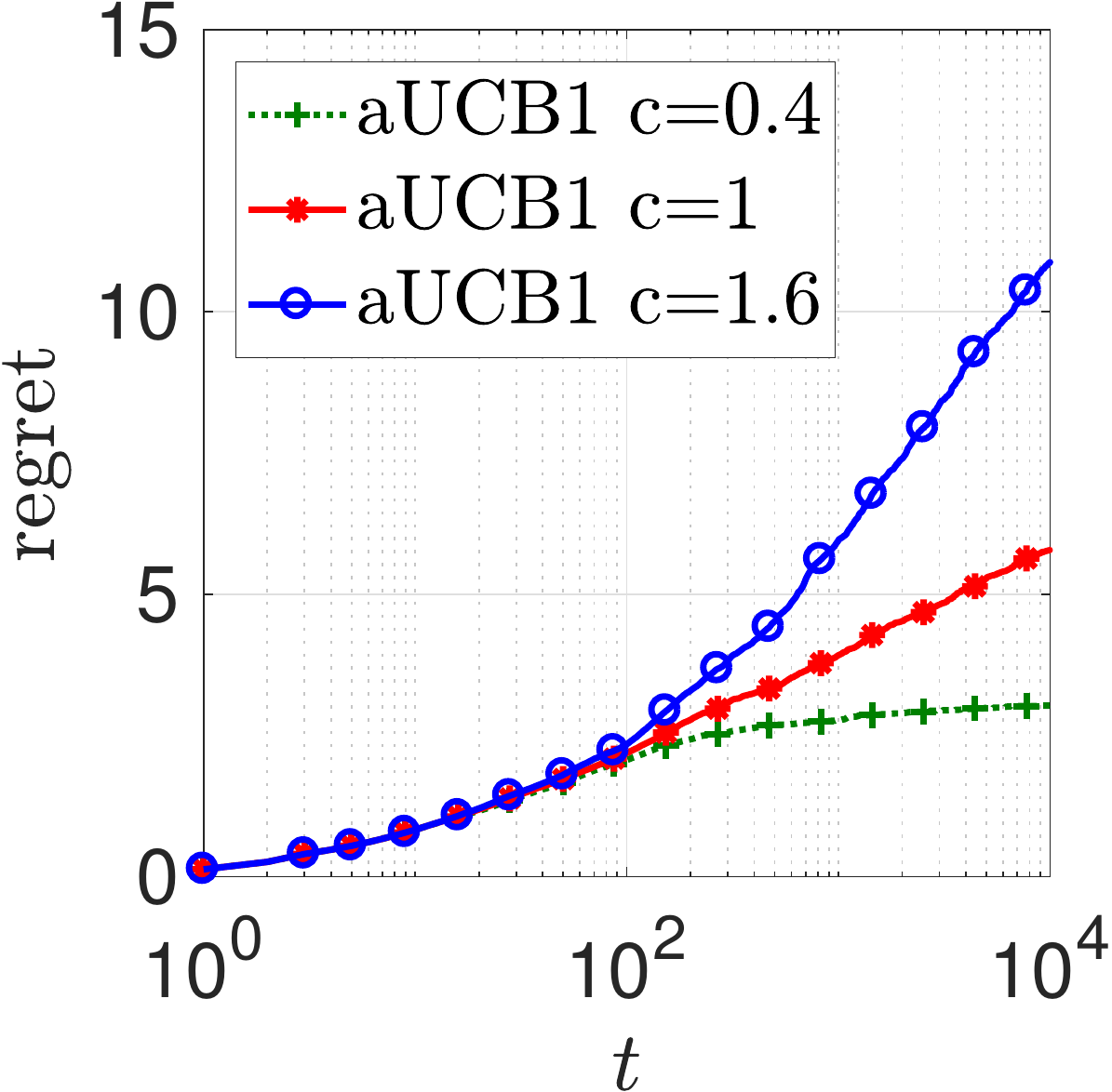}}
	\end{subfigure}
	\hfill
	\begin{subfigure}[t]{0.23\textwidth}
		\raisebox{-\height}{\includegraphics[width=\textwidth]{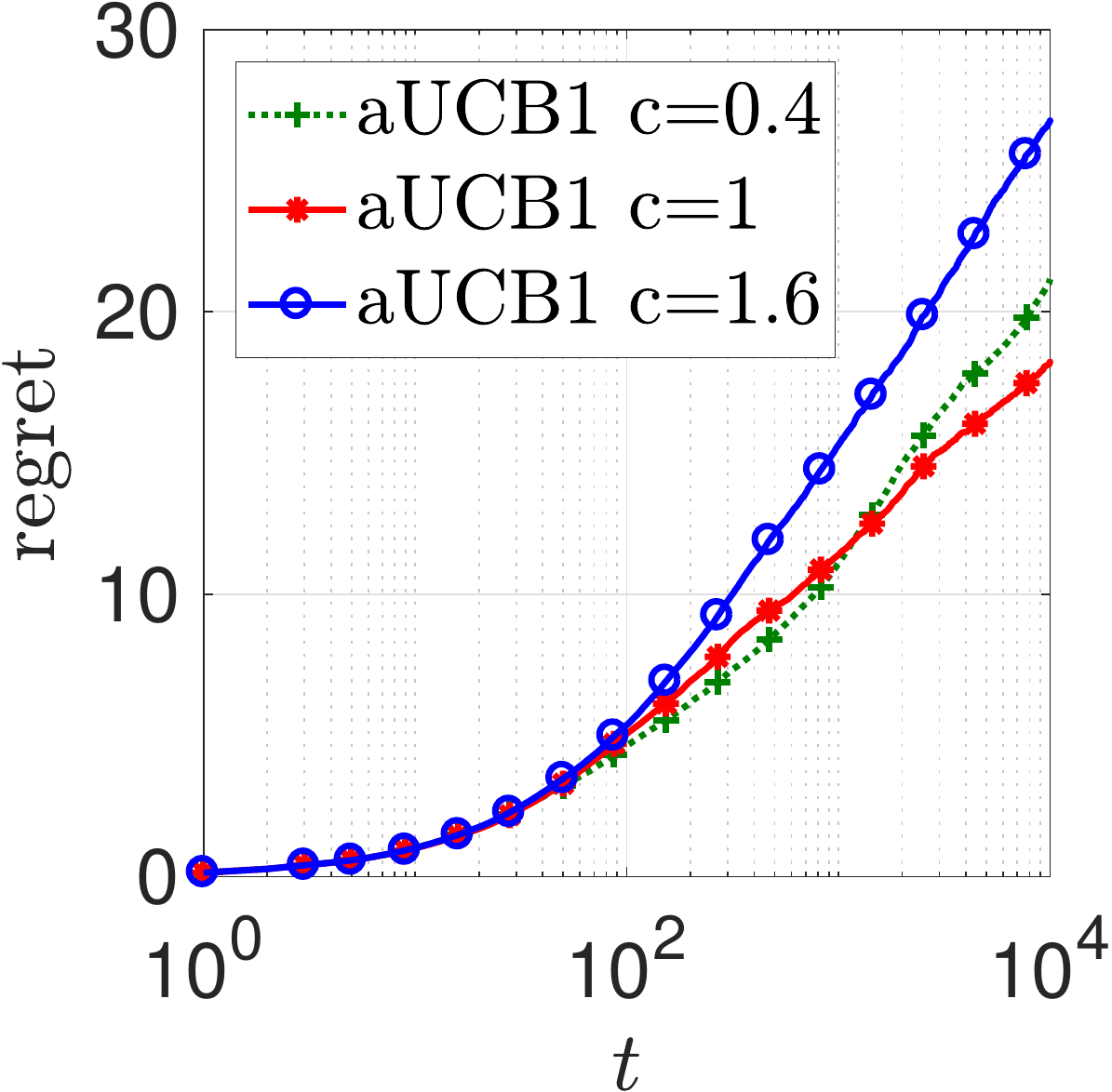}}
	\end{subfigure}
	
	\begin{subfigure}[t]{0.23\textwidth}
		\raisebox{-\height}{\includegraphics[width=\textwidth]{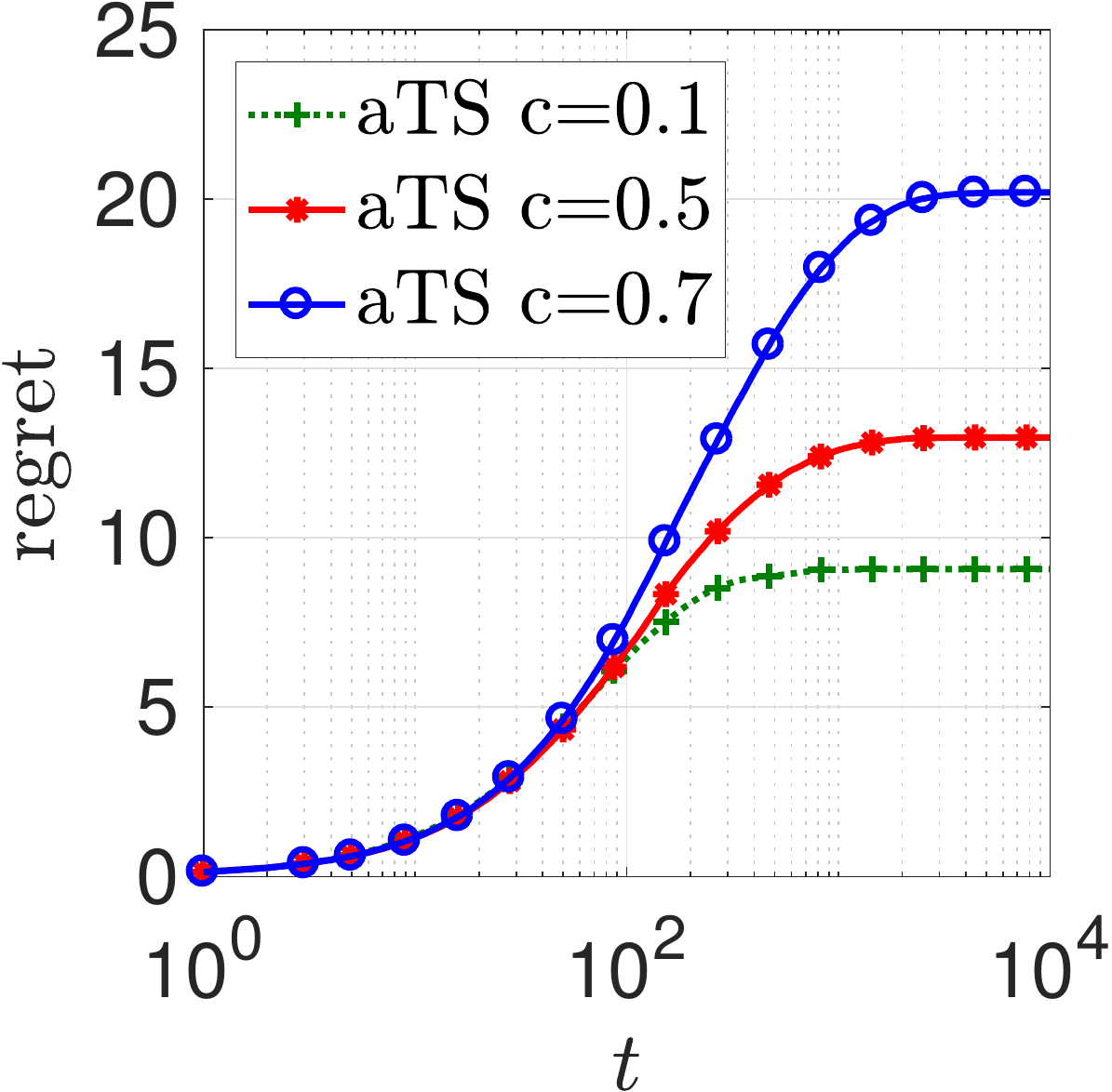}}
	\end{subfigure}
	\hfill
	\begin{subfigure}[t]{0.23\textwidth}
		\raisebox{-\height}{\includegraphics[width=\textwidth]{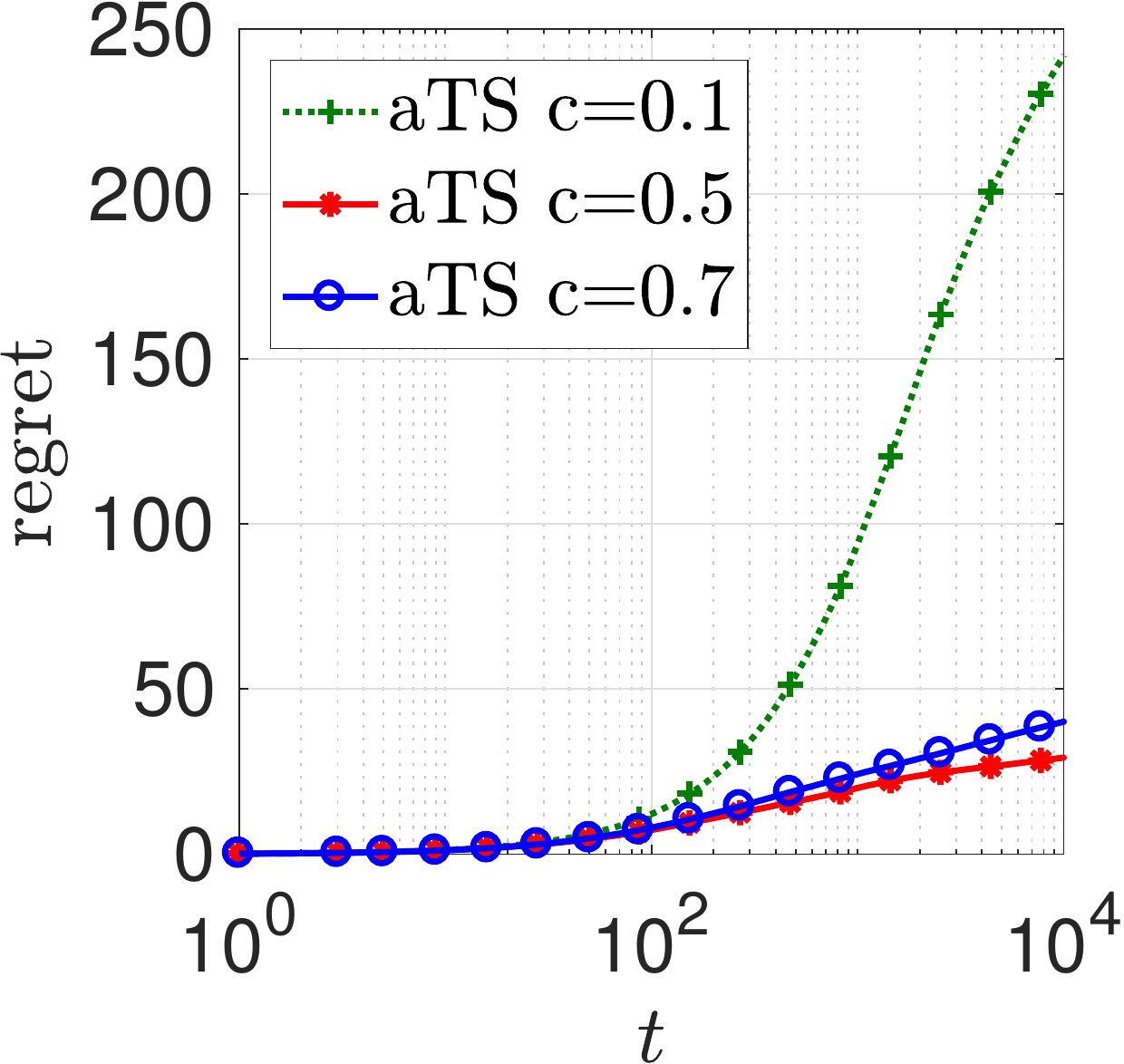}}
	\end{subfigure}
	\hfill
	\begin{subfigure}[t]{0.23\textwidth}
		\raisebox{-\height}{\includegraphics[width=\textwidth]{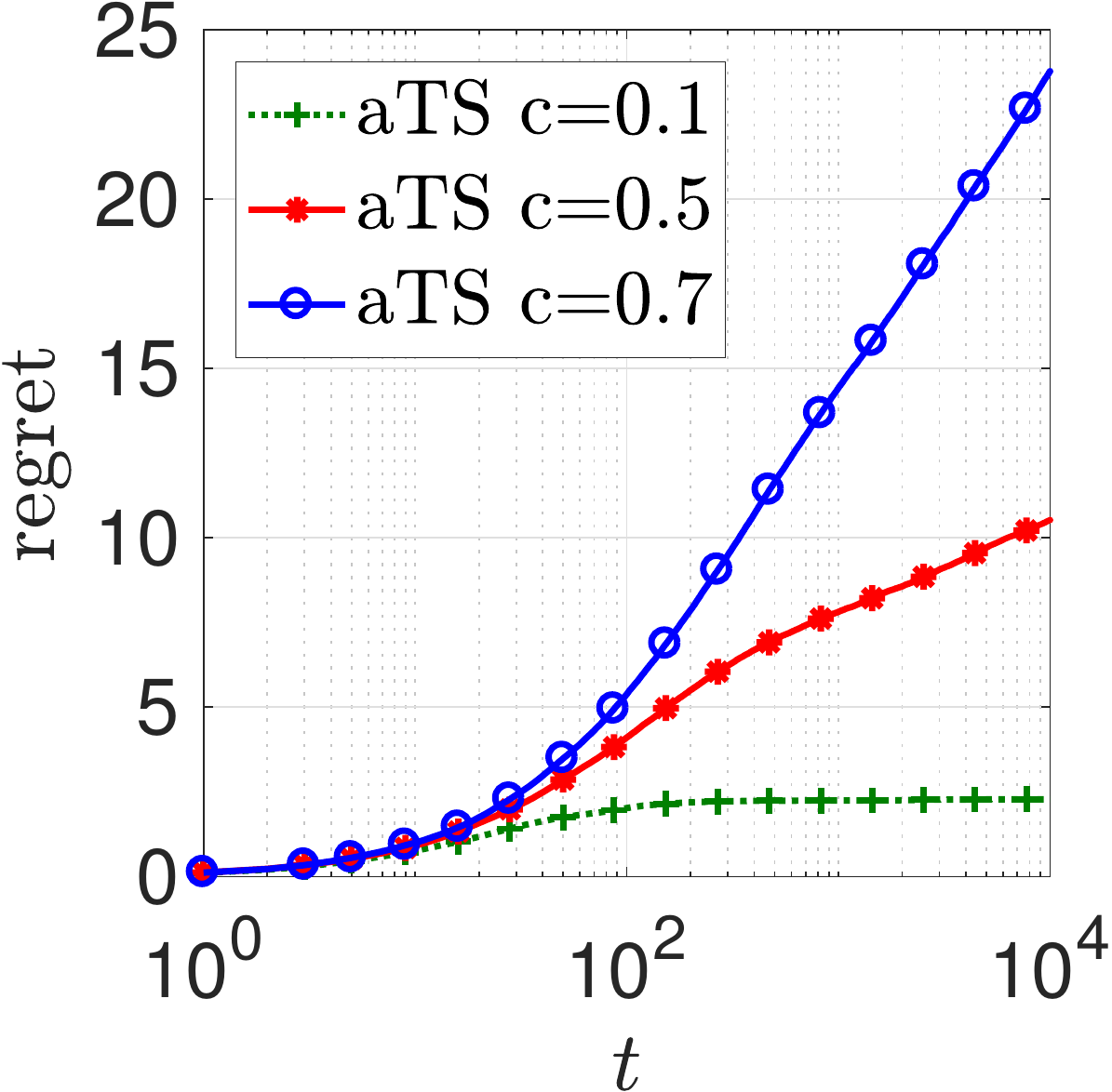}}
	\end{subfigure}
	\hfill
	\begin{subfigure}[t]{0.23\textwidth}
		\raisebox{-\height}{\includegraphics[width=\textwidth]{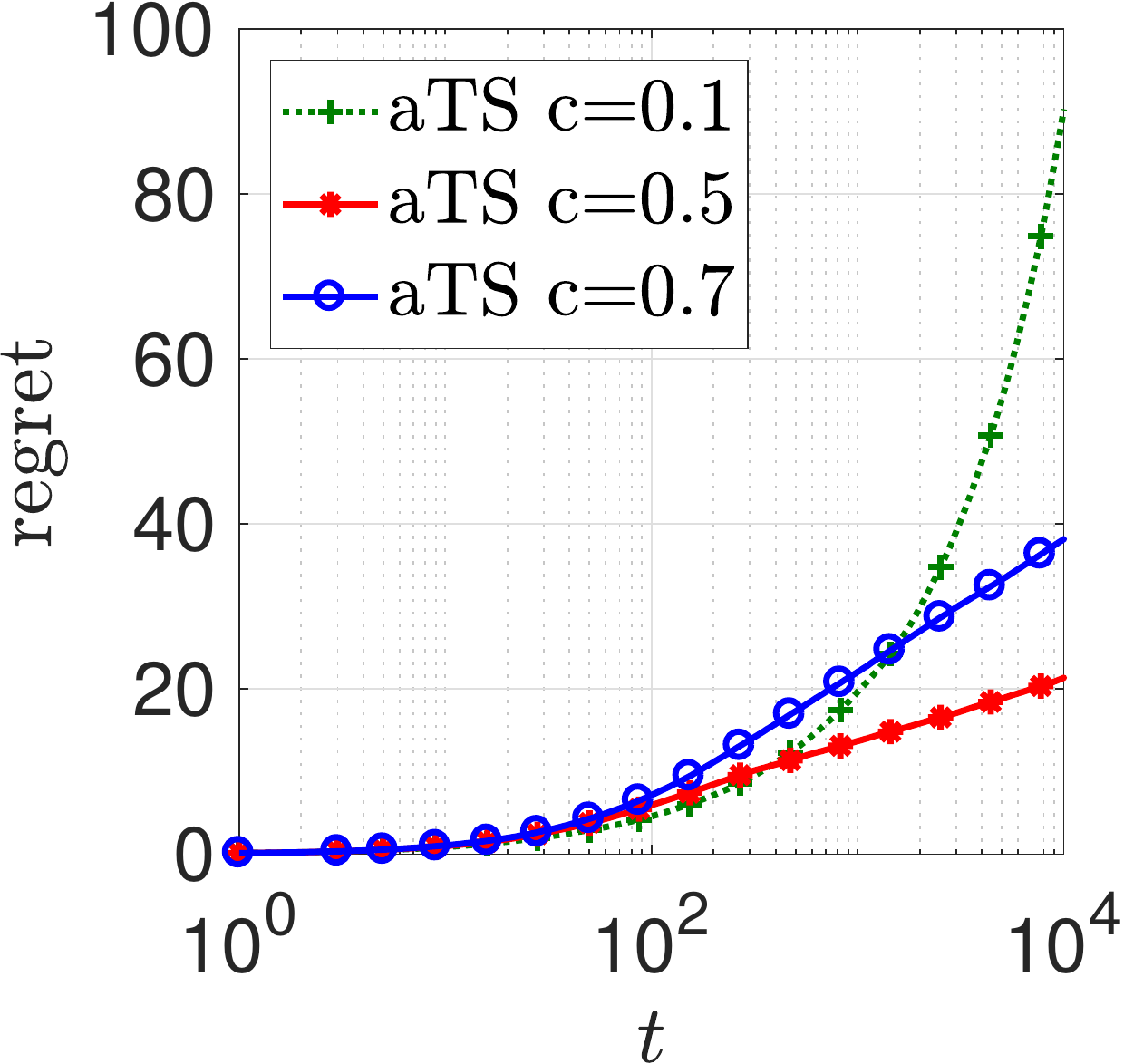}}
	\end{subfigure}
	\vspace{-0.2cm}
	\caption{\small The effect of the tuning parameter $c$ on the performance of different policies under different information arrival processes. \textit{First column on the left}: stationary information arrival process with constant rate $\lambda = \frac{500}{T}$. \textit{Second column}: stationary information arrival process with constant rate $\lambda = \frac{100}{T}$. \textit{Third column}: diminishing information arrival process with rate $\frac{\kappa^{\mathrm{aux}}}{t}$ at each time step $t$ for $\kappa^{\mathrm{aux}} = 8.0$. \textit{Fourth column}: diminishing information arrival process with rate $\frac{\kappa^{\mathrm{aux}}}{t}$ at each time step $t$ for $\kappa^{\mathrm{aux}} = 1.0$. (For description of policies see~\S\ref{subsec:numcomparison}.)}\label{fig-miss-c}
\end{figure}

\vspace{-0.4cm}
\subsection{Misspecification of the Minimal Gap $\Delta$ for $\epsilon_t$-greedy with Adaptive Exploration}\label{section-misspecification-Delta-UCB1}\vspace{-0.1cm}

\noindent \textbf{Setup.} We use the setup from \S\ref{subsec:numcomparison} to study the effect of misspecifying the minimum gap $\Delta$ for $\epsilon_t$-greedy with virtual time indices. We experiment with values of $\hat \Delta$ that are higher than, lower than, and equal to the true value of $\Delta=0.2$. The representative results below detail the empirical regret incurred by using $\hat{\Delta} \in \{0.05, 0.2, 0.35\}$, in different information arrival processes: stationary arrival processes with $\lambda \in \{500/T, 10/ T\}$ and diminishing arrival processes with $\kappa^{\mathrm{aux}} \in \{8.0, 1.0\}$. We used the tuning parameter $c=1.0$ in all cases.

\medskip
\noindent \textbf{Results and discussion.} Plots comparing the regret accumulation with different specified gaps $\hat \Delta$ appear in Figure \ref{fig-miss-Delta}. \Copy{misspecification-Delta-e-greedy}{In all cases, a conservative choice of the gap $\hat \Delta$ (that is, $\hat \Delta = 0.05$) results in over-exploration and additional regret compared with the case that the exact value of the gap is used. On the other hand, a lower value of $\hat \Delta$ (that is, $\hat \Delta=0.4$) results in lower regret when the amount of information is high. The reason is that a higher value of $\hat \Delta$ leads to less exploration, which, when compensated by the presence of many auxiliary observations, leads to better performance overall. However, a higher value of $\hat \Delta$ can result in high regret when the amount of auxiliary information is low.}

\begin{figure} [H]
	\centering
	\begin{subfigure}[t]{0.23\textwidth}
		\raisebox{-\height}{\includegraphics[width=\textwidth]{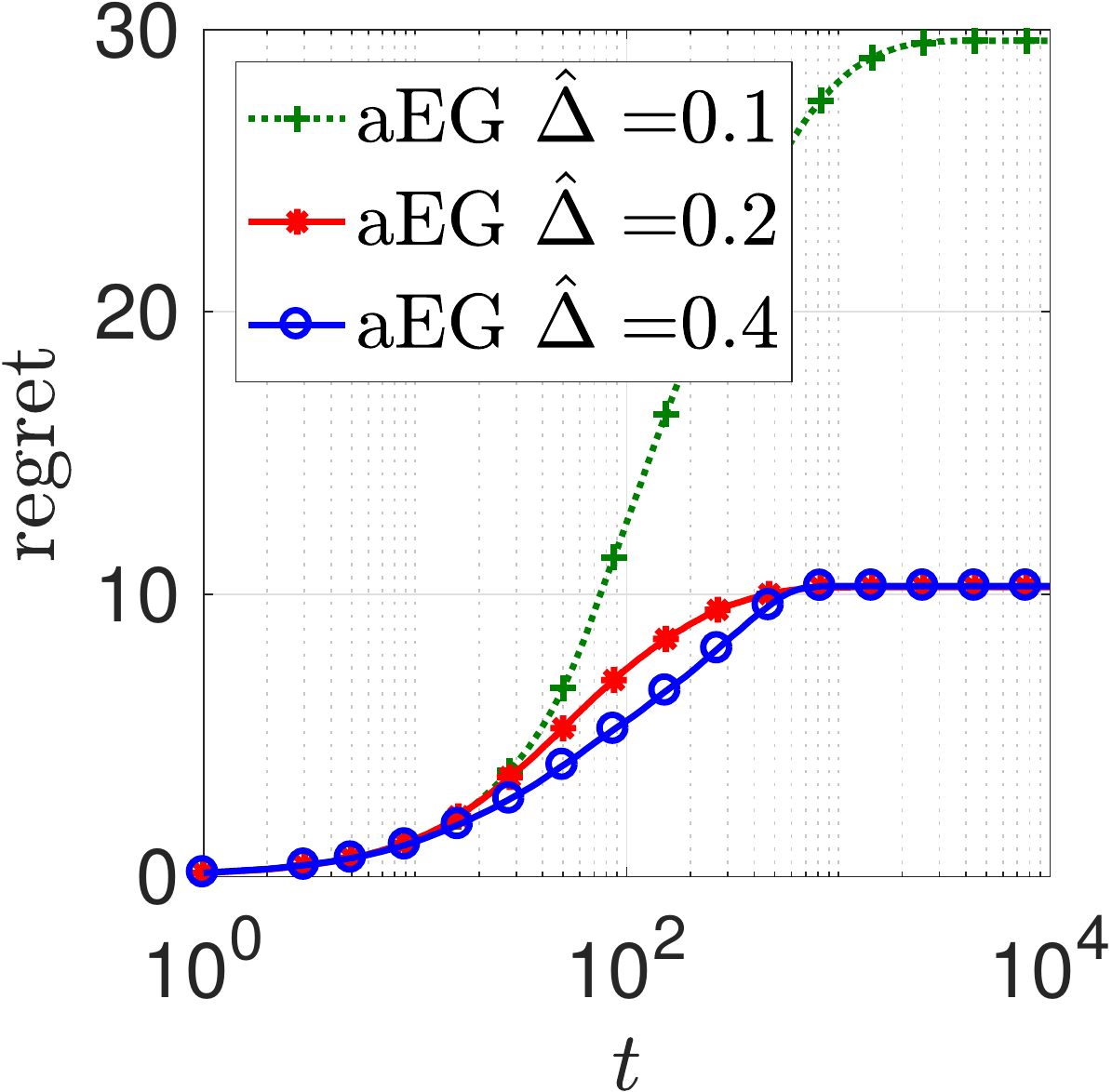}}
	\end{subfigure}
	\hfill
	\begin{subfigure}[t]{0.23\textwidth}
		\raisebox{-\height}{\includegraphics[width=\textwidth]{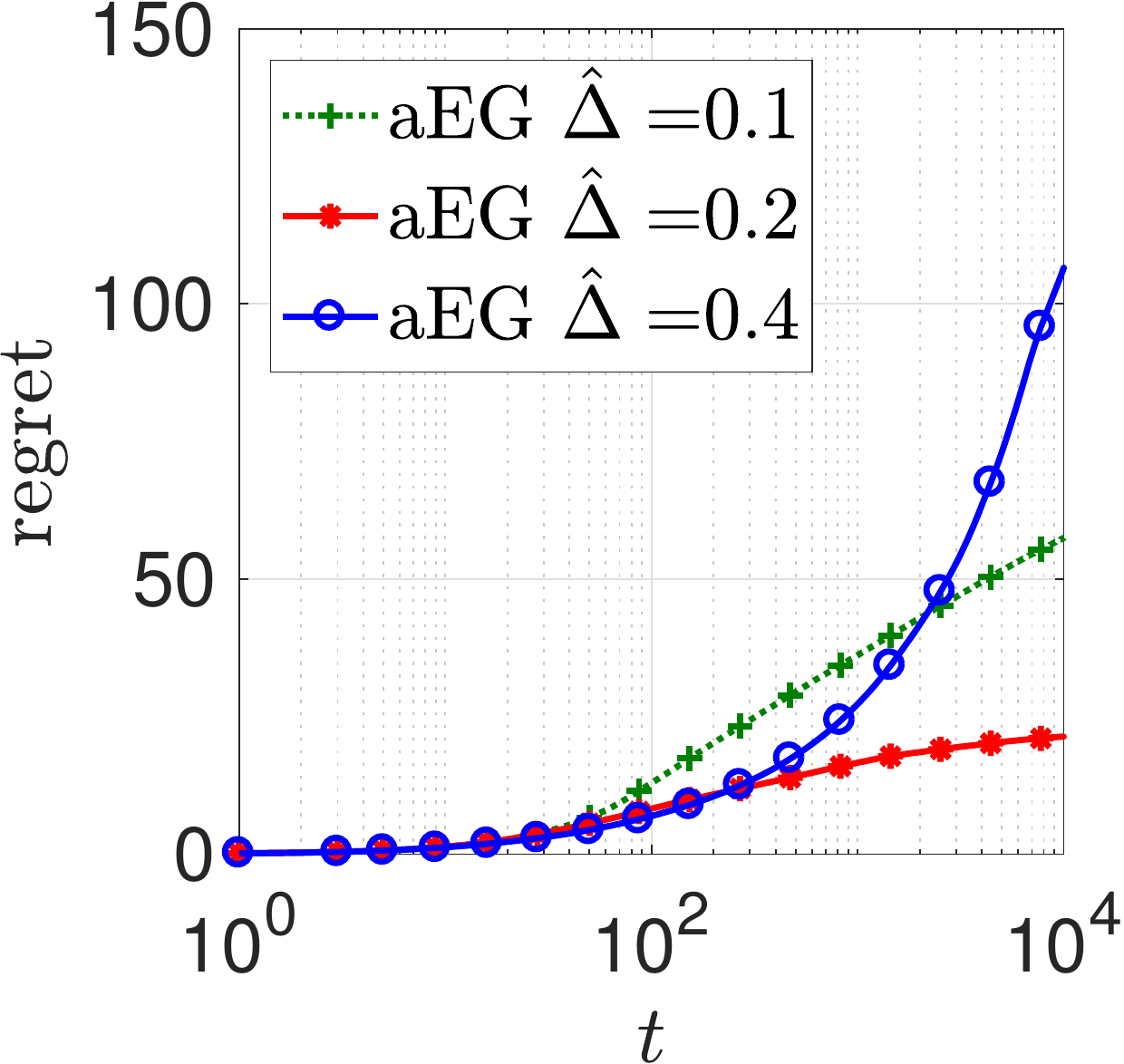}}
	\end{subfigure}
	\hfill
	\begin{subfigure}[t]{0.23\textwidth}
		\raisebox{-\height}{\includegraphics[width=\textwidth]{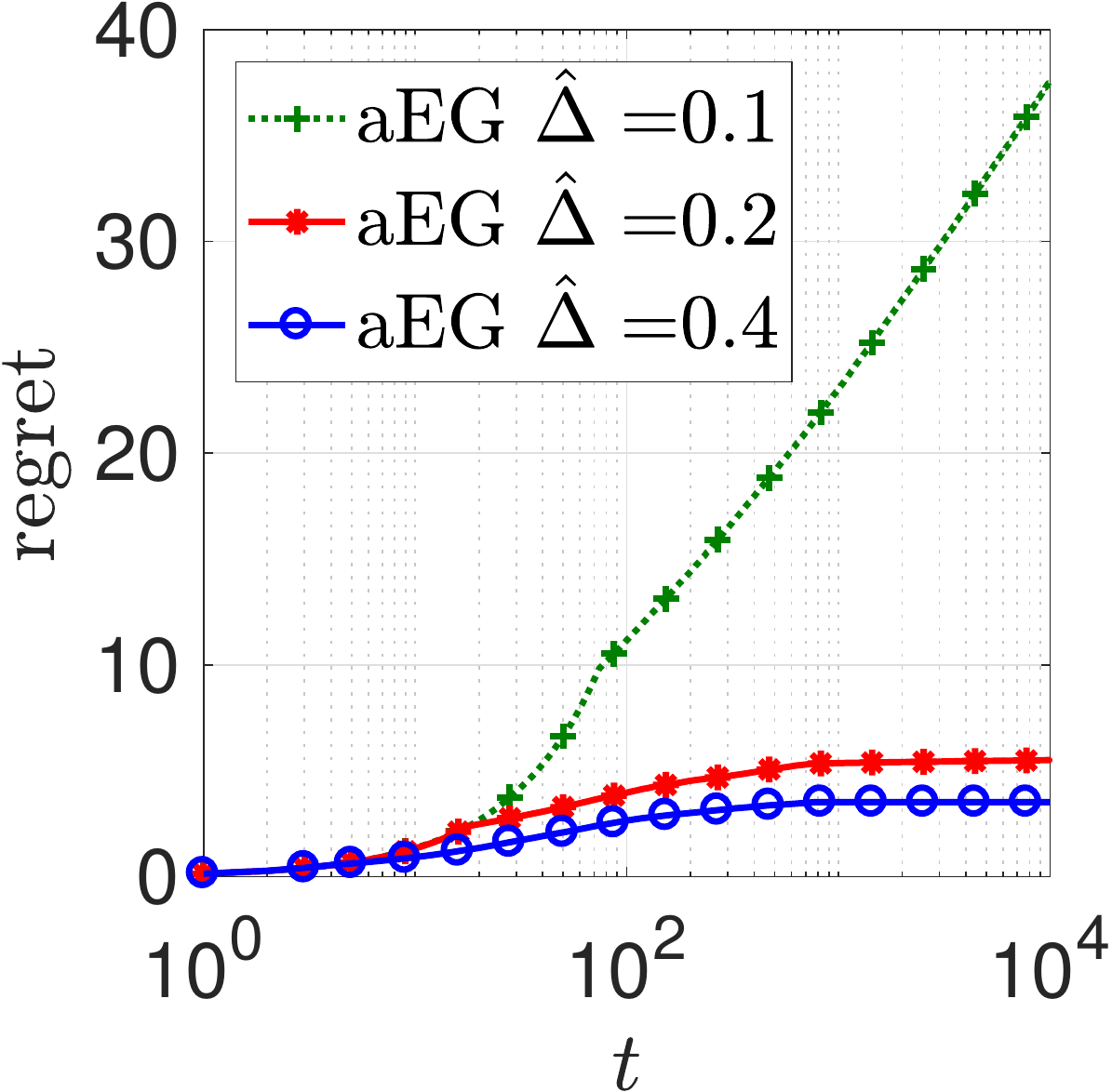}}
	\end{subfigure}
	\hfill
	\begin{subfigure}[t]{0.23\textwidth}
		\raisebox{-\height}{\includegraphics[width=\textwidth]{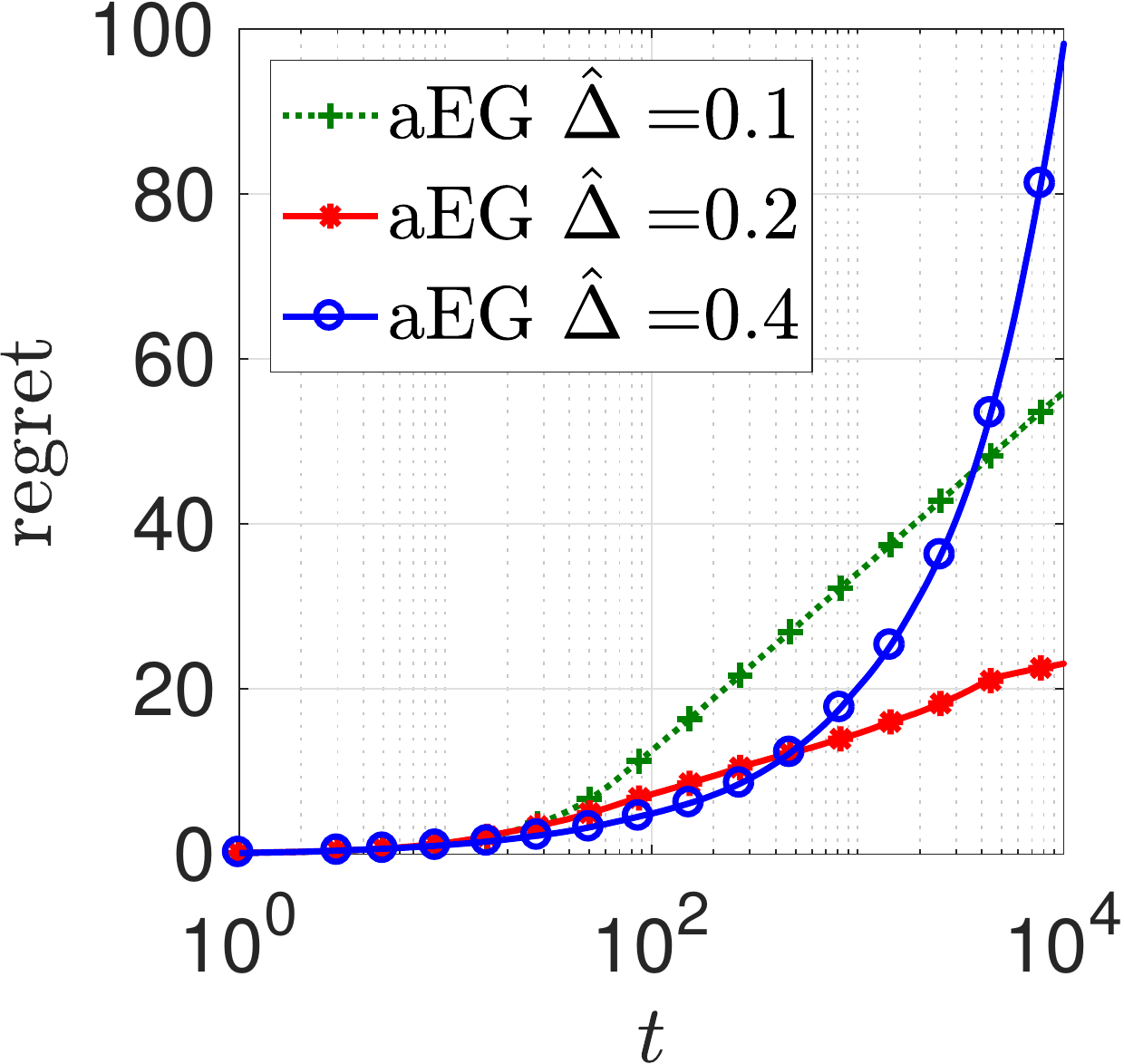}}
	\end{subfigure}
	\caption{\small The effect of the gap parameter $\hat{\Delta}$ on the performance of $\epsilon_t$-greedy with adaptive exploration under different information arrival processes. From left to right: stationary information arrival process with constant rate  $\lambda = \frac{500}{T}$; stationary information arrival process with constant rate $\lambda = \frac{100}{T}$; diminishing information arrival process with rate $\frac{\kappa^{\mathrm{aux}}}{t}$ at each time step $t$ for $\kappa^{\mathrm{aux}} = 8.0$; diminishing information arrival process with rate $\frac{\kappa^{\mathrm{aux}}}{t}$ at each time step $t$ for $\kappa^{\mathrm{aux}} = 1.0$. (For description of the policies see \S\ref{subsec:numcomparison}.)}\label{fig-miss-Delta}
\end{figure}

\vspace{-0.8cm}
\section{Analysis with Content Recommendations Data}\vspace{-0.1cm}\label{sec:realdata}

\subsection{Description of Policies}\label{app-sec-policy-description}
We compare the performance of UCB1 when ignoring auxiliary data, with the one achieved by utilizing auxiliary data based on estimated mappings (aUCB1, see \S\ref{section-ucb1}), and the adaptive 2-UCBs policy (see \S\ref{section-2ucb}). Let $\tilde{t}_t \coloneqq \sum_{s=1}^{t-1} \cC_s$ be the number of clicks up to time $t$. For the standard UCB1, the variables in \eqref{eq-def-of-counters-and-empirical-mean} are set as follows:\vspace{-0.2cm}

\footnotesize
\begin{equation*}
	n^{\known} _{k,t} = \sum\limits_{s=1}^{t-1}  \cC_{k,s} \mathbbm{1}\{\pi_s = k\},
	\qquad
	\bar{X}_{k,t}^{\known} = \frac{\sum\limits_{s=1}^{t-1} \cC_{k,s}\mathbbm{1}\{\pi_s = k\}X_{k,s}}{\sum\limits_{s=1}^{t-1} \cC_{k,s} \mathbbm{1}\{\pi_s = k\}},
	\qquad
	U_{k,t} \coloneqq \bar{X}_{k,t}^{\known} + \sqrt{\frac{c\sigma^2 \log \tilde{t}_t }{n^{\known} _{k,t}}}.
	\vspace{-0.2cm}
\end{equation*}\normalsize
For the aUCB1 that utilizes auxiliary observations, they are set as follows:\vspace{-0.3cm}

\footnotesize
\begin{align*}\label{eq:empirics-update-rule}
	n^{\known} _{k,t} &= \sum\limits_{s=1}^{t-1}  \cC_{k,s} \mathbbm{1}\{\pi_s = k\}  + \sum\limits_{s=1}^{t} \frac{\sigma^2}{\hat{\sigma}^2}h_{k,s}, \qquad
	\bar{X}_{k,t}^{\known} = \frac{\sum\limits_{s=1}^{t-1}\frac{1}{\sigma^2}\cC_{k,s}\mathbbm{1}\{\pi_s = k\}X_{k,s} +  \sum\limits_{s=1}^{t} \sum\limits_{m=1}^{h_{k,t}} \frac{1}{\hat{\sigma}^2}\hat Z_{k,s,m}}{\sum\limits_{s=1}^{t-1} \frac{1}{\sigma^2} \cC_{k,s} \mathbbm{1}\{\pi_s = k\}  + \sum\limits_{s=1}^{t} \frac{1}{\hat{\sigma}^2}h_{k,s} },
	\\
	U_{k,t} &= \bar{X}_{k,t}^{\known} + \sqrt{\frac{c\sigma^2 \log \tilde{t}_t }{n^{\known} _{k,t}}}.\vspace{-0.3cm}
\end{align*}\normalsize
where $\sigma^2 = 1/4$, $\hat \sigma^2 = \hat\alpha_{a,d-1}^2 /4$, and $\hat Z_{k,t,m} = \hat{\alpha}_{a,d} Y_{k,t,m}$. The variables in 2-UCBs (in Table \ref{table:notations}) are defined as follows:

\footnotesize
\begin{align*}
	n_{k,t}^\pi &= \sum\limits_{s=1}^{t-1}  \cC_{k,s} \mathbbm{1}\{\pi_s = k\}, \quad
	&\bar{X}_{k,t}^\pi &= \frac{\sum\limits_{s=1}^{t-1} \cC_{k,s}\mathbbm{1}\{\pi_s = k\}X_{k,s}}{\sum\limits_{s=1}^{t-1} \cC_{k,s} \mathbbm{1}\{\pi_s = k\}},
	\\
	n^{\pi, \mathrm{aux}}_{k,t} &= \sum\limits_{s=1}^{t-1}  \cC_{k,s} \mathbbm{1}\{\pi_s = k\}  + \sum\limits_{s=1}^{t} \frac{\sigma^2}{\hat{\bar{\sigma}}^2}h_{k,s},
	\quad
	&\bar{X}^{\pi, \mathrm{aux}}_{k,t} &= \frac{\sum\limits_{s=1}^{t-1}\frac{1}{\sigma^2}\cC_{k,s}\mathbbm{1}\{\pi_s = k\}X_{k,s} +  \sum\limits_{s=1}^{t} \sum\limits_{m=1}^{h_{k,t}} \frac{1}{\hat{\bar{\sigma}}^2}\hat{\bar{ Z}}_{k,s,m}}{\sum\limits_{s=1}^{t-1} \frac{1}{\sigma^2} \cC_{k,s} \mathbbm{1}\{\pi_s = k\}  + \sum\limits_{s=1}^{t} \frac{1}{\hat{\bar{\sigma}}^2}h_{k,s} },
	\\\quad
	U^{\pi}_{k,t} &=
	\bar{X}^{\pi}_{k,t} + \sqrt{\frac{c\sigma^2 \log \tilde t_t}{n^{\pi}_{k,t}}},
	\quad
	&U^{\pi, \mathrm{aux}}_{k,t}&=
	\bar{X}^{\pi, \mathrm{aux}}_{k,t} + \sqrt{\frac{c\sigma^2 \log \tilde t_t}{n^{\pi, \mathrm{aux}}_{k,t}}},
	\qquad \quad
	U_{k,t}
	=
	\min\l\{ U^{\pi}_{k,t} , U^{\pi, \mathrm{aux}}_{k,t} \r\},
	\vspace{-0.2cm}
\end{align*}\normalsize
where $\hat{\bar{\sigma}}^2 = \bar{\alpha}_{a,d}^2 /4$ and $\hat{\bar{ Z}}_{k,s,m} = \bar{\alpha}_{a,d} Y_{k,t,m}$.
All algorithmic variants make decisions based on the following rule:\vspace{-0.2cm}
\[
\pi_t
=
\begin{cases}
	0 & \text{if } \mathrm{CVR}^{recom}_{a,d,0} > U_{k,t};\\
	1 & \text{o.w.}
\end{cases}\vspace{-0.1cm}
\]
All algorithms were tuned using the parametric value $c=0.05$, which was identified through exhaustive search as the value that optimizes the performance of the standard UCB1 over the data set.
\color{black}

\subsection{Summary of Results}
Figure \ref{fig-colormps} depicts the regret of the policies aUCB1 and 2-UCBs as function of the AIE and RMM. The color of each point is proportional to the regret of the corresponding policy for a certain day and article. The scaling of the colors are provided for each scatter plot. As you can see for both policies mapping misspecification deteriorates performance the most when auxiliary information is more effective. Furthermore, for the cases with mapping misspecification, 2-UCBs policy experiences less performance loss.
\begin{figure} [H]
	\centering
	\begin{subfigure}[t]{0.45\textwidth}
		\raisebox{-\height}{\includegraphics[width=\textwidth]{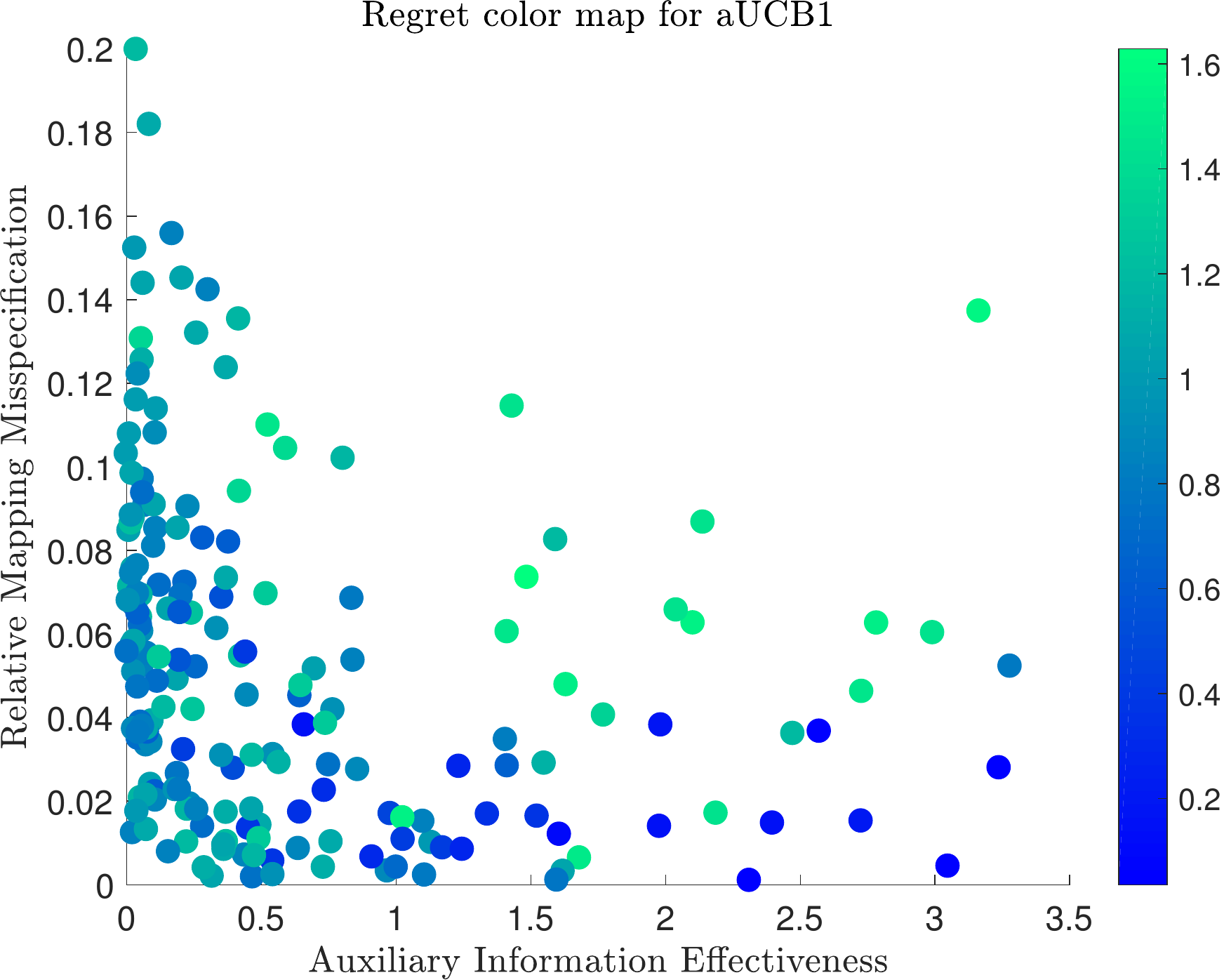}}
	\end{subfigure}
	\hfill
	\begin{subfigure}[t]{0.45\textwidth}
		\raisebox{-\height}{\includegraphics[width=\textwidth]{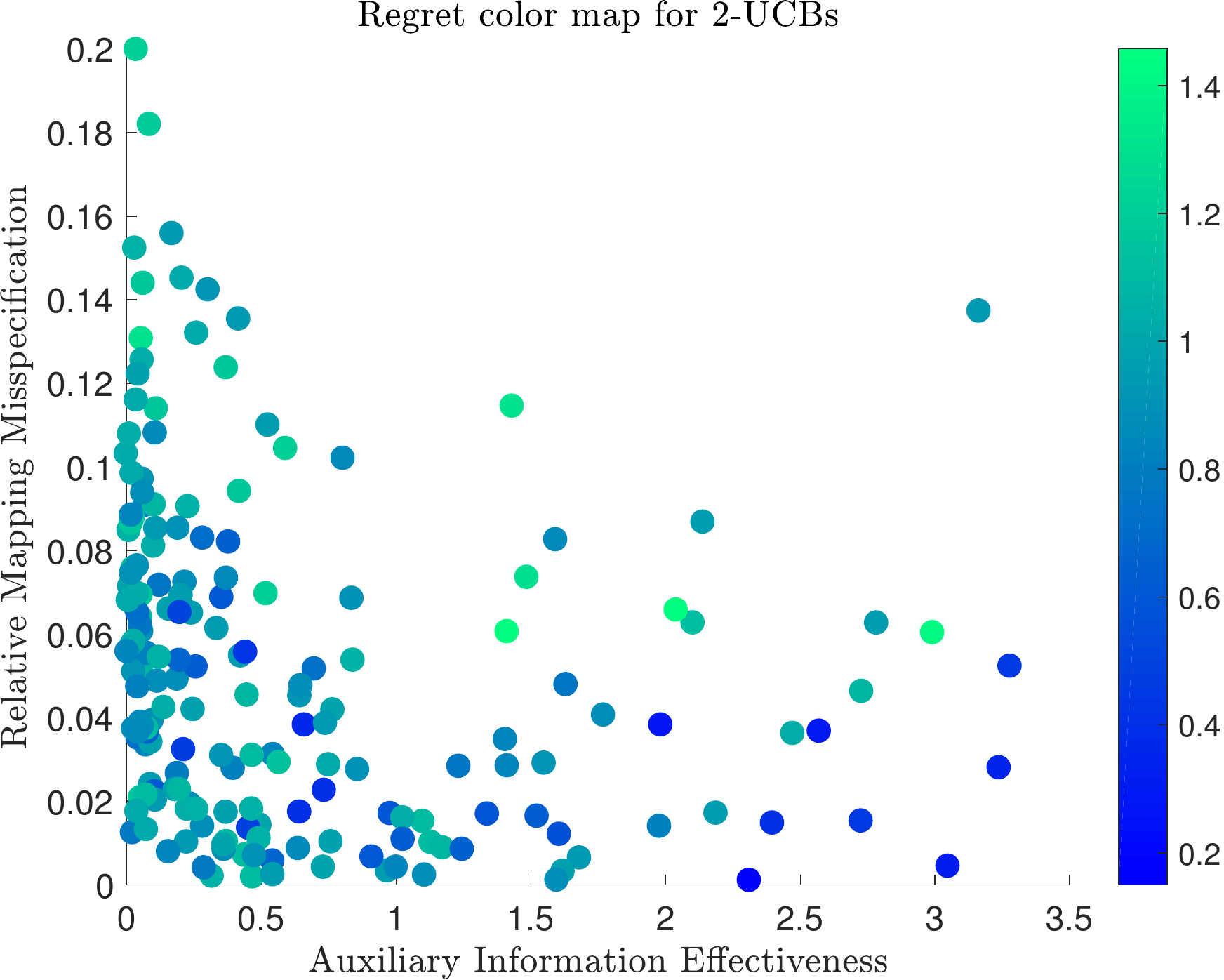}}
	\end{subfigure}
\caption{Regret as a function of AIE and RMM for $\Delta=0.03$ and tuning parameter~$c=0.05$; Colors are proportional to regret. \textit{Left:} aUCB1. \textit{Right:} 2-UCBs}
\label{fig-colormps}
\end{figure}

We used different selections of the tuning parameter $c\in\{0.03,0.05,0.1\}$ and the gap parameter $\Delta \in \{0.01,0.02,0.03,0.04\}$ or when $\Delta$ is $20\%$ of the $\mathrm{CVR}$ of each article and day pair. Table \ref{table-avg-regret}, \ref{table-avg-ri}, and \ref{table-avg-nh} present the cumulative regret (normalized by the cumulative regret of UCB1 with $\Delta=0.03$ and $c=0.05$), the average Relative Improvement (RI), and the average No Harm (NH) rate, respectively; where the No Harm (NH) metric for policy $\pi$ is the fraction of cases in which $\pi$ has a better performance than~UCB1:
\[
\text{NH}(\pi) \coloneqq \frac{\sum_{a,d} \Indlr{\cR^{\pi}_{a,d}\le \cR^{\text{UCB1}}_{a,d}} }{\sum_{a,d} 1}.
\]

\begin{table}[H]
	\centering
	\small
	\begin{tabular}{|l|l|l|l|l|l|l|l|l|l|}
		\hline
		\multirow{2}{*}{\backslashbox{\;\;\;\;\quad\quad}{\,}} &
		\multicolumn{3}{c|}{\boldmath{$c=0.03$}} &
		\multicolumn{3}{c|}{\boldmath{$c=0.05$}} &
		\multicolumn{3}{c|}{\boldmath{$c=0.1$}} \\
		& UCB1 & aUCB1 & 2-UCBs & UCB1 & aUCB1 & 2-UCBs  & UCB1 & aUCB1 & 2-UCBs \\
		\hline
		\boldmath{$\Delta = 0.01$} & 0.44538 & 0.42288 & 0.43988 & 0.44567& 0.43286 & 0.43855  & 0.46515 & 0.45524 & 0.45252\\
		\hline
		\boldmath{$\Delta = 0.02$} & 0.90161 & 0.82164 & 0.86636 & 0.89996 & 0.84585 & 0.83899 & 1.0251 & 0.95521 & 0.94564 \\
		\hline
		\boldmath{$\Delta = 0.03$} & 1.0107 & 0.90583 & 0.95628 & 1 & 0.93476 & 0.91303 & 1.1764 & 1.0839 & 1.0611\\
		\hline
		\boldmath{$\Delta = 0.04$} & 1.1403 & 1.0023 & 1.0732 & 1.1373 & 1.0461 & 1.0057 & 1.4188 & 1.2591 & 1.2093\\
		\hline
		\boldmath{$\Delta = 0.2\mathrm{CVR}_{a,d}$}  & 1.1679 & 1.0093 & 1.1354 & 1.0984 & 1.0115 & 1.0242 & 1.2726 & 1.159 & 1.1108\\
		\hline
	\end{tabular}
	\caption{\small Cumulative regret of UCB1, aUCB1, and 2-UCBs normalized by the cumulative regret of UCB1 in the case of $\Delta=0.03$ and $c=0.05$, under different choices of the gap $\Delta$ and tuning parameter $c$.}
	\label{table-avg-regret}
\end{table}

\begin{table}[H]
	\centering
	\small
	\begin{tabular}{|l|l|l|l|l|l|l|l|l|l|l|}
		\hline
		\multirow{2}{*}{\backslashbox{\,\;\quad}{\,}} &
		\multicolumn{2}{c|}{\boldmath{$\Delta = 0.01$}} &
		\multicolumn{2}{c|}{\boldmath{$\Delta = 0.02$}} &
		\multicolumn{2}{c|}{\boldmath{$\Delta = 0.03$}} &
		\multicolumn{2}{c|}{\boldmath{$\Delta = 0.04$}} &
		\multicolumn{2}{c|}{\boldmath{$\Delta = 0.2\mathrm{CVR}_{a,d}$}} \\
		& aUCB1 & 2-UCBs & aUCB1 & 2-UCBs  & aUCB1 & 2-UCBs  & aUCB1 & 2-UCBs  & aUCB1 & 2-UCBs  \\
		\hline
		\boldmath{$c=0.03$} & 5.0\% & 1.2\% & 8.8\% & 3.9\% & 10.3\% & 5.3\%  & 12.1\% & 5.8\% & 13.5\% & 2.7\%\\
		\hline
		\boldmath{$c=0.05$} & 2.8\% & 1.5\% & 5.0\% & 7.7\% & 6.5\% & 8.6\% & 8.0\% & 11.5\% & 6.7\% & 8.9\% \\
		\hline
		\boldmath{$c=0.1$} & 2.1\% & 2.7\% & 6.8\% & 7.7\% & 7.8\% & 9.8\% & 11.2\% & 14.7\% & 8.9\% & 12.7\%\\
		\hline
	\end{tabular}
\caption{\small Average Relative Improvement (RI) for aUCB1 and 2-UCBs under different choices of the gap $\Delta$ and tuning parameter $c$.}
\label{table-avg-ri}
\end{table}

\begin{table}[H]
	\centering \small
	\begin{tabular}{|l|l|l|l|l|l|l|l|l|l|l|}
		\hline
		\multirow{2}{*}{\backslashbox{\,\;\quad}{\,}} &
		\multicolumn{2}{c|}{\boldmath{$\Delta = 0.01$}} &
		\multicolumn{2}{c|}{\boldmath{$\Delta = 0.02$}} &
		\multicolumn{2}{c|}{\boldmath{$\Delta = 0.03$}} &
		\multicolumn{2}{c|}{\boldmath{$\Delta = 0.04$}} &
		\multicolumn{2}{c|}{\boldmath{$\Delta = 0.2\mathrm{CVR}_{a,d}$}} \\
		& aUCB1 & 2-UCBs & aUCB1 & 2-UCBs  & aUCB1 & 2-UCBs  & aUCB1 & 2-UCBs  & aUCB1 & 2-UCBs  \\
		\hline
		\boldmath{$c=0.03$} & 66.6\% & 60.6\% & 66.1\% & 60.6\% & 65.1\% & 65.6\%  & 73.2\% & 68.6\% & 67.1\% & 62.6\%\\
		\hline
		\boldmath{$c=0.05$} & 53.0\% & 60.6\% & 58.5\% & 72.7\% & 59.5\% & 69.1\% & 64.6\% & 73.7\% & 61.6\% & 71.7\% \\
		\hline
		\boldmath{$c=0.1$} & 58.5\% & 72.2\% & 54.5\% & 77.2\% & 58.5\% & 77.2\% & 59.0\% & 74.2\% & 57.5\% & 76.7\%\\
		\hline
	\end{tabular}
\caption{Average No Harm (NH) for aUCB1 and 2-UCBs under different choices of the gap $\Delta$ and tuning parameter $c$.}
\label{table-avg-nh}
\end{table}
\color{black}

\end{document}